\documentclass[twoside,11pt]{article}

\usepackage{jmlr2e}
\usepackage{mathrsfs,amssymb,amsmath,booktabs,array,xcolor,url,cite,color,soul,multirow,multicol,natbib}
\usepackage{mathbbold,bm,bbm}
\usepackage{slashbox,verbatim}
\usepackage{algpseudocode,algorithm}

\ShortHeadings{Transfer learning in information criteria-based feature selection}{Shaohan Chen, Nikolaos V. Sahinidis and Chuanhou Gao }
\firstpageno{1}

\begin{document}

\title{Transfer Learning in Information Criteria-based Feature Selection}

\author{\name Shaohan Chen \email shaohan\_chen@zju.edu.cn \\
       \addr School of Mathematical Sciences\\
       Zhejiang University\\
       Hangzhou 310027, China
       \AND
       \name Nikolaos V. Sahinidis \email nikos@gatech.edu \\
        \addr H. Milton Stewart School of Industrial \& Systems Engineering and \\
        School of Chemical \& Biomolecular Engineering \\
        Georgia Institute of Technology \\
        Atlanta, GA 30332, USA
       \AND
       \name Chuanhou Gao \email gaochou@zju.edu.cn \\
       \addr School of Mathematical Sciences\\
       Zhejiang University\\
       Hangzhou 310027, China}

\editor{}

\maketitle

\begin{abstract}%   <- trailing '%' for backward compatibility of .sty file
This paper investigates the effectiveness of transfer learning based on information criteria. We propose a procedure that combines transfer learning with Mallows' Cp (TLCp) and prove that it outperforms the conventional Mallows' Cp criterion in terms of accuracy and stability.
 %show that combining transfer learning with Mallows' Cp (TLCp) yields more stable and accurate model selection criterion, when the tuning parameters are well-chosen.
Our theoretical results indicate that, for any sample size in the target domain, the proposed TLCp estimator performs better than the Cp estimator by the mean squared error (MSE) metric {in the case of orthogonal predictors}, provided that i) the dissimilarity between the tasks from source domain and target domain is small, and ii) the procedure parameters (complexity penalties) are tuned according to certain explicit rules.
%In particular, our results theoretically indicate that for any sample size $n$ in the target domain, when the parameters are tuned according to some explicit rules and the design matrix is orthogonal, the proposed TLCp estimator performs better than the original Cp estimator by the MSE metric, provided that the dissimilarity measure between the tasks from source domain and target domain is small.  %the estimator learned from the proposed TLCp model is closer to the true regression vector in terms of the MSE measure compared to the estimator from the original Mallows' Cp model.
Moreover, we show that our transfer learning framework can be extended to other feature selection criteria, such as the Bayesian information criterion. By analyzing the solution of the orthogonalized Cp, we identify an estimator that asymptotically approximates the solution of the Cp criterion in the case of non-orthogonal predictors. {Similar results are obtained for the non-orthogonal TLCp. Finally, simulation studies and applications with real data demonstrate the usefulness of the TLCp scheme.}
\end{abstract}


\section{Introduction}

Machine learning technologies have been remarkably successful in many contemporary industrial applications. However, {supervised machine learning algorithms, such as support vector machines, neural networks, and decision trees}, are fundamentally limited because they demand large amounts of training samples to guarantee good performance \citep{bartlett2002rademacher}. It is either expensive or impossible to collect a huge amount of data in many industrial applications \citep{hill1977introduction,buzzi2010interpolation}. Therefore, it is important to investigate learning methods that perform well in small samples---or even when the dimension of the learning system is much larger than the number of training samples.

Over the past five decades, statisticians provided the following feature selection criteria, which continue to influence modern day machine learning algorithms: Akaike's information criterion \citep{akaike1974new}, Mallows' Cp \citep{mallows1973some}, the Bayesian information criterion \citep{schwarz1978estimating}, the Hannan-Quinn information criterion \citep{hannan1979determination}, and the risk inflation criterion \citep{foster1994risk}. These criteria balance model accuracy and complexity, and are used to produce sparse algebraic models. {Each of these feature selection criteria} can be utilized via a mixed-integer programming (MIP) formulation. \citet{cozad2014learning} recently developed the ALAMO approach, which implements these MIP formulations. ALAMO expands on feature selection criteria by considering a large set of explicit transformations from the original input variables in the models \citep{cozad2014learning,cozad2015combined,wilson2017alamo}. At the core of ALAMO is a nonlinear integer programming-based best subset selection technique, which relies on the global mixed-integer nonlinear programming solver BARON \citep{tawarmalani2005polyhedral}. 
{In addition to information criteria, score-based criteria (Fisher score, mutual information, maximum relevance minimum redundancy (mRMR) \citep{peng2005feature}, etc.), cross-validation and statistical tests are all widely employed as feature selection criteria in practice \citep{borboudakis2019forward}. Moreover, some machine learning-based criteria were proposed to enrich feature selection methods: an SVM-based method \citep{guyon2002gene} and a neural network-based methods \citep{steppe1997feature,setiono1997neural}.}
For an introduction to feature selection and review of methods in the context of supervised learning and unsupervised learning, we refer the reader to \citet{friedman2001elements}, \citet{guyon2003introduction} and \citet{dy2004feature}. 

%The method we are concerned about in this paper is called transfer learning \citep{pan2010survey}, which borrows knowledge from similar tasks to supplement the shortage of data in a target domain. In this work, we embed transfer learning techniques into the best subset selection techniques based on information criteria and investigate the utility of the resulting approach.
%{In applied settings, the size of target samples is often limited or insufficient to make machine learning models (i.e., the introduced feature selection methods) perform well. The method we are concerned about in this paper is called transfer learning \citep{pan2010survey}, which borrows knowledge from similar tasks to supplement the shortage of data in a target domain. In this work, we embed transfer learning techniques into the feature selection criteria and investigate the utility of the resulting approach.}

{In addition to feature selection, another method to address the problem of insufficient data is transfer learning, which utilizes knowledge from similar tasks to overcome the shortage of data in a target domain. In this work, we embed transfer learning techniques into the widely used information criteria for feature selection and investigate the utility of the resulting approach.}

{Transfer learning (referred by some as multi-task learning\footnote{{We use the term transfer learning to refer to techniques that pay attention to the learning performance on the target task alone, while we reserve the term multi-task learning when one wishes to learn both the source and target tasks as well as possible. We refer interested readers to \citet{pan2010survey} for a related discussion.}}) aims to improve learning performance of the target task by extracting (common) knowledge from the related source tasks.} {According to the type of knowledge that can be transferred, we can categorize transfer learning into four cases: instance-based \citep{dai2007boosting}, feature-representation-based \citep{argyriou2007multi},  parameter-based \citep{evgeniou2004regularized} and relational domain-based \citep{mihalkova2007mapping}. For a review of transfer learning and how it works we refer the reader to \citet{pan2010survey}.} 
A large body of researchers have recently explored the benefits of transfer learning techniques both from an experimental and theoretical perspective. \citet{barreto2017successor} showed that combining transfer learning with reinforcement learning frameworks can significantly enhance  performance in navigation tasks. Transfer learning has also been successful for detection problems when integrated with deep convolutional networks \citep{hoo2016deep,wang2019ridesharing}.
Many researchers {proved} that tighter generalization upper bounds can be achieved when transfer learning techniques are utilized \citep{baxter2000model,ando2005framework,maurer2006bounds,ben2003exploiting}. \citet{maurer2013sparse,pontil2013excess} investigated the power of transfer learning techniques to manage the excess risk upper bounds. \citet{kuzborskij2013stability} showed how transfer learning can help accelerate the convergence of the Leave-One-Out error to the generalization error. %\citet{lounici2009taking} studied the advantage of transfer learning in the context of the Lasso, and provided a tighter prediction error upper bound compared to the original Lasso.

{The advantages of applying transfer learning techniques have led to studies in the context of feature selection. Combining transfer learning with LASSO can be beneficial to feature selection by sharing the same sparsity pattern across tasks \citep{obozinski2006multi, yuan2006model, argyriou2007multi, lounici2009taking, liu2009multi, zhang2010probabilistic, wang2016multi}. \citet{lozano2012multi} presented a flexible LASSO-based feature selection framework combined with transfer learning, which can identify common and task-specific patterns across similar tasks. \citet{jebara2004multi} showed that incorporating transfer learning with SVM can be advantageous to identify relevant features. \citet{helleputte2009feature} demonstrated that the common knowledge extracted by transfer learning is useful to guide feature selection in the target domain. \citet{sugiyama2014multi} demonstrated that transfer learning can be used to discover causal features among similar networks. }%Transfer learning is also capable of identifying common mechanism among related tasks to guide the feature selection on the target domain. }

Merely providing tighter upper error bounds is not enough to guarantee that a model selection technique coupled with transfer learning will perform better in real industrial applications with limited data, {because these error bounds only make sense in large data size cases.} {It is still important to investigate how transfer learning affects feature selection and identify conditions under which transfer learning is superior to independent learning in the case of limited data.}
Our main goal is to investigate the effectiveness of transfer learning {for feature selection}. {We expose  parameter tuning rules and conditions under which transfer learning is guaranteed to outperform independent learning in the sense of both accuracy (i.e., leads to smaller mean squared error (MSE)) and stability (i.e.,  comes with higher probability to identify relevant features) under limited sample size.}

%Here, we choose to justify the efficiency of transfer learning in the framework of Mallows' Cp, which aims to identify the simple and more accurate parameter estimation models in the small sample regime.

%A very similar learning terminology called multi-task learning \citep{}, which tries to simultaneously learn both the source and target domain tasks by sharing common information between tasks. Many research works have explored the effectiveness of transfer learning and multi-task learning techniques both in the theoretical and experimental domains. \citep{}

{We choose to combine the transfer learning technique of \citet{evgeniou2004regularized} with the popular information criterion Mallows' Cp, as a representative way to show that transfer learning can facilitate feature selection.} 
The combined technique is referred to as TLCp and aims to provide a simple and accurate parameter estimation method in the small sample regime. We prove that, for any fixed sample size in the target domain, if the tasks in the target domain and source domain are similar enough and the tuning parameters are chosen to satisfy some explicit rules, then the orthogonal TLCp estimator is closer than Cp to the true regression coefficients in terms of the MSE measure. {Moreover, based on the orthogonality assumption, we show that the TLCp estimator identifies important features with higher {probability} than the Cp estimator.}

The main contributions of this paper are as follows. (1) For any sample size in the target domain, we derive an explicit {parameter} tuning rule so that the proposed TLCp procedure can outperform the independent learning (or original Mallows' Cp criterion) in terms of accuracy and stability under the orthogonality assumption. {(2) Our simulation studies and experiments on three real \textcolor{black}{datasets} demonstrate the usefulness of the proposed TLCp framework in practical applications.} (3) We show that our analysis framework, which explores the efficiency of transfer learning, can be extended to other feature selection criteria, such as the Bayesian information criterion. (4) We present a method for producing an estimator that can asymptotically approximate the solution of Mallows' Cp in the non-orthogonal case. {Similarly, we identify an estimator to asymptotically approximate the non-orthogonal TLCp estimator. }

%Furthermore, we show that our analysis framework to explore the efficiency of transfer learning can be extended to other feature selection criteria, such as the Bayesian information criterion.
%Note that this paper presents two new viewpoints on transfer learning: (1) transfer learning will outperform independent learning in finite sample size only when the tuning parameters are well-chosen. (2) The proposed TLCp model is more stable and accurate than the original Mallows' Cp. For the non-orthogonal case, we present a method for producing an estimator that can asymptotically approximate the solution of Mallows' Cp.

 %in our framework, Mallow's Cp measure can only outperform the independent learning algorithms in the limited data size when parameters are well-tuned.

The remainder of this paper is organized as follows. Section $2$ introduces preliminaries of this paper, including the basic concepts of {the information criteria} and the transfer learning method. Section $3$ theoretically analyzes the process of the orthogonal Cp for its ability to identify relevant features. Section $4$ describes the basic framework of the TLCp method, and analyzes its ability to identify important features under the orthogonality assumption.
{Section $5$ discusses extensions of the main ideas of this paper, including the computation of an estimator that can asymptotically approximate the solution of Mallows' Cp in the non-orthogonal case. {Similarly, we identify an estimator to asymptotically approximate the non-orthogonal TLCp estimator.}
In the same section, we provide guidelines for practitioners on how to use the TLCp method. Section $6$ describes the simulations conducted to illustrate some of our results. Section $7$ verifies the effectiveness of the TLCp method by three real data experiments. Section $8$ summarizes the main conclusions of this paper. To improve the readability of this paper, we provide all proofs and {a summary of notations} in the appendix. }

%Section $5$ describes the simulations conducted to illusrate some of our results. Section $6$ discusses extensions of the main ideas of this paper, including the computation of an estimator that can asymptotically approximate the solution of Mallows' Cp in the non-orthogonal design matrix case. Section $7$ summarizes the main conclusions of this paper. To improve the readability of this paper, we provide all proofs in the appendix.

\section{Background}
This section describes the paper's notation and assumptions. It also introduces the concepts behind Mallows' Cp and the transfer learning technique employed in the proposed TLCp model.
\subsection{ Preliminaries}

Assume that the data set in the target domain consists of $n$ samples $(x_1^i, x_2^i, \cdots, x_k^i;y_i)$ for $i=1,\cdots, n$, each of which has $k$ features and satisfies the following true but unknown relationship:
\begin{eqnarray}
\boldsymbol { y } = \boldsymbol { X } \boldsymbol { \beta } + \boldsymbol { \varepsilon }
\end{eqnarray}
where $\boldsymbol { y } : = \left( y _ { 1 }, y _ { 2 }, \cdots, y _ { n } \right) ^ { \top }$ are the responses, $\boldsymbol { \beta} : = \left( \beta _ { 1 }, \beta _ { 2 }, \cdots, \beta _ { k } \right) ^ { \top }$ are the regression coefficients, $\boldsymbol { X } : = (X_1, X_2, \cdots, X_n) ^ { \top }=(W_1, W_2,\cdots, W_k)$ is the design matrix, and $\boldsymbol{\varepsilon} : = \left( \varepsilon _ { 1 }, \varepsilon _ { 2 }, \cdots, \varepsilon _ { n } \right) ^ { \top }$ are the prediction residuals each of which is  Gaussian noise. Without loss of generality, we suppose $\varepsilon_i \sim \mathcal { N } \left( 0 , \sigma_1 ^ { 2 } \right)$ for $i=1,\cdots, n$. {Unless otherwise stated, we assume that the design matrix
$\boldsymbol { X }$ satisfies $\boldsymbol{X}^{\top}\boldsymbol{X}=nI$, where $I$ is the identity matrix, and we refer to the regression problem under this condition as the orthogonal problem.}

\subsection{Mallows' Cp}

%Feature selection provides us a way of improving the interpretability of the resulting models, reducing the cost of measurements for constructing the predictive models and gaining the better model performances. 
%Information criteria for feature selection are widely employed in practice due to their efficiency for producing the accurate and sparse models. Among the information criteria, Mallows' Cp }  

Mallows' Cp is a model fitness metric that has been proposed for identifying a best subset of the regressors. The feature selection procedure based on this metric is defined as follows.

\begin{eqnarray}\label{Cp}
C _ { p } = \min_{\boldsymbol { a }}\frac { ( \boldsymbol { y } - \boldsymbol { X }  \boldsymbol { a }  ) ^ { \top } ( \boldsymbol { y } - \boldsymbol { X } \boldsymbol { a }  ) } { \hat { \sigma }_1 ^ { 2 } } + 2 p - n
\end{eqnarray}
where $\hat { \sigma}_1^{2}$ is an estimator of the true residual variance, $\sigma_1^2$. For simplicity, we assume $\hat { \sigma }_1^{2}\approx\sigma_1^2$. The non-negative integer $p$ indicates the number of nonzero regressors in the regression model and represents the model complexity. This principle helps prevent over-fitting and achieve higher generalization performance compared to traditional regression methods \citep{friedman2001elements,miyashiro2015subset}, such as the ordinary least squares estimation. {Using Cp can improve the interpretability of the resulting model and reduce the cost of measurements to obtain a good predictive model \citep{guyon2003introduction,borboudakis2019forward}.}

{The Mallows' Cp criterion balances goodness-of-fit (i.e., the maximized log-likelihood) and complexity (i.e., the number of regressors) of the model. Other commonly used information criteria are: Akaike's information criterion (AIC) \citep{akaike1974new}, the Bayesian information criterion (BIC) \citep{schwarz1978estimating}, the Hannan-Quinn information criterion (HIC) \citep{hannan1979determination}, and the risk inflation criterion (RIC) \citep{foster1994risk}. All these criteria have the same goodness-of-fit form but employ different model complexity penalties. Information criteria are closely related to cross-validation and statistical tests~\citep{dziak2020sensitivity}. Information criteria can be optimized directly to obtain a best subset of the features. Alternatively, these criteria can be used within other feature selection algorithms to compare different models \citep{borboudakis2019forward}.  }
	
%	Information criteria can be used as the performance functions to compare different competing models, more details about how information criteria works can be found in \citep{}. }
Since Mallows' Cp can be viewed as a representative of the above mathematically similar feature selection criteria, we will use it to investigate the effectiveness of transfer learning. % Our analysis of transfer learning with Mallows' Cp can be easily generalized to other feature selection criteria.
As indicated in \citet{mallows1973some},  the ``minimum Cp'' rule for selecting the best subset of the features for least-squares fitting should not be applied universally. We will identify conditions under which Mallows' Cp fails to identify important features and why the transfer learning techniques will {perform better under the orthogonality assumption}.

\subsection{ Transfer Learning}

Transfer learning aims to improve the learning of predictive functions in a target domain using the knowledge in a source domain \citep{pan2010survey} by transferring the knowledge of four categories: instances, features, parameters and relationships. {In this paper, we focus on the parameter-based transfer learning technique presented by \citet{evgeniou2004regularized}, which shares common knowledge extracted from the source tasks through parameters to be learned so as to improve the performance of Mallows' Cp criterion. Section $4$ will show how this parameter-based transfer learning scheme can work well with information criteria.}

{We consider the following problem setting.}
 There is one learning task from the source domain with data  $\{(\tilde{x}_1^i, \tilde{x}_2^i, \cdots, \tilde{x}_k^i;\tilde{y}_i)\}_{i=1}^{m}$, and another learning task from the target domain with data  $\{(x_1^i, x_2^i, \cdots, x_k^i;y_i)\}_{i=1}^{n}$.%, where these two tasks correspond to a ``similarity'' measure. 
We enhance learning performance in the target domain by sharing common information with the learning task in the source domain.

Concretely, the transfer learning scheme built by \citet{evgeniou2004regularized} has the form:
\begin{eqnarray}\label{transfer learning}
\min_{\boldsymbol{v}_1,\boldsymbol{v}_2,\boldsymbol{w}_0}\sum _ { i = 1 } ^ { n }\lambda_1(y_i-\boldsymbol{w}_1^{\top}X_i)^2 + \sum _ { i = 1 } ^ { m } \lambda_2(\tilde{y_{i}}-\boldsymbol{w}_2^{\top}\tilde{X_{i}})^2 + \frac { \lambda _ { 3 } } { 2 } \sum _ { t = 1 } ^ { 2 } \| \boldsymbol { v }_t \| ^ { 2 }+\gamma\|\boldsymbol{w}_0\|^{2},
\end{eqnarray}
where $\boldsymbol{w}_1=\boldsymbol{w}_0+\boldsymbol{v}_1$ and $\boldsymbol{w}_2=\boldsymbol{w}_0+\boldsymbol{v}_2$ are the regression coefficients with respect to the tasks in the target domain and the source domain, respectively. $\boldsymbol{w}_0$ is a common parameter while $\boldsymbol{v}_1 , \boldsymbol{v}_2$ are specific parameters for the source task and target task.  The tuning parameters $\lambda_1, \lambda_2$ define the weights of the loss functions for the two domains, with $\lambda_1>\lambda_2$ when focusing on the performance in the target domain. {$\lambda_3$ and $\gamma$ are two positive regularization parameters reflecting the importance of the individual and common parts of the models of the two tasks. When $\lambda_3$ approaches $\infty$, which implies that $\boldsymbol{v}_1=\boldsymbol{v}_2=0$, then~(\ref{transfer learning}) treats the target and source tasks identically. When $\gamma$ approaches $\infty$, which means $\boldsymbol{w}_0=0$, then (\ref{transfer learning}) reduces to learning the target and source tasks independently. In general, this transfer learning framework provides an elegant way to share knowledge among tasks through parameters. In Section $4$, we combine this transfer learning technique (\ref{transfer learning}) with the information criteria to improve feature selection. }

{\subsection{Benefits of Combining Transfer Learning with Information Criteria}}

{Intuitively, we can understand the advantage of applying transfer learning to feature selection as utilizing the common sparsity structure extracted from the related tasks to guide feature selection in the target domain.}

{In this work, the combination of the parameter-based transfer learning method with information criteria presents two advantages.} 
{First, the mathematical structure of information criteria allows us to carry a sparsity pattern across related tasks, thus facilitating feature selection in the target domain. Second, the parameter-based transfer learning framework provides an avenue to extract knowledge from similar tasks to improve the target task's learning. Therefore, we can expect that these two techniques will strengthen each other when we combine them.   }

{As shown later (in Section $4$), one of the contributions of this work is to provide explicit rules to tune the hyper-parameters of the combined model. Based on the optimal tuning of the hyper-parameters, we can guarantee that the resulting model is superior to the independent model under some mild conditions. Although we focus on integrating the parameter-based transfer learning technique presented by \citet{evgeniou2004regularized} into Mallows' Cp, we can potentially extend our analysis framework to other feature selection criteria and other transfer learning methods. }

The following section explains the reason behind embedding the transfer learning technique to Mallows' Cp criterion and analyzing the conditions under which Mallows' Cp will miss the right model.

\section{Analysis of Mallows' Cp with Orthogonal Design}

In this section, we exploit the ability of the orthogonal Mallows' Cp to identify important regressors.
{Later, we will investigate the case when Mallows' Cp is not recommended.}

Without loss of generality, we turn to investigate the properties of the modified Mallows' Cp as follows.

\begin{equation}\label{the modified Cp}
\min_{\boldsymbol{a}}~ ( \boldsymbol { y } - \boldsymbol { X }  \boldsymbol { a }  ) ^ { \top } ( \boldsymbol { y } - \boldsymbol { X } \boldsymbol { a }  )+\lambda \|\boldsymbol{a}\|_{0}
\end{equation}
 where $\|\cdot\|_0$ denotes the $\ell _ { 0 }$ norm that refers to the number of nonzero components, and $\lambda>0$ is a parameter used to balance the trade-off between model accuracy and complexity. In particular, when $\lambda=2\hat{\sigma}_1^2$, this formula will reduce to the original Mallows' Cp.
The resulting estimator of model~(\ref{the modified Cp}) has a closed-form solution in the orthogonal case.

\begin{proposition}\label{theorem1}
The solution of the orthogonal Cp criterion (\ref{the modified Cp}) has the following form:
\begin{eqnarray} \label{eq5}
\hat{a}_i=\left\{
\begin{matrix}
	\beta_i+\frac{W_i^{\top}\boldsymbol { \varepsilon }}{n},
		& \text{if~} n\left(\beta_i+\frac{W_i^{\top}\boldsymbol { \varepsilon }}{n}\right)^2>\lambda\\
		0,
		&  \text{otherwise}
	\end{matrix}
	\right.
\end{eqnarray}
for $i=1,\cdots, k$.
\end{proposition}

\begin{proof}
The detailed proof of Proposition \ref{theorem1} can be found in Appendix \ref{appendixC}.
\end{proof}	

Proposition \ref{theorem1} clarifies the discrimination rule of the orthogonal Cp (\ref{the modified Cp}) in order to identify relevant features, which explains how the performance of the orthogonal Cp criterion is affected by the distribution of true regression coefficients. {Remark \ref{remark2} explains Proposition \ref{theorem1} from the viewpoint of the statistical hypothesis test.} %illustrate how the distribution of true regression coefficients can affect the probability of orthogonal Cp to identify features.%below is the probability of the procedure to select features in the regression model depending on the distributions of true regression coefficients.

{\begin{remark}\label{remark2}
 %Notice that the estimated regression coefficients in Proposition \ref{theorem1} can be transformed into the Pearson's correlation coefficient. 
 For each regression coefficient estimate, we construct the $z$-statistic, $z_i=\frac{r_is_{\boldsymbol{y}}}{\sigma_1}-\frac{\sqrt{n}\beta_i}{\sigma_1}$, where $r_i$ is the sample Pearson's correlation coefficient between the $i$-th feature and the response $\boldsymbol{y}$, $i=1,\cdots,k$. Under the null hypothesis that $\beta_i=0$, or equivalently the population Pearson's correlation coefficient equals zero, the $z$-statistic follows the standard normal distribution. Then, Proposition \ref{theorem1} implies that using the orthogonal Cp criterion is equivalent to performing the statistical $z$-test for each feature with the significance level $\alpha_1(\lambda)=2\phi(-\frac{\sqrt{\lambda}}{\sigma_1})$, where $\phi(u)$= $\int_{-\infty}^{u}\frac{1}{\sqrt{2\pi}}\exp{\left\{-\frac{x^2}{2}\right\}}dx$. Appendix A contains more details.
%By considering the relationship between the least squares estimate and the Pearson's correlation coefficient, we can further rephrase (\ref{eq5}) in the Proposition \ref{theorem1} as:
\end{remark}}

\begin{theorem}\label{theorem 2}
	
The probability $Pr^{Cp}\{i\}$ that the orthogonal Cp selects the $i$-th feature in the regression model is
\begin{eqnarray}
Pr^{Cp}\{i\}&=&\frac{\sqrt{n}}{\sqrt{2\pi}\sigma_1}\int_{nx^2>\lambda}\exp\left\{-\frac{1}{2}\frac{\left(x-\beta_i\right)^2}{\frac{\sigma_1^2}{n}}\right\}dx\notag\\&=&1-\int_{\frac{-\sqrt{n}\beta_i-\sqrt{\lambda}}{\sigma_1}}^{\frac{-\sqrt{n}\beta_i+\sqrt{\lambda}}{\sigma_1}}\frac{1}{\sqrt{2\pi}}\exp\left\{-\frac{x^2}{2}\right\}dx\notag
\end{eqnarray}
where $i=1,\cdots,k$.
\end{theorem}	

\begin{proof}
	The detailed proof of Theorem \ref{theorem 2} can be found in Appendix \ref{appendixC}.
\end{proof}	

According to Theorem \ref{theorem 2}, the probability of the orthogonal Cp procedure to select a feature is independent of whether or not the remaining features are chosen. {We will also indicate (in Remark \ref{remark4}) that Theorem \ref{theorem 2} corresponds to the power analysis for the $z$-statistic introduced in Remark \ref{remark2}. This fact inspires us to restudy the orthogonal Cp from the angle of statistical tests directly (see Appendix A).} Additionally, if there is a feature whose regression coefficient equals zero (\textcolor{black}{we refer to it as a ``superfluous feature''}), then by Theorem \ref{theorem 2}, we can estimate that the probability of the orthogonal Cp to select this superfluous feature is approximately $0.16$, if $\lambda=2\sigma_1^2$. For this reason, practitioners often assign $\lambda$ a large value in order to develop sparse models. However, if $\lambda$ is too large, the orthogonal Cp will remove important features. Therefore, determining model parameters is a challenging task in machine learning research. Later, we will show (in Proposition \ref{proposition5}) the advantage of using $\lambda=2\sigma_1^2$ in the original Mallows' Cp in terms of MSE performance. 

{\begin{remark}\label{remark4}
When the $i$-th true regression coefficient $\beta_i\neq0$, then the result in Theorem \ref{theorem 2} corresponds to the power of the hypothesis test (with the null hypothesis $ \beta_i=0$, and the alternative hypothesis $\beta_i\neq0$) concerning the z-statistic introduced in Remark \ref{remark2}, for $i=1,\cdots,k$. Therefore, based on Theorem \ref{theorem 2}, we can estimate the required sample size to achieve the desired power to detect some essential features, i.e., those with an effective size (absolute value of the corresponding regression coefficient) larger than a given threshold. Appendix A contains more details.
\end{remark}}

\begin{proposition}\label{proposition3}
Assume that for the $j$-th feature the coefficient $\beta_j$ satisfies the equality $\beta_j^2=\frac{2\sigma_1^2}{n}$ in the true regression model. If we set $\lambda=2\sigma_1^2$, then the probability of the orthogonal Cp to select the $j$-th feature is $1-[\phi(0)-\phi(-2\sqrt{2})]$, where $\phi(u)$= $\int_{-\infty}^{u}\frac{1}{\sqrt{2\pi}}\exp{\left\{-\frac{x^2}{2}\right\}}dx$.%greater than $0.5$, but less than $0.503$.
\end{proposition}	

\begin{proof}
	The detailed proof of Proposition \ref{proposition3} can be found in Appendix \ref{appendixC}.
\end{proof}	

Proposition \ref{proposition3} reveals that the orthogonal Cp criterion can fail to identify important features whose true regression coefficients are near the {\em critical points} $\pm\sqrt{2/n}\sigma_1$, with probability $0.5$. %When the regression coefficient takes value $\sqrt{2/n}\sigma_1$ or $-\sqrt{2/n}\sigma_1$, which means that the corresponding feature is more significant than the noise, we need to establish a feature selection model (see TLCp model later) that can re-identify these kind of features.%and can suffer from instability and inaccuracy problems in terms of feature selection.
To further illustrate the importance of this problem, we will analyze two particular scenarios below.
First, %we can see from the Proposition \ref{proposition3} that,
{if} the target data size $n$ is very small, or the variance of noise $\sigma_1^2$ in the target domain is very large, this implies the importance of the feature whose coefficient takes the value of $\pm\sqrt{2/n}\sigma_1$. However, Proposition \ref{proposition3} indicates that the orthogonal Cp procedure will miss this feature with a probability of $0.5$.
Therefore, serious problems may arise in applications where the training data is often very limited and has large noise. Second, if the true regression model contains a large number of features with coefficients near $\pm\sqrt{2/n}\sigma_1$, there will inevitably be a large deviation between the orthogonal Cp estimator and the true estimator, since only half of these features will be chosen in this case. To demonstrate the importance of re-identifying these features, we conduct an experiment to compare the performances between the orthogonal Cp criterion and the orthogonal least squares method when there are  {\em{critical features}} (which are defined as features whose coefficients are at or near the critical points) in the true regression model (see Figure \ref{critical points}). Later, we will expound on this problem.
%Proposition \ref{proposition3} reveals the inherent reasons leading to the instability and inaccuracy of the orthogonal Cp in terms of identifying important features, which states the situations in which the procedure will miss some important features with a $0.5$ probability. In particular, the orthogonal Cp fails to identify important features when (1) there is a large proportion of the coefficients near the points $\pm\sqrt{\frac{2}{n}}\sigma_1$ (which will be referred to as ``sensitive points'' hereafter) in the true regression model; (2) the target data size $n$ is small; (3) the variance of noise in the target domain is big. The second and third cases are often encountered in real industrial applications, which makes the Cp criterion suffer under limitations.

%\begin{figure}
%	\centering
%	\includegraphics[height=0.35\textheight, width=0.85 \textwidth]{pictures/critical_point_2.png}
%	\caption{\small The MSE performance comparison between Cp criterion and least square method under the existence of critical points in the true regression coefficients.}
%	\label{critical points}
%\end{figure}

%The MSE performances between the Cp and least square method vary with sample size $n$. More importantly, there is an interaction point of these two MSE curves. After that point, the superiority comparison between these two models is exchanged (see figure $1$).

To demonstrate the advantages of employing the $L_0$-type penalty in the Cp criterion as opposed to the conventional least squares method, we calculate the MSE metric of the orthogonal Cp estimator below. Based on this metric, we will investigate some key factors behind the MSE performance of the orthogonal Cp estimator.

\begin{theorem}\label{theorem4}
The MSE measure of the estimator $\hat{\boldsymbol{a}}$ that minimizes the orthogonal Cp model in (\ref{the modified Cp}) can be calculated as follows.
\begin{eqnarray}
\text{MSE}(\hat{\boldsymbol{a}})&=&\sum_{i=1}^{k}\left[\frac{\sigma_1^2}{n}+\int_{\frac{-\sqrt{n}\beta_i-\sqrt{\lambda}}{\sigma_1}}^{\frac{-\sqrt{n}\beta_i+\sqrt{\lambda}}{\sigma_1}}\left(\beta_i^2-\frac{\sigma_1^2x^2}{n}\right)\frac{1}{\sqrt{2\pi}}\exp\left\{-\frac{x^2}{2}\right\}dx\right]\\&=&\sum_{i=1}^{k}\left[\frac{\sigma_1^2}{n}+\frac{1}{\sqrt{2\pi}\sigma_1}\int_{(y+\sqrt{n}\beta_i)^2<\lambda}\left(\beta_i^2-\frac{1}{n}y^2\right)\exp{\left\{-\frac{y^2}{2\sigma_1^2}\right\}}dy\right]\label{7}
\end{eqnarray}
\end{theorem}	

\begin{proof}
The detailed proof of Theorem \ref{theorem4} can be found in Appendix \ref{appendixC}.
\end{proof}

After closely examining the second equality (\ref{7}) in Theorem \ref{theorem4}, we learn that utilizing the feature selection technique (or applying the $L_0$ penalty term on the regression coefficients in the Cp criterion) has the potential to decrease the MSE metric compared to the least squares method, even when the true regression model is not sparse. Furthermore, the MSE metric of the orthogonal least squares (LS) estimator is $\frac{k\sigma_1^2}{n}$ (which amounts to the case when $\lambda=0$ in (\ref{7})) under our problem settings. Below, we show a theoretical advantage of setting $\lambda=2\sigma_1^2$ in the original Cp criterion.%To understand the philosophy of choosing $\lambda=2\sigma_1^2$ in the original Mallows' Cp criterion, we will analyze the integrand in (\ref{7}). According to our results, selecting $\lambda=2\sigma_1^2$ has the potential to decrease the MSE value of the orthogonal Cp estimator.

\begin{proposition}\label{proposition5}
	Let $f(x):=\left(\beta_i^2-\frac{1}{n}x^2\right)\exp{\left\{-\frac{x^2}{2\sigma_1^2}\right\}}$, where $x\in(-\infty, +\infty)$. Then, the global minimum points of $f(x)$ are $\pm\sqrt{n\beta_i^2+2\sigma_1^2}$. If we set $\lambda=2\sigma_1^2$, then at least one of these two global minimizers belongs to the integral interval $\left(-\sqrt{n}\beta_i-\sqrt{\lambda}, -\sqrt{n}\beta_i+\sqrt{\lambda}\right)$ in (\ref{7}).
\end{proposition}	

\begin{proof}
	The detailed proof of Proposition \ref{proposition5} can be found in Appendix \ref{appendixC}.
\end{proof}

Based on Proposition \ref{proposition5}, we can expect to obtain a lower MSE of the orthogonal Cp estimator (\ref{7}) by setting $\lambda=2\sigma_1^2$ in (\ref{the modified Cp}), because of the inclusion of minimum points in the integral interval. %Thus, a lower MSE value of the orthogonal Cp estimator can be expected in this case.

In general, the Cp criterion outperforms the least squares method under the sparse model assumption (based on the MSE value). This trend is consistent with our results in Theorem \ref{theorem4}, where the integrand in (\ref{7}) is negative if the corresponding regression coefficient is zero. However, we are more interested in the performance of the orthogonal Cp criterion in the presence of critical features in the true regression model.

%when some coefficients are zeros, that is, the true regression model is sparse, it seems that the orthogonal Cp can outperform the leas by the negativity of the integrand in (\ref{7}). However, we are more interested in the case when they exist %We will further discuss this problem in Section $4.4$.

%To understand the effect of having critical points in coefficients on the MSE performance of the orthogonal Cp criterion
To understand the MSE behavior of the orthogonal Cp criterion with critical features in the true regression model, we conduct a numerical experiment to compare the performances between the orthogonal Cp criterion and the orthogonal least squares method in the presence of critical features.

We generate data from $\boldsymbol { y } = \boldsymbol { X } \boldsymbol { \beta } + \boldsymbol { \varepsilon }$ as described previously ($\boldsymbol{X}^{\top}\boldsymbol{X}=nI$), where each element of $\boldsymbol{\varepsilon} : = \left( \varepsilon _ { 1 }, \varepsilon _ { 2 }, \cdots, \varepsilon _ { n } \right) ^ { \top }$ is a standard Gaussian noise. Let $\boldsymbol { \beta }=[1,-0.01, 0.2, -0.21,$ $ 0.19, 0.02]^{\top}$ in which the third, fourth, fifth elements are at (or nearby) the critical points $\pm\sqrt{2\sigma_1^2/n}$ corresponding to the case when $n=50$ and $\sigma_1=1$. We simulated data with $n=(10,15,20,25,\cdots,90,95,100)$. For each sample size, we randomly simulated $5000$ \textcolor{black}{datasets}, and applied the Cp and standard least squares method. In addition, we set the tuning parameter in the Cp model as $\lambda=2$ (the logic behind this choice is found in Proposition \ref{proposition5}). We see an intersection point of the two resulting MSE curves (see (1) in Figure \ref{critical points}). Beyond this point, the performance of the orthogonal Cp criterion no longer surpasses that of the least squares method. Moreover, the maximum difference between the MSE values of the orthogonal Cp estimator and orthogonal least squares estimator occurs at $n=50$ (see picture (2) in Figure \ref{critical points}), which is consistent with our analysis and results.

\begin{figure}[htbp]
	\begin{minipage}[t] {0.5\linewidth}
		\centerline{\includegraphics[height=6.5cm,width=8.0cm]{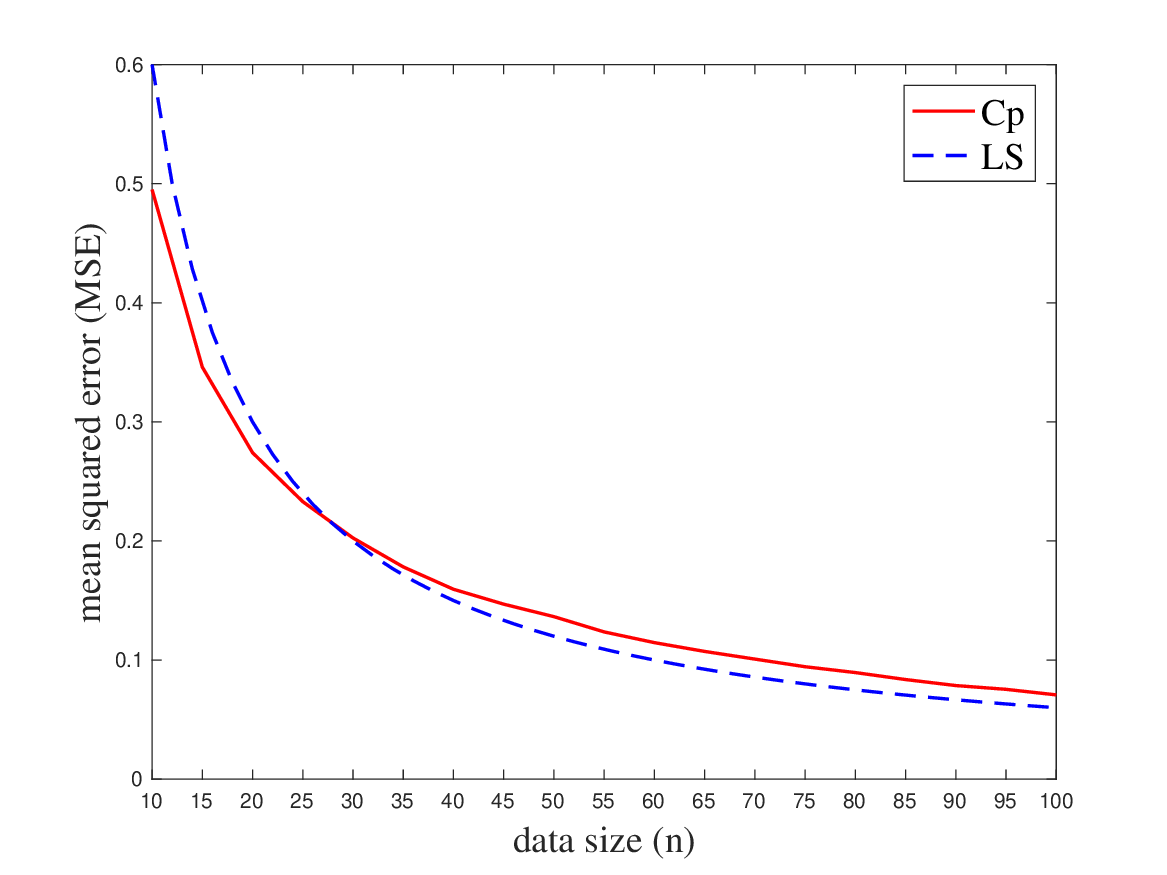}}
		\centerline{(1)}
	\end{minipage}
	\hfill
	\begin{minipage}[t] {0.5\linewidth}
		\centerline{\includegraphics[height=6.5cm,width=8.0cm]{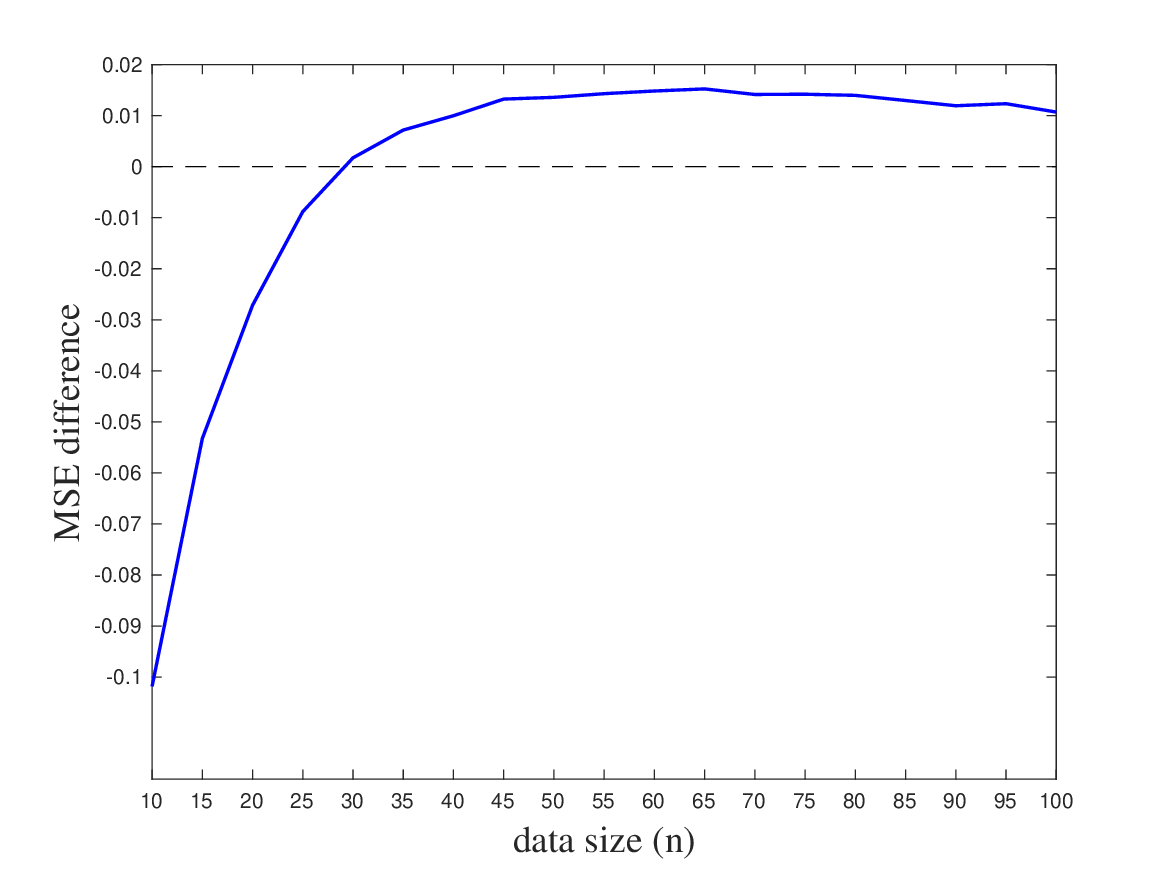}}
		\centerline{(2)}
	\end{minipage}
	\caption{\small The MSE performance comparison between the Cp criterion and the least squares method in the presence of critical points in the true regression coefficients. {Figure} (1) shows an intersection point. {Figure} (2) depicts the difference between the MSE values of the orthogonal Cp estimator and the orthogonal least squares estimator. }
	%	\label{critical points}
	\label{critical points}
\end{figure}

%Even though we have empirically illustrated the that Cp may not always outperform the least-square method in terms of MSE metric, when there exist some critical points in $\boldsymbol{\beta}$, see Figure \ref{critical points}.

%Except for the case when the true regression model is sparse, we can also find that it is possible for the integrand in (\ref{7}) to be negative under some special choices of parameter $\lambda$. This observation motivates us to find out whether the combination between transfer learning and orthogonal Cp (we will refer as orthogonal TLCp hereafter) can produce an estimator with a smaller MSE value.
As this section illustrates, %orthogonal Cp criterion may fail to identify critical features whose true coefficients are $\pm\sqrt{2\sigma_1^2/n}$, which can then lead to poor MSE performance.
failing to identify critical features can lead to poor MSE performance of the orthogonal Cp method.
However, our analysis is based on the assumption that the size of available training data in the target domain $n$ is small. In this case, ignoring the critical features is inappropriate, as these features may have a significant impact on the MSE value. Therefore, our aim is to ameliorate this problem by incorporating transfer learning into the Cp criterion (referred to as TLCp hereafter). Intuitively, we can expect the orthogonal TLCp estimator to get a lower MSE value if it helps re-identify the critical features. %If a feature whose squared coefficient value $\frac{2\sigma_1^2}{n}$ is significant, it cannot be ignored when compared to the value $\frac{\sigma_1^2}{n}$, which is the MSE that contributes from the noise of each dimension.%, by noticing that the MSE value that contributes from the noise for each dimension is only $\frac{\sigma_1^2}{n}$ .
%Intuitively, we can expect the orthogonal TLCp estimator to get a lower MSE value if it can help to ameliorate the shortages of orthgonal Cp as described in Proposition \ref{proposition3}. Since we can think of a feature whose squared coefficient value $\frac{2\sigma_1^2}{n}$ is significant and cannot be ignored, by noticing that the MSE value that contributes from the noise for each dimension is only $\frac{\sigma_1^2}{n}$ .
% Combining these observations and analysis we can %Therefore, orthogonal Cp is superior to the orthogonal least-square methods from the perspective of MSE measure, if the considering true regression model is sparse.
%Based on these observations, we try to incorporate transfer learning with the Cp criterion to discount the instability of the orthogonal Cp estimator to identify all relevant features in the presence of critical points in the true regression coefficients.

\section{TLCp Approach for Feature Selection}

In this section, we describe the developed TLCp model, which provides a remedy to the disadvantage of the Cp criterion. However, the proposed TLCp scheme is not simply a combination of two learning methods. Rather, we will show that the proposed orthogonal TLCp model has the potential to outperform the orthogonal Cp (\ref{the modified Cp}) in virtue of the embedded transfer learning technique. Specifically, our results prove the superiority of the orthogonal TLCp method in both stability (with respect to feature selection) and accuracy (in terms of MSE measure), when the tuning parameters are chosen based on explicit rules that we provide.

\subsection{Transfer Learning with Mallows' Cp}

Before introducing our proposed TLCp learning framework, we will first illustrate the corresponding problem setting. In addition to the training set $\{(x_1^i, x_2^i, \cdots, x_k^i;y_i)\}_{i=1}^{n}$ for the target regression task previously mentioned, there are several source domains in which the corresponding tasks are similar to the target. Our intuitive motivation is to borrow (common) knowledge from the source tasks for enhancing the prediction capacity of the target task. {Here, without loss of generality, we only consider one source task. The TLCp with more than two tasks (abbreviated as ``general TLCp'') will be discussed in Appendix B.}

Specifically, we define the source training set as $\{(\tilde{x}_1^i, \tilde{x}_2^i, \cdots, \tilde{x}_k^i;\tilde{y}_i)\}_{i=1}^{m}$, which are i.i.d. sampled from the true but unknown relation $\boldsymbol { \tilde{y} } = \boldsymbol { \tilde{X} } (\boldsymbol { \beta}+\boldsymbol {\delta}) + \boldsymbol { \eta }$, where $\boldsymbol { \tilde{y} } : = \left( \tilde{y} _ { 1 }, \tilde{y} _ { 2 }, \cdots, \tilde{y} _ { m } \right) ^ { \top }$, $\boldsymbol { \delta} : = \left( \delta _ { 1 }, \delta _ { 2 }, \cdots ,\delta _ { k } \right) ^ { \top }$ quantifies the dissimilarity between the target task and the source task, and $\boldsymbol { \tilde{X} } : = (\tilde{X}_1, \tilde{X}_2, \cdots, \tilde{X}_m) ^ { \top }=(\tilde{W}_1 \tilde{W}_2\cdots \tilde{W}_k)$ is the design matrix for the source task. We also denote the residual vector as $\eta : = \left( \eta _ { 1 }, \eta _ { 2 }, \cdots, \eta _ { m } \right) ^ { \top }$, where $\eta_i \sim \mathcal { N } \left( 0 , \sigma_2 ^ { 2 } \right)$ for $i=1,\cdots, m$.

Now, we can begin to build the TLCp model, which is obtained naturally by embedding the transfer learning technique (\ref{transfer learning}) into Mallows' Cp criterion (\ref{the modified Cp}), resulting in the following model:
\begin{equation}\label{TLCp}
\min_{\boldsymbol{v}_1,\boldsymbol{v}_2,\boldsymbol{w}_0}\sum _ { i = 1 } ^ { n }\lambda_1(y_i-\boldsymbol{w}_1^{\top}X_i)^2 + \sum _ { i = 1 } ^ { m } \lambda_2(\tilde{y_{i}}-\boldsymbol{w}_2^{\top}\tilde{X_{i}})^2 + \frac { 1 } { 2 } \sum _ { t = 1 } ^ { 2 }  \boldsymbol { v }_t^{\top}\boldsymbol{\lambda} _ { 3 }\boldsymbol{v}_t +\lambda_4\bar{p},
\end{equation}	
where $\boldsymbol{w}_1=\boldsymbol{w}_0+\boldsymbol{v}_1$, $\boldsymbol{w}_2=\boldsymbol{w}_0+\boldsymbol{v}_2$ are the regression coefficients of the learning tasks in the target domain and the source domain, respectively. $\boldsymbol{w}_0$ is a common parameter used to share information between two tasks, while $\boldsymbol{v}_1 , \boldsymbol{v}_2$ are individual parameters for the source task and target task, respectively. Moreover, the non-negative integer $\bar{p}$ in (\ref{TLCp}) indicates the number of regressors to be selected in the regression problem either in the target task or source task. We can also see how minimizing the objective function in (\ref{TLCp}) implicitly forces these two tasks to identify the same best subset jointly.
In view of this, $\bar{p}$ already quantifies the model complexity to be reduced, so we omit the regularization term $\|\boldsymbol{w}_0\|$ originating from (\ref{transfer learning}). {For each task, the designed $\boldsymbol{\lambda}_3:=\text{diag}(\lambda_3^1, \cdots, \lambda_3^k)$ is a parameter matrix. Each element of this matrix reflects the significance of the individual part of a regression coefficient for each feature. More specifically, an element of  $\boldsymbol{\lambda}_3$ indicates the 
degree of relatedness of the target and source tasks for the corresponding feature (this point can be checked in Corollary \ref{corollary 13}). In the extreme case where every attribute of the parameter matrix $\boldsymbol{\lambda}_3$ is $\infty$ and $\lambda_1=\lambda_{2}$, the proposed TLCp paradigm is equivalent to the ``aggregate Cp criterion.'' In that case, the Cp problem is trained on the whole dataset formed by combining data for all tasks. When every element of $\boldsymbol{\lambda}_3$ is $0$, then the corresponding TLCp scheme shares no parameters; it only shares the sparsity of tasks. }

The point of developing the TLCp procedure by integrating the transfer learning technique with the Cp criterion is to enhance the capacity of the original Cp to execute feature selection reliably even when the target data size $n$ is very small. Furthermore, we are interested in understanding the interaction between the transfer learning algorithm and the feature selection criteria. In general, a brute combination of two learning procedures is not guaranteed to have better performance than the individual models, unless the tuning parameters are well-chosen.
Therefore, we aim to derive a parameters tuning rule for the orthogonal TLCp procedure, so as to guarantee improved performance for any sample size $n$ in the target domain. %via tuning the parameters suitably.

\subsection{Estimator of Orthogonal TLCp Approach}

%In this section, we prove that the developed TLCp model (\ref{TLCp}) has the potential to outperform the original Cp (\ref{the modified Cp}) in virtue of the embedded transfer learning technique. Specifically, our results validate the superiorities of the proposed TLCp method in both stability (with respect to feature selection) and accuracy (in terms of the MSE measure), when the tuning parameters are well-chosen.
%show that in terms of the MSE measure with respect to the estimated regression coefficients, via tuning the parameters to improve the stability of TLCp to identify important features.%increase the probability of TLCp to identify the important features
%\textcolor{red}{this paragraph should be modified to illustrate that the advantage of transfer learning can be divided into two aspects if parameters are well-chosen, one is to improve the stability of the TLCp in terms of feature selection, the other one is to increase the accuracy of the TLCp in the sense of MSE.}

Similarly to the analysis of the orthogonal Cp, we can now derive in closed form the resulting estimator of the TLCp model (\ref{TLCp}) under the orthogonality assumption. This estimator will be referred to as the orthogonal TLCp estimator hereafter. As we focus on the learning task in the target domain, only the expression of the estimator for the target task will be shown in the following Proposition.

\begin{proposition}\label{TLCp solution}
If the conditions $\boldsymbol{X}^{\top}\boldsymbol{X}=nI$ and $\boldsymbol{\tilde{X}}^{\top}\boldsymbol{\tilde{X}}=mI$ hold, the estimated regression coefficients for the target learning task in the TLCp model are as follows:
\begin{eqnarray}\label{9}
\hat{w}_1^{i}=\left\{
\begin{matrix}
{ {\beta_{i}}+D_1^{i}{\delta}_{i}+(1-D_1^{i})\frac{1}{n}W_{i}^{\top} { \boldsymbol{\varepsilon} } }+D_1^{i}\frac{1}{m}\tilde{W}_{i}^{\top} \boldsymbol{ \eta }, & \text{if~} A_iH_i^2+B_iZ_i^2+C_iJ_i^2>\lambda_4\\
0,
&  \text{otherwise}
\end{matrix}
\right.
\end{eqnarray}
for $i=1, \cdots, k$, where $A_i=\frac{4\lambda_1\lambda_2^2m^2n}{4\lambda_1\lambda_2mn+m\lambda_2\lambda_3^i+n\lambda_1\lambda_3^i}$, $B_i=\frac{4\lambda_2\lambda_1^2mn^2}{4\lambda_1\lambda_2mn+m\lambda_2\lambda_3^i+n\lambda_1\lambda_3^i}$, $C_i=\frac{\lambda_3^{i}}{4\lambda_1\lambda_2mn+m\lambda_2\lambda_3^i+n\lambda_1\lambda_3^i}$ and $D_1^{i}=\frac{\lambda_2\lambda_3^i}{4\lambda_1\lambda_2n+\lambda_2\lambda_3^i+\frac{n}{m}\lambda_1\lambda_3^i}$ are determined by parameters $\lambda_1, \lambda_2, \lambda_3^i$. Above,
$H_i:=\beta_i+\delta_i+\frac{1}{m}\tilde{W}_{i}^{\top} \boldsymbol{ \eta }$, $Z_i:=\beta_i+\frac{1}{n}W_{i}^{\top} { \boldsymbol{\varepsilon} } $ and $J_i:=m\lambda_2H_i+n\lambda_1Z_i$ are random variables which arise due to the noises contained in responses $y_i$ and $\tilde{y}_i$ for the target and source tasks, $i=1, \cdots, k$.
\end{proposition}

\begin{proof}
The detailed proof of Proposition \ref{TLCp solution} can be found in Appendix \ref{appendixC}.
\end{proof}

Proposition \ref{TLCp solution} demonstrates that whether a feature is selected by the orthogonal TLCp procedure depends not only on the data in the target domain (i.e., $Z_i$) but also on the knowledge extracted from the source data (i.e., $H_i$).

The proposed orthogonal TLCp model reduces to the orthogonal Cp in the special case of $\lambda_2=0$. Although the expression of solution (\ref{9}) for the orthogonal TLCp is far more complicated than that of the orthogonal Cp, the expressions of these two estimators share a ``similar'' structure. This observation leads us to explore the essential relationships between these two learning frameworks.

\subsection{Stability Analysis of Orthogonal TLCp in Feature Selection}

In this section, we show the proposed orthogonal TLCp method's {advantages} over the orthogonal Cp criterion in terms of feature selection, by analyzing the probability of the proposed orthogonal TLCp estimator to select every relevant feature in the learned regression model.
\begin{theorem}\label{TLCp probability}
The probability $Pr^{TLCp}\{i\}$ of the orthogonal TLCp  to select the $i$-th feature in the regression model can be calculated as follows.
\begin{align}
&Pr^{TLCp}\{i\}\notag\\
&=Pr\left\{\left(2-\frac{Q_i}{\sqrt{M_iN_i}}\right)\left[\frac{\sqrt{M_i}H_i+\sqrt{N_i}Z_i}{2}\right]^2+\left( 2+\frac{Q_i}{\sqrt{M_iN_i}}\right
)\left[\frac{\sqrt{N_i}Z_i-\sqrt{M_i}H_i}{2}\right]^2>\lambda_4\right\}\label{10}\\
&=\frac{\sqrt{mn}}{\pi\sigma_1\sigma_2\sqrt{M_iN_i}}\iint_{\left(2-\frac{Q_i}{\sqrt{M_iN_i}}\right)x^2+\left( 2+\frac{Q_i}{\sqrt{M_iN_i}}
\right)y^2>\lambda_4}\exp\bigg\{-\frac{1}{2}\bigg[\frac{n}{N_i\sigma_1^2}\left(x+y-\sqrt{N_i}\beta_i\right)^2\label{11}\\
&\qquad\qquad\qquad\qquad\qquad\qquad\qquad\qquad\qquad\qquad+\frac{m}{M_i\sigma_2^2}\left(x-y-\sqrt{M_i}(\beta_i+\delta_i)\right)^2\bigg]\bigg\}dxdy \notag
\end{align}
for $i=1, \cdots, k$, where we define $D_1^i=\frac{\lambda_2\lambda_3^i}{4\lambda_1\lambda_2n+\lambda_2\lambda_3^i+\frac{n}{m}\lambda_1\lambda_3^i}$, $D_2^i=\frac{\lambda_1\lambda_3^i}{4\lambda_1\lambda_2m+\lambda_1\lambda_3^i+\frac{m}{n}\lambda_2\lambda_3^i}$, $D_3^i=\frac{2\lambda_1\lambda_2}{4\lambda_1\lambda_2+\frac{1}{n}\lambda_2\lambda_3^i+\frac{1}{m}\lambda_1\lambda_3^i}$, $\tilde{D}^i=\lambda_1n(D_1^i)^2+\lambda_2m(D_2^i)^2+\lambda_3^i(D_3^i)^2$, $M_i=-\tilde{D}^i+\lambda_2m$, $N_i=-\tilde{D}^i+\lambda_1n$ and $Q_i=-2\tilde{D}^i$.  Also, $H_i=\beta_i+\delta_i+\frac{1}{m}\tilde{W}_{i}^{\top} \boldsymbol{ \eta }$, $Z_i=\beta_i+\frac{1}{n}W_{i}^{\top} { \boldsymbol{\varepsilon} } $ are two random variables which are the same as those given by Proposition \ref{TLCp solution}, where $i=1, \cdots, k$.
\end{theorem}	
%it indeed produces a more stable and accurate learning scheme for feature selection.
\begin{proof}
	The detailed proof of Theorem \ref{TLCp probability} can be found in Appendix \ref{appendixC}.
\end{proof}

Theorem \ref{TLCp probability} reveals the factors that determine whether or not a variable is selected using the proposed TLCp method. We observe that the probability of the orthogonal TLCp to identify a feature (\ref{10}) reduces to the corresponding result for the orthogonal Cp when $\lambda_1=1$, $\lambda_2=0$ and $\lambda_4=\lambda$ ($\lambda$ is a tuning parameter of the original Cp criterion (\ref{the modified Cp})). Notice that the probability of the orthogonal Cp to select the $i$-th feature can be rewritten as $Pr\left\{nZ_i^2>\lambda\right\}$ for $i=1, \cdots, k$, so that the formula (\ref{10}) in Theorem \ref{TLCp probability} is an extension of the result in Theorem \ref{theorem 2}.

%Note that comparing the mechanisms of the orthogonal TLCp and Cp criteria to discriminate relevant features is the key step to exploiting the benefits of combining transfer learning strategy with Cp criterion in terms of feature selection. Therefore, the goal of Corollary \ref{TLCp scale-up} below is to appropriately scale down the formula (\ref{11}) in Theorem \ref{TLCp probability} so that the obtained lower bound looks ``similar'' to the result in Theorem \ref{theorem 2}. Based on that, we can find a way to tune the parameters in order to make the proposed orthogonal TLCp procedure superior to the orthogonal Cp with regard to identifying important features.
In the following section, we explore the benefits of incorporating transfer learning in Cp, by comparing the ability of the orthogonal TLCp estimator to identify relevant features with that of the orthogonal Cp estimator. Toward this goal, we derive an upper bound on the probability of the orthogonal TLCp estimator to select a feature (in Corollary \ref{TLCp scale-up}), so that it shares a ``similar'' structure with the orthogonal Cp estimator presented in Theorem \ref{theorem 2}.

\begin{corollary}\label{TLCp scale-up}
The probability $Pr^{TLCp}\{\bar{i}\}$ of the orthogonal TLCp procedure to miss the $i$-th feature in the learned regression model satisfies the following relationship
\begin{align}
&Pr^{TLCp}\{\bar{i}\}\notag\\
&\leq\frac{\sqrt{2}}{\sqrt{\pi}\sigma_1}\int_{4y^2<\frac{4\lambda_4\sigma_1^2G_i^2}{2-\frac{Q_i}{\sqrt{M_iN_i}}}}  ~~\exp\left\{-\frac{1}{2\sigma_1^2}\left[2y-G_i(\sqrt{M_i}+\sqrt{N_i})\sigma_1\beta_i-G_i\sqrt{M_i}\sigma_1\delta_i\right]^2\right\}dy,\label{12}
\end{align}
where $G_i=\sqrt{\frac{mn}{nM_i\sigma_2^2+mN_i\sigma_1^2}}$, and the equality holds if the dissimilarity between the tasks from target domain and source domain with respect to the $i$-th feature is zero, that is $\delta_i=0$ for $i=1, \cdots, k$. Above, $M_i,N_i$ are from Theorem \ref{TLCp probability} for $i=1, \cdots, k$.
\end{corollary}

\begin{proof}
	The detailed proof of Corollary \ref{TLCp scale-up} can be found in Appendix \ref{appendixC}.
\end{proof}

Corollary \ref{TLCp scale-up} establishes a bridge to find the correlations between the TLCp and Cp estimators under the orthogonality assumption. Remark \ref{remark 1} helps us understand this point.
%In fact, the derived upper bound of the probability of the orthogonal TLCp to identify
\begin{remark}\label{remark 1}
To provide an intuitive guidance for deriving a parameters tuning rule of the orthogonal TLCp procedure, we rewrite the result of Theorem \ref{theorem 2} as follows,
\begin{equation*}
Pr^{Cp}\{\bar{i}\}=1-Pr^{Cp}\{i\}=\frac{\sqrt{n}}{\sqrt{2\pi}\sigma_1}\int_{nx^2<\lambda}\exp\left\{-\frac{1}{2}\frac{\left(x-\beta_i\right)^2}{\frac{\sigma_1^2}{n}}\right\}dx.
\end{equation*}
Furthermore, let $y=\frac{\sqrt{n}}{2}x$ and substitute it into the formula above to obtain
\begin{equation*}
Pr^{Cp}\{\bar{i}\}=\frac{\sqrt{2}}{\sqrt{\pi}\sigma_1}\int_{4y^2<\lambda}\exp\left\{-\frac{\left(2y-\sqrt{n}\beta_i\right)^2}{2\sigma_1^2}\right\}dy,
\end{equation*}
which is ``similar'' to (\ref{12}) in Corollary \ref{TLCp scale-up}.
%We can infer from the results in Theorem \ref{theorem 2} that the probability of orthogonal Cp procedure to miss the $i$-th feature for $i=1,\cdots, k$ can be expressed as
%This observation provides us an intuitive guidance to tune the parameters of orthogonal TLCp model such that it can complement the disadvantages of orthogonal Cp method, i.e., as indicated in Proposition \ref{proposition3}, and then making it possible to outperform the orthogonal Cp criterion.
\end{remark}

%We have now accumulated the sufficient information to formally derive some explicit rules to tune the parameters of orthogonal TLCp procedure for rescuing some important features in the regression model that may be over-deleted by orthogonal Cp. In other words, we hope that the orthogonal TLCp with well-chosen parameters can produce an estimator better than the orthogonal Cp estimator with respect to feature selection.
In light of these observations, we have now accumulated sufficient information to theoretically derive the conditions under which the resulting orthogonal TLCp procedure will outperform the orthogonal Cp by re-identifying some important features in the regression model that may be missed by the orthogonal Cp criterion.

\begin{theorem}\label{tune parameters}
Consider the $\ell$-th feature with true regression coefficient  $\beta_{\ell}\neq0$. The probability of the orthogonal TLCp to identify this feature will be strictly larger than that of the orthogonal Cp, if the tuning parameters of the orthogonal TLCp procedure $\lambda_1, \lambda_2, \lambda_3^{\ell}, \lambda_4$ are chosen to satisfy the conditions
\begin{align}
&(1)~ |\sqrt{n}\beta_{\ell}|<|G_{\ell}\sigma_1|\cdot\left|(\sqrt{M_{\ell}}+\sqrt{N_{\ell}})\beta_{\ell}+\sqrt{M_{\ell}}\delta_{\ell}\right|,\label{13}\\
&(2) ~\lambda_4=\min_{i\in\{1,\cdots,k\}}\left\{\frac{\lambda\left(2-\frac{Q_i}{\sqrt{M_iN_i}}\right)}{4\sigma_1^2(G_i)^2}\right\}
\label{14}
\end{align}
where $\lambda$ is the tuning parameter of the original Cp criterion (\ref{the modified Cp}), and $M_{\ell}, N_{\ell}, Q_{\ell}, G_{\ell}$ are functions of $\lambda_1, \lambda_2, \lambda_3^{\ell}$ as introduced in Theorem \ref{TLCp probability}.
\end{theorem}

\begin{proof}
	The detailed proof of Theorem \ref{tune parameters} can be found in Appendix \ref{appendixC}.
\end{proof}

Theorem \ref{tune parameters} guarantees a set of good parameters to make the orthogonal TLCp superior to the orthogonal Cp in terms of feature selection. Specifically, this result shows that the TLCp estimator will be more stable than the orthogonal Cp estimator when identifying all relevant features (whose regression coefficients are {non-zero}). There is a higher chance for the orthogonal TLCp to identify these features when the tuning parameters satisfy the inequalities (\ref{13}) and (\ref{14}). However, inequalities (\ref{13}) and (\ref{14}), provided by Theorem \ref{tune parameters}, cannot be used as the explicit rules for tuning the parameters of the TLCp model, since the true regression coefficients are unknown.

 %when the tuning parameters and true (but unknown) regression coefficients satisfy the conditions  (\ref{13}) and (\ref{14}) given by Theorem \ref{tune parameters}, the orthogonal TLCp estimator will be more stable than orthogonal Cp estimator in identifying important features, i.e., the features whose regression coefficients are comparatively large, since there is a greater chance for the orthogonal TLCp to identify all the relevant features (whose regression coefficients are strictly more than zero) when the tuning parameters are well-chosen.

To demonstrate the crucial role that Theorem \ref{tune parameters} plays in this paper, Corollary \ref{corollary 13} will show that, based on the result of Theorem \ref{tune parameters}, we can derive an explicit rule to tune the parameters of the orthogonal TLCp model. This way, the resulting estimator will not only perform better in terms of identifying important features, but also have the potential to achieve a lower MSE metric. %Secondly, by Theorem \ref{tune parameters}, we can further realize (in Proposition \ref{proposition14}) that the proposed orthogonal TLCp approach is capable of solving the dilemma originated from the orthogonal Cp, as illustrated in Proposition \ref{proposition3}. Last but not least, we can analyze the capacity of the orthogonal TLCp to eliminate the irrelevant features (whose regression coefficients are zeros) in Section $4.5$(see remark \ref{remark 15}), according to Theorem \ref{tune parameters}.

\subsection{MSE Metric of Orthogonal TLCp Estimator}

%From now on, we begin to investigate the benefits of embedding transfer learning technique into orthogonal Cp criterion from the perspective of MSE measure for the resulting estimator.
In this section, we compare the proposed orthogonal TLCp estimator against the orthogonal Cp estimator in terms of the MSE measure.
To do that, we will explore the correlations between the MSE metric of the estimator induced from the orthogonal TLCp and orthogonal Cp procedures. We will begin by calculating the MSE measure of the orthogonal TLCp estimator.
\begin{theorem}\label{MSE TLCp}
The MSE measure of the orthogonal TLCp estimator $\hat{\boldsymbol{w}}_1$ can be represented as follows.
\begin{align}
&\text{MSE}(\hat{\boldsymbol{w}}_1)\label{MSE expression}\\
&=\sum_{i=1}^{k}\bigg\{\mathbb{E}[(\bar{M}_i{U_i}+\bar{N}_i{V_i}+\beta_i)^2]+\notag\\
&\qquad \iint_{\left(2-\frac{Q_i}{\sqrt{M_iN_i}}\right)x^2+\left( 2+\frac{Q_i}{\sqrt{M_iN_i}}\right)y^2<\lambda_4} \left[\beta_i^2-(\bar{M}_ix+\bar{N}_iy+\beta_i)^2\right]p\left(U_i=x,V_i=y\right)dxdy \bigg\}\notag,
\end{align}
where
\begin{itemize}
	\item $\bar{M}_i=\frac{\sqrt{M_i}-\sqrt{N_i}}{\sqrt{M_iN_i}}D_1^i-\frac{1}{\sqrt{N_i}}$, $\bar{N}_i=\frac{\sqrt{M_i}+\sqrt{N_i}}{\sqrt{M_iN_i}}D_1^i-\frac{1}{\sqrt{N_i}}$ are determined by $M_i, N_i, D_1^i$ as introduced Theorem \ref{TLCp probability}, for $i=1,\cdots, k$.
	\item $U_i=\frac{\sqrt{M_i}H_i+\sqrt{N_i}Z_i}{2}$, $V_i=\frac{-\sqrt{M_i}H_i+\sqrt{N_i}Z_i}{2}$ are two random variables related to $H_i, Z_i$ in Theorem \ref{TLCp probability}, for $i=1,\cdots, k$.
	\item $p\left(U_i=x,V_i=y\right)=T_i\exp\bigg\{-\frac{1}{2}\bigg[\frac{n\left(x+y-\sqrt{N_i}\beta_i\right)^2}{N_i\sigma_1^2}+\frac{m\left(x-y-\sqrt{M_i}(\beta_i+\delta_i)\right)^2}{M_i\sigma_2^2}\bigg]\bigg\}$ is the joint density distribution for $U_i, V_i$, and $T_i=\frac{\sqrt{mn}}{\pi\sigma_1\sigma_2\sqrt{M_iN_i}}$, for $i=1, \cdots, k$.
\end{itemize}
\end{theorem}

\begin{proof}
	The detailed proof of Theorem \ref{MSE TLCp} can be found in Appendix \ref{appendixC}.
\end{proof}

The MSE metric of the orthogonal TLCp estimator $\hat{\boldsymbol{w}}_1$ shares a similar structure with that of the orthogonal Cp estimator $\hat{\boldsymbol{a}}$ being the summation of two terms. The first one represents the MSE measure of the estimator obtained by combining the orthogonal least squares method with transfer learning (LSTL), and the second one implies the potential reduction of the MSE value that benefited from the feature selection technique (or the inclusion of regularization term $\bar{p}$ in the LSTL). Formula (\ref{MSE expression}) reduces to the MSE metric of the orthogonal LSTL estimator if $\lambda_4=0$. In this case, only the first term exists.

%For simplicity, we only use the first term of the formula (\ref{MSE expression}), which represents the MSE metric with respect to orthogonal TLCp estimator. In other words, we first investigate the advantage of using the transfer learning technique in the foundation of simple LS, by comparing the performances between the orthogonal LSTL estimator and the orthogonal LS estimator in the sense of MSE. Specifically, we will exploit whether there exists an explicit rule to tune the parameters $\lambda_1, \lambda_2, \lambda_3^i$ such that the resulting value of the term $\mathbb{E}[(\bar{M}_i{U_i}+\bar{N}_i{V_i}+\beta_i)^2]$ can be strictly less than $\sigma_1^2/n$, for $i=1, \cdots, k$. We answer this question in the following proposition.

\subsection{{Parameter} Tuning for Orthogonal TLCp}

In this section, we formally derive an explicit parameters rule of the orthogonal TLCp model, such that the resulting estimator will outperform the orthogonal Cp estimator in both accuracy and stability.

As illustrated above, the first summation term of formula (\ref{MSE expression}), $\sum_{i=1}^{k}\mathbb{E}[(\bar{M}_i{U_i}+\bar{N}_i{V_i}+\beta_i)^2]$, is the exact MSE value of the orthogonal LSTL estimator.
First, we will investigate the advantage of using the transfer learning technique in the context of the simple least squares method (in Proposition \ref{explicit rule}), by exploiting whether there exists an explicit rule to tune the parameters $\lambda_1, \lambda_2, \lambda_3^i$, such that the resulting value of the term $\sum_{i=1}^{k}\mathbb{E}[(\bar{M}_i{U_i}+\bar{N}_i{V_i}+\beta_i)^2]$ can be minimized and become less than $k\sigma_1^2/n$ (the MSE value of orthogonal LS estimator).

\begin{proposition}\label{explicit rule}
Let the parameters of the orthogonal TLCp procedure be set as $\lambda_1^*=\sigma_2^2, \lambda_2^*=\sigma_1^2, {\lambda_3^i}^*=\frac{4\sigma_1^2\sigma_2^2}{\delta_i^2}$, %Or equivalently, which satisfy  ${D_1^i}^*=\frac{\frac{\sigma_1^2}{n}}{\delta_i^2+\frac{\sigma_1^2}{n}+\frac{\sigma_2^2}{m}}$, where $D_1^{i}=\frac{\lambda_2\lambda_3^i}{4\lambda_1\lambda_2n+\lambda_2\lambda_3^i+\frac{n}{m}\lambda_1\lambda_3^i}$,
for $i=1,\cdots,k$. Then, we guarantee that $\sum_{i=1}^{k}\mathbb{E}[(\bar{M}_i^*{U_i^*}+\bar{N}_i^*{V_i^*}+\beta_i)^2]$ is minimized and less than $k\sigma_1^2/n$ (the MSE value of orthogonal LS estimator). Here $\bar{M}_i^*, \bar{N}_i^*, U_i^*, V_i^*$ represent the corresponding values of $\bar{M}_i, \bar{N}_i, U_i, V_i$ after substituting $\lambda_1^*,\lambda_2^*,{\lambda_3^i}^*$ for $i=1,\cdots,k$ in the expressions of Theorem \ref{MSE TLCp}.
\end{proposition}

\begin{proof}
	The detailed proof of Proposition \ref{explicit rule} can be found in Appendix \ref{appendixC}.
\end{proof}

%\begin{remark}
%The derived parameters tuning rule with respect to the orthogonal LSTL procedure is actually optimal in that it minimizes the MSE metric of the orthogonal LSTL estimator. Check the proof of Proposition \ref{explicit rule} in Appendix and see (\ref{39}) therein.
%We can see from the proof of Proposition \ref{explicit rule} that the derived ${D_1^i}^*$ is actually the optimal solution to minimize $F(D_1^i):=\mathbb{E}(\bar{M}U_i+\bar{N}V_i+\beta_i)^2$ (see (\ref{39})), for $i=1,\cdots,k$. In other words, the way to tune the parameters $\lambda_1, \lambda_2, \lambda_3^i$ provided by proposition  \ref{explicit rule} is optimal in that it minimizes the MSE metric of the orthogonal LSTL estimator, for $i=1,\cdots,k$.
%\end{remark}
%Notice that one basic assumption for all the feature selection learning schemes is the sparsity of the true regression function based on which the training data are sampled. Thus, it is natural to exploit whether or not the orthogonal TLCp estimator will outperform the orthogonal Cp estimator, when there are several irrelevant features (the corresponding true regression coefficients of which are zeros) under the considered circumstances. The following remark will show that when there are superfluous features in the true regression model, the orthogonal TLCp procedure will choose the redundant features with some probability, which indicates a shortcoming of using the orthogonal TLCp method. Without the loss of generality, we focus on the case when there is only one
%redundant feature in the true regression model in remark \ref{remark 10}.
Theorem \ref{tune parameters} guaranteed the existence of a set of good parameters of the orthogonal TLCp model, such that it can outperform the orthogonal Cp criterion in feature selection. Below, we explicitly provide a parameters tuning rule of the orthogonal TLCp procedure (based on the result of Proposition \ref{explicit rule}) to achieve this goal.

%we have provided a theoretical way to tune the parameters $\lambda_1, \lambda_2, \lambda_3^i$ ($i=1,\cdots,k$) in Theorem \ref{tune parameters}, which is to increase the probability of the orthogonal TLCp to identify important features. Under some conditions, the method to tune the parameters given by Proposition \ref{explicit rule} can also lead to the improvement of the orthogonal TLCp to identify important features.
\begin{corollary}\label{corollary 13}
 For any relevant feature, say $t$-th, whose corresponding true regression coefficient $\beta_t\neq0$, set the parameters of orthogonal TLCp model as $\lambda_1^*=\sigma_2^2, \lambda_2^*=\sigma_1^2, {\lambda_3^t}^*=\frac{4\sigma_1^2\sigma_2^2}{\delta_t^2}$ and $\lambda_4^*=\min_{i\in\{1,\cdots,k\}}\left\{\lambda\left(2-\frac{Q_i^*}{\sqrt{M_i^*N_i^*}}\right)/ 4\sigma_1^2(G_i^*)^2\right\}$, where $M_t^*,N_t^*, Q_t^*,G_t^*$ are obtained by substituting $\lambda_1^*,\lambda_2^*, {\lambda_3^t}^*$ into the expressions of $M_t, N_t, Q_t, G_t$ that was previously introduced. Then, there is a constant $\kappa(\sigma_1,\sigma_2,m,n)>0$, such that the probability of the orthogonal TLCp to select the $t$-th feature will be strictly higher than that of the orthogonal Cp, provided that $|\delta_t|\leq\kappa$.
\end{corollary}

\begin{proof}
	The detailed proof of Corollary \ref{corollary 13} can be found in Appendix \ref{appendixC}.
\end{proof}

Corollary \ref{corollary 13} realizes the parameters tuning rule given by Theorem \ref{tune parameters}. Specifically, it reveals that when the dissimilarity of the tasks from the target domain and source domain $\delta_t$, with regards to any relevant feature, is upper bounded by the constant $\kappa$, then we can find a set of good parameters. Based on these parameters, the resulting orthogonal TLCp estimator can outperform the orthogonal Cp estimator in identifying all relevant features and potentially get a lower MSE measure. Moreover, when $\boldsymbol{\delta}=\boldsymbol{0}$, meaning the target and source tasks overlap (in this case, $\lambda_1^*=\sigma_2^2, \lambda_2^*=\sigma_1^2, {\lambda_3^t}^*=\infty$, for $t=1,\cdots,k$, and $\lambda_4^*=2\sigma_1^2\sigma_2^2$ (where we assume $\lambda=2\sigma_1^2$), and the infinity of each $\lambda_3^t$ leads to the vanishment of $\boldsymbol{v}_1$ and $\boldsymbol{v}_2$ in the TLCp model (\ref{TLCp})), the resulting TLCp procedure will {reduce to} Mallows' Cp. This particular scenario illustrates the reasonableness of the proposed criteria for tuning parameters. For simplicity, we set $\lambda=2\sigma_1^2$ hereinafter.

Proposition \ref{proposition14} will validate another advantage of the proposed orthogonal TLCp: it addresses to some extent the problem of the orthogonal Cp, which removes critical features with probability 0.5.

\begin{proposition}\label{proposition14}
For any critical feature (say the $\gamma$-th) whose true regression coefficient satisfies the equality $\beta_{\gamma}^2=\frac{2\sigma_1^2}{n}$, the probability of the orthogonal TLCp to identify this feature is $1-\left[\phi\left({\sqrt{2}-\sqrt{\frac{2(n+m)}{n}}}\right)-\phi\left({-\sqrt{2}-\sqrt{\frac{2(n+m)}{n}}}\right)\right]$, if $\delta_{\gamma}=0$, $\sigma_1=\sigma_2$ and the parameters of orthogonal TLCp, and  $\lambda_1,\lambda_2,\lambda_3^i, \lambda_4$ ($i=1,\cdots, k$) are tuned based on Corollary \ref{corollary 13}, where $\phi(u)$= $\int_{-\infty}^{u}\frac{1}{\sqrt{2\pi}}\exp{\left\{-\frac{x^2}{2}\right\}}dx$.
\end{proposition}

\begin{proof}
	The detailed proof of Proposition \ref{proposition14} can be found in Appendix \ref{appendixC}.
\end{proof}

To quantitatively understand the advantages of the orthogonal TLCp over the orthogonal Cp for identifying critical features, we focus on the particular scenario of $n=m$ in Proposition \ref{proposition14}. In fact, when model parameters are tuned based on the rule given by Corollary \ref{corollary 13}, then the probability of the orthogonal TLCp to select critical features will be strictly higher than $0.5$ and up to $1-[\phi(\sqrt{2}-2)-\phi(-\sqrt{2}-2)]\approx0.72$, provided that the task dissimilarity $\delta_\gamma$ is sufficiently small, which reflects the superiority of the proposed TLCp model.

Although the conclusion of Proposition \ref{proposition14} is proven by assuming $\delta_\gamma=0$, the experimental results (in Figure \ref{bar}, Subsection $6.1$) show that this limitation can be relaxed to some extent, and the resulting orthogonal TLCp estimator can still perform better than the orthogonal Cp estimator when identifying critical features.

In the following remark, we will investigate the performance of the orthogonal TLCp procedure for situations where {superfluous} features are found in the true regression model.
\begin{remark}\label{remark 15}
We set $\lambda=2\sigma_1^2$. If the $s$-th feature in the true regression model is superfluous (that is, its true regression coefficient satisfies $\beta_s=0$), then the probability of the orthogonal TLCp procedure to select the {superfluous} feature is approximately equal to that of the orthogonal Cp, if the task dissimilarity measure $\delta_s$ with respect to the $s$-th feature is sufficiently small and the tuning parameters of orthogonal TLCp model $\lambda_1,\lambda_2,\lambda_3^i, \lambda_4$ ($i=1,\cdots, k$) are tuned based on Corollary \ref{corollary 13}. This argument is easily verified by combining the results of Corollaries \ref{TLCp scale-up} and \ref{corollary 13}, and comparing them to the result in Theorem \ref{theorem 2}.
\end{remark}

Remark \ref{remark 15} indicates the possibility (with some probability) for the orthogonal TLCp to select superfluous features even when model parameters are tuned, based on the rule given by Corollary \ref{corollary 13}. This phenomenon is due to the inherent drawback of the orthogonal Cp, on which the proposed TLCp is based. {Since other information criteria, such as BIC, may be able to better address the problem of selecting superfluous features \citep{dziak2020sensitivity}, we may be able to mitigate this problem at least to some extent by applying transfer learning to other information criteria}. %We leave this issue for future work. 

{Remark \ref{remark19} below builds a connection between the orthogonal TLCp procedure (when $\boldsymbol{\delta}=\boldsymbol{0}$) and the statistical tests, which also facilitates understanding of the advantages of the TLCp procedure over the Cp criterion.}

{\begin{remark}\label{remark19}
For the $i$-th feature, using the orthogonal TLCp (with its parameters tuned optimally based on the rules in Corollary \ref{corollary 13}) to decide whether to select it or not amounts to implementing a chi-squared test with respect to the statistic $A_iH_i^2+B_iZ_i^2+C_iJ_i^2$ for the significance level $\alpha_3=1-F(2;1)$ ($\approx 0.16$) and the power $1-\tilde{\gamma}$ in the special case of $\delta_i=0$. $F(2;1)$ is the cumulative distribution function of the chi-squared distribution with $1$ degree of freedom at value $2$, $\tilde{\gamma}=\phi\left(\sqrt{2}-\frac{\sqrt{m\sigma_1^2+n\sigma_2^2}\beta_i}{\sigma_1\sigma_2}\right)-\phi\left(-\sqrt{2}-\frac{\sqrt{m\sigma_1^2+n\sigma_2^2}\beta_i}{\sigma_1\sigma_2}\right)$, and $A_i=\frac{4\lambda_1\lambda_2^2m^2n}{4\lambda_1\lambda_2mn+m\lambda_2\lambda_3+n\lambda_1\lambda_3^i}$, $B_i=\frac{4\lambda_2\lambda_1^2mn^2}{4\lambda_1\lambda_2mn+m\lambda_2\lambda_3^i+n\lambda_1\lambda_3^i}$ and $C_i=\frac{\lambda_3^{i}}{4\lambda_1\lambda_2mn+m\lambda_2\lambda_3+n\lambda_1\lambda_3^i}$, are functions of the parameters $\lambda_1, \lambda_2, \lambda_3^i$. $H_i=\beta_i+\delta_i+\frac{1}{m}\tilde{W}_{i}^{\top} \boldsymbol{ \eta }$, $Z_i=\beta_i+\frac{1}{n}W_{i}^{\top} { \boldsymbol{\varepsilon} } $ are two random variables that stem from the responses $\boldsymbol{y}$ and $\boldsymbol{\tilde{y}}$ for the target and source tasks, respectively. $J_i=m\lambda_2H_i+n\lambda_1Z_i$. More details and proofs can be found in Appendix A.
\end{remark}}

Now, we begin to exploit the MSE performance of the orthogonal TLCp estimator under the condition that the tuning parameters of the orthogonal TLCp model are chosen based on the rule given by Corollary \ref{corollary 13}.
As stated in Proposition \ref{explicit rule}, we can find the optimal set of model parameters to minimize the first summation term of the MSE metric of the orthogonal TLCp estimator (\ref{MSE expression}). Below, we estimate the second term.

\begin{lemma}\label{lemma17}
We set the parameters of the orthogonal TLCp as $\lambda_1^*=\sigma_2^2, \lambda_2^*=\sigma_1^2, {\lambda_3^i}^*=\frac{4\sigma_1^2\sigma_2^2}{\delta_i^2}$ and $\lambda_4^*=\min_{i\in\{1,\cdots,k\}}\left\{\lambda\left(2-\frac{Q_i^*}{\sqrt{M_i^*N_i^*}}\right)/ 4\sigma_1^2(G_i^*)^2\right\}$ for $i=1,\cdots, k$. In this case, we further denote the second summation term of the MSE of the orthogonal TLCp (\ref{MSE expression}) as
\begin{align}
&\tilde{F}_i(\delta_i):=\notag\\&\iint_{\left(2-\frac{Q_i^*}{\sqrt{M_i^*N_i^*}}\right)x^2+\left(2+\frac{Q_i^*}{\sqrt{M_i^*N_i^*}}\right)y^2<\lambda_4}\left[\beta_i^2-(\bar{M}_i^*x+\bar{N}_i^*y+\beta_i)^2\right]p\left(U_i^*=x,V_i^*=y\right)dxdy\notag,
\end{align}
for $i=1, \cdots,k$.
Then, $\tilde{F}_i(0)\leq \left[\frac{4|\beta_i|\sqrt{2\sigma_1^2}}{\sqrt{\pi\tilde{G}}}-\frac{4\sigma_1^2}{\tilde{G}\sqrt{\pi}}\right]\exp\left\{-\frac{(\sqrt{\tilde{G}}|\beta_i|-\sqrt{2\sigma_1^2})^2}{2\sigma_1^2}\right\}$, where $\tilde{G}=\frac{m\sigma_1^2+n\sigma_2^2}{\sigma_2^2}$, if $\beta_i^2\geq \frac{2\sigma_1^2}{n}$, for $i=1, \cdots,k$.
\end{lemma}

\begin{proof}
	The detailed proof of Lemma \ref{lemma17} can be found in Appendix \ref{appendixC}.
\end{proof}

As illustrated in Proposition \ref{proposition3}, $\pm\sqrt{2/n}\sigma_1$ are two critical points of regression coefficient values for determining whether or not the corresponding features can be identified by the orthogonal Cp criterion. Therefore, in the following theorem, we will check the MSE performances of the orthogonal TLCp estimator when $\beta_i^2> \frac{2\sigma_1^2}{n}$, $\beta_i^2< \frac{2\sigma_1^2}{n}$ and $\beta_i^2= \frac{2\sigma_1^2}{n}$, where $i=1, \cdots,k$.

\begin{theorem}\label{theorem18}
Let the parameters of the orthogonal TLCp be set as $\lambda_1^*=\sigma_2^2, \lambda_2^*=\sigma_1^2, {\lambda_3^i}^*=\frac{4\sigma_1^2\sigma_2^2}{\delta_i^2}$ and $\lambda_4^*=\min_{i\in\{1,\cdots,k\}}\left\{\lambda\left(2-\frac{Q_i^*}{\sqrt{M_i^*N_i^*}}\right)/ 4\sigma_1^2(G_i^*)^2\right\}$ for $i=1,\cdots, k$. Also, we denote
$K=\frac{\frac{\sigma_1^2}{n}}{\frac{\sigma_1^2}{n}+\frac{\sigma_2^2}{m}}$.
 %$\tilde{K}_1=\frac{\tilde{G}}{4\sqrt{n}}-\frac{\sqrt{n}}{4}+\frac{1}{2}$, $\tilde{K}_2=\frac{\sqrt{\pi}\sigma_1^2m}{16\sigma_2^2\sqrt{n}}$,where $\tilde{G}=\frac{m\sigma_1^2+n\sigma_2^2}{\sigma_2^2}$.
Then, there are two positive constants $\rho(\sigma_1,\sigma_2,m,n)$ and $\tilde{\kappa}(\sigma_1,\sigma_2,m,n)$, where $\tilde{\kappa}$ depends on $\rho$, such that the MSE metric of the resulting orthogonal TLCp estimator will be strictly less than that of the orthogonal Cp estimator, provided that $\|\boldsymbol{\delta}\|_2<\tilde{\kappa}$ and
%\begin{equation*}
%	\beta_i^2>\frac{2\sigma_1^2}{n}\left[1+\frac{1+\sqrt{1-4\tilde{K}_2^2+4\tilde{K}_1\tilde{K}_2}}{2\tilde{K}_2}\right]^2 \quad \text{or} \quad \beta_i^2<\rho^2, \quad \text{for } i=1,\cdots, k.
%\end{equation*}
\begin{equation*}
\beta_i^2\geq \frac{2\sigma_1^2}{n}\left[1+\sqrt{-\ln\left(\frac{\sqrt{\pi}}{8}K\right)}\right]^2 \quad \text{or} \quad \beta_i^2<\rho^2, \quad \text{for } i=1,\cdots, k.
\end{equation*}
\end{theorem}

\begin{proof}
	The detailed proof of Theorem \ref{theorem18} can be found in Appendix \ref{appendixC}.
\end{proof}

Theorem \ref{theorem18} suggests the following: when the parameters of the orthogonal TLCp model are tuned as stated in Corollary \ref{corollary 13}, and if the true regression coefficients deviate from the critical points $\pm\sqrt{2/n}\sigma_1$ to a certain extent, then we can theoretically guarantee that the proposed orthogonal TLCp estimator will be superior to the orthogonal Cp estimator in terms of the MSE value. In the simulation part (Section $6$), we will test our theory and investigate whether the orthogonal TLCp estimator can still lead to better MSE performance than the orthogonal Cp estimator when there are several critical features in the true regression model.

%presents the sufficient conditions under which the orthogonal TLCp estimator will be superior to the orthogonal Cp estimator in the sense of MSE metric, which embodies the advantages of using transfer learning technique in the framework of Mallows' Cp criterion. Reiterate that Proposition \ref{proposition14} also indicates the benefits of combining transfer learning with Cp such that it can enhance the probability of the resulting estimator to identify ``important'' features whose regression coefficients are near the critical points $\pm\sqrt{2/n}\sigma_1$. In the following, we will \textcolor{red}{empirically} investigate whether the orthogonal TLCp estimator leads to better MSE performance than orthogonal Cp estimator when there are features whose regression coefficients are near the critical points, due to the difficulty of analyzing the MSE formulas in this case.

\section{Extensions}

%To clarify, the results in Section $3$ do not imply that Mallows' Cp criterion is unsuitable for feature selection tasks. In general, the Cp criterion properly identifies the true set of important features. Our results simply corroborate on the argument stated in \citet{mallows1973some}, that Mallows' Cp criterion is not always recommended (e.g., under limited sample size).

While we verified the effectiveness of the proposed TLCp model, the key assumption in our analytical framework thus far has been the orthogonality of the regression problem, which raises the question of the practicality of the above results. In this section, we will relax the orthogonality assumption. Moreover, we will investigate whether our analytical framework can be generalized to feature selection criteria other than Cp.  %In this section, we discuss whether it is possible to obtain some similar properties of Cp-type models without the orthogonal assumption. %Furthermore, we discuss whether our analysis methods with respect to Mallows' Cp can be generalized to other feature selection criteria.

\subsection{An Asymptotic Solution of Cp in the Non-orthogonal Case}

Since solving a non-orthogonal Cp problem is NP-hard, it is unlikely that a closed-form efficient solution can be derived for it. Thus, we derive an approximation to the solution of Cp in the non-orthogonal case ({the ``approximate Cp'' in this context}).

%As presented in Section $3$, it is easy and straightforward to analyze the solutions of Mallows' Cp criterion in the special case of orthogonal design matrix. However, in general cases, solving the Cp model is a NP-hard problem, and we cannot even provide the closed-form solutions. Any non-singular matrix can be mathematically transformed into the orthogonal one (see Lemma \ref{lemma19} below). Using this technical skill, we propose a solution to asymptotically approximate the real solution of the non-orthogonal Cp problem. %can find a way to (asymptotically) approximate the solution of Cp criterion without orthogonal assumption. In order to achieve this goal, we will first give a lemma below.
To achieve this goal, we first define the non-orthogonal Cp problem as
\begin{equation}\label{non-orthogonal Cp}
\boldsymbol{\hat{\alpha}}=\text{argmin}_{\boldsymbol{\alpha}}~ ( \boldsymbol { \bar{y} } - \boldsymbol { \bar{X} }  \boldsymbol { \alpha }  ) ^ { \top } ( \boldsymbol { \bar{y} } - \boldsymbol { \bar{X} } \boldsymbol { \alpha }  )+\lambda \|\boldsymbol{\alpha}\|_{0},
\end{equation}
where we assume the data points are sampled from $\boldsymbol { \bar{y} } = \boldsymbol { \bar{X} } \boldsymbol { \beta } + \boldsymbol { \varepsilon }$ (each item of vector $\boldsymbol { \varepsilon }$ is i.i.d. distributed with $ \mathcal { N } \left( 0 , \sigma_1 ^ { 2 } \right)$), and the design matrix $\boldsymbol { \bar{X} } $ does not necessarily satisfy $\boldsymbol { \bar{X} }^{\top}\boldsymbol { \bar{X} }=n\boldsymbol{I}$. %Without loss of generality, we assume $\boldsymbol { \bar{X} } $ is full column rank and that the minimum and maximum eigenvalues of $\boldsymbol { \bar{X} }^{\top}\boldsymbol { \bar{X} }$ are of order $n$.

Based on Lemma \ref{lemma19}, we can further denote the orthogonalized version of the Cp problem (\ref{non-orthogonal Cp}) as follows,
\begin{equation}\label{orthogonalized Cp}
\boldsymbol{\hat{\alpha}}_1=\text{argmin}_{\boldsymbol{\alpha}_1}~ ( \boldsymbol { \bar{y} } - \boldsymbol { \bar{X}Q }  \boldsymbol { \alpha }_1  ) ^ { \top } ( \boldsymbol { \bar{y} } - \boldsymbol { \bar{X}Q } \boldsymbol {\alpha}_1  )+\lambda \|\boldsymbol{\alpha}_1\|_{0},
\end{equation}
where $\boldsymbol{Q}$ is an invertible matrix such that $(\boldsymbol{\bar{X}Q})^{\top}\boldsymbol{\bar{X}Q}=n\boldsymbol{I}$.

To find an estimator that can asymptotically approximate the solution of the non-orthogonal Cp problem $\boldsymbol{\hat{\alpha}}$, as the number of data points $n$ goes to infinity, we first solve the orthogonalized Cp model described in (\ref{orthogonalized Cp}). Seeing that $\boldsymbol{\hat{\alpha}}_1$ and $\boldsymbol{\hat{\alpha}}$ belong to different feature spaces with different coordinates, we back-transform the obtained solution $\boldsymbol{\hat{\alpha}}_1$ onto the original feature space in which $\boldsymbol{\hat{\alpha}}$ lies as $\boldsymbol{\hat{\alpha}}_2=\boldsymbol{Q}\boldsymbol{\hat{\alpha}}_1$. $\boldsymbol{\hat{\alpha}}_2$ and $\boldsymbol{\hat{\alpha}}$ are in the same feature space with the some coordinates. Finally, we use $\boldsymbol{\hat{\alpha}}_2$ to estimate $\boldsymbol{\hat{\alpha}}$ if the distance between these two solutions, $\|\boldsymbol{\hat{\alpha}}_2-\boldsymbol{\hat{\alpha}}\|_2$, asymptotically converges to zero, when $n$ goes to infinity. In this section, $\xrightarrow{P}$ denotes convergence in probability.

\begin{lemma}[{\citep[Theorem 15.0.2.]{linearalgebra}}]\label{lemma19}
	If matrix $\boldsymbol{\bar{X}}$ is full column rank with $\text{rank}(\boldsymbol{\bar{X}})=k$, this means we have an invertible matrix $\boldsymbol{Q}$ s.t. $\boldsymbol{Q}^{\top}\boldsymbol{\bar{X}}^{\top}\boldsymbol{\bar{X}Q}=n\boldsymbol{I}$, where $n$ is the number of rows for matrix $\boldsymbol{\bar{X}}$.
\end{lemma}

%\begin{proof}
%	The detailed proof of Lemma \ref{lemma19} can be found in Appendix %\ref{appendixC}.
%\end{proof}

First, we explicitly solve (\ref{orthogonalized Cp}) by applying a method similar to that used in Proposition \ref{theorem1}, given the property that $\boldsymbol{Q}^{\top}\boldsymbol{\bar{X}}^{\top}\boldsymbol{\bar{X}Q}=n\boldsymbol{I}$.
\begin{proposition}\label{proposition20}
	The solution of the orthogonalized Cp problem (\ref{orthogonalized Cp}) can be written as
	\begin{eqnarray}
	\hat{\alpha}_1^{i}=\left\{
	\begin{matrix}
	\tilde{Q_i}^{\top}\boldsymbol { \beta }+\frac{Z_i^{\top}\boldsymbol { \varepsilon }}{n},
	& \text{if~} n\left[\tilde{Q_i}^{\top}\boldsymbol { \beta }+\frac{Z_i^{\top}\boldsymbol { \varepsilon }}{n}\right]^2>\lambda\\
	0,
	&  \text{otherwise}
	\end{matrix}
	\right.
	\end{eqnarray}
	where $i=1,\cdots,k$. $\tilde{Q_i}^{\top}$ is the $i$-th row vector of the invertible matrix $\boldsymbol{Q}^{-1}$, for $i=1,\cdots,k$. $Z_j$ is the $j$-th column vector of the design matrix $\boldsymbol{\bar{X}Q}$, for $j=1,\cdots,k$.
\end{proposition}

\begin{proof}
	The detailed proof of Proposition \ref{proposition20} can be found in Appendix \ref{appendixC}.
\end{proof}

Next, in order to measure the distance between $\boldsymbol{\hat{\alpha}}_2=\boldsymbol{Q}\boldsymbol{\hat{\alpha}}_1$ and $\boldsymbol{\hat{\alpha}}$, we first estimate the distance between $\boldsymbol{\hat{\alpha}}_2$ and the true regression coefficients $\boldsymbol{\beta}$.

\begin{theorem}\label{theorem21}
	The back-transformed solution of the orthogonalized Cp problem (\ref{orthogonalized Cp}) $\boldsymbol{\hat{\alpha}}_2$ converges in probability to the true regression coefficients $\boldsymbol{\beta}$, when $n$ goes to infinity. {Specifically, for any $\eta>0$, with a probability of at least $1-\eta$, there holds $\|\hat{\boldsymbol{\alpha}}_2-\boldsymbol{\beta}\|_2^2\leq\mathcal{O}\left({\frac{1}{n}}\right)$.}
\end{theorem}

\begin{proof}
	The detailed proof of Theorem \ref{theorem21} can be found in Appendix \ref{appendixC}.
\end{proof}

Our ultimate goal is to compare estimates $\boldsymbol{\hat{\alpha}}_2$ and $\boldsymbol{\hat{\alpha}}$. Therefore, in Corollary \ref{corollary22}, we build the relationship between $\boldsymbol{\hat{\alpha}}$ and $\boldsymbol{\beta}$ by utilizing a result from \citet{shao1997asymptotic} (see Theorem 1 therein and notice that the dimension of the true regressors is fixed in our setting). Here, we define symbols that will be used in the theorem below. For the non-orthogonal Cp problem (\ref{non-orthogonal Cp}), $\textcolor{black}{\mathcal{J}}$ is a subset of $\{1,\cdots,k\}$, and $\boldsymbol{\beta}(\textcolor{black}{\mathcal{J}})$ or ($\bar{\boldsymbol{X}}(\textcolor{black}{\mathcal{J}})$) contains the components of $\boldsymbol{\beta}$ (or columns of $\bar{\boldsymbol{X}}$) that are indexed by the integers in $\textcolor{black}{\mathcal{J}}$. {We use $\mathcal{A}$ to denote all nonempty subsets of $\{1,\cdots,k\}$}, $\textcolor{black}{\hat{\mathcal{J}}}$ is the subscripts for nonzero elements of the non-orthogonal Cp estimator ${\hat{\boldsymbol{\alpha}}}$, $\textcolor{black}{\mathcal{J}}^*$ is the subscripts for nonzero elements of the true regression coefficient $\boldsymbol{\beta}$, which is fixed with the increase of $n$. Let  $\mathcal{A}^{c}=\{\textcolor{black}{\mathcal{J}}\in \mathcal{A}| \textcolor{black}{\mathcal{J}}^*\subset \textcolor{black}{\mathcal{J}}\}$, and we assume $\mathcal{A}^{c}$ is nonempty in this context. 
We can have the following Theorem by \citep{shao1997asymptotic} and \citep{nishii1984asymptotic}.

\begin{theorem}\label{theorem22}
	For the non-orthogonal Cp problem (\ref{non-orthogonal Cp}), suppose that the matrix $ \bar{\boldsymbol{X}}^{\top}\bar{\boldsymbol{X}}$ is positive definite, and $\lim_{n\rightarrow \infty}\frac{\bar{\boldsymbol{X}}^{\top}\bar{\boldsymbol{X}}}{n}$ exists and is positive definite. Then, $P_r\{\textcolor{black}{\hat{\mathcal{J}}} \in \mathcal{A}^{c}\}\xrightarrow{P}1 (n\rightarrow\infty)$.
\end{theorem}

\begin{proof}
	The detailed proof of Theorem \ref{theorem22} can be found in Appendix \ref{appendixC}.
\end{proof}

Theorem \ref{theorem22} (together with the discussions in \citep{shao1997asymptotic}) implies that the non-orthogonal Cp criterion may tend to select a correct model with superfluous features if the cardinality of $\mathcal{A}^{c}$ is larger than $1$. However, based on this result, we can prove that the non-orthogonal Cp estimator asymptotically approaches the true regression coefficients when $n\rightarrow\infty$.

\begin{corollary}\label{corollary22}
	For the non-orthogonal Cp problem (\ref{non-orthogonal Cp}), suppose that the matrix $ \bar{\boldsymbol{X}}^{\top}\bar{\boldsymbol{X}}$ is positive definite, and $\lim_{n\rightarrow \infty}\frac{\bar{\boldsymbol{X}}^{\top}\bar{\boldsymbol{X}}}{n}$ exists and is positive definite. Then,
	$\hat{\boldsymbol{\alpha}}\xrightarrow{P}\boldsymbol{\beta}(n\rightarrow\infty)$.
\end{corollary}

\begin{proof}
	The detailed proof of Corollary \ref{corollary22} can be found in Appendix \ref{appendixC}.
\end{proof}

%Corollary \ref{corollary22} indicates that the non-orthogonal Cp procedure is consistent if the regularity condition holds.
%Based on this result, we can now formally characterize the relationship between $\boldsymbol{\hat{\alpha}}$ and $\boldsymbol{\beta}$.

%\begin{corollary}\label{corollary23}
%For the non-orthogonal Cp problem (\ref{non-orthogonal Cp}), if $\sum _ { \tau \in \mathcal { A }/\{\tau^*\} } \frac { 1 } { \left[ n R_n( \tau ) \right] ^ { l } } \rightarrow 0$, $n\rightarrow \infty$, where $R_n(\tau):=\mathbb{E}(L_n(\tau))$, $L_n(\tau)=\frac{\|\bar{\boldsymbol{X}}\boldsymbol{\beta}-\hat{\mu}_n(\tau)\|^2}{n}$, and $\hat{\mu}_n(\tau)$ is the least squares estimator of the true model $\boldsymbol{\bar{X}\beta}$ under index subset $\tau$. $l$ is some fixed positive integer, such that $\mathbb{E}(\varepsilon_1)^{4l}<\infty$. Then, $\boldsymbol{\hat{\alpha}}\xrightarrow{P}\boldsymbol{\beta}$.	
%\end{corollary}

%\begin{proof}
%	The detailed proof of Corollary \ref{corollary23} can be found in Appendix.
%\end{proof}

By combining Theorem \ref{theorem21} and Corollary \ref{corollary22}, we directly obtain the desired result.

\begin{theorem}\label{theorem24}
	For the non-orthogonal Cp problem (\ref{non-orthogonal Cp}), suppose that the matrix $ \bar{\boldsymbol{X}}^{\top}\bar{\boldsymbol{X}}$ is positive definite, and $\lim_{n\rightarrow \infty}\frac{\bar{\boldsymbol{X}}^{\top}\bar{\boldsymbol{X}}}{n}$ exists and is positive definite. Then, $\boldsymbol{\hat{\alpha}}_2\xrightarrow{P}\boldsymbol{\hat{\alpha}}(n\rightarrow\infty)$.		
\end{theorem}

When a problem of the Cp-type criteria is applied to large \textcolor{black}{datasets}, computational requirements increase considerably. Theorem \ref{theorem24} indicates that we can treat $\boldsymbol{\hat{\alpha}}_2$ as an ``estimator'' of $\boldsymbol{\hat{\alpha}}$ under appropriate conditions, meaning that we can study the asymptotic behavior of $\boldsymbol{\hat{\alpha}}$ by analyzing the properties of $\boldsymbol{\hat{\alpha}}_2$. This is significant, since the explicit expression of $\boldsymbol{\hat{\alpha}}_2$ is available, which allows us to exploit and understand the process of feature selection using the Cp criterion. %Therefore, we can expect to acquire some similar properties of non-orthogonal Cp by

{\subsection{Asymptotic Analysis of TLCp in the Non-orthogonal Case}}

{We can naturally extend to the TLCp case our method to find an estimator  approximating the solution of the non-orthogonal Cp problem. We will refer to this extension as ``the approximate TLCp.'' We will show below the detailed procedure to find the approximate TLCp estimator and investigate its asymptotic properties.}

{We denote the proposed TLCp problem (\ref{TLCp}) after orthogonalization as minimizing the following objective function with respect to $\boldsymbol{w}_0, \boldsymbol{v}_1,\boldsymbol{v}_2$,
	\begin{equation}\label{orthogonalized TLCp}
	( \boldsymbol {y} _1 - \boldsymbol { {X}}_1\boldsymbol{Q}_1  \boldsymbol { w }_1  ) ^ { \top } ( \boldsymbol{y}_1  - \boldsymbol { {X}}_1\boldsymbol{Q}_1  \boldsymbol { w }_1  )+ ( \boldsymbol { {y} }_2 - \boldsymbol { {X}}_2\boldsymbol{Q}_2  \boldsymbol { w }_2  ) ^ { \top } ( \boldsymbol{y}_2 - \boldsymbol { {X}}_2\boldsymbol{Q}_2  \boldsymbol { w }_2 ) + \frac { 1 } { 2 } \sum _ { t = 1 } ^ { 2 }  \boldsymbol { v }_t^{\top}\boldsymbol{\lambda} _ { 3 }\boldsymbol{v}_t +\lambda_4\bar{p},
	\end{equation}
	where $\boldsymbol{Q}_1$ and $\boldsymbol{Q}_2$ are two invertible matrices such that $(\boldsymbol{X}_1\boldsymbol{Q}_1)^{\top}\boldsymbol{X}_1\boldsymbol{Q}_1=n\boldsymbol{I}$, $(\boldsymbol{X}_2\boldsymbol{Q}_2)^{\top}\boldsymbol{X}_2\boldsymbol{Q}_2=m\boldsymbol{I}$. $\boldsymbol{w}_1=\boldsymbol{w}_0+\boldsymbol{v}_1$, $\boldsymbol{w}_2=\boldsymbol{w}_0+\boldsymbol{v}_2$ are the regression coefficients of the tasks in the target and source domains, respectively. Here, we also assume the 
target domain samples are i.i.d. sampled from the relation $\boldsymbol {y} _1=\boldsymbol { {X}}_1\boldsymbol{\beta}+\boldsymbol{\varepsilon}$, where $\varepsilon_i \sim \mathcal { N } \left( 0 , \sigma_1 ^ { 2 } \right)$ for $i=1,\cdots, n$. Also, the source domain data are i.i.d. sampled from the relation: $\boldsymbol {y} _2=\boldsymbol { {X}}_2(\boldsymbol{\beta}+\boldsymbol{\delta})+\boldsymbol{\eta}$, where $\eta_i \sim \mathcal { N } \left( 0 , \sigma_2 ^ { 2 } \right)$ for $i=1,\cdots, m$. Other parameters in this model can refer to the corresponding illustrations in Subsection $4.1$.}

{To identify an estimator that can approximate the solution of the non-orthogonal  TLCp problem (\ref{TLCp}), we first solve the orthogonalized TLCp model (\ref{orthogonalized TLCp}). 
}

{\begin{proposition}\label{orthogonalized TLCp solution}
		The solution of the orthogonalized TLCp problem (\ref{orthogonalized TLCp}) can be written as
		\begin{align}
		&\bar{w}_1^{i}= \notag\\
		&\left\{
		\begin{array}{cc}
		(\tilde{{Q}}_1^{i})^{\top}\boldsymbol { \beta }+\frac{(Z_1^i)^{\top}\boldsymbol { \varepsilon }}{n}+D_1^i\left[(\tilde{{Q}}_2^{i})^{\top}(\boldsymbol{\delta}+\boldsymbol{\beta})-(\tilde{{Q}}_1^{i})^{\top}\boldsymbol{\beta}+\frac{(Z_2^i)^{\top}\boldsymbol{\eta}}{m}-\frac{(Z_1^i)^{\top}\boldsymbol { \varepsilon }}{n}\right]
		& \text{if~}
		\tilde{F}(\tilde{H}_i,\tilde{R}_i,\tilde{J}_i)>\lambda_4\notag\\
		%A_i\tilde{H}_i^2+B_i\tilde{Z}_i^2+C_i\tilde{J}_i^2>\lambda_4
		0
		&  \text{otherwise}
		\end{array}
		\right.
		\end{align}
		where  $\tilde{F}(\tilde{H}_i,\tilde{R}_i,\tilde{J}_i)=A_i\tilde{H}_i^2+B_i\tilde{R}_i^2+C_i\tilde{J}_i^2$. Further,  $\tilde{H}_i=(\boldsymbol{\delta}+\boldsymbol{\beta})\tilde{{Q}}_2^{i}+\frac{\boldsymbol{\eta}^{\top}Z_2^i}{m}$,  $\tilde{R}_i=\boldsymbol{\beta}^{\top}\tilde{{Q}}_1^{i}+\frac{\boldsymbol { \varepsilon }^{\top}Z_i^i}{n}$, $\tilde{J}_i=m\lambda_2\tilde{H}_i+n\lambda_1\tilde{R}_i$, $A_i$, $B_i$, $C_i$ and $D_1^i$ are defined as previously, for $i=1,\cdots,k$.
		In the solution formula, $(\tilde{{Q}}_1^i)^{\top}$ is the $i$-th row vector of the invertible matrix $\boldsymbol{Q}_1^{-1}$, and $(\tilde{{Q}}_2^i)^{\top}$ is the $i$-th row vector of the invertible matrix $\boldsymbol{Q}_2^{-1}$ for $i=1,\cdots,k$. Also, $Z_1^{i}$ is the $i$-th column vector of the design matrix $\boldsymbol{X}_1\boldsymbol{Q}_1$, and $Z_2^{i}$ is the $i$-th column vector of the design matrix $\boldsymbol{X}_2\boldsymbol{Q}_2$, for $i=1,\cdots,k$. 
\end{proposition}}

{\begin{proof}
		The detailed proof of Proposition \ref{orthogonalized TLCp solution} can be found in Appendix \ref{appendixC}.
\end{proof}}

{Following the same scheme we applied to the Cp case, we back-transform the solution of the orthogonalized TLCp problem (\ref{orthogonalized TLCp}), which is denoted as $\hat{\boldsymbol{w}}_1=\boldsymbol{Q}_1\bar{\boldsymbol{w}}_1$. This is the approximation of the solution of the non-orthogonal TLCp problem (\ref{TLCp}).}

{\begin{theorem}\label{theorem 28}
		The approximate TLCp estimator $\boldsymbol{\hat{w}}_1$ converges in probability to the true regression coefficients $\boldsymbol{\beta}$, when $n$ goes to infinity. Specifically, for any $\tilde{\eta}>0$, with probability at least $1-\tilde{\eta}$, there holds $\|\hat{\boldsymbol{w}}_1-\boldsymbol{\beta}\|_2^2\leq\mathcal{O}\left({\frac{1}{n}}\right)$.
\end{theorem}}

{\begin{proof}
		The detailed proof of Theorem \ref{theorem 28} can be found in Appendix \ref{appendixC}.	
\end{proof}}

{Theorem \ref{theorem 28} demonstrates that the proposed approximate TLCp procedure still preserves as good asymptotic properties as that of the Cp case. For the sake of completeness, we will illustrate the asymptotic results of the non-orthogonal TLCp estimator in the following remark.}

{\begin{remark}
		Following a similar procedure as in the proof of Corollary \ref{corollary22} (and Theorem \ref{theorem24}), we can further obtain that the solution $\tilde{\boldsymbol{w}}_1^*$ of the non-orthogonal TLCp problem (\ref{TLCp}) (for the target task) converges in probability to the true regression coefficients $\boldsymbol{\beta}$ (thus $\boldsymbol{\hat{w}}_1\xrightarrow{P}\tilde{\boldsymbol{w}}_1^*$), as $n$ goes to infinity. For any fixed $\textcolor{black}{\mathcal{J}}\in\mathcal{A}$, the solution of the non-orthogonal TLCp problem (\ref{TLCp}) has the form $\hat{\boldsymbol{\beta}}(\mathcal{J})+\boldsymbol{C}^{-1}_1(\mathcal{J})\boldsymbol{\lambda}_3(\mathcal{J})[2\boldsymbol{C}_2(\mathcal{J})+(\boldsymbol{C}_2(\mathcal{J})\boldsymbol{C}^{-1}_1(\mathcal{J})+\boldsymbol{I}(\mathcal{J})\boldsymbol{\lambda}_3(\mathcal{J})]^{-1}(b_2(\mathcal{J})-\boldsymbol{C}_2(\mathcal{J})\hat{\boldsymbol{\beta}}(\mathcal{J}))$, where $\boldsymbol{C}_1(\mathcal{J})=2\lambda_1\boldsymbol{X}(\mathcal{J})^{\top}\boldsymbol{X}(\mathcal{J})$, $\boldsymbol{C}_2(\mathcal{J})=2\lambda_2\tilde{\boldsymbol{X}}(\mathcal{J})^{\top}\tilde{\boldsymbol{X}}(\mathcal{J})$, $b_2(\mathcal{J})=2\lambda_2\tilde{\boldsymbol{X}}(\mathcal{J})^{\top}\tilde{\boldsymbol{y}}(\mathcal{J})$ and $\hat{\boldsymbol{\beta}}(\mathcal{J})$ is the least squares estimation of $\boldsymbol{\beta}$ under the index set $\mathcal{J}$. Also, the residual sum of squares for the target task dominates the objective function of (\ref{TLCp}). 
\end{remark}}

\subsection{Feature Selection Using Approximate Cp and TLCp Methods}

{The primary goal of using the approximate Cp and TLCp methods to select features is to retain relevant features and discard superfluous or redundant ones. We achieve this by deriving a cutoff value for each feature using the approximate Cp and TLCp methods. Coefficients with Cp/TLCp estimators below the cutoff will be discarded. In Subsection 6.3, we present several simulation studies to illustrate the effectiveness of this method. }

{For a sufficiently large number of data points $n$, the approximate Cp estimator for the $j$-th feature satisfies $\boldsymbol{\hat{\alpha}}_2^{j}\approx \beta_j+\sum_{i=1}^{k}Q_{ji}\frac{Z_i^{\top}\boldsymbol { \varepsilon }}{n}$, where $Z_i$ is the $i$-th column of $\boldsymbol{\bar{X}Q}$ and $Q_{j \cdot}$ is the $j$-th row of the transformation matrix $\boldsymbol{Q}$, for $j=1,\cdots,k$. 
By derivations similar to the proof of Theorem \ref{theorem21}, we have $\boldsymbol{\hat{\alpha}}_2^{j}\sim\mathcal{N}\left(\beta_j,\frac{\sigma_1^2}{n}\sum_{i=1}^{k}Q^2_{ji}\right), j=1,\cdots,k$, when $n$ is large enough. If the $j$-th feature is \textcolor{black}{superfluous} ($\beta_j=0$),  then $\boldsymbol{\hat{\alpha}}_2^{j}\sim\mathcal{N}\left(0,\frac{\sigma_1^2}{n}\sum_{i=1}^{k}Q^2_{ji}\right)$.
Thus, a natural way to determine the cutoff for this feature is to calculate the corresponding $(1-\tau/2)$-percentile ($u_{\tau/2}$) of the standard normal distribution, which satisfies $P_r\left\{\left|\boldsymbol{\hat{\alpha}}_2^{j}\right|/ \sqrt{\frac{\sigma_1^2}{n}\sum_{i=1}^{k}Q^2_{ji}}>u_{\tau/2}\right\}=\tau$. According to this formula, if we want the probability of a type I error (rejecting the hypothesis when it is true) to be less than $\tau$, $\left|\boldsymbol{\hat{\alpha}}_2^{j}\right|>u_{\tau/2}\sqrt{\frac{\sigma_1^2}{n}\sum_{i=1}^{k}Q^2_{ji}}$ is sufficient. Therefore, we can set the cutoff for the $j$-th feature equal to $U_j:=u_{\tau/2}\sqrt{\frac{\sigma_1^2}{n}\sum_{i=1}^{k}Q^2_{ji}}$, for $j=1,\cdots,k$. \textcolor{black}{To balance the type I error and type II error, we use Mallows' Cp to determine $u_{\tau/2}$ for the threshold $U_j$ on each feature. Theorem \ref{theorem30} guarantees that Mallows' Cp can indeed lead to proper cutoffs.}}

{\begin{definition}
We define the approximate Cp cutoff estimator $\boldsymbol{\tilde{\alpha}}_2$ as $\boldsymbol{\tilde{\alpha}}_2^{j} = \boldsymbol{\hat{\alpha}}_2^{j}$  when $\boldsymbol{\hat{\alpha}}_2^{j} \geq U_j$ and 0 otherwise, $(j=1,\cdots,k)$.
%the vector of approximate Cp estimators with values greater than or equal to the corresponding thresholds $U_j (j=1,\cdots,k)$; all other values are set to zeros.	
\end{definition}}

 %{Note that a hypothesis test with a smaller type I error tends to have a bigger type II error. To balance the type I error and type II error (accepting the hypothesis when it is false) of the hypothesis test on each feature, we use Mallows' Cp to determine $u_{\tau/2}$ for the threshold $U_j$ on each feature. The following theorem guarantees that Mallows' Cp can indeed lead to proper cutoffs.}

 \begin{theorem}\label{theorem30}
Assume that $\lim_{n\rightarrow \infty}\frac{\bar{\boldsymbol{X}}^{\top}\bar{\boldsymbol{X}}}{n}$ exists. Then, the approximate Cp estimator $\boldsymbol{\hat{\alpha}}_2$ asymptotically achieves the lowest value of Mallows' Cp-statistic $\frac{( \boldsymbol { \bar{y} } - \boldsymbol { \bar{X} }  \boldsymbol { \alpha }  ) ^ { \top } ( \boldsymbol { \bar{y} } - \boldsymbol { \bar{X} } \boldsymbol { \alpha }  )}{n}+\frac{2\sigma_1^2}{n}p$ in the sense of probability, when the number of data points $n$ goes to infinity. In other words, The approximate Cp cutoff estimator $\boldsymbol{\tilde{\alpha}}_2$ can also asymptotically achieve the lowest value of Mallows' Cp-statistic in the sense of probability, when the number of data points $n$ goes to infinity, if and only if the discarded attributes of $\boldsymbol{\hat{\alpha}}_2$ correspond to superfluous features.
 \end{theorem}
 	
\begin{proof}
The detailed proof of Theorem \ref{theorem30} can be found in Appendix \ref{appendixC}.
\end{proof} 

{Theorem \ref{theorem30} implies that using Mallows' Cp to determine the cutoffs on feature coefficients estimated by the approximate Cp method balances the type I and II errors, when the number of data points $n$ is large enough.}

{Algorithm 1 summarizes the procedure of using the approximate Cp method to select features. We can intuitively understand the candidate ($1-\tau/2$)-percentiles in Algorithm 1 as follows. Let $u_{\min}:=\min_{j=1,\cdots,k}\left\{u_j\right\}$,  $u_{\max}:=\max_{j=1,\cdots,k}\left\{u_j\right\}$. Note that the approximate Cp estimator indicates the degrees of importance level for each features. We sort all the candidate ($1-\tau/2$)-percentiles in descending order listed as $v_1, v_2, \cdots, v_{k+1}$, where $v_1=u_{\max}+1$, $v_{k+1}=u_{min}$. 
Then, when $u_{\tau/2}\leq u_{\min}$, the algorithm with thresholds $U_j(u_{\tau/2}) (j=1,\cdots,k)$ selects all the features, and when $u_{\tau/2}>u_{\max}$, the algorithm with thresholds $U_j(u_{\tau/2}) (j=1,\cdots,k)$ discards all the features. For cases $u_{\min}<u_{\tau/2}\leq u_{\max}$, the algorithm with thresholds $U_j(u_{\tau/2}=v_{\ell}) (j=1,\cdots,k)$ selects the first important $\ell-1$ features, for $\ell=2,\cdots,k$.} 

\begin{algorithm}
\caption{Using the approximate Cp method to select features}
\begin{algorithmic}
 \State \hspace{-3ex} \textbf{Input:} The approximate Cp estimator $\boldsymbol{\hat{\alpha}}_2$.
 \State \hspace{-3ex} \textbf{Output:}  The threshold on each feature coefficient $U_j(u_{\tau/2}):=u_{\tau/2}\sqrt{\frac{\hat{\sigma}_1^2}{n}\sum_{i=1}^{k}Q^2_{ji}}$ ($j=1,\cdots,k$) and the corresponding approximate Cp cutoff estimator $\boldsymbol{\tilde{\alpha}}_2$.
 		\begin{description}
	    \item[\hspace{-0.5em}] \hspace{-1em} Initialize $\boldsymbol{\tilde{\alpha}}_2(0)=\boldsymbol{0}_{k\times 1}$.
	    \item[\hspace{-0.5em}1:]\hspace{-0.5em} Calculate $k+1$ candidate ($1-\tau/2$)-percentiles: $u_\ell=\left|\boldsymbol{\hat{\alpha}}_2^{\ell}\right|/\sqrt{\frac{\hat{\sigma}_1^2}{n}\sum_{i=1}^{k}Q^2_{\ell i}}$, for $\ell=1,\cdots,k$. Let $u_{k+1}=\text{max}_{\ell=1,\cdots,k}\left\{u_{\ell}\right\}+1$.
	    \item[\hspace{-0.5em}2:]\hspace{-0.5em} Pick the best ($1-\tau/2$)-percentile $u_{\tau/2}$ by Mallows' Cp;
	    \item[\hspace{1em}]\hspace{-1.0em}{\bfseries for} $p=1$ to $k+1:1$
	    \item[\hspace{2em}]\hspace{-1.0em}{\bfseries for} $q=1$ to $k:1$
	    \item[\hspace{3em}]\hspace{-1.0em}{\bfseries if} $\boldsymbol{\hat{\alpha}}_2^{q}\geq U_q(u_p)$ ~$ \verb|\\|$ \textcolor{black}{$U_q(u_p):=u_p\sqrt{\frac{\hat{\sigma}_1^2}{n}\sum_{i=1}^{k}Q^2_{ji}}$ is a candidate threshold;}
	    \item[\hspace{3em}]\hspace{-1.0em} $\boldsymbol{\tilde{\alpha}}_2^{q}(p)=\boldsymbol{\hat{\alpha}}_2^{q}$;
	    \item[\hspace{3em}]\hspace{-1.0em}{\bfseries else} 
	    \item[\hspace{3em}]\hspace{-1.0em} $\boldsymbol{\tilde{\alpha}}_2^{q}(p)=0$;
	    \item[\hspace{3em}]\hspace{-1.0em}{\bfseries end if}
	    \item[\hspace{2em}]\hspace{-1.0em}{\bfseries end for}
	     \item[\hspace{2em}]\hspace{-1.0em}{\bfseries if} $C_p(\boldsymbol{\tilde{\alpha}}_2(p))<C_p(\boldsymbol{\tilde{\alpha}}_2(p-1))$, where $C_p(\cdot)$ denotes the Mallows' Cp statistic.
	      \item[\hspace{3em}]\hspace{-1.0em} $\boldsymbol{\tilde{\alpha}}_2=\boldsymbol{\tilde{\alpha}}_2(p)$;
	      \item[\hspace{3em}]\hspace{-1.0em} $u_{\tau/2}=u_p$.
	       \item[\hspace{2em}]\hspace{-1.0em}{\bfseries end if}
	       \item[\hspace{1em}]\hspace{-1.0em}{\bfseries end for}
	       \end{description}
\end{algorithmic}
\end{algorithm}

{Next, we present the procedure to determine the cutoff on each feature coefficient estimated by the approximate TLCp method. When the number of target samples $n$ is large enough, the approximate TLCp estimator for the $j$-th feature coefficient satisfies $\boldsymbol{\hat{w}}_1^{j}\approx \beta_j+\sum_{i=1}^{k}Q_{1}^{ji}\frac{(Z_1^i)^{\top}\boldsymbol { \varepsilon }}{n} (j=1,\cdots,k)$, where $Z_1^i$ is the $i$-th column of $\boldsymbol{X}_1\boldsymbol{Q}_1$ and $Q_{1}^{j\cdot}$ is the $j$-th row of the transformation matrix $\boldsymbol{Q}_1$ (see the proof of Theorem \ref{theorem 28}). Similar to the case of the approximate Cp method, we can set the cutoff for the $j$-th feature estimated by the approximate TLCp method as $\tilde{U}_j:=\tilde{u}_{\tau/2}\sqrt{\frac{\sigma_1^2}{n}\sum_{i=1}^{k}(Q_1^{ji})^2}$, where $\tilde{u}_{\tau/2}$ is ($1-\tau/2$)-percentile of the standard normal distribution, for $j=1,\cdots,k$. }

{
\begin{definition}
We define the approximate TLCp cutoff estimator $\boldsymbol{\tilde{w}}_1$ as $\boldsymbol{\tilde{w}}_1^j=\boldsymbol{\hat{w}}_1^j$ when $\boldsymbol{\hat{w}}_1^j\geq \tilde{U}_j$ and $0$ otherwise, $(j=1,\cdots,k)$.		
\end{definition}	
}	

{Following the same idea used in Algorithm 1,  we determine the proper value of $\tilde{u}_{\tau/2}$ for the threshold on each feature $\tilde{U}_j$ by the proposed TLCp criterion (\ref{TLCp}). The next theorem states that using the TLCp criterion leads to proper thresholds on the feature coefficients estimated by the approximate TLCp method.}

{
\begin{theorem}\label{theorem31}	
Assume that the number of source samples $m$ satisfies $\lim_{n\rightarrow \infty}m/n=C$, where $C>0$ is a constant and $n$ is the number of target samples. Further, assume $\lim_{n\rightarrow \infty}\frac{{\boldsymbol{X}_1}^{\top}{\boldsymbol{X}_1}}{n+m}$ and $\lim_{n\rightarrow \infty}\frac{{\boldsymbol{X}_2}^{\top}{\boldsymbol{X}}_2}{n+m}$ exist. Then, the approximate TLCp estimators $\boldsymbol{\hat{w}}_1$ and $\boldsymbol{\hat{w}}_2$ with respect to the target and source tasks asymptotically achieve the lowest value of the TLCp-statistic $\frac{1}{n+m}\sum _ { t = 1 } ^ { 2 } \left[\lambda_t( \boldsymbol { y }_t - \boldsymbol { X }_t \boldsymbol { w}_t ) ^ { \top } ( \boldsymbol { y }_t - \boldsymbol { X }_t \boldsymbol { w}_t  )+\frac { 1 } { 2 } \boldsymbol { v }_t^{\top}\boldsymbol{\lambda} _ { 3 }\boldsymbol{v}_t +\frac { 1 } { 2 }\lambda_4\bar{p}\right]$ in the sense of probability, when $n$ goes to infinity. In other words, the approximate TLCp cutoff estimators $\boldsymbol{\tilde{w}}_1$ and $\boldsymbol{\tilde{w}}_2$ for the target and source tasks can also asymptotically achieve the lowest value of the TLCp-statistic in the sense of probability, when $n$ goes to infinity, if and only if the discarded attributes of $\boldsymbol{\hat{w}}_1$ and $\boldsymbol{\hat{w}}_2$ correspond to superfluous features.
\end{theorem}}

\begin{proof}
	The detailed proof of Theorem \ref{theorem31} can be found in Appendix \ref{appendixC}.
\end{proof} 

{Algorithm 2 summarizes the procedure of using the approximate TLCp method to select features for the target task. This algorithm is a natural extension of Algorithm 1 under the framework of TLCp. Note that the feature selection for the target task in Algorithm 2 is related to the source task. Intuitively, we can expect a reliable feature selection if the relative dissimilarity between the target and source task is small.}

\begin{algorithm}
\caption{Using the approximate TLCp method to select features for the target task}
\begin{algorithmic}
\State \hspace{-3ex} \textbf{Input:} The approximate TLCp estimators $\boldsymbol{\hat{w}}_1$ and $\boldsymbol{\hat{w}}_2$.
\State \hspace{-3ex} \textbf{Output:}  The threshold on each feature coefficient $\tilde{U}_j(\textcolor{black}{\tilde{u}_{\tau/2}}):=\tilde{u}_{\tau/2}\sqrt{\frac{\hat{\sigma}_1^2}{n}\sum_{i=1}^{k}(Q_1^{ji})^2}$ ($j=1,\cdots,k$) and the corresponding approximate TLCp cutoff estimator $\boldsymbol{\tilde{w}}_1$ for the target task.
\begin{description}
			\item[\hspace{-0.5em}] \hspace{-1em} Initialize $\boldsymbol{\tilde{w}}_1(0)=\boldsymbol{0}_{k\times 1}$, $\boldsymbol{\tilde{w}}_2(0)=\boldsymbol{0}_{k\times 1}$.
			\item[\hspace{-0.5em}1:]\hspace{-0.5em} Calculate $k+1$ candidate ($1-\tau/2$)-percentiles: $\tilde{u}_\ell=\left|\boldsymbol{\hat{w}}_1^{\ell}\right|/\sqrt{\frac{\hat{\sigma}_1^2}{n}\sum_{i=1}^{k}(Q_1^{\ell i})^2}$, for $\ell=1,\cdots,k$. Let $\tilde{u}_{k+1}=\text{max}_{\ell=1,\cdots,k}\left\{\tilde{u}_{\ell}\right\}+1$.
			\item[\hspace{-0.5em}2:]\hspace{-0.5em} Pick the best ($1-\tau/2$)-percentile $\tilde{u}_{\tau/2}$ by the TLCp criterion;
			\item[\hspace{1em}]\hspace{-1.0em}{\bfseries for} $p=1$ to $k+1:1$
			\item[\hspace{2em}]\hspace{-1.0em}{\bfseries for} $q=1$ to $k:1$
			\item[\hspace{3em}]\hspace{-1.0em}{\bfseries if} $\boldsymbol{\hat{w}}_1^{q}\geq \tilde{U}_q(\tilde{u}_p)$~~$ \verb|\\|$ \textcolor{black}{$\tilde{U}_q(\tilde{u}_p):=\tilde{u}_p\sqrt{\frac{\hat{\sigma}_1^2}{n}\sum_{i=1}^{k}(Q_1^{ji})^2}$ is a candidate threshold;}
			\item[\hspace{3em}]\hspace{-1.0em} $\boldsymbol{\tilde{w}}_1^{q}(p)=\boldsymbol{\hat{w}}_1^{q}$, and
			$\boldsymbol{\tilde{w}}_2^{q}(p)=\boldsymbol{\hat{w}}_2^{q}$;
			\item[\hspace{3em}]\hspace{-1.0em}{\bfseries else} 
			\item[\hspace{3em}]\hspace{-1.0em} $\boldsymbol{\tilde{w}}_1^{q}(p)=0$, and $\boldsymbol{\tilde{w}}_2^{q}(p)=0$;
			\item[\hspace{3em}]\hspace{-1.0em}{\bfseries end if}
			\item[\hspace{2em}]\hspace{-1.0em}{\bfseries end for}
			\item[\hspace{2em}]\hspace{-1.0em}{\bfseries if} $TLC_p(\boldsymbol{\tilde{w}}_1(p),\boldsymbol{\tilde{w}}_2(p))<TLC_p(\boldsymbol{\tilde{w}}_1(p-1),\boldsymbol{\tilde{w}}_2(p-1))$, where $TLC_p(\cdot,\cdot)$ denotes the TLCp statistic with $\boldsymbol{\lambda_3}=\boldsymbol{0}$ (In this case, the TLCp criterion only shares the sparsity of tasks).
			\item[\hspace{3em}]\hspace{-1.0em} $\boldsymbol{\tilde{w}}_1=\boldsymbol{\tilde{w}}_1(p)$;
			\item[\hspace{3em}]\hspace{-1.0em} $\tilde{u}_{\tau/2}=\tilde{u}_p$.
			\item[\hspace{2em}]\hspace{-1.0em}{\bfseries end if}
			\item[\hspace{1em}]\hspace{-1.0em}{\bfseries end for}
		\end{description}
\end{algorithmic}
\end{algorithm}

{\subsection{Practical Considerations for Using the TLCp Methods}}

{In this subsection, we summarize the workflow in using the proposed TLCp methods including the original TLCp method and the approximate TLCp cutoff procedure. }

\begin{minipage}{\textwidth}
	\begin{small}
		\textup{
			\rule{\columnwidth}{1pt} \textbf{Guidelines for applying the TLCp methods to feature selection}\\
			\rule[1ex]{\columnwidth}{1pt}%\\ [-5pt]
			\vspace{-1ex} \textbf{Input:} target training dataset: $\{(x_1^i, x_2^i, \cdots, x_k^i;y_i)\}_{i=1}^{n}$, source dataset: $\{(\tilde{x}_1^i, \tilde{x}_2^i, \cdots, \tilde{x}_k^i;\tilde{y}_i)\}_{i=1}^{m}$. \\
			\textbf{Output:} estimated regression coefficients of target task and selected relevant features.
			\begin{description}
				\item[\hspace{0em}1:]\hspace{0em} Data standardization\footnote{Variable standardization is a necessary preprocessing step in feature selection, aiming to make the threshold independent of the scale of variables. However, we do not standardize binary variables (coded as $0/1$) to preserve their binary meaning.} (using $z$-scores);
				\item[\hspace{-0em}2:]\hspace{-0em} Calculate the relative dissimilarity\footnote{We define the relative dissimilarity of tasks as the scaled dissimilarity of tasks, that is,  $\|\boldsymbol{\hat{\mu}}_t-\boldsymbol{\hat{\mu}}_s\|_2/\|\boldsymbol{\hat{\mu}}_t\|_2$, where $\boldsymbol{\hat{\mu}}_t$ and $\boldsymbol{\hat{\mu}}_s$ are the least squares estimates of the regression coefficient vector for the target and source tasks, respectively.} of the given tasks and apply TLCp methods when the relative dissimilarity is less than $3$\footnote{A dissimilarity $>3$ indicates a significant deviation between two tasks. In this case, we do not transfer knowledge from the source task; instead, we use the original Cp method.}. %\footnote{%The relative dissimilarity threshold of applying the TLCp method depends on the specific problem. 
				%Generally, smaller relative dissimilarity of tasks indicates better performance of the TLCp approach. }
				% (i.e., less than $3$).
				\item[\hspace{-0em}3:]\hspace{-0em}~The detailed procedures of using TLCp methods.
				\item[\hspace{0.5em}3.1:]\hspace{0em}{\bfseries if} $n \geq k$, and the design matrices for the target and source tasks are non-singular, one can use either of the following two methods.
				\item[\hspace{0.5em}]\hspace{0.5em} (1) The original TLCp method;
				\item[\hspace{1em}]\hspace{1em}$\circ$~Tune the parameters of the original TLCp  procedure as the rules stated in Theorem $20$ and solve it.
				\item[\hspace{0.5em}]\hspace{0.5em} (2) The approximate TLCp cutoff method;
				\item[\hspace{1em}]\hspace{1em}$\circ$~ Apply the Gram-Schmidt process\footnote{We use the modified version of the Gram-Schmidt process where features that are almost linearly correlated to previous ones are deleted. When this process is applied to almost linearly dependent vectors and the $i$-th vector is a linear combination of the previous $i-1$ vectors, the process outputs a nearly zero vector in the $i$-th step. We simply discard these vectors because they are \textcolor{black}{redundant}} for the purpose of orthogonalizing the regression problems of the target and source tasks.
				\item[\hspace{1em}]\hspace{1em}$\circ$~ Solve the orthogonalized TLCp problem analytically, then back-transform the obtained solution as the approximate TLCp estimator.
			{\item[\hspace{1em}]\hspace{1em}$\circ$~ Conduct feature selection based on the approximate TLCp estimator by Algorithm 2.}
				\item[\hspace{0.5em}3.2:]\hspace{0em}{\bfseries if}~$n \geq k$, and the design matrices for the target and source tasks are singular.
				\item[\hspace{1em}]\hspace{1em}$\circ$ Delete the redundant features so that the remaining features are linearly independent, then execute $3.1$.
				\item[\hspace{0.5em}3.3:]\hspace{0em}{\bfseries if ~$n<k$}, and assuming the rank of the design matrix is $r$.
				\item[\hspace{1em}]\hspace{1em}$\circ$ Use a projection operator $\tilde{\boldsymbol{Q}} \in \mathbb{R}^{k\times r}$ by using the eigenvalue decomposition method\footnote{We discard eigenvectors whose corresponding eigenvalues are nearly zero.} such that $\tilde{\boldsymbol{Q}}^{\top}\boldsymbol{X}^{\top}\boldsymbol{X}\tilde{\boldsymbol{Q}}=rI$, and then apply $3.1(2)$. %Select $r$ linearly independent features and then apply $3.1$.
				\item[\hspace{-0em}4:]\hspace{-0em} Output the estimated regression coefficients of the target task and the selected features.
			\end{description}
			\rule[0ex]{\columnwidth}{1pt}\hspace{2mm}
		}
	\end{small}
\end{minipage}

{Next, we make a few recommendations for using the proposed TLCp method. 
\begin{itemize}
    \item As shown in the experimental section, based on the tuned parameters, the original TLCp method is effective when the relative task dissimilarity is small (i.e., less than $3$). The approximate cutoff TLCp procedure can perform as well as or better than the original TLCp method. The approximate TLCp cutoff procedure is preferable when users are more concerned about calculation time. 
    %\item According to our experiments, we recommend applying the Gram-Schmidt process to orthogonalize the design matrices when using the approximate TLCp procedure, though there are other techniques, such as SVD, to orthogonalize a matrix.  
    \item For more information about using the TLCp methods with more than two tasks, refer to Appendix B. 
    \item In case the number of features $k$ is significantly larger than the sample size $n$ (the sparsity assumption of the true regression model is required \citep{candes2007dantzig}), we tested the approximate TLCp cutoff procedure by simulations with $k=60$, 90, 300, 3000, 30000 and $n=30$ ($m=30$) (in Subsection $6.4$), also demonstrating the advantage of using our method. Due to the NP-hardness of the original TLCp problem, exact algorithms may be very time-consuming when the number of features is large. However, {we can try to use a solver such as ALAMO \citep{cozad2014learning} to solve the original TLCp problem.} %efficiently solve it, i.e., transforming it into a mixed-integer program (MIP) and using ALAMO \citep{cozad2014learning} to solve it. 
    %This point will be left for our future work.
\end{itemize}}

\subsection{Extension  to Other Feature Selection Criteria}

Under the orthogonality assumption, we can directly generalize the analysis of the Cp problem to other feature selection criteria, such as the Bayesian information criterion (BIC) in the following equation,
\begin{equation}
\text{BIC} = \min_{\boldsymbol { \bar{a} }}\frac { ( \boldsymbol { \bar{y} } - \boldsymbol { \bar{X} }  \boldsymbol { \bar{a} }  ) ^ { \top } ( \boldsymbol { \bar{y} } - \boldsymbol { \bar{X} } \boldsymbol { \bar{a} }  ) } { \hat { \sigma }_1 ^ { 2 } } + p\log{n}
\end{equation}
Specifically, we can directly obtain similar results with respect to the BIC criterion, as illustrated in Proposition \ref{theorem1}, Theorem \ref{theorem 2}, Proposition \ref{proposition3}, and Theorem \ref{theorem4}.

{\begin{remark}\label{remark30}
		We can also analyze BIC from the viewpoint of statistical tests. For instance, under the orthogonality assumption, we can conclude that using BIC amounts to performing a chi-squared test with regard to the statistic $\left(\beta_i\sqrt{n}+\frac{\sum_{j=1}^{n}\varepsilon_jW_i^{j}}{\sqrt{n}}\right)^2$ for each feature, with the significant level $\alpha_4=1-F(\log(n);1)$, where $F(\log(n);1)$ is the cumulative distribution function of the chi-squared distribution with $1$ degree of freedom at value $\text{log}(n)$. Therefore, BIC is more conservative than Cp in selecting features. More details can be found in Appendix A.
\end{remark}}

In the absence of orthogonality, our framework can also be applied to BIC. The proposed TLCp model serves as a a guide on how to proceed. By substituting $\lambda={\hat{\sigma}_1}^2\log{n}$ in Proposition \ref{TLCp solution}, Theorem \ref{TLCp probability}, Corollary \ref{TLCp scale-up}, Theorem \ref{tune parameters}, Theorem \ref{MSE TLCp}, and Corollary \ref{corollary 13}, we can directly acquire the corresponding results for BIC by using transfer learning. Our results in Proposition \ref{proposition14}, Lemma \ref{lemma17}, and Theorem \ref{theorem18} are obtained under the condition that $\lambda=2\sigma_1^2$. We can expect to obtain similar results in the context of BIC by using a similar technical analysis. We will leave this to future work.

\section{Simulation Studies}

{In this section, we conduct simulations to evaluate the performance of the proposed methods under different problem settings. We first use a toy example to support our theoretical results in Corollary \ref{corollary 13} and Theorem \ref{theorem18}. Then, we investigate the effect of sample size and relative task dissimilarity on the performance of the TLCp method in the orthogonal case. Then, we investigate the performance of the approximate Cp and TLCp cutoff methods with feature correlations in the non-orthogonal case. Finally, we compare the performance of these two methods under high-dimensional settings. }
{The source code for reproducing the experimental results is available at \href{https://github.com/Shaohan-Chen/Transfer-learning-in-Mallows-Cp}{``https://github.com/Shaohan-Chen/Transfer-learning-in-Mallows-Cp"}. All experiments in this paper were conducted on a computer with a $6$-core, $2.60$-GHz CPU and $16$-GB memory.}

\subsection{Toy Example}

%In this section, we conduct simple simulations in the orthogonal design matrix case in order to confirm and demonstrate the superiority of the proposed TLCp method. Although the conclusions of Corollary \ref{corollary 13} and Theorem \ref{theorem18} are proven by limiting the length of dissimilarity measure $\|\boldsymbol{\delta}\|_2$ to be sufficiently small, they are still valid when this limitation is loosened to some degree. 

We assume that the target training data are i.i.d. sampled from $\boldsymbol { y } = \boldsymbol { X } \boldsymbol { \beta } + \boldsymbol { \varepsilon }$, where $\boldsymbol { \beta }=[1,0.01,0.005,0.3,0.32, 0.08]^{\top}$, the fourth and fifth elements of which are (or are near) the critical points $\pm\sqrt{2\sigma^2/n}$ when $n=20$ and $\sigma=1$.  $\boldsymbol{\varepsilon} : = \left( \varepsilon _ { 1 }, \varepsilon _ { 2 }, \cdots, \varepsilon _ { n } \right) ^ { \top }$ are the standard Gaussian noises. We generate data from the source domain as $\boldsymbol { \tilde{y} } = \boldsymbol { X } (\boldsymbol { \beta }+\boldsymbol{\delta}) + \boldsymbol { \varepsilon }$. Here, $\boldsymbol { X }$ is first obtained by producing a random matrix $\boldsymbol{Z}$, where each item follows a standard normal distribution. Then, we
find an invertible matrix $\tilde{\boldsymbol{Q}}$ such that $\boldsymbol { X }=\boldsymbol{Z}\tilde{\boldsymbol{Q}}$ satisfies $\boldsymbol{X}^{\top}\boldsymbol{X}=n\boldsymbol{I}$ (see Lemma \ref{lemma19} in Section $5.1$). %This method is similar to the matrix orthogonalization method stated in Lemma \ref{lemma19} in Section $5.1$.
We simulate data with the sample size $n=20$ in the target domain, and $m=20$ in the source domain. We also define the similarity measure between the tasks from target and source domains as $1/\|\boldsymbol{\delta}\|_2$ with $\|\boldsymbol{\delta}\|_2\in[0,5]$ (for our experiment, we uniformly picked $29$ points from $\|\boldsymbol{\delta}\|_2\in[0,5]$). For each $1/\|\boldsymbol{\delta}\|_2$, we randomly simulated $5000$ \textcolor{black}{datasets} and applied the Cp and TLCp criteria. We chose the tuning parameter of the Cp model (\ref{the modified Cp}) as $\lambda=2$, and set the parameters of the TLCp model (\ref{TLCp}) $\lambda_1, \lambda_2, \boldsymbol{\lambda_3}, \lambda_4$ according to the tuning rules stated in Corollary \ref{corollary 13} or Theorem \ref{theorem18}, as $\lambda_1=1, \lambda_2=1, \lambda_3^{i}=4/\delta_i^2(i=1,\cdots,k), \lambda_4\approx2$.

{The probabilities of the orthogonal TLCp method and the Cp criterion to select a feature under several specific task similarities are presented in Figure \ref{bar}. Figures 2(a), 2(b), and 2(c) show that the greater the task similarity is, so are the probabilities of the TLCp to select critical features. The probability of TLCp to select features whose coefficients are small (i.e., the second, third, and sixth ones) is similar to that of Cp. However, the probability of TLCp to identify the critical features is remarkably larger than that of Cp when the task similarity is large (i.e., larger than $2$).
As depicted in figure 2(d), TLCp may choose incorrect models with a high probability if the task similarity is small. These experimental results are consistent with our theoretical results in Corollary \ref{corollary 13}. The observations imply that we can generalize the restriction of the task dissimilarity in Corollary \ref{corollary 13} to a wide range. Table \ref{Table 1} shows the (average) estimated regression coefficients for the TLCp and Cp methods under several task similarities. We see that the TLCp method ranks the features reliably when the task similarity is relatively small.}

%Based on the hyperparameters' tuning rule in Theorem \ref{theorem18}, our TLCp achieves a much lower MSE value than the Cp, provided that the true regression model contains several critical features. 
{In Figure \ref{figure3}, we compare the MSE performance of the orthogonal TLCp estimator to that of the orthogonal Cp estimator and detect changes in variance with the task similarity. {In Figure \ref{figure3}(a)}, based on the hyperparameters' tuning rule in Theorem \ref{theorem18}, the MSE value of TLCp dramatically decreases with the increase of the task similarity. However, suppose the hyperparameters of TLCp are randomly set. In that case, the MSE performance of TLCp is significantly worse than the well-tuned case and performs slightly better than Cp as task similarity grows. These numerical results support the theoretical result in Theorem \ref{theorem18} when there exist critical features in the model.}

%\begin{figure}
%	\centering
%	\includegraphics[height=0.2\textheight, width=0.4 \textwidth]{pictures/bar1.png}
%	\caption{\small The MSE performance comparison between Cp criterion and least square method under the existence of critical points in the true regression coefficients.}
%	\label{bar1}
%\end{figure}
\begin{figure}
\begin{minipage}[t] {0.5\linewidth}
\centerline{\includegraphics[height=6.5cm,width=7.5cm]{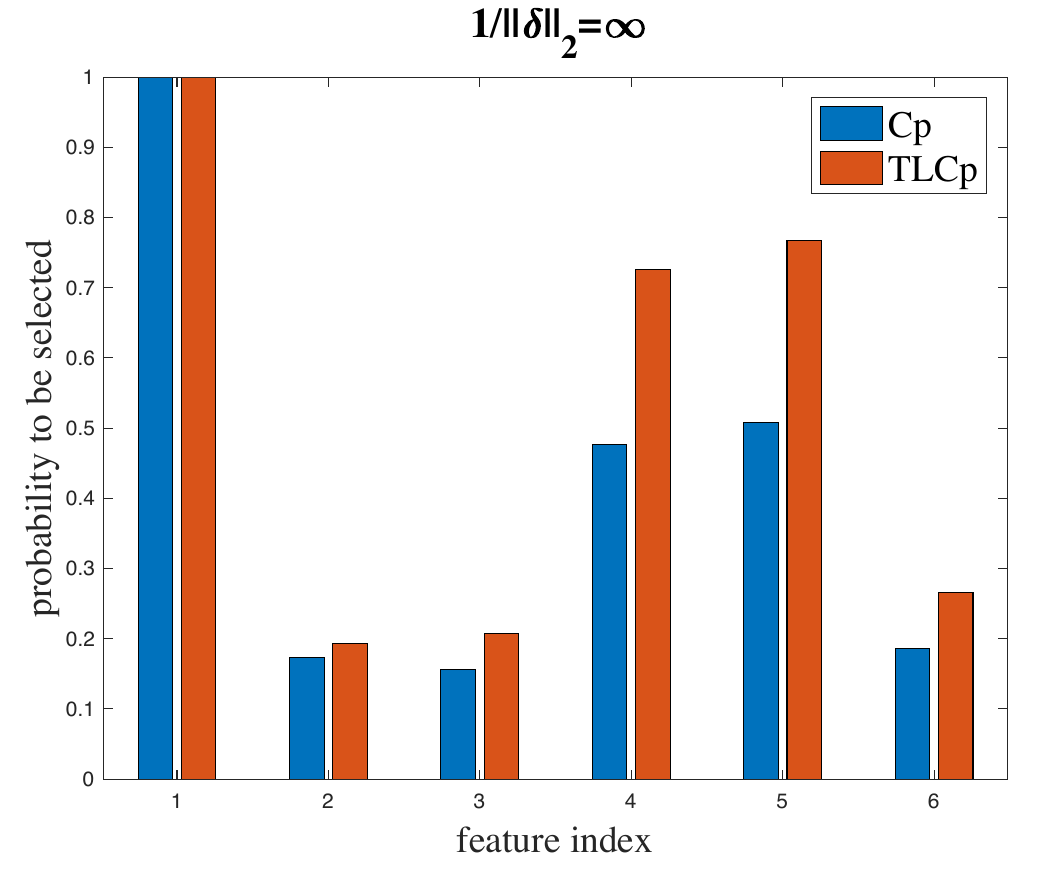}}
\centerline{(a)}
\centerline{  }
\end{minipage}
\hfill
\begin{minipage}[t] {0.5\linewidth}
\centerline{\includegraphics[height=6.5cm,width=7.5cm]{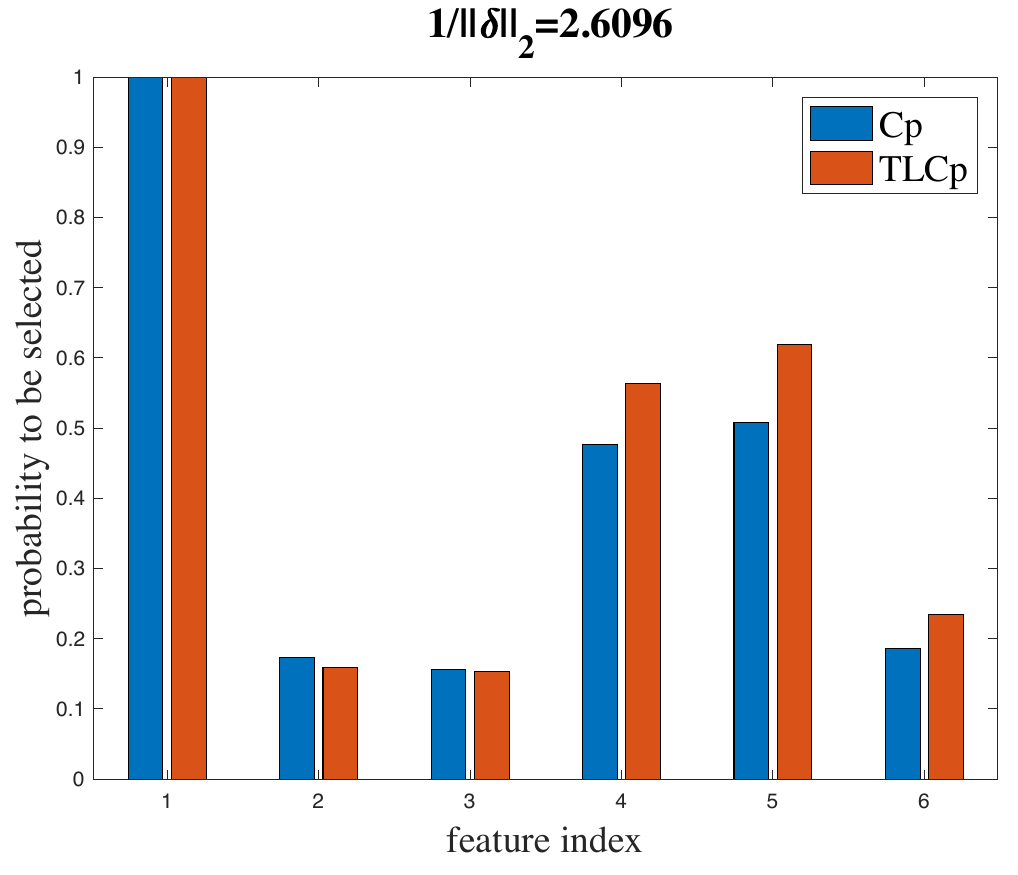}}
\centerline{(b)}
\end{minipage}
\vfill
\begin{minipage}[t] {0.5\linewidth}
\centerline{\includegraphics[height=6.5cm,width=7.5cm]{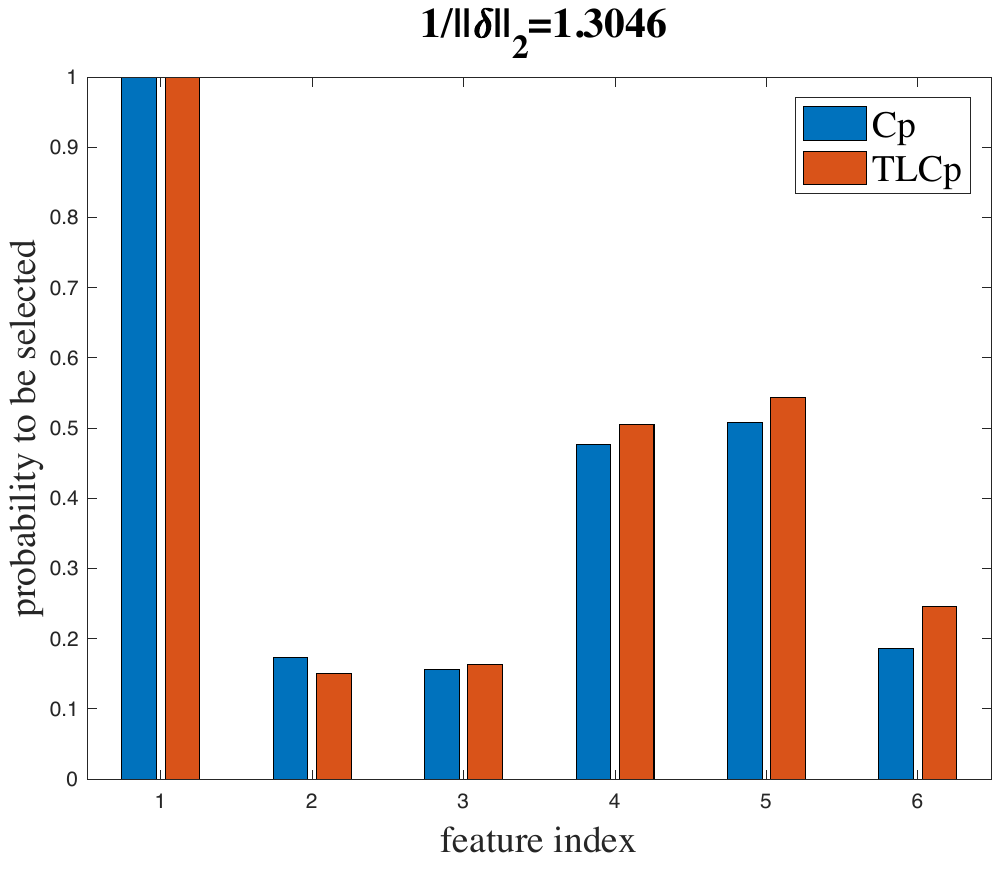}}
\centerline{(c)}
\end{minipage}
\hfill
\begin{minipage}[t] {0.5\linewidth}
\centerline{\includegraphics[height=6.5cm,width=7.5cm]{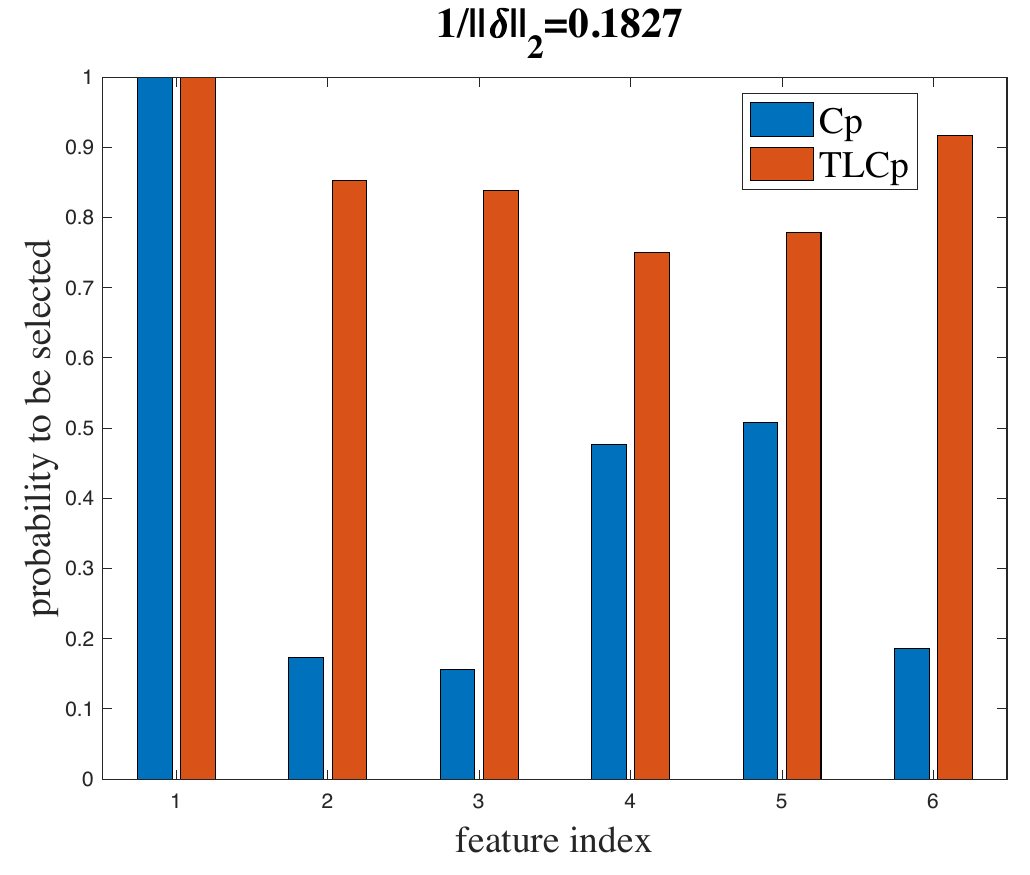}}
\centerline{(d)}
\end{minipage}
\caption{\small The probabilities of the orthogonal TLCp and Cp estimators to identify a feature under a specific similarity value $1/\|\boldsymbol{\delta}\|_2$, when model parameters are well-tuned. The red bar represents the probability of the orthogonal TLCp to select a feature. The blue bar corresponds to the orthogonal Cp case. In (a), (b) and (c), the larger the similarity measure between tasks from the target and source domains, the larger the probabilities of the orthogonal TLCp estimator to identify features with index $4$ and $5$. However, when tasks from the target and source domains differ greatly (see (d)), the orthogonal TLCp estimator may result in undesirable feature selection results.}
\label{bar}
\end{figure}

\begin{figure}[htbp]
	\begin{minipage}[t] {0.5\linewidth}
		\centerline{\includegraphics[height=6.5cm,width=7.5cm]{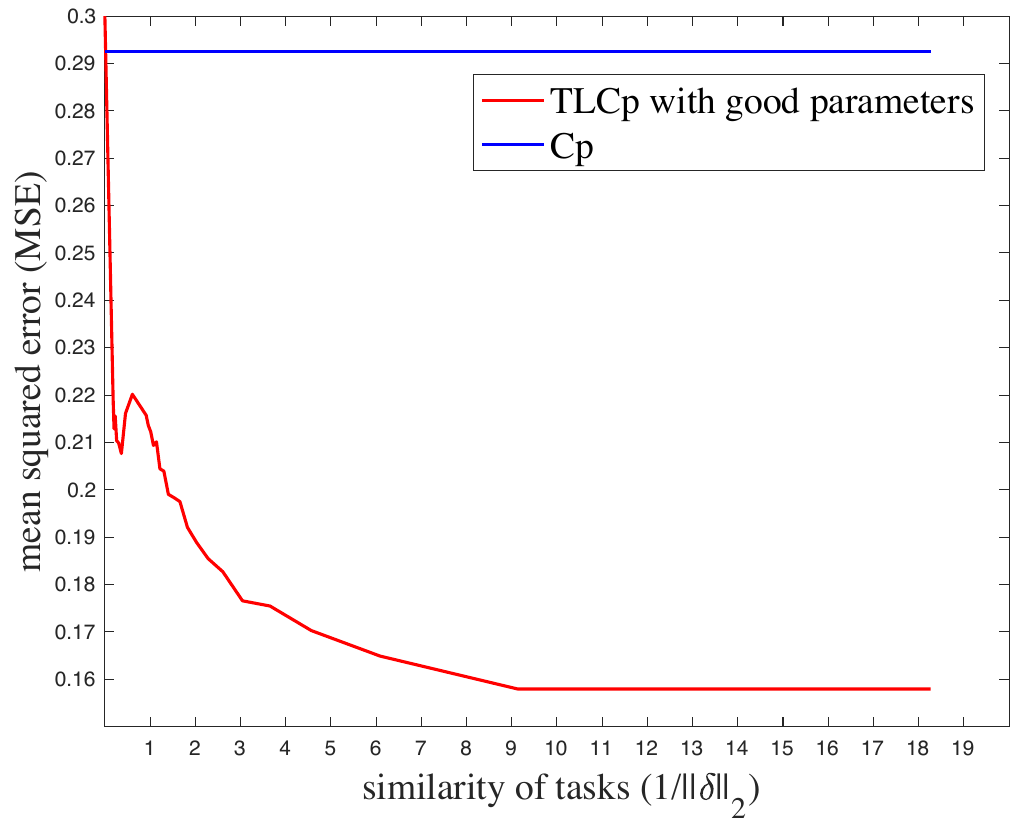}}
		\centerline{(a)}
	\end{minipage}
	\hfill
	\begin{minipage}[t] {0.5\linewidth}
		\centerline{\includegraphics[height=6.5cm,width=7.5cm]{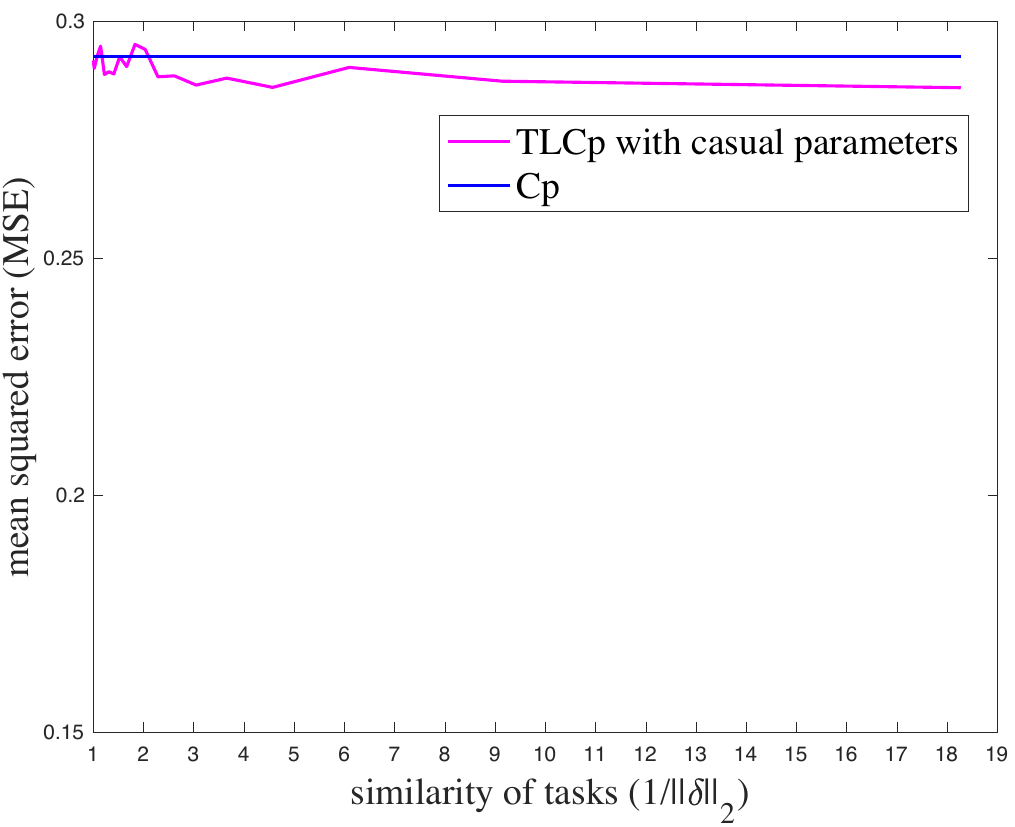}}
		\centerline{(b)}
	\end{minipage}
\caption{\small MSE performances of the orthogonal TLCp {(the non-horizontal line)} and orthogonal Cp {(the horizontal line)} estimators. The picture on the left depicts the MSE performances of these two estimators when the tuning parameters of TLCp are selected according to Theorem \ref{theorem18}: {$\lambda_1=1, \lambda_2=1, \lambda_3^{i}=4/\delta_i^2(i=1,\cdots,k), \lambda_4=2$.} {The picture on the right shows the MSE performance of TLCp with its hyperparameters arbitrarily set to: $\lambda_1=2, \lambda_2=1, \boldsymbol{\lambda}_3=[1,2,3,4,1,1]^\top, \lambda_4=2$. }}
\label{figure3}
\end{figure}

\begin{table}[htbp]
	\centering
	\begin{tabular}{|c|cccccc|}
		\hline
  	    &  $1$ & $2$  & $3$ & $4$ & $5$  & $6$ \\ \hline
	$\boldsymbol{\beta}$ (true model) & $1.0000$  & $0.0100$  & $0.0050$  & $0.3000$ & $0.3200$  & $0.0800$ \\ \hline
	$\text{Cp}$ & $1.0013$  & $0.0090$  & $0.0074$  & $0.2333$ & $0.2492$  & $0.0448$ \\ \hline
	$\text{TLCp}(1/\| \boldsymbol{\delta}\|_2=\infty)$	& $0.9982$  & $0.0095$  & $0.0048$  & $0.2682$ & $0.2909$  & $0.0522$ \\ \hline
	$\text{TLCp}(1/\| \boldsymbol{\delta}\|_2=2.6096)$    & $1.0761$ & $0.0055$  & $0.0048$ & $0.2075$ & $0.2235$ & $0.0552$ \\ \hline
	$\text{TLCp}(1/\| \boldsymbol{\delta}\|_2=1.3046)$	& $1.0595$ & $0.0046$ & $0.0027$ & $0.1731$ & $0.2024$ & $0.0690$  \\ \hline
	$\text{TLCp}(1/\| \boldsymbol{\delta}\|_2=0.1827)$	& $1.0126$ & $-0.0099$ & $0.0097$ & $0.2672$ & $0.2897$ & $0.1484$ \\ \hline
	\end{tabular}
    \caption{Estimated regression coefficients for the orthogonal Cp and TLCp methods.}
    \label{Table 1}
\end{table}

{\subsection{Extension of Toy Example}}

{Following the same simulation design as in the toy example, we assume the true regression coefficients for the target task are $\boldsymbol { \beta }_1=[0.24,0.01,0.005,0.3,0.32,0.08,0,0.26,0.25,0]^{\top}$. The fourth and fifth elements of $\boldsymbol { \beta }_1$ are the critical points when $n=20$. Here, we fix the number of source samples as $m=20$, and $\|\boldsymbol{\delta}\|_2/\|\boldsymbol{\beta}_1\|_2$ is defined as the relative dissimilarity between the target and source tasks. In order to illustrate how the combinations of the relative task dissimilarity and the target sample size affect the performance of the TLCp method, we consider the MSE performance by varying $n$ in $[20,180]$ and uniformly selecting $11$ values from $ [0,4]$ as the relative task dissimilarities. The relative dissimilarity of two tasks equals zero, indicating that the training datasets for these two tasks are sampled from the same distribution. For each target sample size $n$ and the relative task dissimilarity $\|\boldsymbol{\delta}\|_2/\|\boldsymbol{\beta}_1\|_2$, we randomly simulated $20000$ \textcolor{black}{datasets} and applied the TLCp approach. We tune the hyperparameters of the TLCp model using the rule stated in Theorem $20$.}

\begin{figure}[htbp]
	\begin{minipage}{\textwidth}
	\centering{\includegraphics[width=1\textwidth]{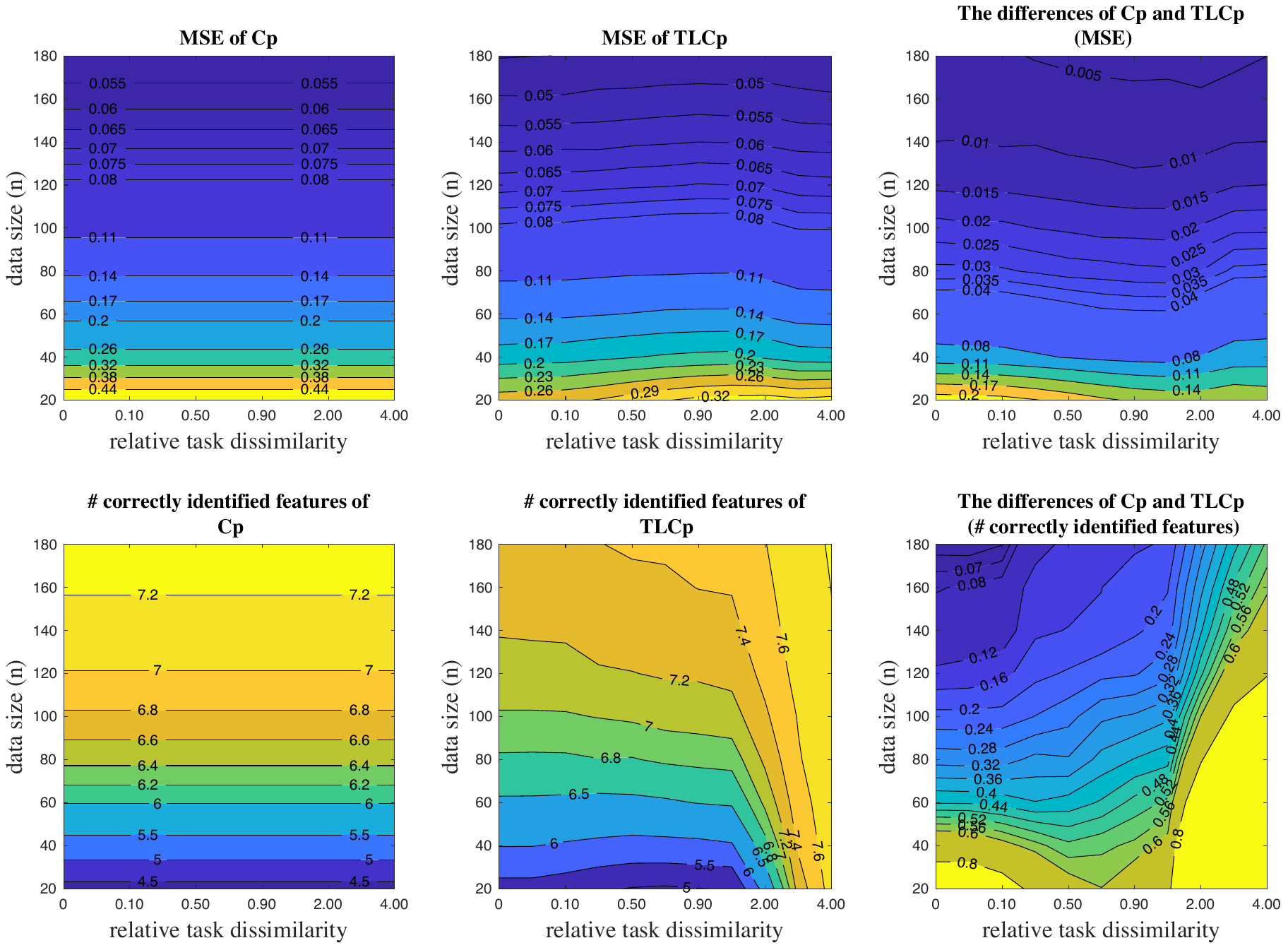}} \\
		\caption{\small With critical features in the true model, the performance (in terms of MSE and number of correctly identified features) of the orthogonal Cp and TLCp methods as the number of target data $n$ and relative task dissimilarity $\|\boldsymbol{\delta}\|_2/\|\boldsymbol{\beta}_1\|_2$ vary.} \label{simulation_expansion1}
\end{minipage}	
\end{figure}

The contour plots in Figure \ref{simulation_expansion1} show the performance (in terms of MSE and the number of correctly identified features) of Cp and TLCp methods when there are critical features in the true regression model. The first row of plots shows the MSE performance of the considered models. As expected, the proposed TLCp method outperforms the Cp criterion at every sample size, and their performance gap shrinks as the sample size grows. In particular, the MSE of the Cp estimator decreases as the number of samples grows, and it remains unchangeable as the relative dissimilarity of tasks varies. As the relative task dissimilarity increases, the MSE of the TLCp estimator increases initially and decreases when the relative dissimilarity of tasks is large enough. This occurs because the TLCp extracts less useful information from the source task as the growth of the relative task dissimilarity. The TLCp stops transferring knowledge from the source task if the dissimilarity grows significantly. %The added source task helps the TLCp treat all the critical features as the significant ones. Thus, no critical features will be over-deleted in this case. The TLCp estimator may still miss the critical features, even though it has a higher probability of identifying all the critical features than the Cp criterion.}

{To further display the benefit of applying the TLCp method, we plot the MSE differences of Cp and TLCp estimators (see the subfigure in the top right). We see that the TLCp significantly outperforms the Cp when both the sample size and the relative dissimilarity of tasks are small. Specifically, TLCp works better than Cp $32\% \sim 45\%$ in terms of MSE when the sample size is $20$, and the relative task dissimilarity is less than $4.00$. We denote the ``effective sample size'' as the number of samples required for Cp and TLCp to perform the same (in the sense of MSE). As the relative dissimilarity of tasks grows, the ``effective sample size'' 
shows a trend from decline to rise (e.g., see the contour line at the level $0.005$ in the top right panel of Figure \ref{simulation_expansion1}). The ``effective sample size'' in this example is approximately $180$ when the relative task dissimilarity is small (i.e., $0.10$) and $165$ when the relative task dissimilarity is relatively large (i.e., larger than $3.00$).}

{The second row of plots in Figure \ref{simulation_expansion1} displays the number of correctly identified features (counted by both the correctly selected relevant features and correctly ignored superfluous features) of the Cp and TLCp methods. The number of correctly identified features in this figure is shown as a function of the target sample size and relative task dissimilarity. We see that the number of correctly identified features of Cp increases as the sample size grows, and it is invariant to the relative dissimilarities. However, as the relative dissimilarity of tasks increases, the number of correctly identified features of the TLCp is ``down and up'' when sample size is relatively small. To further illustrate the advantages of using the TLCp method, the subfigure in the bottom right depicts the differences between Cp and TLCp based on the number of correctly identified features. We see the distinct benefits of TLCp over Cp when sample size and relative dissimilarity of tasks are comparatively small, or when the relative dissimilarity of tasks is relatively large (all the critical features are successfully identified in this case). We can similarly estimate the ``effective sample size'' in terms of the number of correctly selected features (e.g., based on the contour line at the level $0.40$ in the bottom right panel) as $60$ when the relative dissimilarity of tasks is $0.10$ and $180$ when the relative task dissimilarity grows to $3.00$. }

{We also demonstrate the efficacy of the orthogonal TLCp method (with its parameters well-tuned) when the true model is generated randomly. More details can be found in Appendix D.}

\subsection{Efficiency of the Feature Selection Strategies based on the Approximate Cp and TLCp Methods}

{This subsection contains three simulation studies to demonstrate the efficiency of applying the approximate Cp and TLCp cutoff methods (Algorithms 1 and 2) to select features.  The simulation results support our theoretical results in Corollary \ref{corollary22}, Theorem \ref{theorem24} and Theorem \ref{theorem 28}.  Furthermore,  we show that our methods can accurately identify all relevant features in the presence of feature correlations. Finally, we evaluate the performance of the TLCp method against two baseline models.}

{First, we present the two simulations (one with and one without \textcolor{black}{superfluous} features) without feature correlations to verify the effectiveness of Algorithm 1.  In the first study, the training data are i.i.d sampled from $\boldsymbol { y } = \boldsymbol { X } \boldsymbol { \beta } + \boldsymbol { \varepsilon }$, where $\boldsymbol { \beta }=[1 ,0.01,0.005,0.3,0.32,0.08]^{\top}$ and $\boldsymbol{\varepsilon} : = \left( \varepsilon _ { 1 }, \varepsilon _ { 2 }, \cdots, \varepsilon _ { n } \right) ^ { \top }$ are the standard Gaussian noises. We generate the design matrix $\boldsymbol{X}$ with its elements following the standard normal distribution. }

{We simulate $2000$ samples of sizes $n=(20, 60, 100, \cdots, 400)$ from the above model. We first apply the approximate Cp method stated in Subsection $5.1$ to each sample to produce the approximate Cp estimator $\boldsymbol{\hat{\alpha}}_2$. 
	Then, we produce the approximate Cp cutoff estimator $\boldsymbol{\tilde{\alpha}}_2$ by using Algorithm 1 on each sample. We use complete enumeration to obtain \textcolor{black}{the} solution of the original Cp problem $\hat{\boldsymbol{\alpha}}$. That is, for each feature, we obtain $2000$ regression coefficients estimated by each method. Figure \ref{convergence_rate_plot1} (left) depicts the MSE comparison among these estimators when there exist no \textcolor{black}{superfluous} features in the model.}
{ As the data size grows, the logarithm of MSE of these estimators decays at a rate approximately $\mathcal{O}\left({\frac{1}{{n}}}\right)$. We also see that the MSE performance of these Cp-based estimators almost overlap. These observations support the results of Theorem \ref{theorem21}, Corollary \ref{corollary22} and Theorem \ref{theorem24}. Figure \ref{convergence_rate_plot1} (right) shows similar simulation results when there are \textcolor{black}{superfluous} features in the model with $\boldsymbol { \beta }=[1 ,0.01,0.005,0.3,0.32,0,0,0]^{\top}$.  What is slightly different here is that the approximate Cp cutoff estimator (with its MSE value $0.52$) performs better than the other two methods (with MSE value $0.57$ for the approximate Cp estimator and $0.55$ for the original Cp estimator) when the data size is small ($n=20$). \textcolor{black}{Table \ref{Table add} summarizes the relative frequency of each feature by each method when there exist superfluous features. We see that our method is better than the original Cp in discarding superfluous features. Here, the computed average $(1-\tau/2)$-percentile $(u_{\tau/2})$ used to determine the cutoff for each feature in Algorithm 1 is approximately $1.3706$ (that is, $\tau\approx 0.17$).} These observations illustrate the \textcolor{black}{superiority} of the cutoff strategy for the approximate Cp method in Algorithm 1.} %\textcolor{black}{There could be other ways to include superfluous features in the model (e.g., adding some noisy variables with zero mean and varying standard deviations), which will be left for our future work. }

\begin{figure}[htbp]
	\begin{minipage}[t] {0.5\linewidth}
		\centerline{\includegraphics[height=6.5cm,width=7.5cm]{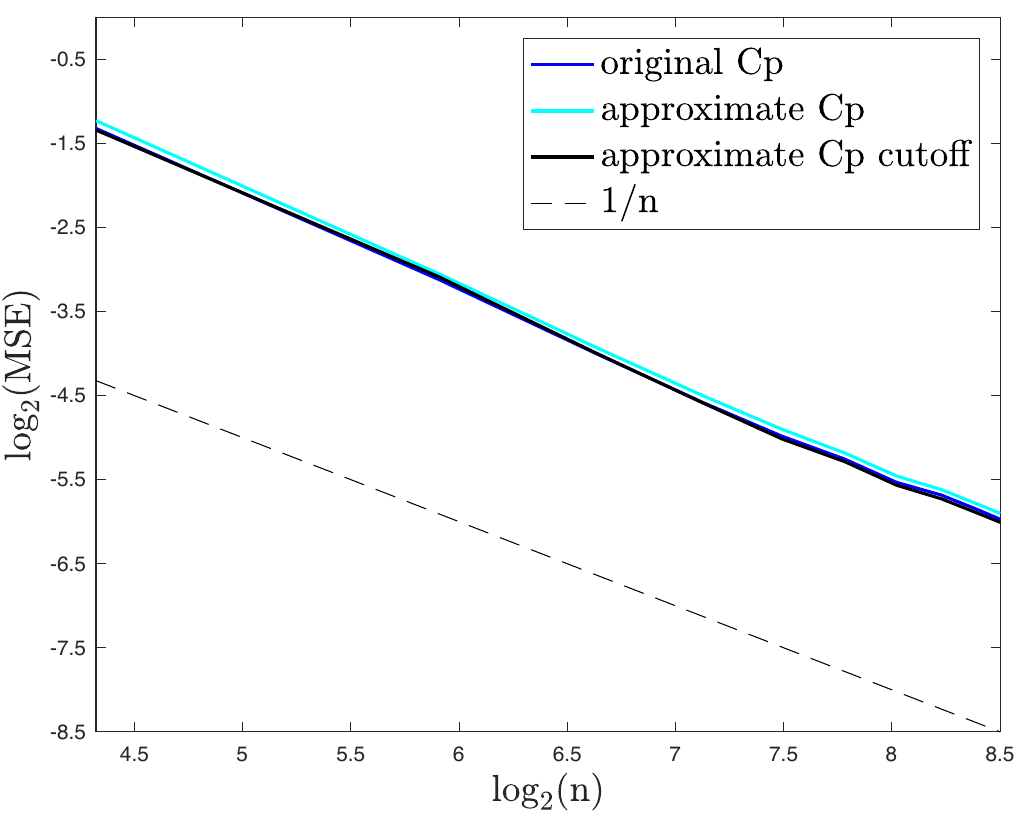}}
	\end{minipage}
	\hfill
	\begin{minipage}[t] {0.5\linewidth}
		\centerline{\includegraphics[height=6.5cm,width=7.5cm]{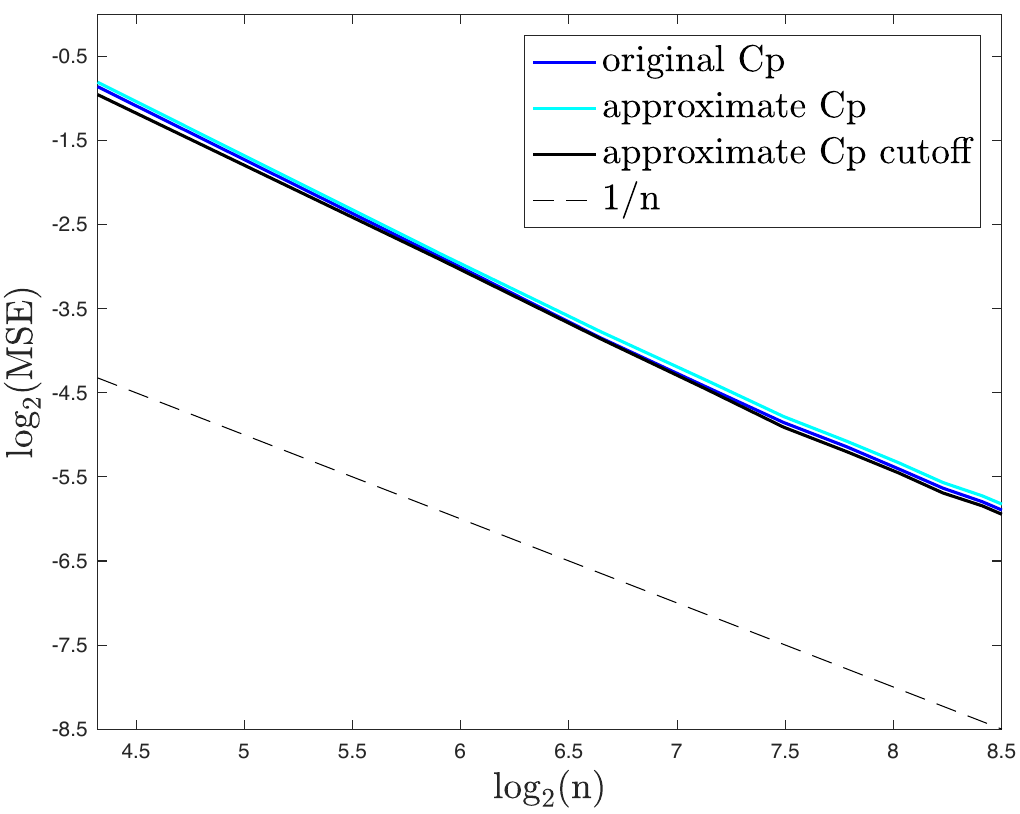}}
	\end{minipage}
	\caption{\small MSE performance comparison of Cp-based methods in the absence of feature correlations. The figure shows how the logarithm of MSE changes with an increasing of $\log_2(n)$ for different methods, without (left) or with (right) {superfluous} features in the true model. The dashed line which decays at $1/n$ is the baseline to compare the convergence rate of these methods. }
	\label{convergence_rate_plot1}
\end{figure}

\begin{table}[]
	\centering
	\begin{tabular}{|c|c|c|c|c|c|c|c|c|}
		\hline
		& 1 & 2 & 3 & 4 & 5 & 6 & 7 & 8 \\ \hline
		\text{original Cp}	&  $0.98$  & $0.19$  & $0.18$   & $0.42$ & $0.46$ & $0.20$ & $0.20$ & $0.20$\\ \hline
		\text{approximate Cp cutoff} & $0.98$	&  $0.20$  & $0.20$   & $0.40$  &$0.43$ & $0.17$ & $0.17$ & $0.15$  \\ \hline
	\end{tabular}
	\caption{The relative frequencies at which Cp-based methods select features in the absence of feature correlations with $n=20$. The last three features are superfluous.}
	\label{Table add}
\end{table}

{Next, we present results from two experiments to compare the MSE performance of Cp-based and TLCp-based methods in the presence of feature correlations (\textcolor{black}{i.e., when there exist redundant features in the model}). In the first experiment, we assume the true regression coefficient vector is $\boldsymbol { \beta }=[0.15, 0.15,0.15,0.3,0.5,0]^{\top}$. Note that the first three correspond to relevant features, the fourth corresponds to critical feature, the fifth corresponds to significantly relevant feature, and the last corresponds to a {superfluous} feature when $n=20$. To generate the feature correlations, we replicate the first column of $\boldsymbol{X}$ three times as the first three columns of the newly created design matrix $\tilde{\boldsymbol{X}}$, and the remaining columns of $\tilde{\boldsymbol{X}}$ are independently generated from the standard normal distribution. To avoid singular data, we add a very small Gaussian noise (with standard Gaussian noise divided by $1000$) to the first three columns of $\tilde{\boldsymbol{X}}$. 
To apply the TLCp-based methods (which includes the approximate TLCp method of Subsection $5.2$, the approximate TLCp cutoff method in Algorithm 2 and the original TLCp method (\ref{TLCp})), we additionally generate source data as $\boldsymbol { \tilde{y} } =\tilde{\boldsymbol { X }} (\boldsymbol { \beta }+\boldsymbol{\delta}) + \boldsymbol { \varepsilon }$. Here, we set the task dissimilarity between the target and source tasks as $\boldsymbol{\delta}=\boldsymbol{0}$, and we set the number of source data $m$ equal to the number of target data. 
We plot the MSE performance of each method under different data sizes in Figure \ref{convergence_rate_plot2}, and we record the relative frequency of each feature by each method in the case of $n=m=20$ in Table \ref{Table1}. From these experimental results, we have the following observations. 1) In Figure \ref{convergence_rate_plot2} (left), our approximate Cp cutoff method performs nearly as well as the original Cp criterion in terms of MSE. In the presence of feature correlations, the logarithm of our Cp-based methods decays approximately at the rate $\mathcal{O}\left({\frac{1}{{n}}}\right)$, which supports the results of Theorem \ref{theorem21}, Corollary \ref{corollary22} and Theorem \ref{theorem24}. 2) In Figure \ref{convergence_rate_plot2} (right), the proposed TLCp-based methods show clear improvement on the Cp-based methods in the sense of MSE, i.e., the TLCp-based methods all have much smaller intercepts than the Cp-based methods. 3) From Table \ref{Table1}, the approximate Cp cutoff method performs slightly better than the original Cp criterion both in identifying relevant features and deleting superfluous ones.  4) From Table \ref{Table1}, our TLCp-based methods identify all relevant features significantly frequently. However, the approximate TLCp cutoff method selects the superfluous feature very frequently. \textcolor{black}{Note that the proposed approximate Cp and TLCp cutoff methods only select one out of the three correlated features due to the modified Gram-Schmidt process.} In the second experiment, we suppose that the true regression coefficients are uniformly drawn from (-1,1) and then held fixed ($\boldsymbol{\beta}=[-0.26, -0.26,  -0.26, -0.91, 0.73, 0.05]^{\top}$).
We follow the same experimental setting as in the first experiment. We also make the first three features identical to each other. The corresponding MSE performance of different methods and the relative frequency of selecting each feature when $n=m=20$ are presented in Figure \ref{convergence_rate_plot3} and Table \ref{Table2}, respectively. }
{Collectively, these experimental results demonstrate the efficiency of the proposed approximate Cp and TLCp methods, and  the resulting feature selection strategies in Algorithms 1 and 2.}

\begin{figure}[htbp]
	\begin{minipage}[t] {0.5\linewidth}
		\centerline{\includegraphics[height=6.5cm,width=7.5cm]{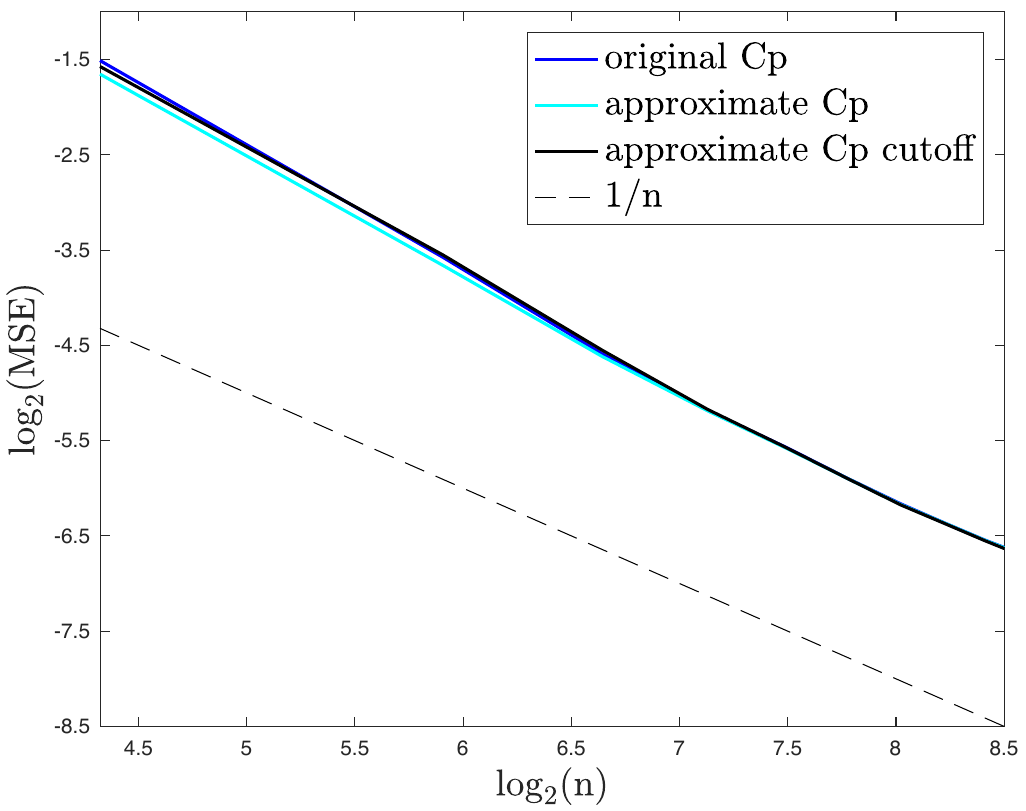}}
		\centerline{(1)}
	\end{minipage}
	\hfill
	\begin{minipage}[t] {0.5\linewidth}
		\centerline{\includegraphics[height=6.5cm,width=7.5cm]{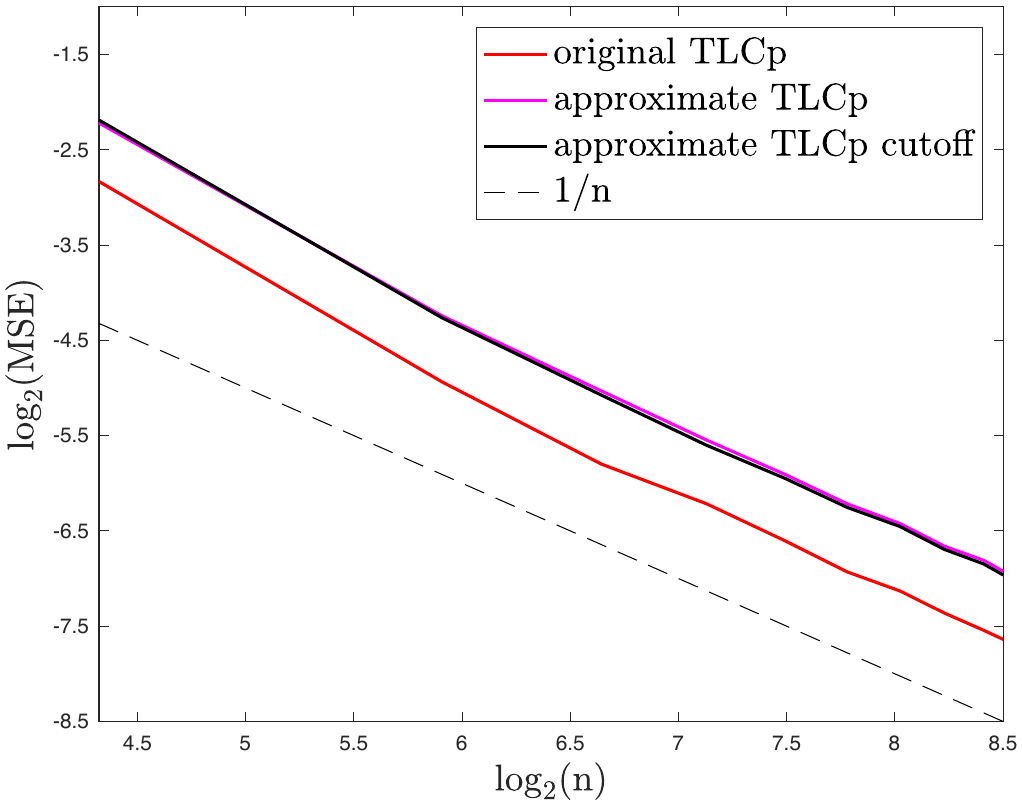}}
		\centerline{(2)}
	\end{minipage}
	\caption{\small MSE performance comparison of Cp-based and TLCp-based methods in the presence of feature correlations when coefficients are fixed. The figure on the left (right) shows how the logarithm of MSE changes with an increasing of $\log_2(n)$ for Cp-based (TLCp-based) methods. The dashed line which decays at $1/n$ is the baseline to compare the convergence rate of these methods. }
	%	\label{critical points}
	\label{convergence_rate_plot2}
\end{figure}

\begin{table}[htbp]
	\centering
	\begin{tabular}{|c|c|c|c|c|}
		\hline
		& 1 or 2 or 3 & 4 & 5 & 6 \\ \hline
		\text{original Cp}	&  $0.72$  & $0.44$  & $0.71$   & $0.20$  \\ \hline
		\text{approximate Cp cutoff}	&  $0.69$  & $0.47$   & $0.77$  &$0.17$   \\ \hline
		\text{original TLCp}&  $0.89$ & $0.59$ & $0.88$  & $0.17$   \\ \hline
		\text{approximate TLCp cutoff}&   $0.91$  & $0.73$  &  $0.93$ & $0.39$   \\ \hline
	\end{tabular}
	\caption{\textcolor{black}{The relative frequencies} at which different methods select different features in the presence of feature correlations for $n=m=20$ with fixed coefficients. By construction of the data, the first three features are almost identical. Thus, at least one of them should be selected by a successful approach.}
	\label{Table1}
\end{table}

\begin{figure}[htbp]
	\begin{minipage}[t] {0.5\linewidth}
		\centerline{\includegraphics[height=6.5cm,width=7.5cm]{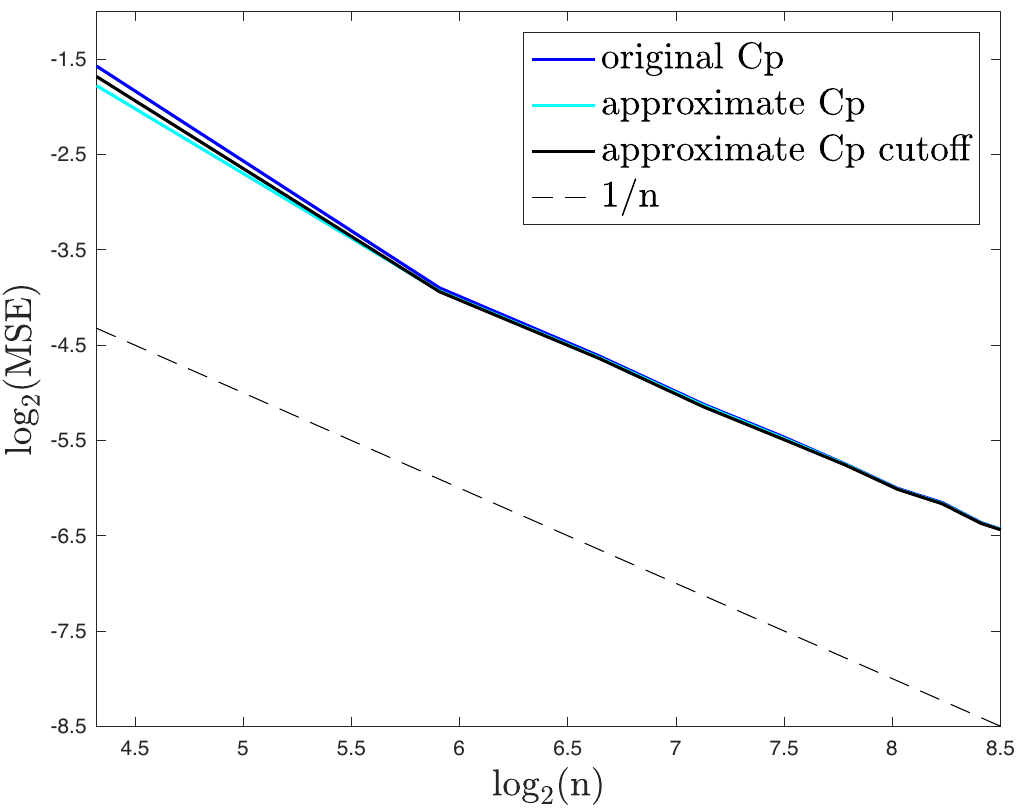}}
		\centerline{(1)}
	\end{minipage}
	\hfill
	\begin{minipage}[t] {0.5\linewidth}
		\centerline{\includegraphics[height=6.5cm,width=7.5cm]{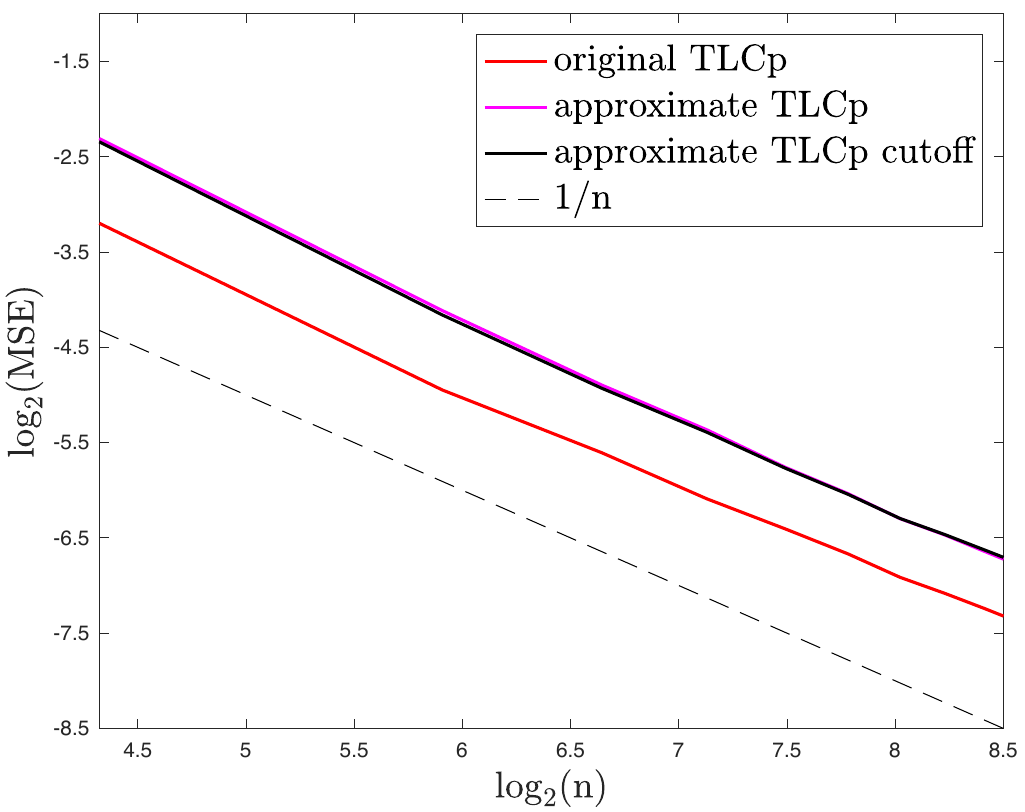}}
		\centerline{(2)}
	\end{minipage}
	%\vfill
	%	\begin{minipage}[t] {0.5\linewidth}
	%	\centerline{\includegraphics[height=6.5cm,width=8.0cm]{simulation/approximate_TLCp1.eps}}
	%	\centerline{(3)}
	%\end{minipage}
	\caption{\small MSE performance comparison of Cp-based and TLCp-based methods in the presence of feature correlations when coefficients are generated randomly. The figure on the left (right) shows how the logarithm of MSE changes with an increasing $\log_2(n)$ for Cp-based (TLCp-based) methods. The dashed line which decays at $1/n$ is the baseline to compare the convergence rate of these methods. }
	\label{convergence_rate_plot3}
\end{figure}

\begin{table}[htbp]
	\centering
	\begin{tabular}{|c|c|c|c|c|}
		\hline
		& 1 or 2 or 3 & 4 & 5 & 6 \\ \hline
		\text{original Cp}	&  $0.94$  & $0.96$  & $0.91$   & $0.20$  \\ \hline
		\text{approximate Cp cutoff}	&  $0.94$  & $0.97$   & $0.92$  &$0.18$   \\ \hline
		\text{original TLCp}&  $1.00$ & $1.00$ & $0.99$  & $0.17$   \\ \hline
		\text{approximate TLCp cutoff}&   $1.00$  & $1.00$  &  $0.99$ & $0.35$   \\ \hline
	\end{tabular}
	\caption{\textcolor{black}{The relative frequencies} at which different methods select different features in the presence of feature correlations for $n=m=20$ with random generated coefficients. By construction of the problems, the first three features are almost identical. Thus, at least one of them should be selected by a successful approach.}
	\label{Table2}
\end{table}

{Finally, following the same experimental setup as in the second simulation study where we assume $\boldsymbol { \beta }=[0.15, 0.15,0.15,0.3,0.5,0]^{\top}$ and $n=m=20$, we compare the proposed TLCp procedures with two baseline methods including the original Cp method and running the original Cp method on the aggregate dataset formed by combining data for both the target and source tasks (referred to as ``aggregate original Cp").}
	
	%original TLCp and approximate TLCp cutoff methods with benchmarks including LASSO \citep{tibshirani1996regression}, stepwise feature selection method (stepwise FS) \citep{draper1998applied}, univariate feature selection method (univariate FS) \citep{guyon2003introduction}, and running the aforementioned methods on the aggregate dataset formed by combining data for both the target and source tasks (referred to as ``aggregate LASSO", ``aggregate stepwise FS" and `` aggregate univariate FS" hereinafter). }

{Table \ref{table4} reports the MSE performance of different methods and the CPU times of each model per run (problem instance) when the task dissimilarity is zero. We observe that the aggregate original Cp method shows a remarkable improvement over original Cp. This occurs because the aggregate methods have twice number of data to begin with when the task dissimilarity is zero. The fact that the proposed original TLCp method performs similarly to the aggregate original Cp method is due to the equivalence between the original TLCp and the aggregate original Cp methods when task dissimilarity is zero. We also find that the approximate TLCp cutoff method significantly outperforms the original Cp method. The original TLCp method performs better than the approximate TLCp cutoff method in this case.  This occurs because the orthogonalization step of the approximate TLCp cutoff method may affect the similarity level of these two tasks. }

{As shown in Table \ref{table4}, the CPU time requirements of the approximate TLCp cutoff method are significantly lower than the other methods since its hyperparameters are specifically determined by Theorem \ref{theorem18} which affords a closed-form solution.}

\begin{table}[htbp]
	\centering
	\begin{tabular}{|c|c|c|}
		\hline
		& \text{MSE}  & \text{CPU s}  \\ \hline
%	\text{LASSO}	& $0.3349$ &$0.2258$  \\ \hline
%	\text{stepwise FS}	&$0.4541$  & $0.0098$  \\ \hline
%	\text{univariate FS}	& $0.4429$ & $0.0150$ \\ \hline
	\text{original Cp}	& $0.35$ & $1.77\times10^{-2}$ \\ \hline
%	\text{aggregate LASSO} &$0.1426$ & $0.2506$ \\ \hline
%	\text{aggregate stepwise FS} &$0.2220$ & $0.0136$ \\ \hline 
%	\text{aggregate univariate FS} & $0.2482$& $0.0112$\\ \hline
		\text{aggregate original Cp} & $0.15$& $3.07\times10^{-2}$\\ \hline
	\text{original TLCp}	& $0.15$ & $3.92\times10^{-2}$  \\ \hline
	\text{approximate TLCp cutoff}	& $0.25$ & $0.50\times10^{-2}$  \\ \hline
	\end{tabular}
\caption{MSE performance comparison and CPU time per run of Cp-based and TLCp-based methods in the presence of feature correlations with $n=m=20$.}\label{table4}
\end{table}

{\subsection{Using TLCp in the High-dimensional Case}}

{In this subsection, we test the performance of the proposed TLCp methods when the number of features $k$ is much larger than $n$. We use $k=60, 90, 300, 3000, 30000$ and fix $n=30$ ($m=30$) in this simulation. Here, we randomly select $4$ attributes of $\boldsymbol { \beta }_3$ to be i.i.d sampled from the standard normal distribution (then held fixed) and the rest are set to zero. % The main goal of this simulation is to compare the MSE performances of the approximate Cp and TLCp methods, due to the high dimensionality. 
In order to investigate how task similarity affects the performance of the proposed TLCp methods, we consider the performance by selecting $6$ different task similarity values. Since $\boldsymbol{X}^{\top}\boldsymbol{X}$ is singular when $n<k$, we will first select a projection operator $\tilde{\boldsymbol{Q}} \in \mathbb{R}^{k\times n}$ by using the eigenvalue decomposition method such that $\tilde{\boldsymbol{Q}}^{\top}\boldsymbol{X}^{\top}\boldsymbol{X}\tilde{\boldsymbol{Q}}=nI$. Then, a closed-form solution $\tilde{\boldsymbol{\alpha}}$ for the orthogonalized Cp problem is obtained. We back-transform this solution and obtain $\tilde{\boldsymbol{\alpha}}_2=\tilde{\boldsymbol{Q}}\tilde{\boldsymbol{\alpha}}$, the approximate Cp estimator. Finally, we execute feature selection by Algorithm 1 to obtain the approximate Cp cutoff estimator. Following a similar procedure and using Algorithm 2, we obtain the corresponding approximate TLCp cutoff estimator in the high-dimensional setting. For each task similarity value $1/\|\boldsymbol{\delta}\|_2$, we randomly simulated $5000$ \textcolor{black}{datasets} and applied the high-dimensional version of least squares \citep{wang2016no}, approximate Cp, approximate TLCp and approximate TLCp cutoff methods. The hyperparameters of the proposed TLCp procedures are tuned based on the rule introduced in Theorem $20$. }

{We report the performance comparison of the methods above in Table \ref{Table high-dimensional-simulation}. We first see that the MSE value of the approximate TLCp estimator decreases with increasing the task similarity. The approximate TLCp method outperforms both the least squares and the approximate Cp methods for each high-dimensional case. Second, we find that the approximate Cp cutoff method clearly improves the non-cutoff counterpart when the number of features is less than $3000$. Similar results are obtained for the TLCp method. These results demonstrate the efficiency of using Algorithms 1 and 2 to select features. Finally, we observe that the performance gap between these models gets smaller with increasing number of features. The models are almost identical when the number of features is larger than $3000$.}%We will investigate the performance of the original TLCp with many more features than examples in our future work.}

%\begin{figure}
%\begin{minipage}{\textwidth}
%	\centering{\includegraphics[height=8cm,width=12.0cm]{simulation/simulation_n<k_1.eps}} 

%\end{minipage}
%	\caption{\small Performance comparison between the approximate Cp and TLCp methods when the number of features $k$ is significantly larger than the sample size $n$. } \label{simultion}
%\end{figure}

\begin{table}[htbp]
	\centering
	\begin{tabular}{|c|c|cccccc|}
		\hline
		& & \multicolumn{6}{c|}{\textcolor{black}{Task Similarity Values}}  \\    Settings & Methods         & $\infty$ & $17.86$ & $8.93$ & $4.47$ & $1.98$ & $0.60$ \\
		\hline
		\multirow{5}{*}{\begin{tabular}[c]{@{}c@{}}$(n,k)=$\\(30,60)\end{tabular}}   & $\text{LS}$   &  $1.69$     &  $1.69$      &   $1.69$   &  $1.69$    &  $1.69$    &   $1.69$   \\
		\cline{2-8}
		& $\text{Cp}$   &  $1.48$     &   $1.48$    &   $1.48$   &   $1.48$   &   $1.48$   &  $1.48$    \\
			\cline{2-8}
		& $\text{Cp cutoff}$ & $0.87$      &   $0.87$    &  $0.87$    &  $0.87$    & $0.87$     & $0.87$     \\
		\cline{2-8}
		& $\text{TLCp}$   &  $1.08$     &   $1.08$    &   $1.08$   &   $1.09$   &   $1.11$   &  $1.24$    \\
		\cline{2-8}
		& $\text{TLCp cutoff}$ & $0.74$      &   $0.74$    &  $0.73$    &  $0.74$    & $0.75$     & $1.01$     \\
		\hline
		\multirow{5}{*}{\begin{tabular}[c]{@{}c@{}}$(n,k)=$\\(30,90)\end{tabular}}   & $\text{LS}$    &  $1.32$     &   $1.30$    & $1.30$      &  $1.30$    &  $1.30$    & $1.30$     \\
			\cline{2-8}
		& $\text{Cp}$   &   $1.25$    &  $1.25$     & $1.25$     & $1.25$     &  $1.25$    &   $1.25$   \\
			\cline{2-8}
		& $\text{Cp cutoff}$ &   $0.90$    &  $0.90$     &   $0.90$   &  $0.90$    &  $0.90$    &  $0.90$    \\
			\cline{2-8}
		& $\text{TLCp}$   &   $1.06$    &  $1.06$     & $1.06$     & $1.06$     &  $1.06$    &   $1.10$   \\
		\cline{2-8}
		& $\text{TLCp cutoff}$ &   $0.83$    &  $0.83$     & $0.83$     & $0.83$     &  $0.83$    &   $0.94$    \\
		\hline
		\multirow{5}{*}{\begin{tabular}[c]{@{}c@{}}$(n,k)=$\\(30,300)\end{tabular}}  & $\text{LS}$    &   $1.27$    &  $1.27$     &   $1.27$   &   $1.27$   &  $1.27$    & $1.27$     \\
			\cline{2-8}
		& $\text{Cp}$   &   $1.26$    &  $1.26$     &   $1.26$   &   $1.26$   & $1.26$     &  $1.26$    \\
			\cline{2-8}
		& $\text{Cp cutoff}$ &  $1.16$     &    $1.16$   &  $1.16$    &   $1.16$   &  $1.16$    &   $1.16$   \\
			\cline{2-8}
		& $\text{TLCp}$   &   $1.21$    &  $1.21$     & $1.21$     & $1.21$     &  $1.22$    &   $1.23$   \\
		\cline{2-8}
		& $\text{TLCp cutoff}$ &   $1.13$    &  $1.13$     &   $1.12$   &  $1.12$    &  $1.11$    &  $1.14$    \\
		\hline
		\multirow{5}{*}{\begin{tabular}[c]{@{}c@{}}$(n,k)=$\\(30,3000)\end{tabular}} & $\text{LS}$    &   $1.25$    &  $1.25$     &  $1.25$    &   $1.25$   &  $1.25$    &   $1.25$   \\
			\cline{2-8}
		& $\text{Cp}$   &  $1.25$     &   $1.25$    &  $1.25$    &  $1.25$    & $1.25$     &  $1.25$    \\
			\cline{2-8}
		& $\text{Cp cutoff}$ & $1.25$      &   $1.25$    &  $1.25$    &  $1.25$    &  $1.25$    &   $1.25$  \\
			\cline{2-8}
		& $\text{TLCp}$   &   $1.25$    &  $1.25$     & $1.25$     & $1.25$     &  $1.25$    &   $1.25$   \\
		\cline{2-8}
		& $\text{TLCp cutoff}$ &   $1.25$    &  $1.25$     &   $1.25$   &  $1.25$    &  $1.25$    &  $1.24$    \\
		\hline
		\multirow{5}{*}{\begin{tabular}[c]{@{}c@{}}$(n,k)=$\\(30,30000)\end{tabular}} & $\text{LS}$    &   $1.25$    &  $1.25$     &  $1.25$    &   $1.25$   &  $1.25$   &   $1.25$   \\
		\cline{2-8}
		& $\text{Cp}$   &  $1.25$     &   $1.25$    &  $1.25$    &  $1.25$    & $1.25$     &  $1.25$    \\
		\cline{2-8}
		& $\text{Cp cutoff}$ & $1.25$      &   $1.25$    &  $1.25$    &  $1.25$    &  $1.25$    &   $1.25$  \\
			\cline{2-8}
		& $\text{TLCp}$   &   $1.25$    &  $1.25$     & $1.25$     & $1.25$     &  $1.25$    &   $1.25$   \\
		\cline{2-8}
		& $\text{TLCp cutoff}$ &   $1.25$    &  $1.25$     &   $1.25$   &  $1.25$    &  $1.25$    &  $1.25$    \\
		\hline
	\end{tabular}
 \caption{MSE performance of least squares (LS), approximate Cp, approximate Cp cutoff,  approximate TLCp and approximate TLCp cutoff estimators under different task similarity values when the number of features $k$ is (significantly) larger than the sample size $n$. The standard deviations of the (mean) MSE for each model is not shown in this table because they are all very small (i.e., less than $0.01$).}
\label{Table high-dimensional-simulation}
\end{table}

{\section{Real Data Applications}}

{We test the original and approximate TLCp cutoff methods in the non-orthogonal case on three real datasets to demonstrate their potential applications in practice. Our methods will be compared with benchmarks including LASSO \citep{tibshirani1996regression}, stepwise feature selection method (stepwise FS) \citep{draper1998applied}, univariate feature selection method (univariate FS) \citep{guyon2003introduction}, and running the aforementioned methods on the aggregate dataset formed by combining data for both the target and source tasks (referred to as ``aggregate LASSO", ``aggregate stepwise FS" and `` aggregate univariate FS" hereinafter). Additionally, we will compare the proposed TLCp methods with two state-of-art multi-task learning methods (referred to as ``the least $\ell_{2,1}$-norm'' \citep{lounici2009taking}  and ``multi-level LASSO'' \citep{lozano2012multi}). We use the software package from \citet{zhou2012mutal} and \citet{matwork2017statistics} %and %\citet{Statovic2020} 
	to solve these two multi-task methods.}

{We implement the aforementioned benchmarks based on the statistics and machine learning toolbox \citep{matwork2017statistics}. For univariate FS, we perform a $t$-test to decide whether the linear relationship (i.e., the Pearson correlation coefficient) between a feature and the response is significant or not. $F$-test is used in the stepwise FS (forward-backward selection) to determine whether a model with more parameters gives a significantly better least-square fit to the data. We use a predetermined significance level of $0.05$.}% Here, we refer the stepwise FS to the forward-backward selection and use a predetermined significance level of $0.05$.}
	
{In all experiments below, we tune the regularization parameter of the least $\ell_{2,1}$-norm method by selecting among the values $\{10^{-6}, 10^{-5}, \cdots, 10^{-1}, 10^{0}, 10^1, 10^2,10^3 \}$ with $5$-fold cross-validation according to \citet{zhou2012mutal} and \citet{argyriou2008convex}. There are two regularization parameters in the multi-level LASSO; we fix the one 
controlling the global sparsity as $1$ and tune the other one by selecting among the values $\{10^{-6}$, $10^{-5}, \cdots$, $10^{-1}$, $10^{0}$, $10^1$, $10^2$, $10^3 \}$ with $5$-fold cross-validation based on \citet{lozano2012multi}.  Following the same hyperparameter tuning protocol as above, we tune the regularization parameters of LASSO and its aggregate method by choosing from the values $\{10^{-6}$, $10^{-5}, \cdots$, $10^{-1}$, $10^{0}$, $10^1$, $10^2$, $10^3 \}$ with $5$-fold cross-validation. We tune the hyperparameters of the proposed TLCp methods with two tasks based on Theorem \ref{theorem18}, as $\lambda_1^*=\hat{\sigma}_2^2, \lambda_2^*=\hat{\sigma}_1^2, {\lambda_3^t}^*=\frac{4\hat{\sigma}_1^2\hat{\sigma}_2^2}{\hat{\delta}_t^2} (t=1,\cdots, k)$ and $\lambda_4^*=\text{min}_{i\in\{1,\cdots,k\}}\left\{\lambda\left(2-\frac{\hat{Q}_i^*}{\sqrt{\hat{M}_i^*\hat{N}_i^*}}\right)/ 4\hat{\sigma}_1^2(\hat{G}_i^*)^2\right\}$, where $\hat{\sigma}_j^2$ is the estimated residual variance ($\hat{\sigma}_j^2=(\boldsymbol{Y}_j-\hat{\boldsymbol{Y}}_j)^{\top}(\boldsymbol{Y}_j-\hat{\boldsymbol{Y}}_j)/(m_j-k)$) and $\hat{\delta}_t (t=1,\cdots,k)$ is the estimated task dissimilarity, both of which are computed based on the training dataset, where $\hat{\boldsymbol{Y}}_j$ is the least squares estimation of $\boldsymbol{Y}_j$  using the training dataset and $m_j-k$ provide the degrees of freedom of the corresponding residuals, for $j=1,2$. Finally, $\hat{M}_t^*=\hat{\sigma}_1^2m\frac{\hat{\delta}_t^2+\frac{\hat{\sigma}_1^2}{n}}{\hat{\delta}_t^2+\frac{\hat{\sigma}_1^2}{n}+\frac{\hat{\sigma}_2^2}{m}},\hat{N}_t^*=\hat{\sigma}_2^2n\frac{\delta_t^2+\frac{\hat{\sigma}_2^2}{m}}{\hat{\delta}_t^2+\frac{\hat{\sigma}_1^2}{n}+\frac{\hat{\sigma}_2^2}{m}}, \hat{Q}_t^*=\frac{-2\hat{\sigma}_1^2\hat{\sigma}_2^2}{\hat{\delta}_t^2+\frac{\hat{\sigma}_1^2}{n}+\frac{\hat{\sigma}_2^2}{m}},\hat{G}_t^*=\sqrt{\frac{mn}{nM_t\hat{\sigma}_2^2+mN_t\hat{\sigma}_1^2}}$, for $t=1,\cdots,k$. For the TLCp with three tasks, we set $\lambda_1=\hat{\sigma}_2^2\hat{\sigma}_3^2$, $\lambda_2=\hat{\sigma}_1^2\hat{\sigma}_3^2$, $\lambda_3=\hat{\sigma}_1^2\hat{\sigma}_2^2$, $\boldsymbol{\gamma}^{i}=12\hat{\sigma}_1^2\hat{\sigma}_2^2\hat{\sigma}_3^2/(\hat{\delta}_1^i+\hat{\delta}_2^i)^2(i=1,\cdots,k)$ and $\lambda_4=\text{min}_{i\in\{1,\cdots,k\}}\left\{\frac{\lambda\left(2-\frac{\tilde{Q}_i}{\sqrt{\tilde{M}_i\tilde{N}_i\tilde{W}_i}}\right)}{4\sigma_1^2(\tilde{G}_i)^2}\right\}$, where $\lambda=2\hat{\sigma}_1^2$, $\tilde{Q}_i=\frac{-2\hat{\sigma}_1^2\hat{\sigma}_2^2\hat{\sigma}_3^2}{(\hat{\delta}_1^i+\hat{\delta}_2^i)^2+\frac{\hat{\sigma}_1^2}{n}+\frac{\hat{\sigma}_2^2}{m_2}+\frac{\hat{\sigma}_3^2}{m_3}}$, $\tilde{M}_i=\frac{\hat{\sigma}_1^2m_2m_3[(\hat{\delta}_1^i+\hat{\delta}_2^i)^2+\frac{\hat{\sigma}_1^2}{n}]}{(\delta_1^i+\delta_2^i)^2+\frac{\hat{\sigma}_1^2}{n}+\frac{\hat{\sigma}_2^2}{m_2}+\frac{\hat{\sigma}_3^2}{m_3}}$, $\tilde{N}_i=\frac{\hat{\sigma}_2^2nm_3[(\hat{\delta}_1^i+\hat{\delta}_2^i)^2+\frac{\hat{\sigma}_2^2}{m_2}]}{(\hat{\delta}_1^i+\hat{\delta}_2^i)^2+\frac{\hat{\sigma}_1^2}{n}+\frac{\hat{\sigma}_2^2}{m_2}+\frac{\hat{\sigma}_3^2}{m_3}}$, $\tilde{W}_i=\frac{\hat{\sigma}_3^2nm_2[(\hat{\delta}_1^i+\hat{\delta}_2^i)^2+\frac{\hat{\sigma}_3^2}{m_3}]}{(\delta_1^i+\delta_2^i)^2+\frac{\hat{\sigma}_1^2}{n}+\frac{\hat{\sigma}_2^2}{m_2}+\frac{\hat{\sigma}_3^2}{m_3}}$, $\tilde{G}_i=\sqrt{\frac{nm_2m_3}{n\tilde{M}_i\hat{\sigma}_2^2\hat{\sigma}_3^2+m_2\tilde{N}_i\hat{\sigma}_1^2\hat{\sigma}_3^2+m_3\tilde{W}_i\hat{\sigma}_1^2\hat{\sigma}_2^2}}$ for $i=1,\cdots,k$, which is a natural extension of the hyperparameter tuning rule for the two-task case. We use enumeration to solve the original TLCp problem with the hyperparameters tuned based on Theorem \ref{theorem18}.}

\subsection{Experiments on Blast Furnace Dataset}

We first verify the effectiveness of the proposed TLCp procedures on a real blast furnace problem. {The experimental datasets are collected from two typical Chinese blast furnaces with an inner volume of about $2500 ~\text{m}^3$ and $750 ~\text{m}^3$, labeled as blast furnaces A and B, respectively \citep{gao2013rule,chen2019linear}. }There are only $395$ valid samples (after omitting some missing values) for furnace A and $800$ valid samples for furnace B. Our target task is to predict the hot metal silicon content for furnace A with the help of one source task from furnace B. Table \ref{Tab1} presents features that are relevant for predicting the hot metal silicon content for these two furnaces. Four lagged terms are also treated as inputs, due to the ($2-8\text{h}$) time delay for furnace outputs to respond to inputs. %We determined that the time delay for both blast furnaces is $4$. %We normalize this dataset to have zero mean and unit variance on each feature. %Besides the benchmarks mentioned in Subsection $6.3$, our TLCp methods will also be compared with two state-of-art multi-task learning methods (referred to as ``the least $\ell_{2,1}$-norm'' \citep{lounici2009taking}  and ``multi-level LASSO'' \citep{lozano2012multi}). \textcolor{black}{The regularization parameter of the least $\ell_{2,1}$-norm method is chosen from the interval $[0.1:0.1:1]$ by $5$-fold cross-validation. We fix one of the regularization parameters and vary the other one  We use exhaustive search to solve the original TLCp problem with the hyperparameters tuned based on Theorem \ref{theorem18}.}} %We also add a bias column to the data of each task to learn the bias. From now on, we will use the $z$-score to standardize the features of data (except for the categorical features) and add a bias column to learn the bias among tasks.}

\begin{table}[htbp]
		\caption{ Input variables for blast furnaces} \label{Tab1}\centering
	\begin{tabular}{@{}p{56mm}@{\hspace{2mm}}p{20mm}@{\hspace{2mm}}p{30mm}}
		\toprule
		Variable name [Unit] & Symbol & Input variable \\
		\midrule[0.06em]
		Blast temperature  [$^\circ \mathrm{C}$] & $x^{(1)}$ & ${q^{-1}}^\S$,$q^{-2}$,$q^{-3}$,$q^{-4}$\\
		Blast volume       [m$^3$/min]   & $x^{(2)}$ & $q^{-1}$,$q^{-2}$,$q^{-3}$,$q^{-4}$\\
		Feed speed         [mm/h]   & $x^{(3)}$ & {$q^{-1}$},$q^{-2}$,$q^{-3}$,$q^{-4}$\\
		Gas permeability   [m$^3$/min$\cdot$kPa]  & $x^{(4)}$ & $q^{-1}$,{$q^{-2}$},$q^{-3}$,$q^{-4}$\\
		Pulverized coal injection  [ton]  & $x^{(5)}$ & {$q^{-1}$},{$q^{-2}$,$q^{-3}$},{$q^{-4}$}\\
		Silicon content    [wt\%]  & $z$ & {$q^{-1}$}\\
		\bottomrule\multicolumn{3}{l}{\hspace{-2mm}\scriptsize{$^\S$
				$q^{-1},\cdots,q^{-4}$ represent delay operators, such as
				$q^{-1}x(t)=x(t-1)$. }}
	\end{tabular}
\end{table}

%{To evaluate the proposed TLCp learning mechanism, we first choose the Cp criterion, LASSO, stepwise FS, and univariate FS as the baseline methods to compare with our method. Among these benchmarks, LASSO \citep{tibshirani1996regression} is capable of delivering both sparse and good predictive performance models \citep{zhao2006model}. We use five-fold cross-validation to estimate the regularization parameter of LASSO. Stepwise FS \citep{borboudakis2019forward} is also an efficient feature selection method that adds variables to and removes variables from a multi-linear model based on their statistical significance (i.e., $p$-values with a threshold $0.05$). Univariate FS \citep{guyon2003introduction} works by selecting variables based on some ranking criterion, i.e., the Pearson correlation coefficient, which is used here. Then, our TLCp method will also be compared with the aggregate baseline methods trained on the merged dataset formed by combining data for target and source tasks. Finally, two state-of-art multi-task learning methods (referred to as ``the least $\ell_{2,1}$-norm'' \citep{lounici2009taking}  and ``multi-level LASSO'' \citep{lozano2012multi}) will be included in our comparisons. We use the software packages from \citet{zhou2012mutal} and \citet{Statovic2020} to solve these two multi-task learning algorithms. The regularization parameters of these two multi-task models are estimated using $10$-fold cross-validation. We use exhaustive search to solve the original TLCp problem here.}

For each target data size ($n=210,250,290$), we randomly split the target dataset (furnace A) $300$ times with $n$ samples as the training set and the remaining $100$ samples as the test set. For each partition, we normalize the dimensions of training samples to have zero mean and unit variance, while the test samples are normalized accordingly. For the proposed TLCp approach, all ($800$) samples of furnace B are treated as the source training \textcolor{black}{dataset}. We use the percentage unexplained variance \citep{bakker2003task} to measure model performance, {denoted as 
the mean squared prediction error on the test set as a percentage of the total data variance for a specific task}. Thus, the percentage unexplained variance can be viewed as a scaled version of the mean squared prediction error. The lower the value of the percentage unexplained variance the better the model performance. One of the advantages of using the measure of percentage unexplained variance is that it is independent of the output scale.
%We define the relative dissimilarity between the target and source tasks as  $\|\boldsymbol{\hat{\mu}}_t-\boldsymbol{\hat{\mu}}_s\|_2/\|\boldsymbol{\hat{\mu}}_t\|_2$, where $\boldsymbol{\hat{\mu}}_t$ and $\boldsymbol{\hat{\mu}}_s$ are the least squares estimations of the regression coefficient vector for the target and source tasks, respectively. 
{We calculate the relative dissimilarity value between these two blast furnace tasks as $0.89$ based on the full datasets. This value is totally independent of the training of our TLCp models. All the hyperparameters of the proposed TLCp models are estimated on the training dataset only for each partition of the data. }

\begin{figure}[htbp]
	\begin{minipage}[t] {\textwidth}
		\centerline{\includegraphics[width=1\textwidth]{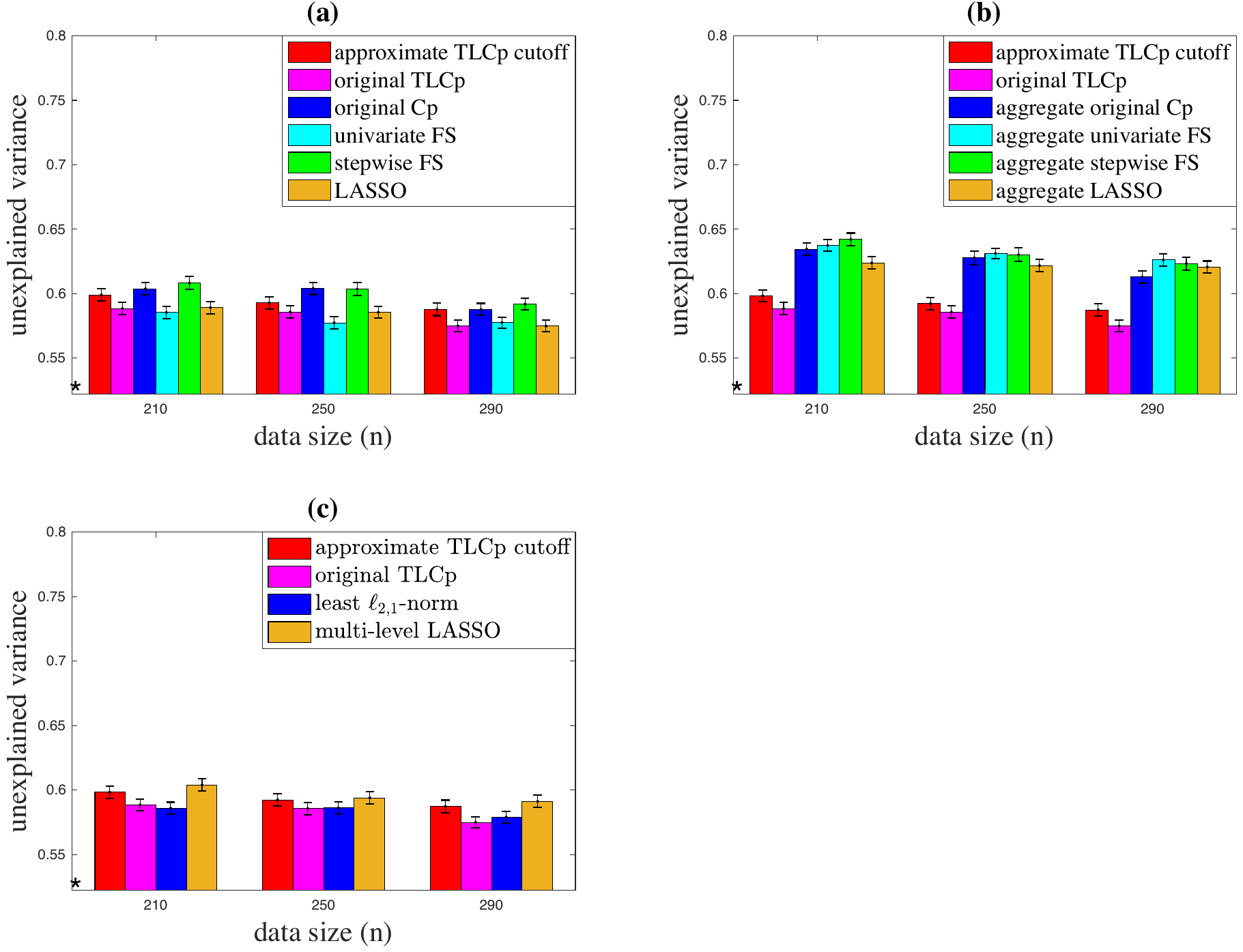}}
	\end{minipage}
	\caption{ \small The unexplained variance performance comparison of the proposed TLCp methods and other benchmarks when the relative task dissimilarity value is $0.89$ for blast furnace data. An aggregate method means running the corresponding non-aggregate method on the aggregate dataset formed by combining data for the target and source tasks. For each model, we plot the error bar to describe the standard deviation of the (mean) unexplained variance. $*$ indicates the ideal unexplained variance for the target task, which is computed by substituting the mean squared prediction error as an unbiased estimation of the residual variance $\sigma_1^2$ using the entire target data set.} \label{fig4}%under different models as a function of the number of target training data. The unexplained variance is computed based on $20$ random splits of the target blast furnace data into the training and test datasets (Smaller values indicate higher predictive accuracy). For each model, we plot the error bar to describe the standard deviation of the (mean) unexplained variance. In addition, $*$ indicates the ideal unexplained variance for the target task, which is computed by substituting the mean squared prediction error as an unbiased estimation of the residual variance $\sigma_1^2$ using the entire target data set.} \label{fig4}
\end{figure}

\begin{table}[htbp]
		\centering
	\begin{tabular}{llll}
		\hline
		\multicolumn{1}{|l|}{} & \multicolumn{1}{l|}{original TLCp} & \multicolumn{1}{c|}{\begin{tabular}[c]{@{}c@{}}approximate TLCp\\ cutoff\end{tabular}} & \multicolumn{1}{l|}{\begin{tabular}[c]{@{}c@{}}CPU time\\ (s)\end{tabular}} \\ \hline
		\multicolumn{1}{|c|}{original Cp} & \multicolumn{1}{c|}{$\textbf{0.02}$} & \multicolumn{1}{c|}{$0.50$}                                                                  & \multicolumn{1}{c|}{1077.73} \\ \hline
		\multicolumn{1}{|c|}{stepwise FS} & \multicolumn{1}{c|}{$\textbf{0.00}$ } & \multicolumn{1}{c|}{$0.26$}                                                                  & \multicolumn{1}{c|}{$0.11$} \\ \hline
		\multicolumn{1}{|c|}{univariate FS} & \multicolumn{1}{c|}{$0.33$} & \multicolumn{1}{c|}{$0.94$}                                                                  & \multicolumn{1}{c|}{$0.11$} \\ \hline
		\multicolumn{1}{|c|}{LASSO} & \multicolumn{1}{c|}{$0.50$} & \multicolumn{1}{c|}{$0.97$}                                                                  & \multicolumn{1}{c|}{$0.36$} \\ \hline
		\multicolumn{1}{|c|}{aggregate original Cp} & \multicolumn{1}{c|}{$\textbf{0.00}$} & \multicolumn{1}{c|}{$\textbf{0.00}$}                                                                  & \multicolumn{1}{c|}{1245.38} \\ \hline
		\multicolumn{1}{|c|}{aggregate stepwise FS} & \multicolumn{1}{c|}{$\textbf{0.00}$} & \multicolumn{1}{c|}{$\textbf{0.00}$}                                                                  & \multicolumn{1}{c|}{$0.22$} \\ \hline
		\multicolumn{1}{|c|}{aggregate univariate FS} & \multicolumn{1}{c|}{$\textbf{0.00}$} & \multicolumn{1}{c|}{$\textbf{0.00}$}                                                                  & \multicolumn{1}{c|}{$0.21$} \\ \hline
		\multicolumn{1}{|c|}{aggregate LASSO} & \multicolumn{1}{c|}{$\textbf{0.00}$} & \multicolumn{1}{c|}{$\textbf{0.00}$}                                                                  & \multicolumn{1}{c|}{$0.32$} \\ \hline
		\multicolumn{1}{|c|}{least $\ell_{2,1}$-norm} & \multicolumn{1}{c|}{$0.25$} & \multicolumn{1}{c|}{$0.90$}                                                                  & \multicolumn{1}{c|}{$0.35$} \\ \hline
	\multicolumn{1}{|c|}{multi-level LASSO} & \multicolumn{1}{c|}{$\textbf{0.00}$} & \multicolumn{1}{c|}{$\text{0.31}$}                                                                  & \multicolumn{1}{c|}{$11.00$} \\ \hline
	\multicolumn{1}{|c|}{original TLCp} & \multicolumn{1}{c|}{$--$} & \multicolumn{1}{c|}{$0.98$}                                                                  & \multicolumn{1}{c|}{$2550.83$} \\ \hline
	\multicolumn{1}{|c|}{approximate TLCp cutoff} & \multicolumn{1}{c|}{$\textbf{0.02}$} & \multicolumn{1}{c|}{$--$}                                                                  & \multicolumn{1}{c|}{$0.24$} \\ \hline                 
	\end{tabular}
 \caption{The table shows the $p$-value of the pairwise $t$-test (in the first two columns) and the CPU time requirements of different methods per run (in the last column) on blast furnace data when $n=290$. Boldface means the proposed TLCp methods statistically outperform the compared methods ($p$-value $<0.05$). }\label{table7}
\end{table}

Figure \ref{fig4} presents the performance comparisons of the two proposed TLCp procedures and other baseline methods. As Figure \ref{fig4}(a) shows, the proposed TLCp schemes outperform the original Cp method for each sample size. In terms of the average excess unexplained variance across three sample sizes, the original TLCp outperforms the original Cp by $20.27\%$ and the approximate TLCp cutoff method by $6.51\%$. {The excess unexplained variance is defined as the unexplained variance difference between the TLCp and the original Cp methods as a percentage of the unexplained variance difference between the original Cp method and the ideal unexplained variance.} We also find that the proposed original TLCp method is competitive with the stepwise FS, LASSO and univariate FS. 
%	and the ideal unexplained variance, the latter of which is the theoretical best performance of a model. Moreover, the original TLCp method is competitive with methods that are not Cp-based, i.e., the stepwise FS and LASSO, and performs dramatically better than univariate FS. We also observe that the approximate TLCp method, which identifies an estimator that asymptotically approximates the solution of the original TLCp problem when $n\rightarrow \infty$, achieves satisfactory performance when the sample sizes are limited. 
%These observations demonstrate the superiority of the two proposed TLCp frameworks. } 

{To show the capacity of the proposed TLCp schemes to leverage related tasks, in Figure \ref{fig4}(b), we compare the performance of the TLCp procedures and the aggregate benchmarks. We see that our TLCp methods consistently produce significant improvements over the aggregate benchmarks. This performance demonstrates the ability of the proposed TLCp mechanisms to efficiently capture useful information shared among the target and source tasks. Furthermore,  our TLCp procedures outperform the multi-level LASSO, and perform similarly to the least $\ell_{2,1}$-norm method in the blast furnace problem (see Figure \ref{fig4}(c)). }%One of advantages of applying our methods is that their hyperparameters are already determined by Theorem \ref{theorem18}.}

{To conduct a rigorous comparison, we perform the paired Student's $t$-test to test the null hypothesis that the population mean of the proposed TLCp method's unexplained variance is strictly greater than that of the compared method. A $p$-value of $0.05$ is considered statistically significant. The results are listed in Table \ref{table7}, where the proposed TLCp methods that are significantly better than the compared methods are shown in bold for each comparison. The corresponding CPU time requirements per run are also listed in Table \ref{table7}. We see that the computational requirements of the approximate TLCp cutoff method are remarkably lower than the original Cp, the aggregate original Cp, the original TLCp methods, the multi-level LASSO, and comparable to other methods. This occurs because the approximate TLCp cutoff method does not need cross-validation to tune the hyperparameters. Instead, its hyperparameters are predetermined by Theorem \ref{theorem18}. Also, in comparison to the original TLCp method, the approximate TLCp cutoff method comes with a closed-form solution and avoids the numerical solution of the problem. These two factors make the approximate TLCp cutoff method more computationally efficient than most of the compared methods.}

%{In addition,  the proposed tuning rules, our TLCp procedures outperform the existing two multi-task learning methods, the least $\ell_{2,1}$-norm method and the multi-level LASSO in the blast furnace problem (see Figure \ref{fig4} (c)). }

\subsection{Experiments on School Dataset}

{In this subsection, we evaluate the performance of the proposed TLCp method on school data used by \citet{bakker2003task}, \citet{argyriou2008convex} and \citet{zhou2012mutal}. The data consists of examination scores of $15362$ students from $139$ secondary schools in London during $1985$, $1986$, and $1987$. {Following the data pre-processing method used in \citet{bakker2003task} and \citet{argyriou2008convex}, we transformed the categorical attributes of this school data to binary variables, with a total of $27$ features}. 
Without loss of generality, our target task is to predict students' exam scores from the first school (which contains $200$ samples). The relative dissimilarities between the selected target task and all the candidate $138$ source tasks are within the interval $[0.41,2.85]$.  %Therefore, these $139$ tasks are similar to each other based on the relative task dissimilarity definition here. %This similarity is consistent with the claim in \citep{bakker2003task, argyriou2008convex} that these learning tasks are very similar, though the definition of the tasks' dissimilarity used here is different. 
%The similarity of these tasks has also been verified by \citet{bakker2003task} and \citet{argyriou2008convex} using different metric of task dissimilarity. 
%To show how the proposed TLCp schemes will behave as the relative task dissimilarity changes, we choose the $4$-th and $20$-th tasks as two source tasks for the TLCp procedure, with the corresponding task dissimilarity values $0.85$ and $2.85$. 
The design matrices for these tasks are all singular (implying redundancy of the given features). Thus, we will delete some redundant features for each task beforehand to make the corresponding design matrix full-rank. }

{In this context, experiments will be conducted on three different sample sizes. For each target sample size ($n= 130, 150, 170$), we divide the target data set into $10000$ random splits with $n$ samples as the training data and the remaining $30$ samples as the test data. For each partition, we standardize the continuous variables to have zero mean and unit variance, and not standardize the binary variables but code them as $0/1$ to retain the interpretation of the variables. {We choose three source tasks ($18$-th, $37$-th and $20$-th) with their relative task dissimilarities $0.58$, $1.51$, and $2.85$ as the representative cases to demonstrate the effectiveness of our methods when compared with other methods (see Figures \ref{school1}, \ref{school2} and \ref{school3}). Furthermore, Figure \ref{school4} presents the performance of our TLCp model for the target task with respect to increasing relative task dissimilarities.}}%Our TLCp procedures will be compared with the aforementioned benchmarks in the example, and the tuning parameters of these models will be determined as previously. The percentage unexplained variance is also used as the model performance measurement here. }

{Figure \ref{school1} presents performance comparison of different methods when the relative task dissimilarity is $0.58$.  Panel (a) shows that, when the relative task dissimilarity value is small, both the proposed original TLCp and approximate TLCp cutoff methods achieve significant improvement over the benchmarks for all three sample sizes. {In terms of the average excess unexplained variance across three data sizes}, the original TLCp model improves the original Cp by $55.81\%$ and the approximate TLCp cutoff method by $59.72\%$. Panel (b) compares the performance of TLCp methods with the aggregate results obtained by training benchmarks on the aggregate dataset formed by combining data for both the target and source tasks. We see that, when the relative task dissimilarity is small, the aggregate benchmarks remarkably outperform the individual counterparts. This observation illustrates the appropriateness of the given task dissimilarity measure. We also observe that the aggregate original Cp method outperforms the original TLCp method in this case. This occurs because the  applied parameter tuning rules for the original TLCp method are ``sub-optimal" in the non-orthogonal case. Panel (c) indicates that our TLCp methods are significantly better than the least $\ell_{2,1}$-norm method and the multi-level LASSO method when the relative task dissimilarity is small. Panel (d) shows that our TLCp methods with three tasks (where the $27$-th task with the relative task dissimilarity $0.71$ is treated as the second source task) perform similarly to the case of two tasks. This occurs because TLCp may not extract further information from the third task.}

{The experimental results of Figure \ref{school1} are verified by the $p$-value of the pairwise Student's $t$-test shown in Table \ref{table8}. The CPU time requirements per run of each method are also listed in Table \ref{table8}. We find that the proposed approximate TLCp cutoff method achieves the least computational requirement among all the compared algorithms.}
%The original TLCp approach is remarkably competitive with univariate FS, stepwise FS, and the LASSO trained on the aggregated dataset. In addition, it seems that the aggregate Cp criterion can perform almost as well as the original TLCp approach, which may be because the relative dissimilarity between the target and source task is very small. Thus, in this case, the original TLCp method is almost equivalent to the aggregate original Cp criterion. Based on the parameters tuning rule previously introduced, the estimated regularization parameter $\hat{\boldsymbol{\gamma}}$ of the TLCp procedure which controls the significance of the task-specific part of the regression coefficient vector, will approach $\infty$ when the relative dissimilarity between target and source tasks is close to $0$. This mechanism forces the TLCp procedure to learn these two tasks without distinction. We also compare the original TLCp to the other two multi-task learning methods introduced above (which are two generalizations of LASSO in the multi-task settings). We see that our TLCp method is significantly better than both these benchmarks as the target sample size grows. Finally, our experiment for the TLCp using three tasks also demonstrates the potential advantages of the proposed TLCp framework for considering increasingly related tasks.}

{We see the similar performance trends of these models when the relative task dissimilarity grows to $1.51$ (see Figure \ref{school2}) and $2.85$ (see Figure \ref{school3}). In terms of the average excess unexplained variance across the three data sizes, the original TLCp model improves the original Cp by $20.45\%$ (when the relative task dissimilarity is $1.51$) and $11.34\%$ (when the dissimilarity is $2.85$); the approximate TLCp cutoff model improves the original Cp by $40.83\%$ (when the dissimilarity is $1.51$) and $15.45\%$ (when the dissimilarity is $2.85$). The improvement of our TLCp methods over the original Cp method reduces as the relative task dissimilarity increases, because TLCp tends to extract less information from the source task if the relative task dissimilarity is large. We also find that the performance gap between the aggregate methods and the individual counterparts shrinks as the the relative task dissimilarity increases. This fact again demonstrates the rationality of the proposed relative task dissimilarity measure. When the relative task dissimilarity grows to $1.51$, our TLCp schemes perform similarly to the least $\ell_{2,1}$-norm method and the multi-level LASSO (see panel (c) in Figure \ref{school2}). Our method shows clear improvement over the least $\ell_{2,1}$-norm method and the multi-level LASSO when the relative task dissimilarity approaches $2.85$ (see panel (c) in Figure \ref{school3}). Finally, our experimental results for the TLCp methods with three tasks illustrate the remarkable advantages of integrating increasingly related tasks into the proposed TLCp schemes.}

{We also performed the paired Student's $t$-test to verify the experimental results of Figures \ref{school2} and \ref{school3} in Tables \ref{table9} and \ref{table10}. See Appendix E for details.}

{Finally, Figure \ref{school4} depicts the performance comparison of the approximate TLCp cutoff method and the original Cp method across a range of the relative task dissimilarities on the school dataset. For each relative task dissimilarity, we take $10000$ partitions of the target dataset with size $n=170$ to compute the average unexplained variance of the model. Here, we choose $47$ source tasks whose data size is no less than $130$ with the relative task dissimilarities varying from $0.41$ to $2.85$ to show the effectiveness of our TLCp cutoff method. We first observe a tendency for the proposed approximate TLCp cutoff method to perform better when the relative task dissimilarity is smaller. Note that the relative task dissimilarities on this dataset are relatively small (i.e., less than $3.00$). Thus, as the relative task dissimilarity increases, the performance curve of the TLCp increases slowly. Second, the approximate TLCp cutoff method generally outperforms the original Cp method except for very few cases where the proposed method performs slightly worse than the original Cp criterion. This occurs because there is high variance across the tasks \citep{argyriou2008convex} even though the relative task dissimilarity is small. }

%{When the relative task dissimilarity value grows to $2.85$ (see Figure \ref{school3}), the unexplained variance value of the original TLCp increases slightly. However, the former still performs better than most of the benchmarks in this case. The original TLCp improves the original Cp $30\%$ in terms of the average excess unexplained variance across three data sizes. This behavior demonstrates the robustness of the original TLCp approach compared to the relative task dissimilarity measure. We also find that the unexplained variance values of the aggregate methods are generally larger than the corresponding results in Figure \ref{school1}, again demonstrating the rationality of the given relative task dissimilarity definition.}

%{It is worth noting that the proposed approximate TLCp is also competitive with the benchmarks in this dataset, and it even performs better than the original TLCp when the target training data reaches $170$ in both the relative task dissimilarity values. Therefore, we recommend using the approximate TLCp approach when the users are particularly concerned with the cost of computation.}

\begin{figure}[htbp]
	\begin{minipage}[t] {\textwidth}
		\centerline{\includegraphics[width=1\textwidth]{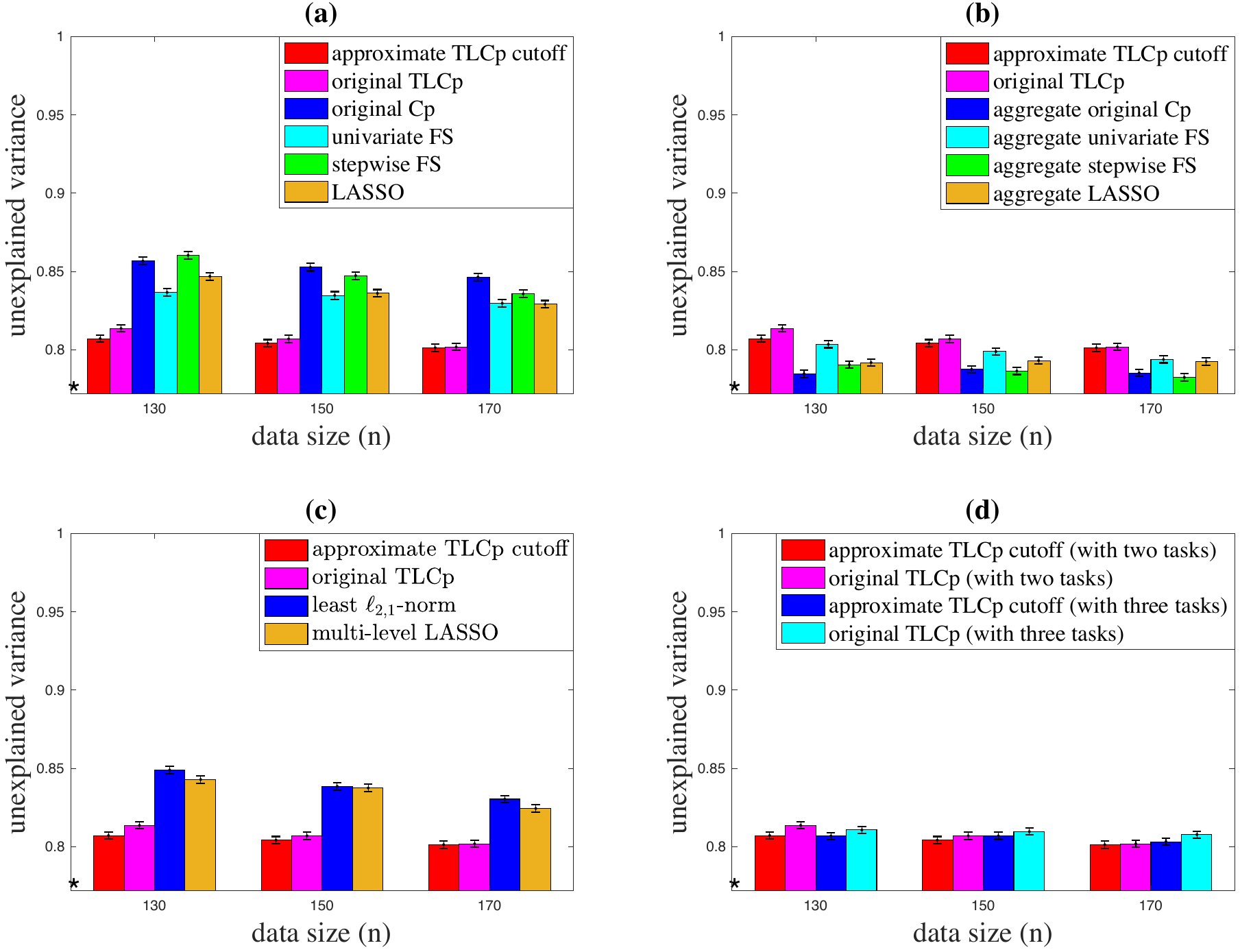}}
	\end{minipage}
	\caption{ \small Unexplained variance performance comparison of the proposed TLCp methods and other benchmarks when the relative task dissimilarity value is $0.58$ for school data. Smaller values indicate higher predictive accuracy. An aggregate method means running the corresponding non-aggregate method on the aggregate dataset formed by combining data for the target and source tasks.} \label{school1}
		
		%under different models as the number of target training data varies. The unexplained variance is computed based on $200$ random splits of the school data w.r.t. the $1$-st task (the target one) into the training and test datasets, the $4$-th task acts as the source task. Smaller values indicate higher predictive accuracy. The $11$-th task (with relative task dissimilarity value $0.44$) is treated as the second source task for the the TLCp with three tasks. For each model, we plot the error bar to describe the standard deviations of the (mean) unexplained variance. In addition, $*$ indicates the ideal unexplained variance for the target task as mentioned before.} \label{school1}
\end{figure}

\begin{table}[htbp]
	\centering
	\begin{tabular}{llll}
		\hline
		\multicolumn{1}{|l|}{} & \multicolumn{1}{l|}{original TLCp} & \multicolumn{1}{c|}{\begin{tabular}[c]{@{}c@{}}approximate TLCp\\ cutoff\end{tabular}} & \multicolumn{1}{l|}{\begin{tabular}[c]{@{}c@{}}CPU time\\ (s)\end{tabular}} \\ \hline
		\multicolumn{1}{|c|}{original Cp} & \multicolumn{1}{c|}{$\textbf{0.00}$} & \multicolumn{1}{c|}{$\textbf{0.00}$}                                                                  & \multicolumn{1}{c|}{$0.97$} \\ \hline
		\multicolumn{1}{|c|}{stepwise FS} & \multicolumn{1}{c|}{$\textbf{0.00}$ } & \multicolumn{1}{c|}{$\textbf{0.00}$}                                                                  & \multicolumn{1}{c|}{$0.77$} \\ \hline
		\multicolumn{1}{|c|}{univariate FS} & \multicolumn{1}{c|}{$\textbf{0.00}$} & \multicolumn{1}{c|}{$\textbf{0.00}$}                                                                  & \multicolumn{1}{c|}{$0.61$} \\ \hline
		\multicolumn{1}{|c|}{LASSO} & \multicolumn{1}{c|}{$\textbf{0.00}$} & \multicolumn{1}{c|}{$\textbf{0.00}$}                                                                  & \multicolumn{1}{c|}{$1.12$} \\ \hline
		\multicolumn{1}{|c|}{aggregate original Cp} & \multicolumn{1}{c|}{${1.00}$} & \multicolumn{1}{c|}{${1.00}$}                                                                  & \multicolumn{1}{c|}{$0.15$} \\ \hline
		\multicolumn{1}{|c|}{aggregate stepwise FS} & \multicolumn{1}{c|}{$\text{1.00}$} & \multicolumn{1}{c|}{$\text{1.00}$}                                                                  & \multicolumn{1}{c|}{$0.33$} \\ \hline
		\multicolumn{1}{|c|}{aggregate univariate FS} & \multicolumn{1}{c|}{$\text{0.99}$} & \multicolumn{1}{c|}{$\text{0.99}$}                                                                  & \multicolumn{1}{c|}{$0.02$} \\ \hline
		\multicolumn{1}{|c|}{aggregate LASSO} & \multicolumn{1}{c|}{$\text{1.00}$} & \multicolumn{1}{c|}{$\text{1.00}$}                                                                  & \multicolumn{1}{c|}{$0.11$} \\ \hline
		\multicolumn{1}{|c|}{least $\ell_{2,1}$-norm} & \multicolumn{1}{c|}{$\textbf{0.00}$} & \multicolumn{1}{c|}{$\textbf{0.00}$}                                                                  & \multicolumn{1}{c|}{$0.68$} \\ \hline
		\multicolumn{1}{|c|}{multi-level LASSO} & \multicolumn{1}{c|}{$\textbf{0.00}$} & \multicolumn{1}{c|}{$\textbf{0.00}$}                                                                  & \multicolumn{1}{c|}{$1.32$} \\ \hline
		\multicolumn{1}{|c|}{original TLCp} & \multicolumn{1}{c|}{$--$} & \multicolumn{1}{c|}{$0.43$}                                                                  & \multicolumn{1}{c|}{$0.58$} \\ \hline
		\multicolumn{1}{|c|}{approximate TLCp cutoff} & \multicolumn{1}{c|}{$\text{0.57}$} & \multicolumn{1}{c|}{$--$}                                                                  & \multicolumn{1}{c|}{$0.01$} \\ \hline       
		\multicolumn{1}{|c|}{\begin{tabular}[c]{@{}c@{}}original TLCp\\ with three tasks\end{tabular}} & \multicolumn{1}{c|}{$\textbf{0.03}$} & \multicolumn{1}{c|}{$\textbf{0.02}$}                                                                  & \multicolumn{1}{c|}{$0.10$} \\ \hline
		\multicolumn{1}{|c|}{\begin{tabular}[c]{@{}c@{}}approximate TLCp cutoff\\ with three tasks\end{tabular}} & \multicolumn{1}{c|}{$\text{0.39}$} & \multicolumn{1}{c|}{$0.33$}                                                                  & \multicolumn{1}{c|}{$0.04$} \\ \hline                 
	\end{tabular}
	\caption{The table shows the $p$-value of the pairwise $t$-test (in the first two columns) and the CPU time requirements per run of different methods (in the last column) on school data with the relative task dissimilarity $0.58$ when $n=170$. Boldface means the proposed TLCp methods statistically outperform the compared methods ($p$-value $<0.05$). }\label{table8}
\end{table}

\begin{figure}[htbp]
	\begin{minipage}[t] {\textwidth}
		\centerline{\includegraphics[width=1\textwidth]{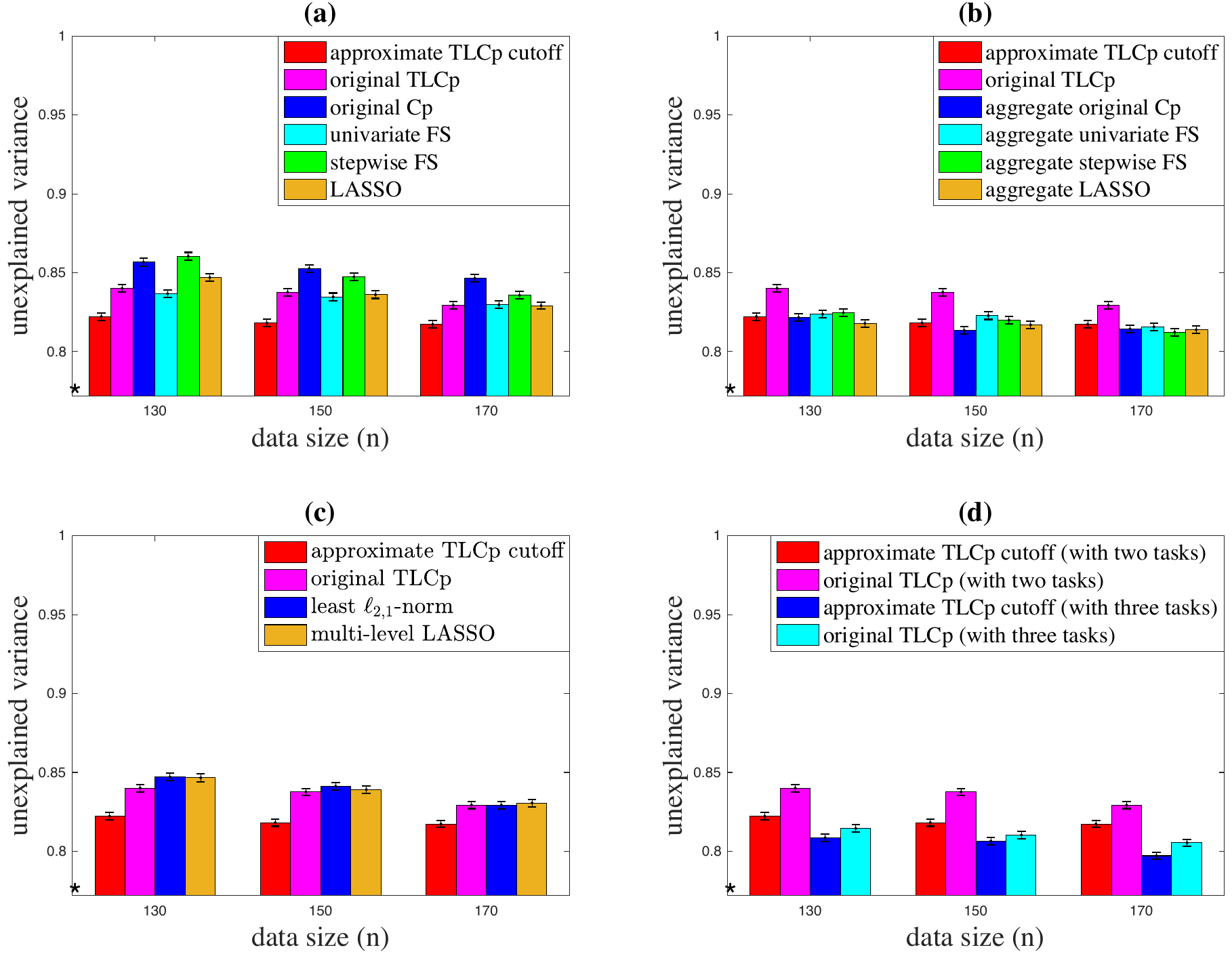}}
	\end{minipage}
	\caption{ \small The unexplained variance performance comparison of the proposed TLCp methods and other benchmarks when the relative task dissimilarity value is $1.51$ for school data. Smaller values indicate higher predictive accuracy. An aggregate method means running the corresponding non-aggregate method on the aggregate dataset formed by combining data for the target and source tasks.} \label{school2}%Unexplained variance performance under different models. As introduced above, we make $200$ random splits of the school data w.r.t. the $1$-st task to calculate the unexplained variance. The $20$-th task is treated as the (first) source task. For the TLCp with three tasks, we use the $11$-th task as the second source task.} \label{school2}
\end{figure}

\begin{figure}[htbp]
	\begin{minipage}[t] {\textwidth}
		\centerline{\includegraphics[width=1\textwidth]{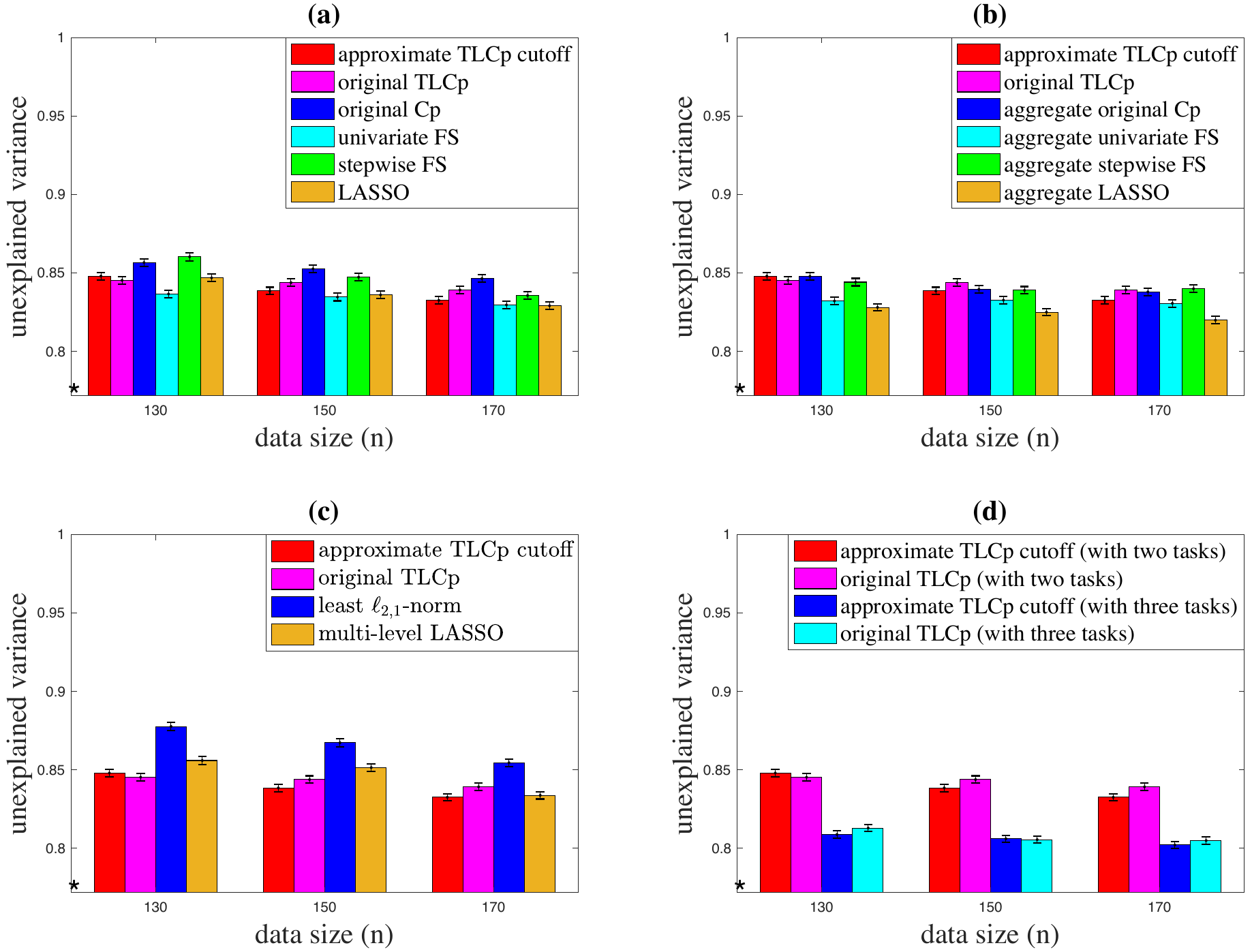}}
	\end{minipage}
	\caption{ \small Unexplained variance performance comparison of the proposed TLCp methods and other benchmarks when the relative task dissimilarity value is $2.85$ for the school dataset. Smaller values indicate higher predictive accuracy. An aggregate method means running the corresponding non-aggregate method on the aggregate dataset formed by combining data for the target and source tasks.} \label{school3}
\end{figure}

\begin{figure}[htbp]
	\begin{minipage}[t] {\textwidth}
		\centerline{\includegraphics[width=0.7\textwidth]{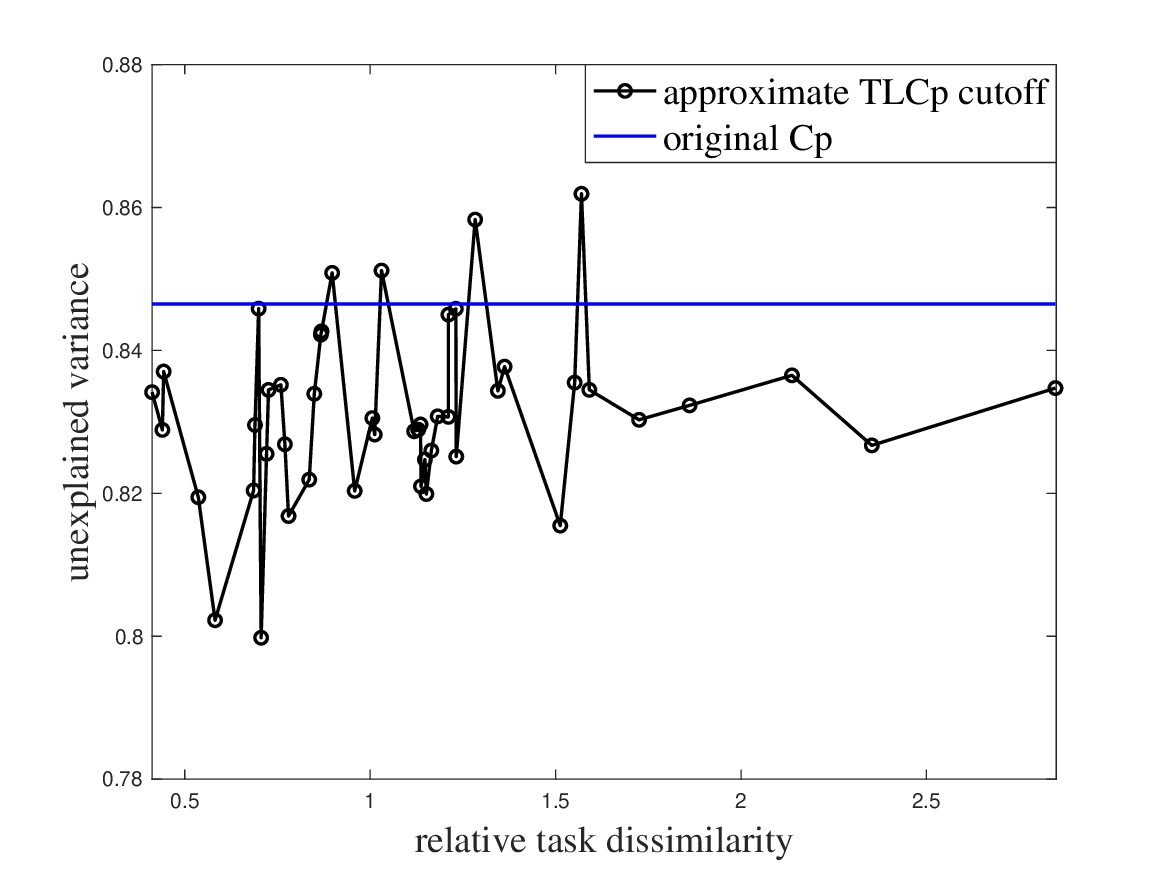}}
	\end{minipage}
	\caption{ \small The figure shows how the performance of the approximate TLCp cutoff method changes with increasing relative task dissimilarities on school data. The horizontal line indicates the performance of the original Cp criterion.} \label{school4}
\end{figure}

\subsection{Experiments on Parkinson's Dataset}

{ We finally test the proposed TLCp methods using the ``Parkinson's telemonitoring dataset'' from the UCI Machine Learning Repository \citep{tsanas2009accurate}. This dataset consists of a range of biomedical voice measurements from $42$ people with early-stage Parkinson's disease recruited for a six-month trial of a telemonitoring device for remote symptom progression monitoring. Our main goal is to predict the monitor UPDRS score for each person from the given features including the time interval from baseline recruitment date and $16$ biomedical voice measures. Thus, there are $42$ tasks. Without loss of generality, we choose data from the first person as the target dataset. Then, we have $41$ candidate source datasets. The relative dissimilarities between the selected target task and the remaining $41$ candidate source tasks are within the interval of $[0.13,21.22]$, indicating the divergent levels of similarity between tasks. To show the performance of the TLCp procedures when the relative task dissimilarity varies greatly, we choose the records from the $24$-th, $3$-rd and $36$-th persons as the source datasets for the TLCp, with the corresponding relative task dissimilarities $0.20$, $2.02$ and $21.22$, respectively. }

{Considering there are fewer than $150$ records for each selected task, the experiments are conducted only on two different sample sizes. For each sample size ($n=100,110$), we randomly split the target dataset {$5000$} times with $n$ samples as the training set and the remaining $30$ as the test set. Following the same experimental design used in the last two examples, our TLCp procedures will be compared to the benchmarks, and the corresponding tuning parameters will be determined as previously. The percentage unexplained variance is used to measure the prediction performance of different models.}

%{All experimental results on the Parkinson's data are shown in Figure \ref{parkinsons1}, Figure \ref{parkinsons2} and Figure \ref{parkinsons3}, with the relative task dissimilarities $0.20$, $2.02$ and $21.22$, respectively. 
{We first observe from these results that, when the relative task dissimilarity is relatively small, i.e., $0.20$ in Figure \ref{parkinsons1}, the proposed TLCp methods remarkably outperform the original Cp criterion for both sample sizes. Specifically, the original TLCp improves the original Cp and the approximate TLCp cutoff method by $28.74\%$ and $23.00\%$, respectively, {in terms of the average excess unexplained variance across two sample sizes}. We also find that the aggregate original Cp method performs significantly better than the individual counterpart in this case. This observation implies the appropriateness of the proposed relative task dissimilarity metric. }

{The experimental results of Figure \ref{parkinsons1} are verified by the $p$-value of the pairwise Student's $t$-test shown in Table \ref{table11}. The CPU time requirements per run of each method are also listed in Table \ref{table11}. We find that the proposed approximate TLCp cutoff method has the least computational requirements among all the compared algorithms.}

The proposed TLCp methods perform better than the original Cp criterion when the relative task dissimilarity is relatively large, i.e., $2.02$ in Figure \ref{parkinsons2}. In this case, the original TLCp improves the original Cp and the approximate TLCp cutoff method by $6.85\%$ and $5.61\%$, respectively, in terms of the average excess unexplained variance across two sample sizes. This behavior demonstrates the capacity of the proposed TLCp method to leverage the related tasks. However, when the relative task dissimilarity grows greatly, i.e., $21.22$ in Figure \ref{parkinsons3}, the original TLCp method performs slightly worse than the original Cp. This occurs because the derived parameter tuning rules for the original TLCp procedure are sub-optimal in the non-orthogonal case. However, the approximate TLCp cutoff method performs as well as the original Cp method when the relative task dissimilarity is $21.22$. This occurs because the approximate TLCp cutoff method stops extracting knowledge from the source task when the relative task dissimilarity is large enough. From Figures \ref{parkinsons2} and \ref{parkinsons3}, we also find that the proposed TLCp methods significantly outperform the aggregate methods when the relative task dissimilarity grows significantly. This observation indicates that our TLCp methods are more robust and reliable than the aggregate methods over varying ranges of the relative task dissimilarity values. 
	%The aggregate methods may slightly outperform the original TLCp when the relative task dissimilarity is small. This occurs because the original TLCp method is almost equivalent to the aggregate Cp criterion when the relative task dissimilarity reaches zero, and the applied parameter tuning rules for the original TLCp method are ``sub-optimal'' in the non-orthogonal case.
% In general, our original TLCp method is more robust than the aggregate methods regarding the relative task dissimilarity values. On the other hand, the performance changes of these aggregate methods shown in Figures \ref{parkinsons1}, \ref{parkinsons2} and \ref{parkinsons3} partly demonstrate the rationality of the given relative task dissimilarity definition.}

\begin{figure}[htbp]
	\begin{minipage}[t] {\textwidth}
		\centerline{\includegraphics[width=1\textwidth]{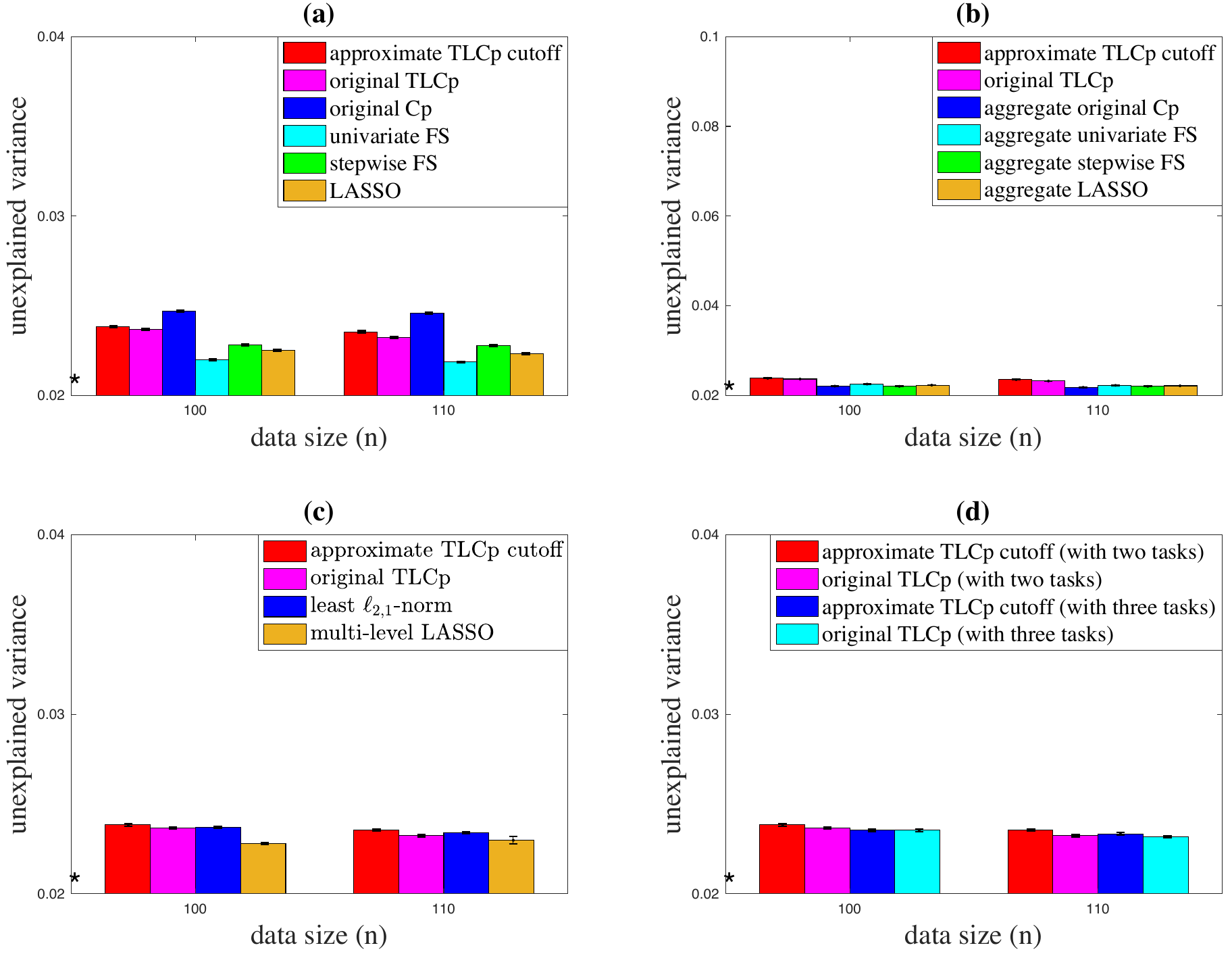}}
	\end{minipage}
	\caption{ \small The unexplained variance performance comparison o the proposed TLCp methods and other benchmarks when the relative task dissimilarity is $0.20$ for Parkinson's data. Smaller values indicate higher predictive accuracy. An aggregate method means running the corresponding non-aggregate method on the aggregate dataset formed by combining data for the target and source tasks.} \label{parkinsons1}
\end{figure}

\begin{table}[htbp]
	\centering
	\begin{tabular}{llll}
		\hline
		\multicolumn{1}{|l|}{} & \multicolumn{1}{l|}{original TLCp} & \multicolumn{1}{c|}{\begin{tabular}[c]{@{}c@{}}approximate TLCp\\ cutoff\end{tabular}} & \multicolumn{1}{l|}{\begin{tabular}[c]{@{}c@{}}CPU time\\ (s)\end{tabular}} \\ \hline
		\multicolumn{1}{|c|}{original Cp} & \multicolumn{1}{c|}{$\textbf{0.00}$} & \multicolumn{1}{c|}{$\textbf{0.00}$}                                                                  & \multicolumn{1}{c|}{$41.77$} \\ \hline
		\multicolumn{1}{|c|}{stepwise FS} & \multicolumn{1}{c|}{$\text{1.00}$ } & \multicolumn{1}{c|}{$\text{1.00}$}                                                                  & \multicolumn{1}{c|}{$0.12$} \\ \hline
		\multicolumn{1}{|c|}{univariate FS} & \multicolumn{1}{c|}{$\text{1.00}$} & \multicolumn{1}{c|}{$\text{1.00}$}                                                                  & \multicolumn{1}{c|}{$0.11$} \\ \hline
		\multicolumn{1}{|c|}{LASSO} & \multicolumn{1}{c|}{$\text{1.00}$} & \multicolumn{1}{c|}{$\text{1.00}$}                                                                  & \multicolumn{1}{c|}{$0.33$} \\ \hline
		\multicolumn{1}{|c|}{aggregate original Cp} & \multicolumn{1}{c|}{${1.00}$} & \multicolumn{1}{c|}{${1.00}$}                                                                  & \multicolumn{1}{c|}{$25.08$} \\ \hline
		\multicolumn{1}{|c|}{aggregate stepwise FS} & \multicolumn{1}{c|}{$\text{1.00}$} & \multicolumn{1}{c|}{$\text{1.00}$}                                                                  & \multicolumn{1}{c|}{$0.12$} \\ \hline
		\multicolumn{1}{|c|}{aggregate univariate FS} & \multicolumn{1}{c|}{$\text{1.00}$} & \multicolumn{1}{c|}{$\text{1.00}$}                                                                  & \multicolumn{1}{c|}{$0.13$} \\ \hline
		\multicolumn{1}{|c|}{aggregate LASSO} & \multicolumn{1}{c|}{$\text{1.00}$} & \multicolumn{1}{c|}{$\text{1.00}$}                                                                  & \multicolumn{1}{c|}{$0.27$} \\ \hline
		\multicolumn{1}{|c|}{least $\ell_{2,1}$-norm} & \multicolumn{1}{c|}{$\textbf{0.02}$} & \multicolumn{1}{c|}{$\text{0.94}$}                                                                  & \multicolumn{1}{c|}{$0.23$} \\ \hline
		\multicolumn{1}{|c|}{multi-level LASSO} & \multicolumn{1}{c|}{$\text{0.88}$} & \multicolumn{1}{c|}{$\text{1.00}$}                                                                  & \multicolumn{1}{c|}{$2.00$} \\ \hline
		\multicolumn{1}{|c|}{original TLCp} & \multicolumn{1}{c|}{$--$} & \multicolumn{1}{c|}{$\text{0.99}$}                                                                  & \multicolumn{1}{c|}{$102.40$} \\ \hline
		\multicolumn{1}{|c|}{approximate TLCp cutoff} & \multicolumn{1}{c|}{$\textbf{0.01}$} & \multicolumn{1}{c|}{$--$}                                                                  & \multicolumn{1}{c|}{$0.11$} \\ \hline       
		\multicolumn{1}{|c|}{\begin{tabular}[c]{@{}c@{}}original TLCp\\ with three tasks\end{tabular}} & \multicolumn{1}{c|}{$0.74$} & \multicolumn{1}{c|}{$\text{1.00}$}                                                                  & \multicolumn{1}{c|}{$64.99$} \\ \hline
		\multicolumn{1}{|c|}{\begin{tabular}[c]{@{}c@{}}approximate TLCp cutoff\\ with three tasks\end{tabular}} & \multicolumn{1}{c|}{$\text{0.11}$} & \multicolumn{1}{c|}{$0.99$}                                                                  & \multicolumn{1}{c|}{$0.15$} \\ \hline                 
	\end{tabular}
	\caption{The table shows the $p$-value of the pairwise $t$-test (in the first two columns) and the CPU time evaluation of different methods (in the last column) on Parkinson's data with the relative task dissimilarity $0.20$. Boldface means the proposed TLCp methods statistically outperform the compared methods ($p$-value $<0.05$). }\label{table11}
\end{table}

\begin{figure}[htbp]
	\begin{minipage}[t] {\textwidth}
		\centerline{\includegraphics[width=1\textwidth]{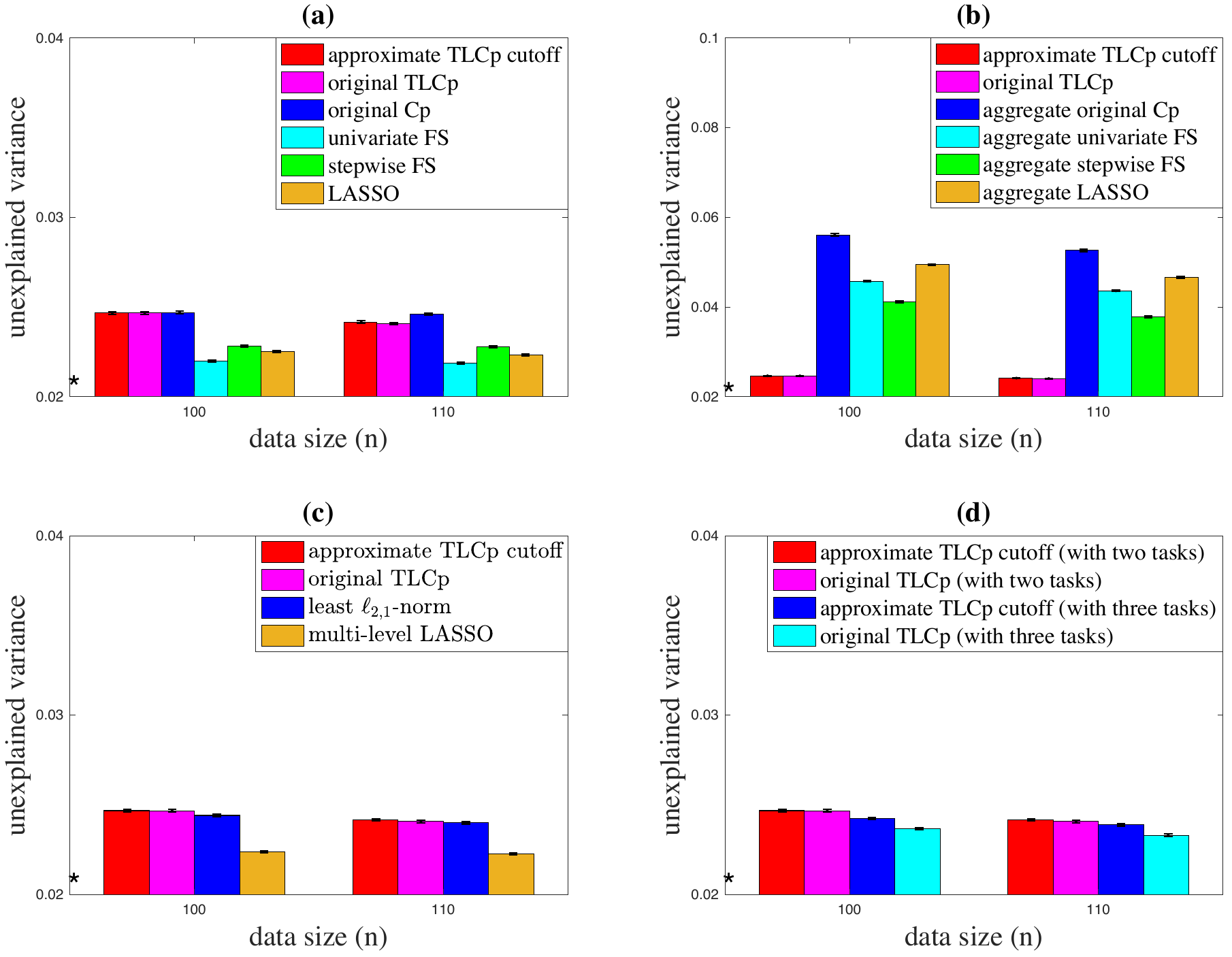}}
	\end{minipage}
	\caption{ \small The unexplained variance performance of the proposed TLCp methods and other benchmarks when the relative task dissimilarity is $2.02$ for Parkinson's data. Smaller values indicate higher predictive accuracy. An aggregate method means running the corresponding non-aggregate method on the aggregate dataset formed by combining data for the target and source tasks.} \label{parkinsons2}
\end{figure}
	
	\begin{figure}[htbp]
		\begin{minipage}[t] {\textwidth}
			\centerline{\includegraphics[width=1\textwidth]{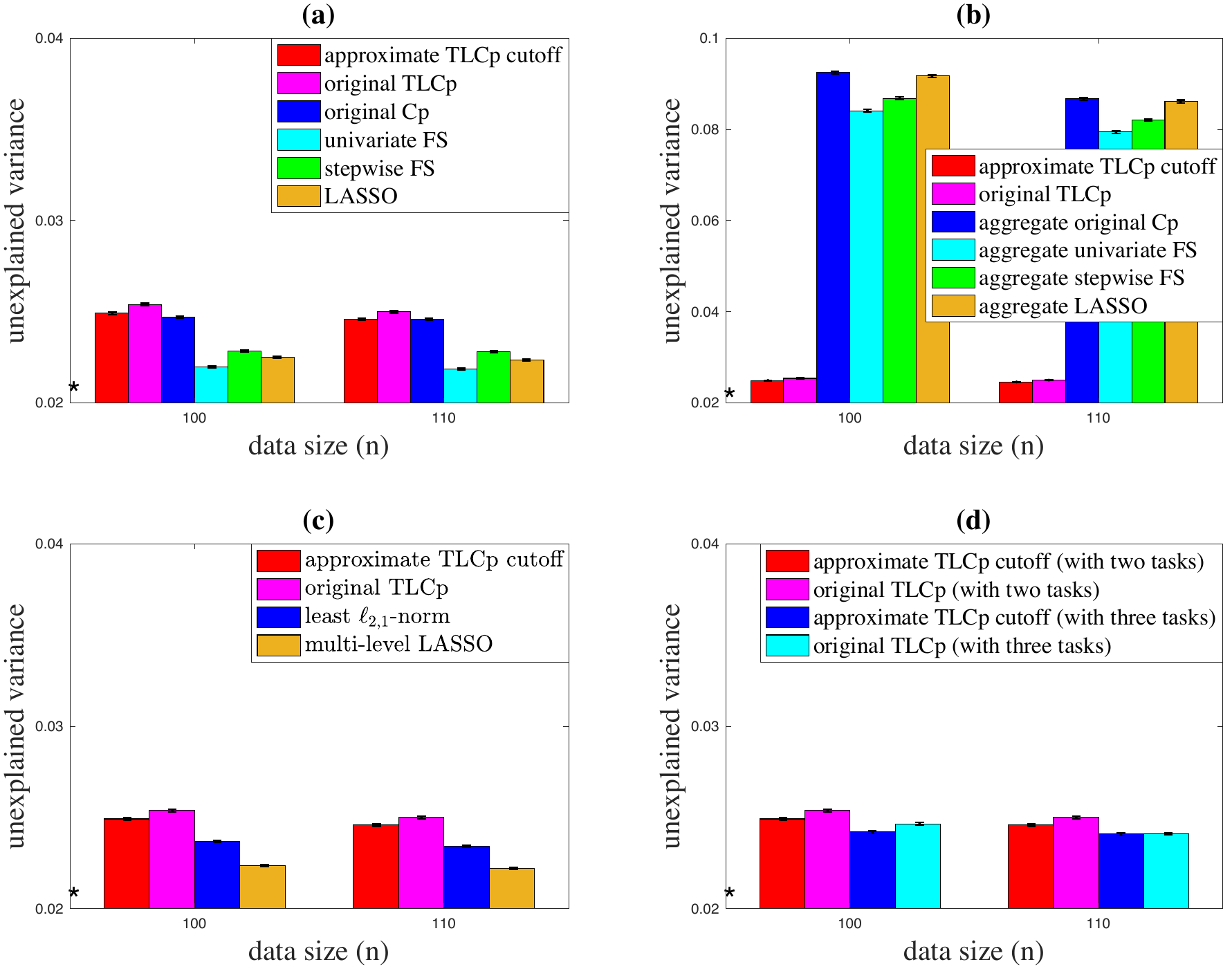}}
		\end{minipage}
		\caption{ \small The unexplained variance performance comparison of the proposed TLCp methods and other benchmarks when the relative task dissimilarity is $21.22$ for Parkinson's data. Smaller values indicate higher predictive accuracy. An aggregate method means running the corresponding non-aggregate method on the aggregate dataset formed by combining data for the target and source tasks.} \label{parkinsons3}
	\end{figure}

\begin{figure}[htbp]
	\begin{minipage}[t] {\textwidth}
		\centerline{\includegraphics[width=0.7\textwidth]{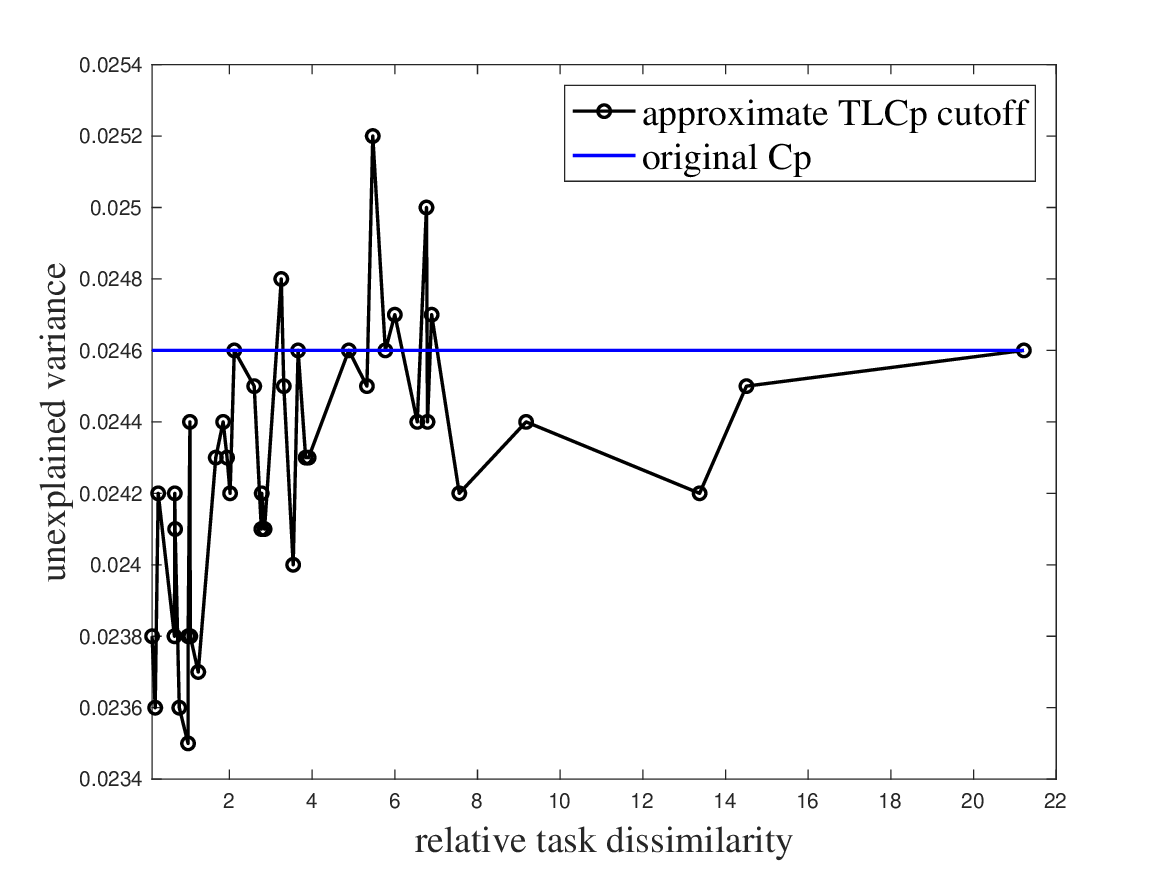}}
	\end{minipage}
	\caption{ \small This figure shows how the performance of the approximate TLCp cutoff method changes as the growth of the relative task dissimilarities on Parkinson's data. the horizontal line indicates the performance of the original Cp criterion.} \label{parkinson4}
\end{figure}

{Our TLCp methods perform similarly to the least $\ell_{2,1}$-norm method and the multi-level LASSO when the relative task dissimilarity is relatively small (see panel (c) of Figures \ref{parkinsons1} and \ref{parkinsons2}). As shown in panel (c) of Figure \ref{parkinsons3}, our methods perform slightly worse than the least $\ell_{2,1}$-norm method and the multi-level LASSO when the relative task dissimilarity grows significantly. This occurs because the applied parameter tuning rules for the TLCp methods are ``sub-optimal'' when the task dissimilarity is large. However, it is worth noting that our methods do not need to use cross-validation to tune the hyperparameters. In particular, the approximate TLCp cutoff method is approximately $20$ times faster than the multi-level LASSO (see Table \ref{table11}). Furthermore, the approximate TLCp cutoff method comes with a closed-form solution. These facts demonstrate the advantages of applying our method. } Finally, our experimental results for the TLCp methods with three tasks (where $33$-th task with the relative task dissimilarity $0.13$ is treated as the second source task) illustrate the potential benefits of incorporating increasingly related source tasks with the TLCp procedures (see panel (d) of Figures \ref{parkinsons1}, \ref{parkinsons2} and \ref{parkinsons3}).

{We also performed the paired Student's $t$-test to check the experimental results of Figures \ref{parkinsons2} and \ref{parkinsons3} in Tables \ref{table12} and \ref{table13}. For details, see Appendix E.}

{Figure \ref{parkinson4} compares the approximate TLCp cutoff method and the original Cp criterion as the relative task dissimilarities increase. For each relative task dissimilarity, we take $5000$ partitions of the target dataset with size $n=110$ to calculate the average unexplained variance of the model. We choose all $41$ available source tasks to investigate the performance of the approximate TLCp cutoff method. We first observe a clear tendency that, as the relative task dissimilarity grows, the approximate TLCp cutoff method's unexplained variance increases when the relative task dissimilarity is less than $7.00$. Then, the unexplained variance of the proposed method decreases when the relative task dissimilarity is larger than $7.00$. Finally, it converges to the performance of the original Cp when the relative task dissimilarity grows significantly, i.e., $21.22$. Second, the proposed approximate TLCp cutoff method clearly improves the original Cp criterion when the relative dissimilarity is less than $3.00$. This observation perfectly matches the simulation result shown in Subsection $6.2$. We can intuitively understand the above observations. First, the proposed TLCp method extracts useful knowledge from the source task when the relative task dissimilarity is small (i.e., less than $3.00$). As dissimilarity grows, the TLCp method distills less knowledge from the source task. Our TLCp method stops transferring knowledge from the source task if the relative task dissimilarity grows significantly.   }

The aforementioned experimental result demonstrates the proposed TLCp methods' superiority over the original Cp method when the relative task dissimilarity is relatively small. However, {this does not imply that the proposed TLCp methods are always better than the other feature selection methods in any problem}.  For example, as shown in panel (a) of Figure \ref{parkinsons1}, our TLCp methods may perform slightly worse than the other benchmarks that are not Cp-based when the relative task dissimilarity is small in Parkinson's data. {Our proposed dissimilarity metric helps distinguish cases when transfer learning techniques can be useful. }

{\subsection{Discussion}}
The following statements summarize the conclusions drawn from the experiments of this section: 
\begin{itemize}
	
\item 	
As the relative task dissimilarities grow, the approximate TLCp cutoff method's unexplained variance increases at first, then it falls, and finally converges to the original Cp criterion's performance. 

\item 
The proposed TLCp methods (including the original TLCp and the approximate TLCp cutoff methods) generally outperform the original Cp criterion when the relative task dissimilarity is small (i.e., less than $3.00$). 
	
\item 
The proposed TLCp methods perform similarly to the aggregate methods (running the benchmarks on the aggregate dataset formed by combining data for the target and source tasks) when the relative task dissimilarity is small. However, our TLCp methods show remarkable improvements over the aggregate methods when the relative task dissimilarity is large. 

\item 
Our TLCp methods perform as well as or better than the least $\ell_{2,1}$-norm method and the multi-level LASSO when the relative task dissimilarity is relatively small. Based on Theorem \ref{theorem18}, our methods do not need to use cross-validation to tune the hyperparameters. In particular, the approximate TLCp cutoff method comes with a closed-form solution and is significantly more efficient than the other two multi-task learning methods.

\item 
Our experimental results for the TLCp methods with three tasks illustrate the potential advantages of integrating increasingly related tasks into the proposed TLCp schemes.
\end{itemize}
%{The results from Figures \ref{parkinsons1}, \ref{parkinsons2} and \ref{parkinsons3} also demonstrate that the approximate TLCp procedure can almost achieve the same performance as the original TLCp method when the target sample size grows to $110$. This behavior also agrees with the asymptotic property of the approximate TLCp method. }

\section{Conclusions}

Our paper explores the effectiveness of the transfer learning technique in the context of Mallows' Cp criterion from both a theoretical and empirical perspective. Our results show that if the parameters (complexity penalties) of the orthogonal TLCp are well-chosen and the (relative) dissimilarity between the learning tasks of the source and target domains is small, then the proposed orthogonal TLCp estimator is superior to the orthogonal Cp estimator both in identifying important features and obtaining a lower MSE value. Moreover, when our learning framework is applied to exploit the orthogonal Cp criterion, it can be extended to BIC.
We also provide a feasible estimator to asymptotically approximate the non-orthogonal Cp solution by studying the orthogonalized Cp estimator, {similarly to the case of the non-orthogonal TLCp.} Finally, the proposed dissimilarity metric is remarkably useful in terms of identifying conditions under which transfer learning can succeed.

%In future work, we will continue to investigate the benefits of integrating transfer learning with other feature selection criteria, such as the Akaike's information criterion, the Hannan-Quinn information criterion, and the risk inflation criterion.
%Additionally, we will apply our analysis to investigate the performance of the non-orthogonal Cp criterion and its corresponding transfer learning algorithm.
%Also, the analysis scheme in this paper has the potential to be applied to investigate the advantages of transfer learning in the context of the non-orthogonal Cp criterion. We will leave this for future work.

\acks{This project was supported by the National Natural Science Foundation of China under Grant No. 12071428 and 62111530247, and the
Zhejiang Provincial Natural Science Foundation of China under Grant No. LZ20A010002.}


%\newpage

%\appendix
%\section*{Appendix A.}
%\label{app:theorem}

% Note: in this sample, the section number is hard-coded in. Following
% proper LaTeX conventions, it should properly be coded as a reference:

%In this appendix we prove the following theorem from
%Section~\ref{sec:textree-generalization}:

%In this appendix we prove the following theorem from
%Section~6.2:

%\noindent
%{\bf Theorem} {\it Let $u,v,w$ be discrete variables such that $v, w$ do
%not co-occur with $u$ (i.e., $u\neq0\;\Rightarrow \;v=w=0$ in a given
%dataset $\dataset$). Let $N_{v0},N_{w0}$ be the number of data points for
%which $v=0, w=0$ respectively, and let $I_{uv},I_{uw}$ be the
%respective empirical mutual information values based on the sample
%$\dataset$. Then
%\[
%	N_{v0} \;>\; N_{w0}\;\;\Rightarrow\;\;I_{uv} \;\leq\;I_{uw}
%\]
%with equality only if $u$ is identically 0.} \hfill\BlackBox

%\noindent
%{\bf Proof}. We use the notation:
%\[
%P_v(i) \;=\;\frac{N_v^i}{N},\;\;\;i \neq 0;\;\;\;
%P_{v0}\;\equiv\;P_v(0)\; = \;1 - \sum_{i\neq 0}P_v(i).
%\]
%These values represent the (empirical) probabilities of $v$
%taking value $i\neq 0$ and 0 respectively.  Entropies will be denoted
%by $H$. We aim to show that $\fracpartial{I_{uv}}{P_{v0}} < 0$....\\

%{\noindent \em Remainder omitted in this sample. See http://www.jmlr.org/papers/ for full paper.}

\vskip 0.2in
\bibliography{reference}

\begin{thebibliography}{61}
\providecommand{\natexlab}[1]{#1}
\providecommand{\url}[1]{\texttt{#1}}
\expandafter\ifx\csname urlstyle\endcsname\relax
  \providecommand{\doi}[1]{doi: #1}\else
  \providecommand{\doi}{doi: \begingroup \urlstyle{rm}\Url}\fi

\bibitem[Akaike(1974)]{akaike1974new}
Hirotugu Akaike.
\newblock A new look at the statistical model identification.
\newblock \emph{IEEE transactions on automatic control}, 19\penalty0
  (6):\penalty0 716--723, 1974.

\bibitem[Ando and Zhang(2005)]{ando2005framework}
Rie~Kubota Ando and Tong Zhang.
\newblock A framework for learning predictive structures from multiple tasks
  and unlabeled data.
\newblock \emph{Journal of Machine Learning Research}, 6\penalty0
  (Nov):\penalty0 1817--1853, 2005.

\bibitem[Argyriou et~al.(2007)Argyriou, Evgeniou, and
  Pontil]{argyriou2007multi}
Andreas Argyriou, Theodoros Evgeniou, and Massimiliano Pontil.
\newblock Multi-task feature learning.
\newblock In \emph{Advances in neural information processing systems}, pages
  41--48, 2007.

\bibitem[Argyriou et~al.(2008)Argyriou, Evgeniou, and
  Pontil]{argyriou2008convex}
Andreas Argyriou, Theodoros Evgeniou, and Massimiliano Pontil.
\newblock Convex multi-task feature learning.
\newblock \emph{Machine learning}, 73\penalty0 (3):\penalty0 243--272, 2008.

\bibitem[Bakker and Heskes(2003)]{bakker2003task}
Bart Bakker and Tom Heskes.
\newblock Task clustering and gating for bayesian multitask learning.
\newblock \emph{Journal of Machine Learning Research}, 4\penalty0
  (May):\penalty0 83--99, 2003.

\bibitem[Barreto et~al.(2017)Barreto, Dabney, Munos, Hunt, Schaul, van Hasselt,
  and Silver]{barreto2017successor}
Andr{\'e} Barreto, Will Dabney, R{\'e}mi Munos, Jonathan~J Hunt, Tom Schaul,
  Hado~P van Hasselt, and David Silver.
\newblock Successor features for transfer in reinforcement learning.
\newblock In \emph{Advances in neural information processing systems}, pages
  4055--4065, 2017.

\bibitem[Bartlett and Mendelson(2002)]{bartlett2002rademacher}
Peter~L Bartlett and Shahar Mendelson.
\newblock Rademacher and {G}aussian complexities: Risk bounds and structural
  results.
\newblock \emph{Journal of Machine Learning Research}, 3\penalty0
  (Nov):\penalty0 463--482, 2002.

\bibitem[Baxter(2000)]{baxter2000model}
Jonathan Baxter.
\newblock A model of inductive bias learning.
\newblock \emph{Journal of Artificial Intelligence Research}, 12:\penalty0
  149--198, 2000.

\bibitem[Ben-David and Schuller(2003)]{ben2003exploiting}
Shai Ben-David and Reba Schuller.
\newblock Exploiting task relatedness for multiple task learning.
\newblock In \emph{Learning Theory and Kernel Machines}, pages 567--580.
  Springer, 2003.

\bibitem[Borboudakis and Tsamardinos(2019)]{borboudakis2019forward}
Giorgos Borboudakis and Ioannis Tsamardinos.
\newblock Forward-backward selection with early dropping.
\newblock \emph{The Journal of Machine Learning Research}, 20\penalty0
  (1):\penalty0 276--314, 2019.

\bibitem[Buzzi-Ferraris and Manenti(2010)]{buzzi2010interpolation}
Guido Buzzi-Ferraris and Flavio Manenti.
\newblock \emph{Interpolation and regression models for the chemical engineer:
  Solving numerical problems}.
\newblock John Wiley \& Sons, 2010.

\bibitem[Candes et~al.(2007)Candes, Tao, et~al.]{candes2007dantzig}
Emmanuel Candes, Terence Tao, et~al.
\newblock The dantzig selector: Statistical estimation when p is much larger
  than n.
\newblock \emph{The annals of Statistics}, 35\penalty0 (6):\penalty0
  2313--2351, 2007.

\bibitem[Chen and Gao(2020)]{chen2019linear}
Shaohan Chen and Chuanhou Gao.
\newblock Linear priors minded and integrated for transparency of blast furnace
  black-box svm model.
\newblock \emph{IEEE Transactions on Industrial Informatics}, 16\penalty0
  (6):\penalty0 3862--3870, 2020.

\bibitem[Cozad et~al.(2014)Cozad, Sahinidis, and Miller]{cozad2014learning}
Alison Cozad, Nikolaos~V Sahinidis, and David~C Miller.
\newblock Learning surrogate models for simulation-based optimization.
\newblock \emph{AIChE Journal}, 60\penalty0 (6):\penalty0 2211--2227, 2014.

\bibitem[Cozad et~al.(2015)Cozad, Sahinidis, and Miller]{cozad2015combined}
Alison Cozad, Nikolaos~V Sahinidis, and David~C Miller.
\newblock A combined first-principles and data-driven approach to model
  building.
\newblock \emph{Computers \& Chemical Engineering}, 73:\penalty0 116--127,
  2015.

\bibitem[Dai et~al.(2007)Dai, Yang, Xue, and Yu]{dai2007boosting}
Wenyuan Dai, Qiang Yang, Gui-Rong Xue, and Yong Yu.
\newblock Boosting for transfer learning.
\newblock In \emph{Proceedings of the 24th international conference on Machine
  learning}, pages 193--200, 2007.

\bibitem[Draper and Smith(1998)]{draper1998applied}
Norman~R Draper and Harry Smith.
\newblock \emph{Applied regression analysis}, volume 326.
\newblock John Wiley \& Sons, 1998.

\bibitem[Dy and Brodley(2004)]{dy2004feature}
Jennifer~G Dy and Carla~E Brodley.
\newblock Feature selection for unsupervised learning.
\newblock \emph{Journal of machine learning research}, 5\penalty0
  (Aug):\penalty0 845--889, 2004.

\bibitem[Dziak et~al.(2020)Dziak, Coffman, Lanza, Li, and
  Jermiin]{dziak2020sensitivity}
John~J Dziak, Donna~L Coffman, Stephanie~T Lanza, Runze Li, and Lars~S Jermiin.
\newblock Sensitivity and specificity of information criteria.
\newblock \emph{Briefings in bioinformatics}, 21\penalty0 (2):\penalty0
  553--565, 2020.

\bibitem[Evgeniou and Pontil(2004)]{evgeniou2004regularized}
Theodoros Evgeniou and Massimiliano Pontil.
\newblock Regularized multi--task learning.
\newblock In \emph{Proceedings of the tenth ACM SIGKDD international conference
  on Knowledge discovery and data mining}, pages 109--117. ACM, 2004.

\bibitem[Foster and George(1994)]{foster1994risk}
Dean~P Foster and Edward~I George.
\newblock The risk inflation criterion for multiple regression.
\newblock \emph{The Annals of Statistics}, pages 1947--1975, 1994.

\bibitem[Friedman et~al.(2001)Friedman, Hastie, and
  Tibshirani]{friedman2001elements}
Jerome Friedman, Trevor Hastie, and Robert Tibshirani.
\newblock \emph{The elements of statistical learning}, volume~1.
\newblock Springer series in statistics New York, 2001.

\bibitem[Gao et~al.(2013)Gao, Ge, and Jian]{gao2013rule}
Chuanhou Gao, Qinghuan Ge, and Ling Jian.
\newblock Rule extraction from fuzzy-based blast furnace svm multiclassifier
  for decision-making.
\newblock \emph{IEEE Transactions on Fuzzy Systems}, 22\penalty0 (3):\penalty0
  586--596, 2013.

\bibitem[Glaeser and Scrimshaw(2013)]{linearalgebra}
Katrina Glaeser and Travis Scrimshaw.
\newblock \emph{Linear algebra}.
\newblock Davis California, 2013.

\bibitem[Guyon and Elisseeff(2003)]{guyon2003introduction}
Isabelle Guyon and Andr{\'e} Elisseeff.
\newblock An introduction to variable and feature selection.
\newblock \emph{Journal of machine learning research}, 3\penalty0
  (Mar):\penalty0 1157--1182, 2003.

\bibitem[Guyon et~al.(2002)Guyon, Weston, Barnhill, and Vapnik]{guyon2002gene}
Isabelle Guyon, Jason Weston, Stephen Barnhill, and Vladimir Vapnik.
\newblock Gene selection for cancer classification using support vector
  machines.
\newblock \emph{Machine learning}, 46\penalty0 (1-3):\penalty0 389--422, 2002.

\bibitem[Hannan and Quinn(1979)]{hannan1979determination}
Edward~J Hannan and Barry~G Quinn.
\newblock The determination of the order of an autoregression.
\newblock \emph{Journal of the Royal Statistical Society. Series B
  (Methodological)}, pages 190--195, 1979.

\bibitem[Helleputte and Dupont(2009)]{helleputte2009feature}
Thibault Helleputte and Pierre Dupont.
\newblock Feature selection by transfer learning with linear regularized
  models.
\newblock In \emph{Joint European Conference on Machine Learning and Knowledge
  Discovery in Databases}, pages 533--547. Springer, 2009.

\bibitem[Hill(1977)]{hill1977introduction}
Charles~G Hill.
\newblock \emph{An introduction to chemical engineering kinetics \& reactor
  design}.
\newblock John Wiley \& Sons, Hoboken, NJ. USA, 1977.

\bibitem[Hoo-Chang et~al.(2016)Hoo-Chang, Roth, Gao, Lu, Xu, Nogues, Yao,
  Mollura, and Summers]{hoo2016deep}
Shin Hoo-Chang, Holger~R Roth, Mingchen Gao, Le~Lu, Ziyue Xu, Isabella Nogues,
  Jianhua Yao, Daniel Mollura, and Ronald~M Summers.
\newblock Deep convolutional neural networks for computer-aided detection: Cnn
  architectures, dataset characteristics and transfer learning.
\newblock \emph{IEEE transactions on medical imaging}, 35\penalty0
  (5):\penalty0 1285, 2016.

\bibitem[Jebara(2004)]{jebara2004multi}
Tony Jebara.
\newblock Multi-task feature and kernel selection for {SVMs}.
\newblock In \emph{Proceedings of the twenty-first international conference on
  Machine learning}, page~55. ACM, 2004.

\bibitem[Kuzborskij and Orabona(2013)]{kuzborskij2013stability}
Ilja Kuzborskij and Francesco Orabona.
\newblock Stability and hypothesis transfer learning.
\newblock In \emph{International Conference on Machine Learning}, pages
  942--950, 2013.

\bibitem[Liu et~al.(2009)Liu, Ji, and Ye]{liu2009multi}
Jun Liu, Shuiwang Ji, and Jieping Ye.
\newblock Multi-task feature learning via efficient l 2, 1-norm minimization.
\newblock In \emph{Proceedings of the twenty-fifth conference on uncertainty in
  artificial intelligence}, pages 339--348. AUAI Press, 2009.

\bibitem[Lounici et~al.(2009)Lounici, Pontil, Tsybakov, and Van
  De~Geer]{lounici2009taking}
Karim Lounici, Massimiliano Pontil, Alexandre~B Tsybakov, and Sara Van De~Geer.
\newblock Taking advantage of sparsity in multi-task learning.
\newblock \emph{arXiv preprint arXiv:0903.1468}, 2009.

\bibitem[Lozano and Swirszcz(2012)]{lozano2012multi}
Aurelie~C Lozano and Grzegorz Swirszcz.
\newblock Multi-level lasso for sparse multi-task regression.
\newblock In \emph{Proceedings of the 29th International Coference on
  International Conference on Machine Learning}, pages 595--602, 2012.

\bibitem[Mallows(1973)]{mallows1973some}
Colin~L Mallows.
\newblock Some comments on {Cp}.
\newblock \emph{Technometrics}, 15\penalty0 (4):\penalty0 661--675, 1973.

\bibitem[Mathworks(2017)]{matwork2017statistics}
Mathworks.
\newblock Statistics and machine learning toolbox user's guide, 2017.

\bibitem[Maurer(2006)]{maurer2006bounds}
Andreas Maurer.
\newblock Bounds for linear multi-task learning.
\newblock \emph{Journal of Machine Learning Research}, 7\penalty0
  (Jan):\penalty0 117--139, 2006.

\bibitem[Maurer et~al.(2013)Maurer, Pontil, and
  Romera-Paredes]{maurer2013sparse}
Andreas Maurer, Massi Pontil, and Bernardino Romera-Paredes.
\newblock Sparse coding for multitask and transfer learning.
\newblock In \emph{International Conference on Machine Learning}, pages
  343--351, 2013.

\bibitem[Mihalkova et~al.(2007)Mihalkova, Huynh, and
  Mooney]{mihalkova2007mapping}
Lilyana Mihalkova, Tuyen Huynh, and Raymond~J Mooney.
\newblock Mapping and revising markov logic networks for transfer learning.
\newblock In \emph{AAAI}, volume~7, pages 608--614, 2007.

\bibitem[Miyashiro and Takano(2015)]{miyashiro2015subset}
Ryuhei Miyashiro and Yuichi Takano.
\newblock Subset selection by {Mallows' Cp}: A mixed integer programming
  approach.
\newblock \emph{Expert Systems with Applications}, 42:\penalty0 325--331, 2015.

\bibitem[Nishii(1984)]{nishii1984asymptotic}
Ryuei Nishii.
\newblock Asymptotic properties of criteria for selection of variables in
  multiple regression.
\newblock \emph{The Annals of Statistics}, pages 758--765, 1984.

\bibitem[Obozinski et~al.(2006)Obozinski, Taskar, and
  Jordan]{obozinski2006multi}
Guillaume Obozinski, Ben Taskar, and Michael Jordan.
\newblock Multi-task feature selection.
\newblock \emph{Statistics Department, UC Berkeley, Tech. Rep}, 2\penalty0
  (2.2), 2006.

\bibitem[Pan and Yang(2010)]{pan2010survey}
Sinno~Jialin Pan and Qiang Yang.
\newblock A survey on transfer learning.
\newblock \emph{IEEE Transactions on knowledge and data engineering},
  22\penalty0 (10):\penalty0 1345--1359, 2010.

\bibitem[Peng et~al.(2005)Peng, Long, and Ding]{peng2005feature}
Hanchuan Peng, Fuhui Long, and Chris Ding.
\newblock Feature selection based on mutual information criteria of
  max-dependency, max-relevance, and min-redundancy.
\newblock \emph{IEEE Transactions on pattern analysis and machine
  intelligence}, 27\penalty0 (8):\penalty0 1226--1238, 2005.

\bibitem[Pontil and Maurer(2013)]{pontil2013excess}
Massimiliano Pontil and Andreas Maurer.
\newblock Excess risk bounds for multitask learning with trace norm
  regularization.
\newblock In \emph{Conference on Learning Theory}, pages 55--76, 2013.

\bibitem[Schwarz et~al.(1978)]{schwarz1978estimating}
Gideon Schwarz et~al.
\newblock Estimating the dimension of a model.
\newblock \emph{The annals of statistics}, 6\penalty0 (2):\penalty0 461--464,
  1978.

\bibitem[Setiono and Liu(1997)]{setiono1997neural}
Rudy Setiono and Huan Liu.
\newblock Neural-network feature selector.
\newblock \emph{IEEE transactions on neural networks}, 8\penalty0 (3):\penalty0
  654--662, 1997.

\bibitem[Shao(1997)]{shao1997asymptotic}
Jun Shao.
\newblock An asymptotic theory for linear model selection.
\newblock \emph{Statistica sinica}, pages 221--242, 1997.

\bibitem[Steppe and Bauer~Jr(1997)]{steppe1997feature}
Jean~M Steppe and Kenneth~W Bauer~Jr.
\newblock Feature saliency measures.
\newblock \emph{Computers \& Mathematics with Applications}, 33\penalty0
  (8):\penalty0 109--126, 1997.

\bibitem[Sugiyama et~al.(2014)Sugiyama, Azencott, Grimm, Kawahara, and
  Borgwardt]{sugiyama2014multi}
Mahito Sugiyama, Chlo{\'e}-Agathe Azencott, Dominik Grimm, Yoshinobu Kawahara,
  and Karsten~M Borgwardt.
\newblock Multi-task feature selection on multiple networks via maximum flows.
\newblock In \emph{Proceedings of the 2014 SIAM International Conference on
  Data Mining}, pages 199--207. SIAM, 2014.

\bibitem[Tawarmalani and Sahinidis(2005)]{tawarmalani2005polyhedral}
Mohit Tawarmalani and Nikolaos~V Sahinidis.
\newblock A polyhedral branch-and-cut approach to global optimization.
\newblock \emph{Mathematical Programming}, 103\penalty0 (2):\penalty0 225--249,
  2005.

\bibitem[Tibshirani(1996)]{tibshirani1996regression}
Robert Tibshirani.
\newblock Regression shrinkage and selection via the lasso.
\newblock \emph{Journal of the Royal Statistical Society: Series B
  (Methodological)}, 58\penalty0 (1):\penalty0 267--288, 1996.

\bibitem[Tsanas et~al.(2009)Tsanas, Little, McSharry, and
  Ramig]{tsanas2009accurate}
Athanasios Tsanas, Max~A Little, Patrick~E McSharry, and Lorraine~O Ramig.
\newblock Accurate telemonitoring of parkinson's disease progression by
  noninvasive speech tests.
\newblock \emph{IEEE transactions on Biomedical Engineering}, 57\penalty0
  (4):\penalty0 884--893, 2009.

\bibitem[Wang et~al.(2019)Wang, Geng, Ma, Zhang, and Yang]{wang2019ridesharing}
Leye Wang, Xu~Geng, Xiaojuan Ma, Daqing Zhang, and Qiang Yang.
\newblock Ridesharing car detection by transfer learning.
\newblock \emph{Artificial Intelligence}, 2019.

\bibitem[Wang et~al.(2016{\natexlab{a}})Wang, Chang, Li, Sheng, and
  Chen]{wang2016multi}
Sen Wang, Xiaojun Chang, Xue Li, Quan~Z Sheng, and Weitong Chen.
\newblock Multi-task support vector machines for feature selection with shared
  knowledge discovery.
\newblock \emph{Signal Processing}, 120:\penalty0 746--753, 2016{\natexlab{a}}.

\bibitem[Wang et~al.(2016{\natexlab{b}})Wang, Dunson, and Leng]{wang2016no}
Xiangyu Wang, David Dunson, and Chenlei Leng.
\newblock No penalty no tears: Least squares in high-dimensional linear models.
\newblock In \emph{International Conference on Machine Learning}, pages
  1814--1822. PMLR, 2016{\natexlab{b}}.

\bibitem[Wilson and Sahinidis(2017)]{wilson2017alamo}
Zachary~T Wilson and Nikolaos~V Sahinidis.
\newblock The {ALAMO} approach to machine learning.
\newblock \emph{Computers \& Chemical Engineering}, 106:\penalty0 785--795,
  2017.

\bibitem[Yuan and Lin(2006)]{yuan2006model}
Ming Yuan and Yi~Lin.
\newblock Model selection and estimation in regression with grouped variables.
\newblock \emph{Journal of the Royal Statistical Society: Series B (Statistical
  Methodology)}, 68\penalty0 (1):\penalty0 49--67, 2006.

\bibitem[Zhang et~al.(2010)Zhang, Yeung, and Xu]{zhang2010probabilistic}
Yu~Zhang, Dit-Yan Yeung, and Qian Xu.
\newblock Probabilistic multi-task feature selection.
\newblock In \emph{Advances in neural information processing systems}, pages
  2559--2567, 2010.

\bibitem[Zhou et~al.(2011)Zhou, Chen, and Ye]{zhou2012mutal}
J.~Zhou, J.~Chen, and J.~Ye.
\newblock \emph{MALSAR: Multi-tAsk Learning via StructurAl Regularization}.
\newblock Arizona State University, 2011.
\newblock URL \url{http://www.public.asu.edu/~jye02/Software/MALSAR}.

\end{thebibliography}
% Alternatively
\appendix

{\section{The statistical angle to understand the Cp and TLCp criteria}\label{appendixA}}

{\subsection{Connections between the Cp criterion and statistical tests}}

{We first illustrate that using the orthogonal Cp criterion to selection features is equivalent to performing the statistical tests. }

{(1) we can rephrase Proposition \ref{theorem1} as, 
\begin{eqnarray}
\hat{a}_i=\left\{
\begin{matrix}
\beta_i+\frac{W_i^{\top}\boldsymbol { \varepsilon }}{n},
& \text{if~} r_{i}>\frac{\sqrt{\lambda}}{s_{\boldsymbol{y}}} ~\text{or}~ r_{i}<-\frac{\sqrt{\lambda}}{s_{\boldsymbol{y}}}
\\
0,
&  \text{otherwise}
\end{matrix}
\right.
\end{eqnarray}
where $r_i$ denotes the Pearson's correlation coefficient between the $i$-th feature $W_i$ and the response $\boldsymbol{y}$, and $s_{\boldsymbol{y}}:=\sqrt{\boldsymbol{y}^{\top}\boldsymbol{y}}$ , for $i=1,\cdots, k$.
Notice that the sample Pearson's correlation coefficient can be calculated as $r_i=\frac{W_i\boldsymbol{y}}{\sqrt{W_i^{\top}W_i}\sqrt{\boldsymbol{y}^{\top}\boldsymbol{y}}}$ by the orthogonality assumption and if $\boldsymbol{y}$ is standardized beforehand.}

{Therefore, under the orthogonality assumption, using the Cp criterion amounts to performing univariate feature selection with Pearson's correlation coefficient. Therefore, we can further show that using the orthogonal Cp criterion is equivalent to utilizing a specific significance level on $p$-value of the $z$-test for each feature.}

{Construct the $z$-statistic,  $z_i=\frac{r_is_{\boldsymbol{y}}}{\sigma_1}-\frac{\sqrt{n}\beta_i}{\sigma_1}$ for the $i$-th regression coefficient estimate. Under the null hypothesis that $\beta_i=0$, or equivalently the corresponding population Pearson's correlation coefficient equals zero, the $z$-statistic follows the standard normal distribution. Then, the significance level of this test (or the probability of falsely selecting a  superfluous feature) can be calculated as $\alpha_1(\lambda)=2\phi(-\frac{\sqrt{\lambda}}{\sigma_1})$, which is obtained by substituting one of the critical values $r_0=-\frac{\sqrt{\lambda}}{s_{\boldsymbol{y}}}$ of the sample Pearson's correlation coefficient $r_i$ into the $z$-statistic, where $\phi(u)$= $\int_{-\infty}^{u}\frac{1}{\sqrt{2\pi}}\exp{\left\{-\frac{x^2}{2}\right\}}dx$. Therefore, using the orthogonal Cp criterion is equivalent to performing the statistical $z$-test for each feature with the significance level $\alpha_1$. So Remark \ref{remark2} holds. Note that the significance level $\alpha_1$ is exactly the result of Theorem \ref{theorem 2} when $\beta_i=0$. }

{In particular, if fixing $\lambda=2\sigma_1^2$, then our analysis above indicates that using Mallows' Cp criterion amounts to performing the statistical $z$-test for each feature with the significance level $\alpha_1(2\sigma_1^2)\approx 0.16$.}

{(2) When the $i$-th true regression coefficient $\beta_i\neq0$, then the result in Theorem \ref{theorem 2} actually corresponds to the power of the hypothesis test (with the null hypothesis $H_0: \beta_i=0$, and the alternative hypothesis $H_1: \beta_i\neq0$) w.r.t. the z-statistic introduced above, for $i=1,\cdots,k$. }

{Note that the power refers to the probability of the hypothesis test to identify a relevant feature correctly. Based on the aforementioned analysis, $H_0$ will be rejected if $\frac{r_is_{\boldsymbol{y}}}{\sigma_1}>\frac{\sqrt{\lambda}}{\sigma_1}$ or $\frac{r_is_{\boldsymbol{y}}}{\sigma_1}<-\frac{\sqrt{\lambda}}{\sigma_1}$. Then, the power can be computed as 
\begin{eqnarray}
B(\beta_i)&=&P_r\left\{\frac{r_is_{\boldsymbol{y}}}{\sigma_1}>\frac{\sqrt{\lambda}}{\sigma_1} ~\text{or}~\frac{r_is_{\boldsymbol{y}}}{\sigma_1}<-\frac{\sqrt{\lambda}}{\sigma_1}~\bigg|~\beta_i\not=0 \right\} \notag\\
&=&P_r\left\{\frac{r_is_{\boldsymbol{y}}-\sqrt{n}\beta_i}{\sigma_1}>\frac{\sqrt{\lambda}-\sqrt{n}\beta_i}{\sigma_1} ~\text{or}~\frac{r_is_{\boldsymbol{y}}-\sqrt{n}\beta_i}{\sigma_1}<\frac{-\sqrt{\lambda}-\sqrt{n}\beta_i}{\sigma_1}~\bigg|~\beta_i\not=0 \right\}\notag\\
&=&1-P_r\left\{\frac{-\sqrt{\lambda}-\sqrt{n}\beta_i}{\sigma_1}\leq\frac{r_is_{\boldsymbol{y}}-\sqrt{n}\beta_i}{\sigma_1}\leq \frac{\sqrt{\lambda}-\sqrt{n}\beta_i}{\sigma_1}~\bigg|~\beta_i\not=0\right\}\notag\\
&=&1-\left[\phi\left(\frac{\sqrt{\lambda}-\sqrt{n}\beta_i}{\sigma_1}\right)-\phi\left(\frac{-\sqrt{\lambda}-\sqrt{n}\beta_i}{\sigma_1}\right)\right].
\end{eqnarray}
Thus, Remark \ref{remark4} is shown.}

{(3) The analysis above inspires us to directly restudy the orthogonal Cp criterion using the statistical tests. For the $i$-th feature ($i=1,\cdots,k$), the orthogonal Cp will select this feature if
\begin{equation}\label{equation5}
\sum_{j=1}^{n}y_j^2-\sum_{j=1}^{n}(y_j-W_i^{j}\hat{\beta_i})^2>\lambda,
\end{equation}
where $\hat{\beta}_i$ is the least squares estimate of the true regression coefficient $\beta_i$. Specifically, $\hat{\beta}_i=\beta_i+\frac{\sum_{j=1}^{n}\varepsilon_jW_i^{j}}{n}$ under the orthogonality assumption.
By expanding the left-hand side of Eq. (\ref{equation5}) and combining the similar terms together, we can further rewrite it as follows,
\begin{equation}\label{equation6}
\left(\beta_i\sqrt{n}+\frac{\sum_{j=1}^{n}\varepsilon_jW_i^{j}}{\sqrt{n}}\right)^2>\lambda.
\end{equation}
Then, we notice that the left-hand side of Eq. (\ref{equation6}) is a statistic that follows the scaled noncentral chi-squared distribution, whereas the right-hand side corresponds to the critical value. Remark \ref{remark29} below demonstrates the inherent equivalence between the orthogonal Cp criterion and the statistical tests, which is consistence with the corresponding results in Section $3$.}

{\begin{remark}\label{remark29}
For the $i$-th feature, using the orthogonal Cp criterion to determine whether to choose it or not is equivalent to performing a chi-squared test that corresponds to the statistic $\left(\beta_i\sqrt{n}+\frac{\sum_{j=1}^{n}\varepsilon_jW_i^{j}}{\sqrt{n}}\right)^2$ ($\sim \sigma_1^2\chi^2\left(1,\frac{\beta_i^2n}{\sigma_1^2}\right)$) for this feature, with the significance level $\alpha_2(\lambda)=1-F(\frac{\lambda}{\sigma_1^2};1)$ and the power $1-\gamma(\lambda)$, where $F(\frac{\lambda}{\sigma_1^2};1)$ is the cumulative distribution function of the chi-squared distribution with $1$ degree of freedom at the value $\frac{\lambda}{\sigma_1^2}$, and
	$\gamma(\lambda)=\phi\left(\frac{\sqrt{\lambda}}{\sigma_1}-\frac{\beta_i\sqrt{n}}{\sigma_1}\right)-\phi\left(-\frac{\sqrt{\lambda}}{\sigma_1}-\frac{\beta_i\sqrt{n}}{\sigma_1}\right)$. In particular, by fixing $\lambda=2\sigma_1^2$, we can conclude that the orthogonal Mallows' Cp criterion amounts to performing the hypothesis test w.r.t. the above statistic for the $i$-th feature with the significance level $\alpha_2(2\sigma_1^2)=1-F(2;1)$ ($\approx 0.16$) and the power $1-\gamma(2\sigma_1^2)$, where $\gamma(2\sigma_1^2)=\int_{-\sqrt{2}-\frac{\beta_i\sqrt{n}}{\sigma_1}}^{\sqrt{2}-\frac{\beta_i\sqrt{n}}{\sigma_1}}\frac{1}{\sqrt{2\pi}}\exp\left\{-\frac{z^2}{2}\right\}dz$. 
\end{remark}}

\begin{proof}	
{To prove this remark, we first notice that, when the null hypothesis $\beta_i=0$ is true, the statistic $Z=\left(\beta_i\sqrt{n}+\frac{\sum_{j=1}^{n}\varepsilon_jW_i^{j}}{\sqrt{n}}\right)^2$ follows a scaled chi-distribution with $1$ degree of freedom. That is, $\left(\frac{\sum_{j=1}^{n}\varepsilon_jW_i^{j}}{\sqrt{n}}\right)^2\sim \sigma_1^2\chi^2\left(1\right)$. For the $i$-th feature ($i=1,\cdots,k$), the orthogonal Cp will reject the null hypothesis that $\beta_i=0$, if
\begin{equation}\notag
\left(\frac{\sum_{j=1}^{n}\varepsilon_jW_i^{j}}{\sqrt{n}}\right)^2>\lambda,
\end{equation}	
or, equivalently,
\begin{equation}\notag
\left(\frac{\sum_{j=1}^{n}\varepsilon_jW_i^{j}}{\sigma_1\sqrt{n}}\right)^2>\frac{\lambda}{\sigma_1^2}.
\end{equation}	
Therefore, the significance level is $\alpha_2(\lambda)=1-F(\frac{\lambda}{\sigma_1^2};1)$, where $F(\frac{\lambda}{\sigma_1^2};1)$ is the cumulative distribution function of the chi-squared distribution with $1$ degree of freedom at the value $\frac{\lambda}{\sigma_1^2}$. }

{Second, when the alternative hypothesis $\beta_i\neq0$ is true, the statistic $Z$ follows a scaled noncentral chi-squared distribution, that is, $Z\sim \sigma_1^2\chi^2\left(1,\frac{\beta_i^2n}{\sigma_1^2}\right)$. In order to facilitate the computations, we first consider the probability density function of the scaled statistic $\tilde{Z}=\frac{Z}{\sigma_1^2}\sim\chi^2\left(1,\frac{\beta_i^2n}{\sigma_1^2}\right)$,
\begin{equation}\label{equation7}
f_{\tilde{Z}}(\tilde{z},1,\mu)=\frac{1}{2}\exp\left\{-\frac{\tilde{z}+\mu}{2}\right\}\left(\frac{\tilde{z}}{\mu}\right)^{-\frac{1}{4}}I_{-\frac{1}{2}}(\sqrt{\tilde{z}\mu}),
\end{equation}
where $\mu$ is the noncentral parameter denoted as $\mu=\frac{\beta_i^2n}{\sigma_1^2}$ and  $I_{v}(y)=\left(\frac{y}{2}\right)^{v}\sum_{j=0}^{\infty}\frac{(y^2/4)^j}{j!\Gamma(v+j+1)}$ is a modified Bessel function of the first kind. Further, we can rewrite Eq. (\ref{equation7}) (by noticing that $I_{-\frac{1}{2}}(y)=\sqrt{\pi}\left(\frac{y}{2}\right)^{-\frac{1}{2}}cosh(y^2)$) as
\begin{equation}\notag
f_{\tilde{Z}}(\tilde{z},1,\mu)=\frac{1}{2\sqrt{\tilde{z}}}[\phi(\sqrt{\tilde{z}}-\sqrt{\mu})+\phi(\sqrt{\tilde{z}}+\sqrt{\mu})],
\end{equation}
where $\phi(\cdot)$ is the standard normal density as defined previously.
Note that the null hypothesis $\beta_i=0$ is accepted for the generalized orthogonal Cp if $Z<\lambda$, or equivalently $\tilde{Z}<\frac{\lambda}{\sigma_1^2}$.
Then, the probability of the generalized orthogonal Cp to falsely delete a relevant feature can be computed as follows,
\begin{equation}\label{equation8}
\gamma(\lambda)=\int_{0}^{\frac{\lambda}{\sigma_1^2}}\frac{1}{2\sqrt{\tilde{z}}}[\phi(\sqrt{\tilde{z}}-\sqrt{\mu})+\phi(\sqrt{\tilde{z}}+\sqrt{\mu})]d\tilde{z}.
\end{equation}
We can rewrite Eq. (\ref{equation8}) as $\gamma=\gamma_1(\lambda)+\gamma_2(\lambda)$, where $\gamma_1(\lambda)$ and $\gamma_2$ are denoted as $\gamma_1(\lambda)=\int_{0}^{\frac{\lambda}{\sigma_1^2}}\frac{1}{2\sqrt{z}}\phi(\sqrt{z}-\sqrt{\mu})dz$, $\gamma_2(\lambda)=\int_{0}^{\frac{\lambda}{\sigma_1^2}}\frac{1}{2\sqrt{z}}\phi(\sqrt{z}+\sqrt{\mu})dz$.
Further, letting $z_1=\sqrt{z}-\sqrt{\mu}$, $z_2=-\sqrt{z}-\sqrt{\mu}$, we have
\begin{equation}\notag
\gamma_1(\lambda)=\int_{-\frac{\beta_i\sqrt{n}}{\sigma_1}}^{\frac{\sqrt{\lambda}}{\sigma_1}-\frac{\beta_i\sqrt{n}}{\sigma_1}}\frac{1}{\sqrt{2\pi}}\exp\left\{-\frac{z_1^2}{2}\right\}dz_1,
\end{equation}
and
\begin{equation}\notag
\gamma_2(\lambda)=\int_{-\frac{\sqrt{\lambda}}{\sigma_1}-\frac{\beta_i\sqrt{n}}{\sigma_1}}^{-\frac{\beta_i\sqrt{n}}{\sigma_1}}\frac{1}{\sqrt{2\pi}}\exp\left\{-\frac{z_2^2}{2}\right\}dz_2.
\end{equation}
Thus, there holds
\begin{eqnarray}\notag
\gamma(\lambda)&=&\int_{-\frac{\sqrt{\lambda}}{\sigma_1}-\frac{\beta_i\sqrt{n}}{\sigma_1}}^{\frac{\sqrt{\lambda}}{\sigma_1}-\frac{\beta_i\sqrt{n}}{\sigma_1}}\frac{1}{\sqrt{2\pi}}\exp\left\{-\frac{z^2}{2}\right\}dz\\
&=&\phi\left(\frac{\sqrt{\lambda}}{\sigma_1}-\frac{\beta_i\sqrt{n}}{\sigma_1}\right)-\phi\left(-\frac{\sqrt{\lambda}}{\sigma_1}-\frac{\beta_i\sqrt{n}}{\sigma_1}\right).
\end{eqnarray}
Finally, we can calculate the power as $1-\gamma(\lambda)$.	}	
\end{proof}

{We can also rephrase the orthogonal BIC criterion from the standpoint of statistical tests as Remark \ref{remark29} by replacing $\lambda$ with $\text{log}(n)$.}

{\subsection{The TLCp criterion and statistical tests}}

{We can understand the advantages of TLCp procedure over the Cp criterion from the standpoint of statistical tests. }

{As was illustrated in Remark \ref{remark19} in section $4.5$, when $\boldsymbol{\delta}=0$, using the orthogonal TLCp procedure (with its parameters optimal tuning based on the rules given in Corollary \ref{corollary 13}) amounts to implementing a chi-squared test for each feature with the significance level $0.16$ and the power $1-\tilde{\gamma}$, where $\tilde{\gamma}=\phi\left(\sqrt{2}-\frac{\sqrt{m\sigma_1^2+n\sigma_2^2}\beta_i}{\sigma_1\sigma_2}\right)-\phi\left(-\sqrt{2}-\frac{\sqrt{m\sigma_1^2+n\sigma_2^2}\beta_i}{\sigma_1\sigma_2}\right)$. Therefore, under the orthogonality assumption and if the task dissimilarity is sufficient small, the TLCp procedure is more able than the Cp criterion to detect the true associations in the hypothesis tests.}

\begin{proof}{\bf of Remark \ref{remark19}}
{First, for the $i$-th feature ($i=1,\cdots,k$), The orthogonal TLCp criterion will reject the null hypothesis ($\beta_i=0$) if
\begin{equation}\notag
\sum _ { i = 1 } ^ { n }\lambda_1y_i^2+\sum _ { i = 1 } ^ { n }\lambda_2\tilde{y}_i^2-\sum _ { j = 1 } ^ { n }\lambda_1(y_j-\boldsymbol{w}_i^{j}X_i^{j})^2 - \sum _ { j = 1 } ^ { m } \lambda_2(\tilde{y_{j}}-\boldsymbol{\tilde{w}}_i^{j}\tilde{X_{i}}^{j})^2 - \frac { 1 } { 2 } \sum _ { t = 1 } ^ { 2 } \lambda_3^{i}v_t^2-\lambda_4>0.
\end{equation}
Following similar techniques as in the proof of Proposition \ref{TLCp solution}, this relationship can be simplified to 
\begin{equation}\label{equation10}
A_iH_i^2+B_iZ_i^2+C_iJ_i^2>\lambda_4,
\end{equation}
where $A_i=\frac{4\lambda_1\lambda_2^2m^2n}{4\lambda_1\lambda_2mn+m\lambda_2\lambda_3+n\lambda_1\lambda_3^i}$, $B_i=\frac{4\lambda_2\lambda_1^2mn^2}{4\lambda_1\lambda_2mn+m\lambda_2\lambda_3^i+n\lambda_1\lambda_3^i}$ and $C_i=\frac{\lambda_3^{i}}{4\lambda_1\lambda_2mn+m\lambda_2\lambda_3+n\lambda_1\lambda_3^i}$. Besides, $H_i=\beta_i+\delta_i+\frac{1}{m}\tilde{W}_{i}^{\top} \boldsymbol{ \eta }$, $Z_i=\beta_i+\frac{1}{n}W_{i}^{\top} { \boldsymbol{\varepsilon} }$ and $J_i=m\lambda_2H_i+n\lambda_1Z_i$ are three random variables.}

{Then, when the null hypothesis $\beta_i=0$ is true, assuming $\delta_i=0$, and setting the tuning parameters of the orthogonal TLCp as $\lambda_1^*=\sigma_2^2, \lambda_2^*=\sigma_1^2, {\lambda_3^i}^*=+\infty$ and $\lambda_4^*=2\sigma_1^2\sigma_2^2$, there holds
\begin{equation}
A_iH_i^2+B_iZ_i^2+C_iJ_i^2\sim \sigma_1^2\sigma_2^2\chi^2(1),
\end{equation}
which means $\frac{A_iH_i^2+B_iZ_i^2+C_iJ_i^2}{\sigma_1^2\sigma_2^2} \sim \chi^2(1)$. Meanwhile, we have $\frac{A_iH_i^2+B_iZ_i^2+C_iJ_i^2}{\sigma_1^2\sigma_2^2}>2$ by Eq. ({\ref{equation10}}).
Therefore, the significance level is $\alpha_3=1-F(2;1)$ ($\approx 0.16$), where $F(2;1)$ is the cumulative distribution function of the chi-squared distribution with $1$ degree of freedom at the value $2$.}

{Second, when the alternative hypothesis $\beta_i\neq0$ is true, $\delta_i=0$, and we tune the parameters of the orthogonal TLCp as $\lambda_1^*=\sigma_2^2, \lambda_2^*=\sigma_1^2, {\lambda_3^i}^*=+\infty$ and $\lambda_4^*=2\sigma_1^2\sigma_2^2$, then there holds
\begin{equation}\label{equation12}
A_iH_i^2+B_iZ_i^2+C_iJ_i^2\sim \sigma_1^2\sigma_2^2\chi^2\left(1,\frac{(m\sigma_1^2+n\sigma_2^2)\beta_i^2}{\sigma_1^2\sigma_2^2}\right).
\end{equation}
Further, we notice that the null hypothesis $\beta_i=0$ is accepted for the orthogonal TLCp if $\frac{A_iH_i^2+B_iZ_i^2+C_iJ_i^2}{\sigma_1^2\sigma_2^2}<2$, Then, the probability of the orthogonal TLCp to falsely delete this relevant feature can be calculated as follows,
\begin{equation}
\tilde{\gamma}=\phi\left(\sqrt{2}-\frac{\sqrt{m\sigma_1^2+n\sigma_2^2}\beta_i}{\sigma_1\sigma_2}\right)-\phi\left(-\sqrt{2}-\frac{\sqrt{m\sigma_1^2+n\sigma_2^2}\beta_i}{\sigma_1\sigma_2}\right),
\end{equation}
by applying the same approach as used in the proof of Remark \ref{remark29}.
Therefore, the power is $1-\tilde{\gamma}$.	}
\end{proof}

{It seems very complicated to directly analyze the orthogonal TLCp procedure by statistical tests when $\boldsymbol{\delta}\neq0$ (i.e., as was done in Remark \ref{remark19}). This is due to the difficulty of estimating the distribution of the statistic $A_iH_i^2+B_iZ_i^2+C_iJ_i^2$ if $\boldsymbol{\delta}\neq0$ (for $i=1,\cdots,k$). Note that $H_i=\beta_i+\delta_i+\frac{1}{m}\tilde{W}_{i}^{\top} \boldsymbol{ \eta }\sim \mathcal { N }\left(\beta_i+\delta_i,\frac{\sigma_2^2}{m}\right)$, $Z_i=\beta_i+\frac{1}{n}W_{i}^{\top} { \boldsymbol{\varepsilon} }\sim \mathcal { N }\left(\beta_i,\frac{\sigma_1^2}{n}\right)$ and $J_i=m\lambda_2H_i+n\lambda_1Z_i\sim\mathcal { N }\left(m\lambda_2(\beta_i+\delta_i)+n\lambda_1\beta_i,m\lambda_2^2\sigma_2^2+n\lambda_1^2\sigma_1^2\right)$, and $J_i$ depends on $H_i$ and $Z_i$. However, we can still understand Eq. (\ref{11}) in Theorem \ref{TLCp probability} as the significance level for the orthogonal TLCp in the case of $\boldsymbol{\delta}\neq0$, i.e., let $\beta_i=0$ in Eq. (\ref{11}), for $i=1,\cdots,k$. }

{\section{Analysis of the general TLCp approach}\label{appendixB}}

{\subsection{The orthogonal TLCp estimator in general cases}}

{First, the general TLCp problem with $\ell$ tasks can be stated as follows,
\begin{equation}\label{general TLCp problem}
\text{min}_{\boldsymbol{v}_1,\cdots,\boldsymbol{v}_{\ell},\boldsymbol{\alpha}_0}\sum _ { i = 1 } ^ { n }\lambda_1(y_1^{i}-\boldsymbol{\alpha}_1^{\top}X_1^{i})^2 + \sum_{h=2}^{\ell}\lambda_h\sum _ { j= 1 } ^ { m_h }({y_h^{j}}-\boldsymbol{\alpha}_h^{\top}{X_{h}^{j}})^2 + \frac { 1 } { 2 } \sum _ { t = 1 } ^ {\ell}  \boldsymbol { v }_t^{\top}\boldsymbol{\gamma}\boldsymbol{v}_t +\lambda_{\ell+1}\bar{p},
\end{equation}
where we assume that the database in the target domain consists of $n$ samples $(X_1^i;y_1^{i})_{i=1}^{n}$ satisfying the true but unknown relationship: $\boldsymbol { y }_1 = \boldsymbol { X }_1 \boldsymbol { \beta } + \boldsymbol { \varepsilon }$,  $\varepsilon_i \sim \mathcal { N } \left( 0 , \sigma_1 ^ { 2 } \right)$ for $i=1,\cdots, n$. Further, there are $\ell-1$ source domain datasets each of which has $m_h$ samples $(X_h^i;y_h^{i})_{i=1}^{m_h}$ and satisfy the following true but unknown correlation functions: $\boldsymbol { y }_h= \boldsymbol { X }_h (\boldsymbol { \beta }+\boldsymbol{\delta_h}) + \boldsymbol { \eta }_h$,  $\eta_h^{i} \sim \mathcal { N } \left( 0 , \sigma_h ^ { 2 } \right)$ for $i=1,\cdots, m$, $h=2,\cdots,\ell$. Furthermore, we suppose the regression coefficient vectors for the $\ell$ regression models with the forms $\boldsymbol{\alpha}_1=\boldsymbol{\alpha}_0+\boldsymbol{v}_1$, $\cdots$,
$\boldsymbol{\alpha}_{\ell}=\boldsymbol{\alpha}_0+\boldsymbol{v}_{\ell}$. For each task, $\boldsymbol{\gamma}:=\text{diag}(\boldsymbol{\gamma}^1,\cdots,\boldsymbol{\gamma}^k)$ is a parameter matrix each element of which reflects the significance of the individual part of a regression coefficient for each feature. Similar to the structure of the Cp criterion, we introduce the non-negative integer $\bar{p}$ in (\ref{general TLCp problem}) to control the number of regressors to be selected among all the tasks. More illustrations of these parameters can be found in Subsection $4.1$.}%For the extreme case that each attribute of the parameter matrix $\boldsymbol{\gamma}$ is $\infty$ and $\lambda_1=\cdots=\lambda_{\ell}$, then the proposed general TLCp paradigm is equivalent to the ``aggregate Cp criterion'', that is,  the original Cp criterion is trained on the whole dataset formed by combining data for all $\ell$ tasks.}

{Second, we hope to indicate that the optimal solution of the orthogonal general TLCp problem for the target task in (\ref{general TLCp problem}) owning the form of
\begin{eqnarray}\notag
\hat{\boldsymbol{\alpha}}_1^{i}=\left\{
\begin{matrix}
{\beta_{i}}+R_1^i(\frac{\boldsymbol{\varepsilon}^{\top}W_{1}^{i}}{ n} )+R_2^{i}\left(\delta_1^i+\frac{\boldsymbol{ \eta}_2^{\top}{W}_{2}^{i}}{m_2}\right)+\cdots +R_{\ell}^i\left(\delta_{\ell}^i+\frac{\boldsymbol{\eta}_{\ell}^{\top}W_\ell^i}{m_{\ell}}\right) & F(Z_1^i,\cdots,Z_\ell^i)<-\lambda_{\ell+1}\\
0
&  \text{otherwise}
\end{matrix}
\right.
\end{eqnarray}
where each weight $R_j^i$ ($j=1,\cdots,\ell$) is determined by the model parameters $\lambda_1,\cdots,\lambda_{\ell+1},\boldsymbol{\gamma}$ satisfying $R_1^i+R_2^i+\cdots+R_{\ell}^i=1$, for $i=1,\cdots,k$. Here, $F(Z_1^i,\cdots,Z_\ell^i)$ is a quadratic form with respect to the random variables $Z_1^i=\beta_i+\frac{\boldsymbol{\varepsilon}^{\top}W_{1}^{i}}{ n}, Z_2^i=\beta_i+\delta_1^i+\frac{\boldsymbol{ \eta}_2^{\top}{W}_{2}^{i}}{m_2},\cdots, Z_\ell^i=\beta_i+ \delta_{\ell}^i+\frac{\boldsymbol{\eta}_{\ell}^{\top}W_\ell^i}{m_{\ell}}$, for $i=1,\cdots,k$. For the orthogonal general TLCp problem (\ref{general TLCp problem}),
the condition to determine whether the $i$-th regressor will be picked, $\{F(Z_1^i,\cdots,Z_\ell^i)+\lambda_{\ell+1}<0\}$, is equivalent to whether selecting this regressor will make the objective value of (\ref{general TLCp problem}) smaller than the value when it is not selected.}

{As a special case, we will show below the explicit expression of the solution of (\ref{general TLCp problem}) when $\ell=3$. Moreover, we will test this approach empirically using simulated as well as real data.  }

{\begin{proposition}
	The estimated regression coefficients for the target task in the general TLCp model when two source tasks are considered has the expression below, if the conditions $\boldsymbol{X}_1^{\top}\boldsymbol{X}_1=nI$, $\boldsymbol{{X}}_2^{\top}\boldsymbol{{X}}_2=mI$ and $\boldsymbol{X}_3^{\top}\boldsymbol{X}_3=qI$ hold.
	\begin{eqnarray}\notag
	\hat{\boldsymbol{\alpha}}_1^{i}=\left\{
	\begin{matrix}
	{ {\beta_{i}}+S_1^{i}(\frac{1}{n}{ \boldsymbol{\varepsilon}^{\top}W_{1}^{i}} })+S_2^{i}\left(\delta_1^i+\frac{1}{m}\boldsymbol{ \eta}^{\top}{W}_{2}^{i}\right)+S_3^i\left(\delta_2^i+\frac{1}{q}\boldsymbol{\zeta}^{\top}W_3^i\right) &  F(Z_1^i,H_1^i,H_2^i)<-\lambda_5\\
	0
	&  \text{otherwise}
	\end{matrix}
	\right.
	\end{eqnarray}
	for $i=1, \cdots, k$, where $S_1^i=1-\frac{\lambda_4^i}{2n\lambda_1}(K_1-K_2L)$, $S_2^i=\frac{\lambda_4^i}{2n\lambda_1}(K_1-K_2L-K_2L_2)$, $S_3^i=\frac{\lambda_4^i}{2n\lambda_1}K_2L_2$, thus, $S_1^i+S_2^i+S_3^i=1$. Additionally, $F(Z_1^i,H_1^i,H_2^i):=R_1^i(Z_1^i)^2+R_2^i(H_1^i)^2+R_3^i(H_2^i)^2+2P_1^iZ_1^iH_1^i+2P_2^iZ_1^iH_2^i+2P_3^iH_1^iH_2^i$, where $Z_1^i=\beta_i+\frac{1}{n}{ \boldsymbol{\varepsilon}^{\top}W_{1}^{i}}$, $H_1^i=\beta_i+\delta_1^i+\frac{1}{m}\boldsymbol{ \eta}^{\top}{W}_{2}^{i}$, $H_2^i=\beta_i+\delta_2^i+\frac{1}{q}\boldsymbol{\zeta}^{\top}W_3^i$ are three random variables. Also, $R_1^i=C_1M_2^2+C_2L^2+C_3Q_1^2$, $R_2^i=C_1M_2^2+C_2L_1^2+C_3Q_2^2$, $R_3^i=C_1M_3^2+C_2L_2^2+C_3Q_3^2$; $P_1^i=C_1M_1M_2-C_2LL_1+C_3Q_1Q_2$, $P_2^i=-C_1M_1M_3+C_2LL_2-C_3Q_1Q_3$, and $P_3^i=-C_1M_2M_3-C_2L_1L_2-C_3Q_2Q_3$, where $M_1=K_1-K_2L$, $M_2=K_2L-K_2L_2-K_1$, $M_3=K_2L_2$; $Q_1=K_1-K_2L-L$, $Q_2=K_2L_1-K_1+L_1$, $Q_3=K_2L_2+L_2$.
	Among them, $K_1=\frac{2mn\lambda_1\lambda_2}{2mn\lambda_1\lambda_2+m\lambda_2\lambda_4^i}$, $K_2=\frac{2mn\lambda_1\lambda_2+n\lambda_1\lambda_4^i}{2mn\lambda_1\lambda_2+m\lambda_2\lambda_4^i}$, $L_1=\frac{1}{J^i}(8mnq\lambda_1\lambda_2\lambda_3+2mq\lambda_2\lambda_3\lambda_4^i+2mn\lambda_1\lambda_2\lambda_4^i)$, $L_2=\frac{1}{J^i}(4mnq\lambda_1\lambda_2\lambda_3+2mq\lambda_2\lambda_3\lambda_4^i)$, $L=\frac{1}{J^i}(4mnq\lambda_1\lambda_2\lambda_3+2mn\lambda_1\lambda_2\lambda_4^i)$, where $J^i=12mnq\lambda_1\lambda_2\lambda_3+4mq\lambda_2\lambda_3\lambda_4^i+4mn\lambda_1\lambda_2\lambda_4^i+4qn\lambda_1\lambda_3\lambda_4^i+(q\lambda_3+n\lambda_1+m\lambda_2)(\lambda_4^i)^2$. $C_1=\frac{(\lambda_4^i)^2+n\lambda_1\lambda_4^i}{2n\lambda_1}$, $C_2=\frac{(\lambda_4^i)^2+m\lambda_2\lambda_4^i}{2m\lambda_2}$, $C_3=\frac{(\lambda_4^i)^2+q\lambda_3\lambda_4^i}{2q\lambda_3}$.
\end{proposition}}

\begin{proof}
	{We denote by $\boldsymbol{\tilde{1}}_i$ the indicator function of whether the $i$-th feature is selected by the general orthogonal TLCp model with two source tasks (\ref{general TLCp problem}) or not. Specifically,
	\begin{eqnarray}\notag
	\boldsymbol{\tilde{1}}_i=
	\left\{
	\begin{matrix}
	0
	& \text{if~} \|\boldsymbol{\alpha}_1^i\|_0=\|\boldsymbol{\alpha}_2^i\|_0=\|\boldsymbol{\alpha}_3^i\|_0=0\\
	1
	&  \text{ortherwise}
	\end{matrix}
	\right.
	\end{eqnarray}
	Then, the general orthogonal TLCp model in (\ref{general TLCp problem}) is equivalent to minimizing the following objective function,
	\begin{equation*}
	\sum_{i=1}^{k}\left\{f_i(\lambda_1,W_1^i,\boldsymbol{\alpha}_1^i)+g_i(\lambda_2,{W}_2^i,\boldsymbol{\alpha}_2^i)+\tilde{g}_i(\lambda_3,{W}_3^i,\boldsymbol{\alpha}_3^i)+h_i(\lambda_4^i,\boldsymbol{v}_1^i,\boldsymbol{v}_2^i,\boldsymbol{v}_3^i,\lambda_5,\tilde{\boldsymbol{1}}_i)\right\}
	\end{equation*}
	where
	\begin{equation*}
	f_i(\lambda_1,W_1^i,\boldsymbol{\alpha}_1^i)=n\lambda_1\left(-2\beta_i\boldsymbol{\alpha}_1^i- \frac{2}{n}\boldsymbol{\varepsilon}^{\top}\boldsymbol{\alpha}_1^iW_{1}^{i}+(\boldsymbol{\alpha}_1^i)^2\right),
	\end{equation*}
	\begin{equation*}
	g_i(\lambda_2,{W}_2^i,\boldsymbol{\alpha}_2^i)=m\lambda_2\left(-2\beta_i\boldsymbol{\alpha}_2^i- \frac{2}{m}\boldsymbol{\eta}^{\top}\boldsymbol{\alpha}_2^iW_{2}^{i}+(\boldsymbol{\alpha}_2^i)^2\right),
	\end{equation*}
	\begin{equation*}
	\tilde{g}_i(\lambda_3,{W}_3^i,\boldsymbol{\alpha_3}^i)=q\lambda_3\left(-2\beta_i\boldsymbol{\alpha_3}^i- \frac{2}{q}\boldsymbol{\zeta}^{\top}\boldsymbol{\alpha_3}^iW_{3}^{i}+(\boldsymbol{\alpha_3}^i)^2\right),
	\end{equation*}
	and
	\begin{equation*}
	h_i(\lambda_4^i,\boldsymbol{v}_1^i,\boldsymbol{v}_2^i,\boldsymbol{v}_3^i,\lambda_5,\tilde{\boldsymbol{1}}_i)=\frac{1}{2}\lambda_4^i\left[(\boldsymbol{v}_1^i)^2+(\boldsymbol{v}_2^i)^2+(\boldsymbol{v}_3^i)^2\right]+\lambda_5\boldsymbol{\tilde{1}}_i.
	\end{equation*}
	Due to the independence of each summand in the objective function above, the general orthogonal TLCp problem (\ref{general TLCp problem}) further amounts to $k$ one-dimensional optimization problems below,
	\begin{equation}\label{objective function}
	\text{min}_{\boldsymbol{v}_1^i,\boldsymbol{v}_2^i,\boldsymbol{v}_3^i,\boldsymbol{\alpha}_0^i}~\left\{f_i({\lambda}_1,W_1^i,\boldsymbol{\alpha}_1^i)+g_i({\lambda}_2,{W}_2^i,\boldsymbol{\alpha}_2^i)+\tilde{g}_i(\lambda_3,{W}_3^i,\boldsymbol{\alpha}_3^i)+h_i(\lambda_4^i,\boldsymbol{v}_1^i,\boldsymbol{v}_2^i,\boldsymbol{v}_3^i,\lambda_5,\tilde{\boldsymbol{1}}_i)\right\}
	\end{equation}
	for $i=1,\cdots,k$, in the sense that they have the same solution.\\
	For the $i$-th problem above, if $\boldsymbol{\tilde{1}}_i=1$, and making the gradient of the corresponding objective function equal to zero, that is, the estimated $i$-th regression coefficients $\alpha_1^i$, $\alpha_2^i$ and $\alpha_3^i$ for the target and source tasks satisfying the following equations,
	\begin{eqnarray}\label{linear euqations for TLCp}
	\left\{
	\begin{array}{l}
	2n\lambda_1\boldsymbol{\alpha}_0^i+(2n\lambda_1+\lambda_4^i)\boldsymbol{ v }_1^i=2\lambda_1(n\beta_i+\boldsymbol{\varepsilon}^{\top}W_{1}^{i}),\\
	2m\lambda_2\boldsymbol{\alpha}_0^i+(2m\lambda_2+\lambda_4^i)\boldsymbol{ v }_2^i=2\lambda_2(m(\beta_i+\delta_1^i)+\boldsymbol{\eta}^{\top}W_{2}^{i}),\\
	2q\lambda_3\boldsymbol{\alpha}_0^i-(2q\lambda_3+\lambda_4^i)\boldsymbol{v}_1^i-(2q\lambda_3+\lambda_4^i)\boldsymbol{v}_2^i=2\lambda_3(q(\beta_i+\delta_2^i)+\boldsymbol{\zeta}^{\top}W_3^i).
	\end{array}
	\right.
	\end{eqnarray}
	By solving these linear equations, we have
	\begin{equation}\label{first solution}
	\hat{\boldsymbol{\alpha}}_1^{i}=
	{ {\beta_{i}}+S_1^{i}(\frac{1}{n}{ \boldsymbol{\varepsilon}^{\top}W_{1}^{i}} })+S_2^{i}\left(\delta_1^i+\frac{1}{m}\boldsymbol{ \eta}^{\top}{W}_{2}^{i}\right)+S_3^i\left(\delta_2^i+\frac{1}{q}\boldsymbol{\zeta}^{\top}W_3^i\right),
	\end{equation}
	for $i=1, \cdots, k$, where $S_1^i=1-\frac{\lambda_4^i}{2n\lambda_1}(K_1-K_2L)$, $S_2^i=\frac{\lambda_4^i}{2n\lambda_1}(K_1-K_2L-K_2L_2)$, $S_3^i=\frac{\lambda_4^i}{2n\lambda_1}K_2L_2$, thus, $S_1^i+S_2^i+S_3^i=1$, which means $\hat{\boldsymbol{\alpha}}_1^{i}$ is a convex combination of the three random variables $\frac{1}{n}{ \boldsymbol{\varepsilon}^{\top}W_{1}^{i}}$, $\delta_1^i+\frac{1}{m}\boldsymbol{ \eta}^{\top}{W}_{2}^{i}$ and $\delta_2^i+\frac{1}{q}\boldsymbol{\zeta}^{\top}W_3^i$. Among them,  $K_1=\frac{2mn\lambda_1\lambda_2}{2mn\lambda_1\lambda_2+m\lambda_2\lambda_4^i}$, $K_2=\frac{2mn\lambda_1\lambda_2+n\lambda_1\lambda_4^i}{2mn\lambda_1\lambda_2+m\lambda_2\lambda_4^i}$, $L_1=\frac{1}{J^i}(8mnq\lambda_1\lambda_2\lambda_3+2mq\lambda_2\lambda_3\lambda_4^i+2mn\lambda_1\lambda_2\lambda_4^i)$, $L_2=\frac{1}{J^i}(4mnq\lambda_1\lambda_2\lambda_3+2mq\lambda_2\lambda_3\lambda_4^i)$, $L=\frac{1}{J^i}(4mnq\lambda_1\lambda_2\lambda_3+2mn\lambda_1\lambda_2\lambda_4^i)$, where $J^i=12mnq\lambda_1\lambda_2\lambda_3+4mq\lambda_2\lambda_3\lambda_4^i+4mn\lambda_1\lambda_2\lambda_4^i+4qn\lambda_1\lambda_3\lambda_4^i+(q\lambda_3+n\lambda_1+m\lambda_2)(\lambda_4^i)^2$. \\
	Similarly, we have
	\begin{equation}\label{second solution}
	\hat{\boldsymbol{\alpha}}_2^{i}=\beta_i+\frac{\lambda_4^iL}{2m\lambda_2}\left(\frac{1}{n}{ \boldsymbol{\varepsilon}^{\top}W_{1}^{i}}\right)+\left(1-\frac{\lambda_4^iL_1}{2m\lambda_2}\right)\left(\delta_1^i+\frac{1}{m}\boldsymbol{ \eta}^{\top}{W}_{2}^{i}\right)+\frac{\lambda_4^iL_2}{2m\lambda_2}\left(\delta_2^i+\frac{1}{q}\boldsymbol{\zeta}^{\top}W_3^i\right),
	\end{equation}
	%where $L_1=L+L_2$. %Thus, $\hat{\boldsymbol{\alpha}}_2^{i}$ is also a convex combination of the three terms $\frac{1}{n}{ \boldsymbol{\varepsilon}^{\top}W_{1}^{i}}$, $\delta_1^i+\frac{1}{m}\boldsymbol{ \eta}^{\top}{W}_{2}^{i}$ and $\delta_2^i+\frac{1}{q}\boldsymbol{\zeta}^{\top}W_3^i$.
	\begin{equation}\label{third solution}
	\hat{\boldsymbol{\alpha}}_3^{i}=\beta_i+\frac{\lambda_4^iQ_1}{2q\lambda_3}\left(\frac{1}{n}{ \boldsymbol{\varepsilon}^{\top}W_{1}^{i}}\right)+\frac{\lambda_4^iQ_2}{2q\lambda_3}\left(\delta_1^i+\frac{1}{m}\boldsymbol{ \eta}^{\top}{W}_{2}^{i}\right)+\left(1-\frac{\lambda_4^iQ_3}{2q\lambda_3}\right)\left(\delta_2^i+\frac{1}{q}\boldsymbol{\zeta}^{\top}W_3^i\right),
	\end{equation}
	where $L_1=L+L_2$, $Q_1=K_1-K_2L-L$, $Q_2=K_2L_1-K_1+L_I$ and $Q_3=K_2L_2+L_2$. Thus, $\hat{\boldsymbol{\alpha}}_2^{i}$ and $\hat{\boldsymbol{\alpha}}_3^{i}$ are also two convex combinations of the three random variables $\frac{1}{n}{ \boldsymbol{\varepsilon}^{\top}W_{1}^{i}}$, $\delta_1^i+\frac{1}{m}\boldsymbol{ \eta}^{\top}{W}_{2}^{i}$ and $\delta_2^i+\frac{1}{q}\boldsymbol{\zeta}^{\top}W_3^i$.\\
	The estimators for the $i$-th individual parameters are
	\begin{equation}\label{forth solution}
	\hat{\boldsymbol{v}}_1^{i}=M_1\left(\frac{1}{n}{ \boldsymbol{\varepsilon}^{\top}W_{1}^{i}}\right)+M_2\left(\delta_1^i+\frac{1}{m}\boldsymbol{ \eta}^{\top}{W}_{2}^{i}\right)-M_3\left(\delta_2^i+\frac{1}{q}\boldsymbol{\zeta}^{\top}W_3^i\right),
	\end{equation}
	where $M_1=K_1-K_2L$, $M_2=K_2L-K_2L_2-K_1$, $M_3=K_2L_2$.
	\begin{equation}\label{fifth solution}
	\hat{\boldsymbol{v}}_2^{i}=-L\left(\frac{1}{n}{ \boldsymbol{\varepsilon}^{\top}W_{1}^{i}}\right)+L_1\left(\delta_1^i+\frac{1}{m}\boldsymbol{ \eta}^{\top}{W}_{2}^{i}\right)-L_2\left(\delta_2^i+\frac{1}{q}\boldsymbol{\zeta}^{\top}W_3^i\right),
	\end{equation}
	and
	\begin{equation}\label{sixth solution}
	\hat{\boldsymbol{v}}_3^{i}=-\hat{\boldsymbol{v}}_2^{i}-\hat{\boldsymbol{v}}_1^{i}.
	\end{equation}
	By substituting the relations (\ref{first solution}), (\ref{second solution}), (\ref{third solution}), (\ref{forth solution}), (\ref{fifth solution}), (\ref{sixth solution}) into the objective function in (\ref{objective function}), we have
	\begin{align}
	&f_i(\lambda_1,W_1^i,\boldsymbol{\alpha}_1^i)+g_i(\lambda_2,{W}_2^i,\boldsymbol{\alpha}_2^i)+\tilde{g}_i(\boldsymbol{\alpha}_3,{W}_3^i,\alpha_3^i)+h_i(\lambda_4^i,\boldsymbol{v}_1^i,\boldsymbol{v}_2^i,\boldsymbol{v}_3^i,\lambda_5,\tilde{\boldsymbol{1}}_i)\notag\\
	&=R_1^i(Z_1^i)^2+R_2^i(H_1^i)^2+R_3^i(H_2^i)^2+2P_1^iZ_1^iH_1^i+2P_2^iZ_1^iH_2^i+2P_3^iH_1^iH_2^i+\lambda_5,\notag
	\end{align}
	where $Z_1^i=\beta_i+\frac{1}{n}{ \boldsymbol{\varepsilon}^{\top}W_{1}^{i}}$, $H_1^i=\beta_i+\delta_1^i+\frac{1}{m}\boldsymbol{ \eta}^{\top}{W}_{2}^{i}$, $H_2^i=\beta_i+\delta_2^i+\frac{1}{q}\boldsymbol{\zeta}^{\top}W_3^i$ are three random variables. Also, $R_1^i=C_1M_2^2+C_2L^2+C_3Q_1^2$, $R_2^i=C_1M_2^2+C_2L_1^2+C_3Q_2^2$, $R_3^i=C_1M_3^2+C_2L_2^2+C_3Q_3^2$; $P_1^i=C_1M_1M_2-C_2LL_1+C_3Q_1Q_2$, $P_2^i=-C_1M_1M_3+C_2LL_2-C_3Q_1Q_3$, and $P_3^i=-C_1M_2M_3-C_2L_1L_2-C_3Q_2Q_3$.\\
	Assuming $\boldsymbol{\tilde{1}}_i=0$ in the $i$-th optimization problem (\ref{objective function}), which means the estimators for the parameters $\boldsymbol{\alpha}_0^i,\boldsymbol{v}_1^i,\boldsymbol{v}_2^i,\boldsymbol{v}_3^i$ satisfying $\tilde{\boldsymbol{\alpha}}_0^i=\tilde{\boldsymbol{v}}_1^i=\tilde{\boldsymbol{v}}_2^i=\tilde{\boldsymbol{v}}_3^i=0$, there holds
	\begin{equation*}
	f_i(\lambda_1,W_1^i,\tilde{\boldsymbol{\alpha}}_1^i)+g_i(\lambda_2,{W}_2^i,\tilde{\boldsymbol{\alpha}}_2^i)+\tilde{g}_i(\lambda_3,{W}_3^i,\tilde{\boldsymbol{\alpha}}_3^i)+h_i(\lambda_4^i,\tilde{\boldsymbol{v}}_1^i,\tilde{\boldsymbol{v}}_2^i,\tilde{\boldsymbol{v}}_3^i,\lambda_5,0)=0.
	\end{equation*}
	Therefore, for the $i$-th optimization problem (\ref{objective function}), the corresponding regressor will be selected if the random variable  $F(Z_1^i,H_1^i,H_2^i)+\lambda_5:=R_1^i(Z_1^i)^2+R_2^i(H_1^i)^2+R_3^i(H_2^i)^2+2P_1^iZ_1^iH_1^i+2P_2^iZ_1^iH_2^i+2P_3^iH_1^iH_2^i+\lambda_5$ is smaller than $0$.
	That is,
	\begin{eqnarray}\notag
	\hat{\boldsymbol{\alpha}}_1^{i}=\left\{
	\begin{matrix}
	{ {\beta_{i}}+S_1^{i}(\frac{1}{n}{ \boldsymbol{\varepsilon}^{\top}W_{1}^{i}} })+S_2^{i}\left(\delta_1^i+\frac{1}{m}\boldsymbol{ \eta}^{\top}{W}_{2}^{i}\right)+S_3^i\left(\delta_2^i+\frac{1}{q}\boldsymbol{\zeta}^{\top}W_3^i\right) &  F(Z_1^i,H_1^i,H_2^i)<-\lambda_5\\
	0
	&  \text{otherwise}
	\end{matrix}
	\right.
	\end{eqnarray}
	Finally, we can acquire the desired optimal solution for the general orthogonal TLCp model with two source tasks by integrating these $k$ solutions together.}
\end{proof}

{\begin{remark}
	For the orthogonal TLCp in the general case, the estimated regression coefficient for the $i$-th relevant feature equals the convex combination among the random variables $\beta_i+\frac{\boldsymbol{\varepsilon}^{\top}W_{1}^{i}}{ n}, \beta_i+\delta_1^i+\frac{\boldsymbol{ \eta}_2^{\top}{W}_{2}^{i}}{m_2},\cdots, \beta_i+ \delta_{\ell}^i+\frac{\boldsymbol{\eta}_{\ell}^{\top}W_\ell^i}{m_{\ell}}$, for $i=1,\cdots,k$. This can be verified by solving linear equations similar to (\ref{linear euqations for TLCp}). However, we will omit this part in our manuscript due to the tedious nature of the derivations. For the practitioners who are interested in the specific solution of the orthogonal general TLCp problem (\ref{general TLCp problem}), we recommend using symbolic computation toolboxes.
\end{remark}}

{\subsection{Simulation studies of the general TLCp approach}} % nikos

{For illustrating how the proposed general TLCp method can be used in practice, we will first test our method with simulated datasets. We assume the target training data are i.i.d. sampled from $\boldsymbol { y } = \boldsymbol { X } \boldsymbol { \beta } + \boldsymbol { \varepsilon }$, where $\boldsymbol { \beta }=[1 ~0.01~ 0.005~ 0.3~ 0.32~ 0.08]^{\top}$, the fourth and fifth elements of which are (or are near) the critical points $\pm\sqrt{2\sigma_1^2/n}$ when $n=20$ and $\sigma_1=1$. Additionally, we generate two source datasets as $\boldsymbol { \tilde{y}_1 } = \boldsymbol { X } (\boldsymbol { \beta }+\boldsymbol{\delta}_1) + \boldsymbol { \eta}_2$ and $\boldsymbol { \tilde{y}_2 } = \boldsymbol { X } (\boldsymbol { \beta }+\boldsymbol{\delta}_2) + \boldsymbol { \eta }_3$ where $\sigma_2=\sigma_3=1$. Here, $\boldsymbol { X }$ (which satisfies $\boldsymbol { X }^{\top}\boldsymbol { X }=nI$), $\boldsymbol{\delta}_1$ and $\boldsymbol{\delta}_2$ are obtained according to the method stated in Subsection $6.1$ and $n=m_2=m_3=20$. For each task similarity measure $1/\|\boldsymbol{\delta}_i\|_2$ ($i=1,2$), we randomly simulated $5000$ datasets and applied the general TLCp approach. Here, we choose the tuning parameters of the general TLCp model with $3$ tasks in (\ref{general TLCp problem}) as $\lambda_1=\sigma_2^2\sigma_3^2$, $\lambda_2=\sigma_1^2\sigma_3^2$, $\lambda_3=\sigma_1^2\sigma_2^2$, $\boldsymbol{\gamma}^{i}=12\sigma_1^2\sigma_2^2\sigma_3^2/(\delta_1^i+\delta_2^i)^2(i=1,\cdots,k)$ and $\lambda_4=\text{min}_{i\in\{1,\cdots,k\}}\left\{\frac{\lambda\left(2-\frac{\tilde{Q}_i}{\sqrt{\tilde{M}_i\tilde{N}_i\tilde{W}_i}}\right)}{4\sigma_1^2(\tilde{G}_i)^2}\right\}$, where $\lambda=2\sigma_1^2$, $\tilde{Q}_i=\frac{-2\sigma_1^2\sigma_2^2\sigma_3^2}{(\delta_1^i+\delta_2^i)^2+\frac{\sigma_1^2}{n}+\frac{\sigma_2^2}{m_2}+\frac{\sigma_3^2}{m_3}}$, $\tilde{M}_i=\frac{\sigma_1^2m_2m_3[(\delta_1^i+\delta_2^i)^2+\frac{\sigma_1^2}{n}]}{(\delta_1^i+\delta_2^i)^2+\frac{\sigma_1^2}{n}+\frac{\sigma_2^2}{m_2}+\frac{\sigma_3^2}{m_3}}$, $\tilde{N}_i=\frac{\sigma_2^2nm_3[(\delta_1^i+\delta_2^i)^2+\frac{\sigma_2^2}{m_2}]}{(\delta_1^i+\delta_2^i)^2+\frac{\sigma_1^2}{n}+\frac{\sigma_2^2}{m_2}+\frac{\sigma_3^2}{m_3}}$, $\tilde{W}_i=\frac{\sigma_3^2nm_2[(\delta_1^i+\delta_2^i)^2+\frac{\sigma_3^2}{m_3}]}{(\delta_1^i+\delta_2^i)^2+\frac{\sigma_1^2}{n}+\frac{\sigma_2^2}{m_2}+\frac{\sigma_3^2}{m_3}}$, and $\tilde{G}_i=\sqrt{\frac{nm_2m_3}{n\tilde{M}_i\sigma_2^2\sigma_3^2+m_2\tilde{N}_i\sigma_1^2\sigma_3^2+m_3\tilde{W}_i\sigma_1^2\sigma_2^2}}$ for $i=1,\cdots,k$, which are the natural extensions of the parameter selection in the two-task case. }

{As demonstrated in Figure \ref{mse1}, with the increase of the task similarity measure (which equals $1/\|\delta_1^i+\delta_2^i\|_2$ if two source tasks are considered), the proposed general TLCp model with $3$ tasks performs significantly better than the TLCp model with $2$ tasks and also dramatically better than the Cp criterion in the sense of MSE. This result demonstrates the effectiveness of the proposed general TLCp method (\ref{general TLCp problem}) together with the corresponding parameter tuning rule. This behavior also motivates us to further explore the advantages of the general TLCp framework with more tasks. Generally, for the general TLCp problem with $\ell$ tasks, we can guess the corresponding tuning parameters as follows: $\lambda_j=\sigma_1^2*\cdots*\sigma_{j-1}^2*\sigma_{j+1}^2*\cdots*\sigma_{\ell}^2$ ($j=1,\cdots,\ell$), $\boldsymbol{\gamma}^i=2*\ell!/(\delta_1^i+\cdots+\delta_{\ell-1}^i)^2(i=1,\cdots,k)$ and for the regularization parameter $\lambda_{\ell+1}$, we can calculate it as $\lambda_{\ell+1}\approx2\sigma_1^2*\cdots*\sigma_{\ell}^2$, if the dissimilarities among tasks are very small. The optimal setting of the regularization parameter in the two-task case is approximately $2\sigma_1^2\sigma_2^2$ when the task dissimilarity is small enough. Thus, the formula for the general $\ell$ tasks case is a natural extension of that for two cases. The theoretical verification of the optimality of the parameter tuning rule introduced above will be left for future work.}

\begin{figure}[htbp]
 \begin{minipage}{\textwidth}
	\centering{\includegraphics[height=8cm,width=12.0cm]{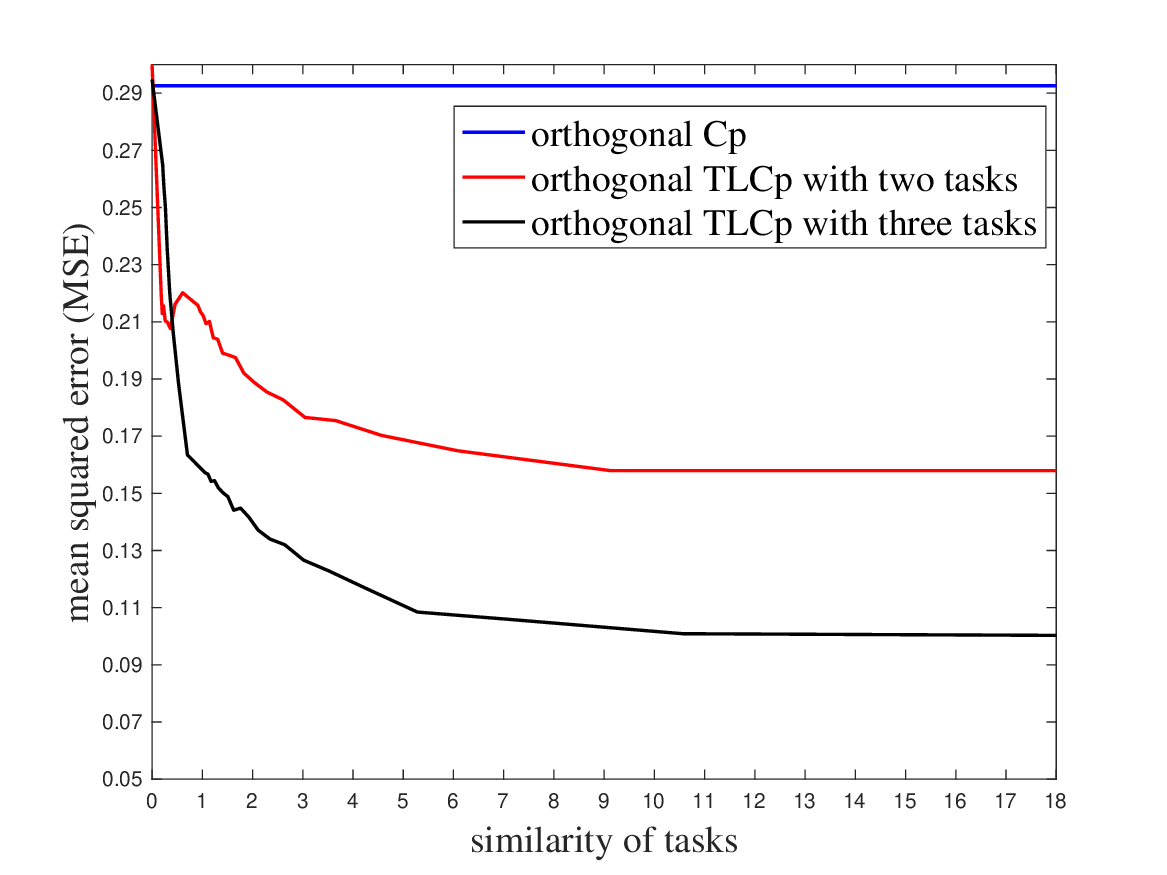}} \\
		\caption{ \small MSE performance comparison between the orthogonal Cp criterion and the orthogonal TLCp method with two and three tasks. Note that, for each dimension ($i$-th), we define the similarity of tasks as $1/\|\delta^i\|_2$ if only one source task is considered, and $1/\|\delta_1^i+\delta_2^i\|_2$ if two source tasks are considered.} \label{mse1}
\end{minipage}	
\end{figure}

{We also evaluated the proposed general TLCp framework with $3$ tasks on two real datasets (school data and Parkinson's data), see Section $7$.} 

\section{Proofs of main results}\label{appendixC}

In this appendix, we provide the proofs of all theoretical results in our article.\\

\begin{proof}{\bf of Proposition \ref{theorem1}}
\\
To obtain the optimal solution for the orthogonal Cp model (\ref{the modified Cp}),
notice that, because $\boldsymbol{X}^{\top}\boldsymbol{X}=nI$ and $\boldsymbol { y } = \boldsymbol { X } \boldsymbol { \beta } + \boldsymbol { \varepsilon }$, the expansion of the objective function in (\ref{the modified Cp}) can be written as
\begin{equation*}
\boldsymbol { y }^{\top}\boldsymbol { y }-2n\boldsymbol{\beta}^{\top}\boldsymbol{ a }-2\boldsymbol{\varepsilon}^{\top} \boldsymbol { X } \boldsymbol { a }+n\boldsymbol { a }^ { \top }\boldsymbol{a}+\lambda\|\boldsymbol{a}\|_0.
\end{equation*}
In this case, the orthogonal Cp criterion (\ref{the modified Cp}) is equivalent to
\begin{equation*}
\min_{\boldsymbol{a}}~\boldsymbol { y }^{\top}\boldsymbol { y }+\sum_{i=1}^{k}\left\{-2n\beta_{i}a_{i}-2\boldsymbol{\varepsilon}^{\top}W_{i}a_{i}+na_{i}^2+\lambda\|a_i\|_0\right\},
\end{equation*}
where $W_i$ indicates the $i$-th column of the design matrix $\boldsymbol{X}$, for $i=1,\cdots, k$.

Due to the independence of each summation term in the objective function above, the orthogonal Cp problem  (\ref{the modified Cp}) is further equivalent to the following $k$ one-dimensional optimization problems
\begin{equation}\label{g}
\min_{a_i}~ g(a_i)=-2n\beta_{i}a_{i}-2\boldsymbol{\varepsilon}^{\top}W_{i}a_{i}+na_{i}^2+\lambda\|a_i\|_0
\end{equation}
for $i=1,\cdots,k$,  in the sense that they have the same solution.

Now, we turn our attention to the solution of the one-dimensional optimization problems. For the $i$-th problem, if $\|a_i\|_0=1$, then we can easily get the estimator $a_i^{o}=\beta_i+\frac{\boldsymbol{\varepsilon}^{\top}W_i}{n}$ by requiring the derivative of $g(a_i)$ be zero, and $g(a_i^{o})=-n\beta_i^2-2\beta_i\boldsymbol{\varepsilon}^{\top}W_i-\frac{\boldsymbol{\varepsilon}^{\top}W_{i}W_{i}^{\top}\boldsymbol{\varepsilon}}{n}+\lambda$. If $\|a_i\|_0=0$, then $a_i=0$ and $g(0)=0$. By comparing the objective values in these two cases and picking the smaller one, we can get the optimal solution of the $i$-th problem as follows
\begin{eqnarray}\notag
\hat{a}_i=\left\{
\begin{matrix}
\beta_i+\frac{W_i^{\top}\boldsymbol { \varepsilon }}{n},
& \text{if~} n\left(\beta_i+\frac{W_i^{\top}\boldsymbol { \varepsilon }}{n}\right)^2>\lambda\\
0,
&  \text{otherwise}
\end{matrix}
\right.
\end{eqnarray}
where $i=1,\cdots,k$.
Finally, we can acquire the desired optimal solution for the orthogonal Cp by collating these $k$ solutions together.
\end{proof}

\begin{proof}{\bf of Theorem \ref{theorem 2}}
\\
Given the closed-form of solution in proposition 1 above, we can calculate the probability $Pr^{Cp}\{i\}$ of the orthogonal Cp to identify the $i$-th feature for $i=1,\cdots,k$ as follows
\begin{eqnarray}\notag
Pr^{Cp}\{i\}&=&Pr\{\hat{a}_i\neq0\}\\&=&Pr\left\{n\left(\beta_i+\frac{W_i^{\top}\boldsymbol { \varepsilon }}{n}\right)^2>\lambda\right\}\notag
\end{eqnarray}
For $\boldsymbol{\varepsilon}\sim \mathcal { N } \left( 0 , \sigma_1 ^ { 2 }I_n \right)$, we can easily have $\beta_i+\frac{W_i^{\top}\boldsymbol { \varepsilon }}{n}\sim \mathcal{N}\left( \beta_i, \frac{\sigma_1^2}{n} \right)$ by considering the orthogonality assumption.
It follows that
\begin{eqnarray}\notag
Pr\left\{n\left(\beta_i+\frac{W_i^{\top}\boldsymbol { \varepsilon }}{n}\right)^2>\lambda\right\}&=&\frac{\sqrt{n}}{\sqrt{2\pi}\sigma_1}\int_{nx^2>\lambda}\exp\left\{-\frac{1}{2}\frac{\left(x-\beta_i\right)^2}{\frac{\sigma_1^2}{n}}\right\}dx\notag
\end{eqnarray}
Now, we begin to derive the second equality in Theorem \ref{theorem 2}.

Let $W_i^{\top}\boldsymbol { \varepsilon }:=\sqrt{\sigma_1^2n}\theta$, where $\theta\sim\mathcal{N}(0,1)$. Therefore, we have
\begin{eqnarray}\notag
Pr^{Cp}\{i\}&=&Pr\left\{n\beta_i^2+2\sqrt{\sigma_1^2n}\beta_i\theta+\sigma_1^2\theta^2-\lambda>0\right\}\\&=&1-
Pr\left\{\frac{-\sqrt{n}\beta_i-\sqrt{\lambda}}{\sigma_1}<\theta<\frac{-\sqrt{n}\beta_i+\sqrt{\lambda}}{\sigma_1}\right\}\notag\\&=&1-\int_{\frac{-\sqrt{n}\beta_i-\sqrt{\lambda}}{\sigma_1}}^{\frac{-\sqrt{n}\beta_i+\sqrt{\lambda}}{\sigma_1}}\frac{1}{\sqrt{2\pi}}\exp\left\{-\frac{x^2}{2}\right\}~dx,\notag
\end{eqnarray}	
{where the second equality is obtained by solving the quadratic equation for $\theta$.}
\end{proof}

\begin{proof}{\bf of Proposition \ref{proposition3}}
\\
We can prove this proposition by considering the following two cases.

If the $j$-th true regression coefficient satisfies $\beta_j=\sqrt{\frac{2}{n}}\sigma_1$, and the assumption
$\lambda=2\sigma_1^2$ holds, then substituting these two conditions into the second equality in Theorem \ref{theorem 2}, we have
\begin{eqnarray}
Pr^{Cp}\{j\}&=&1-\int_{\frac{-\sqrt{n}\beta_j-\sqrt{\lambda}}{\sigma_1}}^{\frac{-\sqrt{n}\beta_j+\sqrt{\lambda}}{\sigma_1}}\frac{1}{\sqrt{2\pi}}\exp\left\{-\frac{x^2}{2}\right\}~dx \\&=&1-\int_{-2\sqrt{2}}^{0}\frac{1}{\sqrt{2\pi}}\exp{\left\{-\frac{x^2}{2}\right\}}dx\notag\\ &=&1-[\phi(0)-\phi(-2\sqrt{2})]\notag %\\ &=& 0.5\notag
\end{eqnarray}
By numerical calculation of this integral, we find $Pr^{Cp}\{j\}=0.5023(\pm0.0001)$.
Due to symmetry of the expressions, the same result holds for the case when $\beta_j=-\sqrt{\frac{2}{n}}\sigma_1$.	
\end{proof}	

\begin{proof}{\bf of Theorem \ref{theorem4}}
\\
Let $\boldsymbol{1}_i$ be the indicator function of whether the $i$-th feature is selected by the orthogonal Cp (\ref{the modified Cp}) or not. To be specific,
\begin{equation}\notag
\boldsymbol{1}_i=
\left\{
\begin{matrix}
0
& \text{if~} a_i=0\\
1
&  \text{otherwise}
\end{matrix}
\right.
\end{equation}
The MSE value for the estimator $\hat{\boldsymbol{a}}$ from the orthogonal Cp is
\begin{eqnarray}\notag
\text{MSE}(\hat{\boldsymbol{a}})&=&\mathbb{E}(\boldsymbol{\beta}-\hat{\boldsymbol{a}})^{\top}(\boldsymbol{\beta}-\hat{\boldsymbol{a}})\\&=&\mathbb{E}\left\{\sum_{i=1}^{k}\left[\beta_i-\boldsymbol{1}_i\left(\beta_i+\frac{W_i^{\top}\boldsymbol { \varepsilon }}{n}\right)\right]^2\right\}\notag \\&=&\sum_{i=1}^{k}\mathbb{E}\left[\beta_i-\boldsymbol{1}_i\left(\beta_i+\frac{W_i^{\top}\boldsymbol { \varepsilon }}{n}\right)\right]^2\notag \\&=&\sum_{i=1}^{k}\left\{Pr\{\boldsymbol{1}_i=1\}\mathbb{E}\left[\left.\frac{(W_i^{\top}\boldsymbol { \varepsilon })^2}{n^2}\right|\boldsymbol{1}_i=1\right]+Pr\{\boldsymbol{1}_i=0\}\beta_i^2\right\}\notag
\end{eqnarray}	
where the second equality is obtained by substituting   $\hat{\boldsymbol{a}}$ in Proposition \ref{theorem1}, and the last equality is due to the
law of total expectation.

For simplicity, we further define the random event $\boldsymbol{H}$ as
\begin{equation}
\boldsymbol{H}=\left\{\theta_i>\frac{-\sqrt{n}\beta_i+\sqrt{\lambda}}{\sigma_1} ~\text{or}~ \theta_i<\frac{-\sqrt{n}\beta_i-\sqrt{\lambda}}{\sigma_1}\right\}
\end{equation}
For $\boldsymbol{\varepsilon}\sim \mathcal { N } \left( 0 , \sigma_1 ^ { 2 }I_n \right)$, we have
\begin{eqnarray}
\mathbb{E}\left[\left.\frac{(W_i^{\top}\boldsymbol { \varepsilon })^2}{n^2} \right|\boldsymbol{1}_i=1\right]&=&\mathbb{E}\left[\left.\frac{\sigma_1^2\theta_i^2}{n}\right|\boldsymbol{H}\right]\notag\\&=&\notag \frac{\sigma_1^2}{n}\int_{-\infty}^{+\infty}x^2dPr\left\{x|\boldsymbol{H}\right\}\\ \notag&=& \frac{\sigma_1^2}{n}\frac{1}{Pr\{\boldsymbol{H}\}}\left[1-\int_{\frac{-\sqrt{n}\beta_i-\sqrt{\lambda}}{\sigma_1}}^{\frac{-\sqrt{n}\beta_i+\sqrt{\lambda}}{\sigma_1}}\frac{x^2}{\sqrt{2\pi}}\exp\left\{-\frac{x^2}{2}\right\}~dx\right]
\end{eqnarray}	
where $\theta_i=\frac{\boldsymbol { \varepsilon }^{\top}W_i}{\sqrt{\sigma_1^2n}}\sim \mathcal { N } \left( 0 , 1 \right)$, and $Pr\{\cdot | \boldsymbol{H}\}$ is the probability measure defined, for each set $\boldsymbol{A}$, as $Pr\{\boldsymbol{A} | \boldsymbol{H}\}=\frac{Pr\{\boldsymbol{A}\cap \boldsymbol{H}\}}{Pr\{\boldsymbol{H}\}}$.

Together with the results in Theorem \ref{theorem 2}, it follows that
\begin{equation*}
 Pr\{\boldsymbol{1}_i=1\}=Pr\{\boldsymbol{H}\}.
 \end{equation*}
%\begin{equation*}
%Pr\{\boldsymbol{1}_i=0\}=\int_{\frac{-\sqrt{n}\beta_i-\sqrt{\lambda}}{\sigma_1}}^{\frac{-\sqrt{n}\beta_i+\sqrt{\lambda}}{\sigma_1}}\frac{1}{\sqrt{2\pi}}\exp\left\{-\frac{x^2}{2}\right\}~dx,
%\end{equation*}
 Combining the previous results together, we have
\begin{eqnarray}\notag
\text{MSE}(\hat{\boldsymbol{a}})&=&\sum_{i=1}^{k}\frac{\sigma_1^2}{n}\left[1-\int_{\frac{-\sqrt{n}\beta_i-\sqrt{\lambda}}{\sigma_1}}^{\frac{-\sqrt{n}\beta_i+\sqrt{\lambda}}{\sigma_1}}\frac{x^2}{\sqrt{2\pi}}\exp\left\{-\frac{x^2}{2}\right\}~dx\right]+\beta_{i}^2Pr\{\boldsymbol{1}_i=0\}\\\notag &=&\sum_{i=1}^{k}\left[\frac{\sigma_1^2}{n}+\int_{\frac{-\sqrt{n}\beta_i-\sqrt{\lambda}}{\sigma_1}}^{\frac{-\sqrt{n}\beta_i+\sqrt{\lambda}}{\sigma_1}}\left(\beta_i^2-\frac{\sigma_1^2x^2}{n}\right)\frac{1}{\sqrt{2\pi}}\exp\left\{-\frac{x^2}{2}\right\}dx\right]\notag
\end{eqnarray}
Letting $x=\frac{y}{\sigma_1}$, we can rewrite the expression of $\text{MSE}(\hat{\boldsymbol{a}})$ as follows
\begin{equation*}
\text{MSE}(\hat{\boldsymbol{a}})=\sum_{i=1}^{k}\left[\frac{\sigma_1^2}{n}+\frac{1}{\sqrt{2\pi}\sigma_1}\int_{(y+\sqrt{n}\beta_i)^2<\lambda}\left(\beta_i^2-\frac{1}{n}y^2\right)\exp{\left\{-\frac{y^2}{2\sigma_1^2}\right\}}dy\right]
\end{equation*}
This completes the proof.
\end{proof}	

\begin{proof}{\bf of Proposition \ref{proposition5}}
\\
Firstly, to obtain the global minimizers of $f(x)$ in the interval $(-\infty, +\infty)$, we can analyze the zeros of its derivative.

Notice that
\begin{equation*}
\frac{df(x)}{dx}=x\left(\frac{x^2}{n\sigma_1^2}-\frac{\beta_i^2}{\sigma_1^2}-\frac{2}{n}\right)\exp\left\{-\frac{x^2}{2\sigma_1^2}\right\}
\end{equation*}
The zeros of this function are at $x_1=0$, $x_2=-\sqrt{n\beta_i^2+2\sigma_1^2}$ and $x_3=\sqrt{n\beta_i^2+2\sigma_1^2}$.
It is easy to check that $x_2, x_3$ are two global minima of $f(x)$, and $x_1$ is the global maximum.

 If we let $\lambda=2\sigma_1^2$, then the integral interval of $\int_{(y+\sqrt{n}\beta_i)^2<\lambda}\left(\beta_i^2-\frac{1}{n}y^2\right)\exp{\left\{-\frac{y^2}{2\sigma_1^2}\right\}}dy$ in (\ref{7}) becomes $\left(-\sqrt{n}\beta_i-\sqrt{2\sigma_1^2}, ~-\sqrt{n}\beta_i+\sqrt{2\sigma_1^2}\right)$, in which at least one of $x_2, x_3$ is included.
\end{proof}

\begin{proof}{\bf of Proposition \ref{TLCp solution}}
\\
Let $\boldsymbol{\bar{1}}_i$ be the indicator function of whether the $i$-th feature is selected by the TLCp model (\ref{TLCp}) or not. Specifically,
\begin{eqnarray}\notag
\boldsymbol{\bar{1}}_i=
\left\{
\begin{matrix}
0
& \text{if~} \|\boldsymbol{w}_1^i\|_0=\|\boldsymbol{w}_2^i\|_0=0\\
1
&  \text{ortherwise}
\end{matrix}
\right.
\end{eqnarray}
Then, the proposed TLCp model in (\ref{TLCp}) is equivalent to
\begin{eqnarray}
\min_{\boldsymbol{v}_1,\boldsymbol{v}_2,\boldsymbol{w}_0}\sum _ { i = 1 } ^ { n }\lambda_1(y_i-\boldsymbol{w}_1^{\top}X_i)^2 + \sum _ { i = 1 } ^ { m } \lambda_2(\tilde{y_{i}}-\boldsymbol{w}_2^{\top}\tilde{X_{i}})^2 + \frac { 1 } { 2 } \sum _ { t = 1 } ^ { 2 }  \boldsymbol { v }_t^{\top}\boldsymbol{\lambda} _ { 3 }\boldsymbol{v}_t+\lambda_4\sum_{i=1}^{k}\boldsymbol{\bar{1}}_i
\end{eqnarray}
Notice that $\boldsymbol{X}^{\top}\boldsymbol{X}=nI$, $\boldsymbol{\tilde{X}}^{\top}\boldsymbol{\tilde{X}}=mI $ and $\boldsymbol{w}_1=\boldsymbol{w}_0+\boldsymbol{v}_1$, $\boldsymbol{w}_2=\boldsymbol{w}_0+\boldsymbol{v}_2$. Then, we can rewrite the objective function as follows
\begin{equation*}
 \lambda_1\boldsymbol{y}^{\top}\boldsymbol{y}+\lambda_2\boldsymbol{\tilde{y}}^{\top}\boldsymbol{\tilde{y}}+\sum_{i=1}^{k}\left\{f_i(\lambda_1,\boldsymbol{y},W_i,w_0^i,v_1^i)+g_i(\lambda_2,\boldsymbol{\tilde{y}},\tilde{W}_i,w_0^i,v_2^i)+h_i(\lambda_3^i,v_1^i,v_2^i,\lambda_4,\bar{\boldsymbol{1}}_i)\right\}
\end{equation*}
where
\begin{equation*}
f_i(\lambda_1,\boldsymbol{y},W_i,w_0^i,v_1^i)=-2\lambda_1\left(\boldsymbol{y}^{\top}W_iw_0^i+\boldsymbol{y}^{\top}W_iv_1^i-nw_0^iv_1^i\right)+\lambda_1n\left[(w_0^i)^2+(v_1^i)^2\right]
\end{equation*}
\begin{equation*}
g_i(\lambda_2,\boldsymbol{\tilde{y}},\tilde{W}_i,w_0^i,v_2^i)=-2\lambda_2\left(\boldsymbol{\tilde{y}}^{\top}\tilde{W}_iw_0^i+\boldsymbol{\tilde{y}}^{\top}\tilde{W}_iv_2^i-mw_0^iv_2^i\right)+\lambda_2m\left[(w_0^i)^2+(v_2^i)^2\right]
\end{equation*}
and
\begin{equation*}
h_i(\lambda_3^i,v_1^i,v_2^i,\lambda_4,\bar{\boldsymbol{1}}_i)=\frac{1}{2}\lambda_3^i\left[(v_1^i)^2+(v_2^i)^2\right]+\lambda_4\boldsymbol{\bar{1}}_i.
\end{equation*}
Similar to the argument used in the proof of Theorem \ref{theorem1}, and because of the independence of each summation term in the objective function above, the orthogonal TLCp is further equivalent to $k$ one-dimensional optimization problem below
\begin{equation}\label{TLCp 2}
\min_{\boldsymbol{v}_1^i,\boldsymbol{v}_2^i,\boldsymbol{w}_0^i}~\left\{f_i(\lambda_1,\boldsymbol{y},W_i,w_0^i,v_1^i)+g_i(\lambda_2,\boldsymbol{\tilde{y}},\tilde{W}_i,w_0^i,v_2^i)+h_i(\lambda_3^i,v_1^i,v_2^i,\lambda_4,\bar{\boldsymbol{1}}_i)\right\}
\end{equation}
for $i=1,\cdots,k$, in the sense that they have the same solution.

For the $i$-th problem above, if $\boldsymbol{\bar{1}}_i=1$, and making the gradient of the corresponding objective function equal to zero, we can easily obtain the estimators
with respect to the $i$-th coefficients $w_1^i$ and $w_2^i$ for the target and source domains as follows
\begin{equation}\label{15}
\bar{w}_1^{i}=\bar{w}_0^i+\bar{v}_1^i= \beta_i+D_1^i\delta_i+(1-D_1^i)\frac{1}{n}W_i^{\top}\boldsymbol { \varepsilon } +D_1^i\frac{1}{m}\tilde{W}_i^{\top}\boldsymbol { \eta },
\end{equation}
\begin{equation}\label{16}
\bar{w}_2^{i}=\bar{w}_0^i+\bar{v}_2^i= \beta_i+(1-D_2^i)\delta_i+(1-D_2^i)\frac{1}{m}\tilde{W}_i^{\top}\boldsymbol { \eta } +D_2^i\frac{1}{n}W_i^{\top}\boldsymbol { \varepsilon },
\end{equation}
and the estimators for the $i$-th individual parameters are
\begin{equation}\label{17}
\bar{v}_1^i=-D_3^i\left(\delta_i+\frac{1}{m}\tilde{W}_i^{\top}\boldsymbol { \eta }-\frac{1}{n}W_i^{\top}\boldsymbol { \varepsilon }\right),
\end{equation}
\begin{equation}\label{18}
\bar{v}_2^i=D_3^i\left(\delta_i+\frac{1}{m}\tilde{W}_i^{\top}\boldsymbol { \eta }-\frac{1}{n}W_i^{\top}\boldsymbol { \varepsilon }\right),
\end{equation}
where $D_1^i=\frac{\lambda_2\lambda_3^i}{4\lambda_1\lambda_2n+\lambda_2\lambda_3^i+\frac{n}{m}\lambda_1\lambda_3^i}$, $D_2^i=\frac{\lambda_1\lambda_3^i}{4\lambda_1\lambda_2m+\lambda_1\lambda_3^i+\frac{m}{n}\lambda_2\lambda_3^i}$ and $D_3^i=\frac{2\lambda_1\lambda_2}{4\lambda_1\lambda_2+\frac{1}{n}\lambda_2\lambda_3^i+\frac{1}{m}\lambda_1\lambda_3^i}$.

Then, substituting the relations (\ref{15}), (\ref{16}), (\ref{17}) and (\ref{18}) into the objective function in (\ref{TLCp 2}), we have
\begin{align}\label{22}
&f_i(\lambda_1,\boldsymbol{y},W_i,\bar{w}_0^i,\bar{v}_1^i)+g_i(\lambda_2,\boldsymbol{\tilde{y}},\tilde{W}_i,\bar{w}_0^i,\bar{v}_2^i)+h_i(\lambda_3^i,\bar{v}_1^i,\bar{v}_2^i,\lambda_4,1)
\\ \notag&=(\tilde{D}^i-\lambda_2m)H_i^2+(\tilde{D}^i-\lambda_1n)Z_i^2-2\tilde{D}^iZ_iH_i+\lambda_4,
\end{align}
where $\tilde{D}^i=\lambda_1n(D_1^i)^2+\lambda_2m(D_2^i)^2+\lambda_3^i(D_3^i)^2$, and $H_i=\beta_i+\delta_i+\frac{1}{m}\tilde{W}_{i}^{\top} \boldsymbol{ \eta }$, $Z_i=\beta_i+\frac{1}{n}W_{i}^{\top} { \boldsymbol{\varepsilon} } $ are two random variables that stem from the responses $\boldsymbol{y}$ and $\boldsymbol{\tilde{y}}$ for the target and source tasks respectively.

Further, we notice that $2D_3^i+D_2^i+D_1^i=1$ and rearrange the last expression to obtain
\begin{eqnarray}\notag
f_i(\lambda_1,\boldsymbol{y},W_i,\bar{w}_0^i,\bar{v}_1^i)+g_i(\lambda_2,\boldsymbol{\tilde{y}},\tilde{W}_i,\bar{w}_0^i,\bar{v}_2^i)+h_i(\lambda_3^i,\bar{v}_1^i,\bar{v}_2^i,\lambda_4,1)=\lambda_4-A_iH_i^2-B_iZ_i^2-C_iJ_i^2
\end{eqnarray}
where $A_i=\frac{4\lambda_1\lambda_2^2m^2n}{4\lambda_1\lambda_2mn+m\lambda_2\lambda_3^i+n\lambda_1\lambda_3^i}$, $B_i=\frac{4\lambda_2\lambda_1^2mn^2}{4\lambda_1\lambda_2mn+m\lambda_2\lambda_3^i+n\lambda_1\lambda_3^i}$, and $C_i=\frac{\lambda_3^{i}}{4\lambda_1\lambda_2mn+m\lambda_2\lambda_3^i+n\lambda_1\lambda_3^i}$ are functions with respect to the parameters $\lambda_1, \lambda_2, \lambda_3^i$, while $J_i=m\lambda_2H_i+n\lambda_1Z_i$.

If $\boldsymbol{\bar{1}}_i=0$ in the $i$-th optimization problem, which means the estimators for the parameters $w_0^i,v_1^i,v_2^i$ satisfying $\bar{w}_0^i=\bar{v}_1^i=\bar{v}_2^i=0$. So the corresponding objective value is
\begin{equation*}
f_i(\lambda_1,\boldsymbol{y},W_i,\bar{w}_0^i,\bar{v}_1^i)+g_i(\lambda_2,\boldsymbol{\tilde{y}},\tilde{W}_i,\bar{w}_0^i,\bar{v}_2^i)+h_i(\lambda_3^i,\bar{v}_1^i,\bar{v}_2^i,\lambda_4,0)=0
\end{equation*}
Finally, we can derive the optimal solution for the $i$-th optimization problem (\ref{TLCp 2}) by finding two estimators $\hat{w}_1^i, \hat{w}_2^i$ that can pick the smaller one between the random value $\lambda_4-A_iH_i^2-B_iZ_i^2-C_iJ_i^2$ and $0$, i.e.,
\begin{equation*}
\hat{w}_1^{i}=\left\{
\begin{matrix}
 {\beta_{i}}+D_1^{i}{\delta}_{i}+(1-D_1^{i})\frac{1}{n}W_{i}^{\top} { \boldsymbol{\varepsilon} } +D_1^{i}\frac{1}{m}\tilde{W}_{i}^{\top} \boldsymbol{ \eta }, & \text{if~} A_iH_i^2+B_iZ_i^2+C_iJ_i^2>\lambda_4\\
0,
&  \text{otherwise}
\end{matrix}
\right.
\end{equation*}
Finally, we can acquire the desired optimal solution for the orthogonal TLCp model by collating these $k$ solutions together.
\end{proof}

\begin{proof}{\bf of Theorem \ref{TLCp probability}}
\\
To prove the first equality in Theorem \ref{TLCp probability}, firstly, the equality (\ref{22}) in the proof of Proposition \ref{TLCp solution} implies that the probability of the orthogonal TLCp to identify the $i$-th feature is
\begin{equation*}
Pr\{\bar{\boldsymbol{1}}_i=1\}=Pr \left\{(\tilde{D}^i-\lambda_2m)H_i^2+(\tilde{D}^i-\lambda_1n)Z_i^2-2\tilde{D}^iZ_iH_i+\lambda_4<0\right\},
\end{equation*}
where $\tilde{D}^i=\lambda_1n(D_1^i)^2+\lambda_2m(D_2^i)^2+\lambda_3^i(D_3^i)^2$, and $H_i=\beta_i+\delta_i+\frac{1}{m}\tilde{W}_{i}^{\top} \boldsymbol{ \eta }$, $Z_i=\beta_i+\frac{1}{n}W_{i}^{\top} { \boldsymbol{\varepsilon} } $ are two random variables.

For simplicity, we let $M_i=-\tilde{D}^i+\lambda_2m$, $N_i=-\tilde{D}^i+\lambda_1n$ and $Q_i=-2\tilde{D}^i$ for $i=1,\cdots,k$. According to the definitions of $D_1^i, D_2^i, D_3^i$ in Proposition \ref{TLCp solution}, we can see that $M_i>0, ~N_i>0,~ Q_i<0$ for $i=1,\cdots,k$.
Thus, we can rewrite the probability of the orthogonal TLCp to identify the $i$-th feature as
\begin{equation*}
Pr\{\bar{\boldsymbol{1}}_i=1\}=Pr\left\{ M_iH_i^2+N_iZ_i^2-Q_iH_iZ_i>\lambda_4\right\}.
\end{equation*}
Letting $\bar{X}_i=\sqrt{M_i}H_i, \bar{Y}_i=\sqrt{N_i}Z_i$, the probability can be rewritten as
\begin{equation*}
Pr\{\bar{\boldsymbol{1}}_i=1\}=Pr\left\{\bar{X}_i^2+\bar{Y}_i^2-\frac{Q_i}{\sqrt{M_iN_i}}\bar{X}_i\bar{Y}_i>\lambda_4\right\}.
\end{equation*}
Finally, letting $\bar{X}_i=U_i-V_i$ and $\bar{Y}_i=U_i+V_i$, we obtain
\begin{equation}\label{23}
Pr\{\bar{\boldsymbol{1}}_i=1\}=P_r\left\{\left(2-\frac{Q_i}{\sqrt{M_iN_i}}\right)U_i^2+\left(2+\frac{Q_i}{\sqrt{M_iN_i}}\right)V_i^2>\lambda_4\right\},
\end{equation}
for $i=1,\cdots,k$.

Notice that $U_i=\frac{\sqrt{M_i}H_i+\sqrt{N_i}Z_i}{2}$ and $V_i=\frac{-\sqrt{M_i}H_i+\sqrt{N_i}Z_i}{2}$ for $i=1,\cdots, k$. Substituting these two equations into (\ref{23}), we obtain the first equality in Theorem \ref{TLCp probability}.

Now, we prove the second equality in this theorem.
Because $H_i\sim \mathcal { N } \left( \beta_i+\delta_i , ~\frac{\sigma_2^2}{m} \right)$, and $Z_i\sim \mathcal { N } \left( \beta_i , ~\frac{\sigma_1^2}{n} \right)$, we have:
\begin{equation*}
U_i\sim \mathcal { N } \left( \frac{1}{2}\left[\sqrt{M_i}(\beta_i+\delta_i)+\sqrt{N_i}\beta_i\right], ~\frac{1}{4}\left(\frac{M_i\sigma_2^2}{m}+\frac{N_i\sigma_1^2}{n}\right) \right),
\end{equation*}
\begin{equation*}
V_i\sim \mathcal { N } \left( \frac{1}{2}\left[-\sqrt{M_i}(\beta_i+\delta_i)+\sqrt{N_i}\beta_i\right], ~\frac{1}{4}\left(\frac{M_i\sigma_2^2}{m}+\frac{N_i\sigma_1^2}{n}\right) \right).
\end{equation*}
Moreover, we can calculate the covariance matrix between the random variables $U_i$ and $V_i$ as
\begin{eqnarray}\notag
\boldsymbol{\Sigma}_{U_i,V_i}=\left( \begin{matrix} { \frac{1}{4}\left(\frac{M_i\sigma_2^2}{m}+\frac{N_i\sigma_1^2}{n}\right) } & { \frac{1}{4}\left(-\frac{M_i\sigma_2^2}{m}+\frac{N_i\sigma_1^2}{n}\right) } \\ { \frac{1}{4}\left(-\frac{M_i\sigma_2^2}{m}+\frac{N_i\sigma_1^2}{n}\right) } & { \frac{1}{4}\left(\frac{M_i\sigma_2^2}{m}+\frac{N_i\sigma_1^2}{n}\right) } \end{matrix} \right)
\end{eqnarray}
By definition, the joint distribution for $U_i, V_i$ has the density
\begin{equation*}
p\left( U_i=x,V_i=y\right) = \frac { \exp \left( - \frac { 1 } { 2 } ( \mathbf { x } - \boldsymbol { \mu } ) ^ { \mathrm { T } } \boldsymbol{\Sigma}_{U_i,V_i}^ { - 1 } ( \mathbf { x } - \boldsymbol { \mu } ) \right) } { \sqrt { ( 2 \pi ) ^ { 2 } | \boldsymbol{\Sigma}_{U_i,V_i} | } },
\end{equation*}
where $\mathbf { x }=(x ~y)^{\top}$, $\boldsymbol{\mu}=\left( \frac{1}{2}\left[\sqrt{M_i}(\beta_i+\delta_i)+\sqrt{N_i}\beta_i\right]~ \frac{1}{2}\left[-\sqrt{M_i}(\beta_i+\delta_i)+\sqrt{N_i}\beta_i\right]\right)^{\top}$.

Plugging the covariance matrix $\boldsymbol{\Sigma}_{U_i,V_i}$ in the density function $p\left( U_i=x,V_i=y\right)$, together with (\ref{23}), results in the probability of the orthogonal TLCp to select the $i$-th feature as follows:
\begin{align}\notag
P_r\{\bar{\boldsymbol{1}}_i=1\}&=\iint_{(2-\frac{Q_i}{\sqrt{M_iN_i}})x^2+( 2+\frac{Q_i}{\sqrt{M_iN_i}}
	)y^2>\lambda_4}p\left(U_i=x,V_i=y\right)dxdy \notag\\
&=\frac{\sqrt{mn}}{\pi\sigma_1\sigma_2\sqrt{M_iN_i}}\iint_{\left(2-\frac{Q_i}{\sqrt{M_iN_i}}\right)x^2+\left( 2+\frac{Q_i}{\sqrt{M_iN_i}}
	\right)y^2>\lambda_4}\exp\bigg\{\bigg[\frac{-n\left(x+y-\sqrt{N_i}\beta_i\right)^2}{2N_i\sigma_1^2}\notag\\
&\qquad\qquad\qquad\qquad\qquad\qquad\qquad\qquad-\frac{m\left(x-y-\sqrt{M_i}(\beta_i+\delta_i)\right)^2}{2M_i\sigma_2^2}\bigg]\bigg\}dxdy,\notag
\end{align}
where $i=1,\cdots, k$.
This proves the second equality in the theorem.
\end{proof}	

\begin{proof}{\bf of Corollary \ref{TLCp scale-up}}
\\
The probability of the orthogonal TLCp procedure to miss the $i$-th feature is
\begin{align}
&P_r\{\bar{\boldsymbol{1}}_i=0\}\notag\\&=\frac{\sqrt{mn}}{\pi\sigma_1\sigma_2\sqrt{M_iN_i}}\iint_{\left(2-\frac{Q_i}{\sqrt{M_iN_i}}\right)x^2+\left( 2+\frac{Q_i}{\sqrt{M_iN_i}}
	\right)y^2<\lambda_4}\exp\bigg\{-\frac{1}{2}\bigg[\frac{n}{N_i\sigma_1^2}\left(x+y-\sqrt{N_i}\beta_i\right)^2\label{aa}\\
&\qquad\qquad\qquad\qquad\qquad\qquad\qquad\qquad\qquad\qquad+\frac{m}{M_i\sigma_2^2}\left(x-y-\sqrt{M_i}(\beta_i+\delta_i)\right)^2\bigg]\bigg\}dxdy,\notag
\end{align}
for $i=1,\cdots,k$.
Due to the non-negativity of the integral function in (\ref{aa}), we can obtain an upper bound for this integral by scaling up the integral region.
To be specific, firstly, we stretch the ellipse integral region $\left\{(x,y) \left| \left(2-\frac{Q_i}{\sqrt{M_iN_i}}\right)x^2+\left(2+\frac{Q_i}{\sqrt{M_iN_i}}
\right)y^2<\lambda_4\right.\right\}$ into a strip region$\left\{(x,y) \left| \left(2-\frac{Q_i}{\sqrt{M_iN_i}}\right)x^2<\lambda_4\right.\right\}$, thus obtaining
\begin{align}
&P_r\{\bar{\boldsymbol{1}}_i=0\}\notag\\&\leq \frac{\sqrt{mn}}{\pi\sigma_1\sigma_2\sqrt{M_iN_i}}\iint_{\left(2-\frac{Q_i}{\sqrt{M_iN_i}}\right)x^2<\lambda_4}\exp\bigg\{-\frac{1}{2}\bigg[\frac{n}{N_i\sigma_1^2}\left(x+y-\sqrt{N_i}\beta_i\right)^2\notag\\
&\qquad\qquad\qquad\qquad\qquad\qquad\qquad\qquad\qquad+\frac{m}{M_i\sigma_2^2}\left(x-y-\sqrt{M_i}(\beta_i+\delta_i)\right)^2\bigg]\bigg\}dxdy\notag\\&=\frac{\sqrt{mn}}{\pi\sigma_1\sigma_2\sqrt{M_iN_i}}\int_{\left(2-\frac{Q_i}{\sqrt{M_iN_i}}\right)x^2<\lambda_4}dx\int_{-\infty}^{+\infty}\exp\bigg\{-\frac{1}{2}\bigg[\frac{n}{N_i\sigma_1^2}\left(x+y-\sqrt{N_i}\beta_i\right)^2\notag\\
&\qquad\qquad\qquad\qquad\qquad\qquad\qquad\qquad\qquad\qquad+\frac{m}{M_i\sigma_2^2}\left(x-y-\sqrt{M_i}(\beta_i+\delta_i)\right)^2\bigg]\bigg\}dy,\label{24}
\end{align}
for $i=1,\cdots, k$.

Now, we turn attention to the second integral in the equation above. By separating the factors of the function in the second integral, we obtain
\begin{align}
&\int_{-\infty}^{+\infty}\exp\bigg\{-\frac{1}{2}\bigg[\frac{n}{N_i\sigma_1^2}\left(x+y-\sqrt{N_i}\beta_i\right)^2+\frac{m}{M_i\sigma_2^2}\left(x-y-\sqrt{M_i}(\beta_i+\delta_i)\right)^2\bigg]\bigg\}dy\notag\\&=\int_{-\infty}^{+\infty}\exp\left\{-\frac{1}{2}\left[K_1y^2+2K_2(x)y+K_3(x)\right]\right\}dy,\label{25}
\end{align}
where $K_1=\frac{n}{N_i\sigma_1^2}+\frac{m}{M_i\sigma_2^2}$, $K_2(x)=\frac{n}{N_i\sigma_1^2}\left(x-\sqrt{N_i}\beta_i\right)-\frac{m}{M_i\sigma_2^2}\left(x-\sqrt{M_i}(\beta_i+\delta_i)\right)$, and $K_3(x)=\frac{n}{N_i\sigma_1^2}\left(x-\sqrt{N_i}\beta_i\right)^2+\frac{m}{M_i\sigma_2^2}\left(x-\sqrt{M_i}(\beta_i+\delta_i)\right)^2$.
Further, we can rewrite (\ref{25}) as
\begin{align}
&\int_{-\infty}^{+\infty}\exp\left\{-\frac{1}{2}\left[K_1\left(y+\frac{K_2(x)}{K_1}\right)^2+\left(K_3(x)-\frac{K_2^2(x)}{K_1}\right)\right]\right\}dy\notag\\&=\int_{-\infty}^{+\infty}\sqrt{\frac{2\pi}{K_1}}\exp\left\{-\frac{1}{2}\left(K_3(x)-\frac{K_2^2(x)}{K_1}\right)\right\}\sqrt{\frac{K_1}{2\pi}}\exp\left\{-\frac{K_1}{2}\left(y+\frac{K_2(x)}{K_1}\right)^2\right\}dy.\label{26}
\end{align}
By the definition of probability distribution, for any fixed $x$, we have
\begin{equation}\notag
\int_{-\infty}^{+\infty}\sqrt{\frac{K_1}{2\pi}}\exp\left\{-\frac{K_1}{2}\left(y+\frac{K_2(x)}{K_1}\right)^2\right\}dy=1.
\end{equation}
Substituting (\ref{25}), (\ref{26}) into (\ref{24}), gives
\begin{align}
&P_r\{\bar{\boldsymbol{1}}_i=0\}\notag\\&\leq \sqrt{\frac{2\pi}{K_1}}\cdot\frac{\sqrt{mn}}{\pi\sigma_1\sigma_2\sqrt{M_iN_i}}\int_{\left(2-\frac{Q_i}{\sqrt{M_iN_i}}\right)x^2<\lambda_4}\exp\left\{-\frac{1}{2}\left(K_3(x)-\frac{K_2^2(x)}{K_1}\right)\right\}dx\notag\\&=\frac{\sqrt{2}}{\sqrt{\pi}}G_i\int_{\left(2-\frac{Q_i}{\sqrt{M_iN_i}}\right)x^2<\lambda_4}\exp\left\{-\frac{G_i^2}{2}\left[2x-(\sqrt{M_i}+\sqrt{N_i})\beta_i-\sqrt{M_i}\delta_i\right]^2\right\}dx\label{28},
\end{align}
where $G_i=\sqrt{\frac{mn}{nM_i\sigma_2^2+mN_i\sigma_1^2}}$ for $i=1,\cdots, k$.

Finally, letting $y=\sigma_1G_ix$ and using variable substitution in (\ref{28}), we obtain
\begin{align}
&P_r\{\bar{\boldsymbol{1}}_i=0\}\notag\\&\leq\frac{\sqrt{2}}{\sqrt{\pi}\sigma_1}\int_{4y^2<\frac{4\lambda_4\sigma_1^2G_i^2}{2-\frac{Q_i}{\sqrt{M_iN_i}}}}  ~~\exp\left\{-\frac{1}{2\sigma_1^2}\left[2y-G_i(\sqrt{M_i}+\sqrt{N_i})\sigma_1\beta_i-G_i\sqrt{M_i}\sigma_1\delta_i\right]^2\right\}dy,\notag
\end{align}
for $i=1,\cdots, k$, thus proving this corollary.
\end{proof}

\begin{proof}{\bf of Theorem \ref{tune parameters}}
\\
For any fixed feature, for example,  the $\ell$-th one in the model, by Corollary \ref{TLCp scale-up} we have
\begin{align}
&P_r\{\bar{\boldsymbol{1}}_{\ell}=0\}\notag\\&\leq
\frac{\sqrt{2}}{\sqrt{\pi}\sigma_1}\int_{4y^2<\frac{4\lambda_4\sigma_1^2G_{\ell}^2}{2-\frac{Q_\ell}{\sqrt{M_{\ell}N_\ell}}}}  ~~\exp\left\{-\frac{1}{2\sigma_1^2}\left[2y-G_{\ell}(\sqrt{M_\ell}+\sqrt{N_\ell})\sigma_1\beta_\ell-G_\ell\sqrt{M_{\ell}}\sigma_1\delta_\ell\right]^2\right\}dy\notag\\&=\frac{\sqrt{2}}{2\sqrt{\pi}\sigma_1}\int_{\left(z_1+G_{\ell}(\sqrt{M_\ell}+\sqrt{N_\ell})\sigma_1\beta_\ell+G_\ell\sqrt{M_\ell}\sigma_1\delta_{\ell}\right)^2<\frac{4\lambda_4\sigma_1^2G_{\ell}^2}{2-\frac{Q_\ell}{\sqrt{M_{\ell}N_\ell}}} }  \exp\left\{-\frac{z_1^2}{2\sigma_1^2}\right\}dz_1,\label{31}
\end{align}
where the second equality is obtained by substituting  $2y-G_{\ell}(\sqrt{M_\ell}+\sqrt{N_\ell})\sigma_1\beta_{\ell}-G_{\ell}\sqrt{M_\ell}\sigma_1\delta_{\ell}=z_1$ into the the first inequality.

On the other hand, by Remark \ref{remark 1}, we have
\begin{equation}\label{32}
P_r\{\boldsymbol{1}_\ell=0\}=\frac{\sqrt{2}}{\sqrt{\pi}\sigma_1}\int_{4y^2<\lambda}\exp\left\{-{\frac{1}{2\sigma_1^2}}\left(2y-\sqrt{n}\beta_\ell\right)^2\right\}dy.
\end{equation}
Then, if we let $z_2=2y-\sqrt{n}\beta_\ell$ and substitute it into (\ref{32}) above, we obtain
\begin{equation}\label{33}
P_r\{\boldsymbol{1}_\ell=0\}=\frac{\sqrt{2}}{2\sqrt{\pi}\sigma_1}\int_{(z_2+\sqrt{n}\beta_\ell)^2<\lambda}\exp\left\{-\frac{z_2^2}{2\sigma_1^2}\right\}dz_2
\end{equation}
Comparing (\ref{31}) and (\ref{33}), we can see that, when parameters $\lambda_1,\lambda_2, \lambda_3^{\ell}, \lambda_4$ satisfy the conditions  $|\sqrt{n}\beta_\ell|<|G_\ell\sigma_1|\cdot\left|(\sqrt{M_\ell}+\sqrt{N_\ell})\beta_\ell+\sqrt{M_\ell}\delta_\ell\right|$ and  $\left(4\lambda_4\sigma_1^2G_{\ell}^2\right)/\left(2-\frac{Q_\ell}{\sqrt{M_{\ell}N_\ell}}\right)\leq\lambda$ or when we choose $\lambda_4=\min_{i\in\{1,\cdots,k\}}\left\{\lambda\left(2-\frac{Q_i}{\sqrt{M_iN_i}}\right)/
4\sigma_1^2(G_i)^2\right\}$, where $\lambda$ is a parameter in Cp criterion(\ref{the modified Cp}), then together with the definition of probability for normal distribution, we obtain  $P_r\{\bar{\boldsymbol{1}}_\ell=0\}<P_r\{\boldsymbol{1}_\ell=0\}$.
In other words, $ P_r\{\bar{\boldsymbol{1}}_\ell=1\}>P_r\{\boldsymbol{1}_\ell=1\}$.
\end{proof}

\begin{proof}{\bf of Theorem \ref{MSE TLCp}}
\\
As before, we define
\begin{eqnarray}\notag
\boldsymbol{\bar{1}}_i=
\left\{
\begin{matrix}
0
& \text{if~} \|\boldsymbol{w}_1^i\|_0=\|\boldsymbol{w}_2^i\|_0=0\\
1
&  \text{ortherwise}
\end{matrix}
\right.
\end{eqnarray}
By the definition of MSE for the estimator $\hat{\boldsymbol{w}_1}$ from the orthogonal TLCp, we have
\begin{eqnarray}\notag
\text{MSE}(\hat{\boldsymbol{w}}_1)&=&\mathbb{E}(\boldsymbol{\beta}-\hat{\boldsymbol{w}}_1)^{\top}(\boldsymbol{\beta}-\hat{\boldsymbol{w}}_1)\notag\\&=&\mathbb{E}\left\{\sum_{i=1}^{k}\left[\beta_i-\bar{\boldsymbol{1}}_i\left({\beta_{i}}+D_1^{i}{\delta}_{i}+(1-D_1^{i})\frac{1}{n}W_{i}^{\top} { \boldsymbol{\varepsilon} } +D_1^{i}\frac{1}{m}\tilde{W}_{i}^{\top} \boldsymbol{ \eta }\right)\right]^2\right\}\notag \\ &=&\sum_{i=1}^{k}\mathbb{E}\left[\beta_i-\bar{\boldsymbol{1}}_i\left({\beta_{i}}+D_1^{i}{\delta}_{i}+(1-D_1^{i})\frac{1}{n}W_{i}^{\top} { \boldsymbol{\varepsilon} } +D_1^{i}\frac{1}{m}\tilde{W}_{i}^{\top} \boldsymbol{ \eta }\right)\right]^2\notag\\&=&\sum_{i=1}^{k}\left\{P_r\{\bar{\boldsymbol{1}}_i=1\}\mathbb{E}\left[\left.R_i^2\right|\bar{\boldsymbol{1}}_i=1\right]+P_r\{\bar{\boldsymbol{1}}_i=0\}\beta_i^2\right\},\label{35}
\end{eqnarray}
where $R_i=-D_1^i\delta_i-(1-D_1^i)\frac{W_i^{\top} \boldsymbol{\varepsilon} }{n}-D_1^i\frac{\tilde{W}_i^{\top}\boldsymbol{\eta}}{m}$, for $i=1,\cdots, k$. The second equality is obtained by substituting  $\hat{\boldsymbol{w}}_1$ in Proposition \ref{TLCp solution}, while the last equality is due to the law of total expectation.

In order to facilitate the calculation with respect to $\mathbb{E}\left[\left.R_i^2\right|\bar{\boldsymbol{1}}_i=1\right]$ in (\ref{35}), we rewrite $R_i$ as $R_i=\bar{M}_iU_i+\bar{N}_iV_i+\beta_i$, where $\bar{M}_i=\frac{\sqrt{M_i}-\sqrt{N_i}}{\sqrt{M_iN_i}}D_1^i-\frac{1}{\sqrt{N_i}}$, $\bar{N}_i=\frac{\sqrt{M_i}+\sqrt{N_i}}{\sqrt{M_iN_i}}D_1^i-\frac{1}{\sqrt{N_i}}$ are determined by $M_i, N_i, D_1^i$ as introduced in the proof of Theorem \ref{TLCp probability}, and $U_i=\frac{\sqrt{M_i}H_i+\sqrt{N_i}Z_i}{2}$, $V_i=\frac{-\sqrt{M_i}H_i+\sqrt{N_i}Z_i}{2}$ are two random variables related to $H_i, Z_i$ in Theorem \ref{TLCp probability}, for $i=1,\cdots, k$.

By the definition of conditional expectation, we have
\begin{align}
&\mathbb{E}\left[\left.R_i^2\right|\bar{\boldsymbol{1}}_i=1\right]\notag\\&
=\mathbb{E}\left[\left.(\bar{M}_iU_i+\bar{N}_iV_i+\beta_i)^2\right|\bar{\boldsymbol{1}}_i=1\right]\notag\\
&=\frac{1}{P_r\{\bar{\boldsymbol{1}}_i=1\}}\iint_{\mathbb{R}^2}(\bar{M}_ix+\bar{N}_iy+\beta_i)^2p(U_i=x,V_i=y, \bar{\boldsymbol{1}}_i=1)~dxdy\notag\\
&=\frac{1}{P_r\{\bar{\boldsymbol{1}}_i=1\}}\iint_{\left(2-\frac{Q_i}{\sqrt{M_iN_i}}\right)x^2+\left( 2+\frac{Q_i}{\sqrt{M_iN_i}}
	\right)y^2>\lambda_4}(\bar{M}_ix+\bar{N}_iy+\beta_i)^2p(U_i=x,V_i=y)~dxdy\notag\\
&=\mathbb{E}[(\bar{M}_i{U_i}+\bar{N}_i{V_i}+\beta_i)^2]-\notag\\
&\qquad\frac{1}{P_r\{\bar{\boldsymbol{1}}_i=1\}}\iint_{\left(2-\frac{Q_i}{\sqrt{M_iN_i}}\right)x^2+\left( 2+\frac{Q_i}{\sqrt{M_iN_i}}
	\right)y^2<\lambda_4}(\bar{M}_ix+\bar{N}_iy+\beta_i)^2p(U_i=x,V_i=y)~dxdy\label{36}
\end{align}
where $p(U_i=x,V_i=y)$ is the joint density distribution for $U_i, V_i$, which has been derived in Theorem \ref{TLCp probability}, for $i=1,\cdots, k$.

Also, by the proof of Theorem \ref{TLCp probability},
\begin{equation}\label{37}
P_r\{\bar{\boldsymbol{1}}_i=0\}=\iint_{\left(2-\frac{Q_i}{\sqrt{M_iN_i}}\right)x^2+\left( 2+\frac{Q_i}{\sqrt{M_iN_i}}
	\right)y^2<\lambda_4}p(U_i=x,V_i=y)~dxdy.
\end{equation}
Substituting (\ref{36}) and (\ref{37}) into (\ref{35}), we obtain the desired result below
\begin{align}
&MSE(\hat{\boldsymbol{w}}_1)\notag\\
&=\sum_{i=1}^{k}\bigg\{\mathbb{E}[(\bar{M}_i{U_i}+\bar{N}_i{V_i}+\beta_i)^2]+\notag\\
&\qquad \qquad \iint_{\left(2-\frac{Q_i}{\sqrt{M_iN_i}}\right)x^2+\left( 2+\frac{Q_i}{\sqrt{M_iN_i}}\right)y^2<\lambda_4} \left[\beta_i^2-(\bar{M}_ix+\bar{N}_iy+\beta_i)\right]p\left(U_i=x,V_i=y\right)dxdy \bigg\}\notag.
\end{align}
\end{proof}	

\begin{proof}{\bf of Proposition \ref{explicit rule}}
\\
Firstly, we notice that $\bar{M}_i=\frac{\sqrt{M_i}-\sqrt{N_i}}{\sqrt{M_iN_i}}D_1^i-\frac{1}{\sqrt{N_i}}$, $\bar{N}_i=\frac{\sqrt{M_i}+\sqrt{N_i}}{\sqrt{M_iN_i}}D_1^i-\frac{1}{\sqrt{N_i}}$, where $D_1^i=\frac{\lambda_2\lambda_3^i}{4\lambda_1\lambda_2n+\lambda_2\lambda_3^i+\frac{n}{m}\lambda_1\lambda_3^i}$, for $i=1, \cdots, k$.

Also, by the proof of Theorem \ref{TLCp probability}, we can see that
\begin{equation*}
U_i\sim \mathcal { N } \left( \frac{1}{2}\left[\sqrt{M_i}(\beta_i+\delta_i)+\sqrt{N_i}\beta_i\right], ~\frac{1}{4}\left(\frac{M_i\sigma_2^2}{m}+\frac{N_i\sigma_1^2}{n}\right) \right),
\end{equation*}
\begin{equation*}
V_i\sim \mathcal { N } \left( \frac{1}{2}\left[-\sqrt{M_i}(\beta_i+\delta_i)+\sqrt{N_i}\beta_i\right], ~\frac{1}{4}\left(\frac{M_i\sigma_2^2}{m}+\frac{N_i\sigma_1^2}{n}\right) \right),
\end{equation*}
for $i=1, \cdots, k$.

Then, by expanding the term $\mathbb{E}(\bar{M}U_i+\bar{N}V_i+\beta_i)^2$ and merging the similar items together, we have
\begin{align}
&\mathbb{E}(\bar{M}U_i+\bar{N}V_i+\beta_i)^2\notag\\
&=\frac{N_i\sigma_1^2}{4n}(\bar{M}_i+\bar{N}_i)^2+\frac{M_i\sigma_2^2}{4m}(\bar{M}_i-\bar{N}_i)^2+\frac{H^2(\bar{M}_i,\bar{N}_i,\delta_i, \beta_i)}{4}+\beta_i^2+\beta_iH(\bar{M}_i,\bar{N}_i,\delta_i, \beta_i),\label{38}
\end{align}
where $H(\bar{M}_i,\bar{N}_i,\delta_i, \beta_i)=[\sqrt{M_i}(\beta_i+\delta_i)+\sqrt{N_i}\beta_i]\bar{M}_i+[-\sqrt{M_i}(\beta_i+\delta_i)+\sqrt{N_i}\beta_i]\bar{N}_i$, for $i=1, \cdots, k$.
	
Further, we notice that $\bar{M}_i+\bar{N}_i=\frac{2}{\sqrt{N_i}}D_1^i-\frac{2}{\sqrt{N}_i}$, $\bar{M}_i-\bar{N}_i=-\frac{2}{\sqrt{M_i}}D_1^i$, and substitute these two equations into (\ref{38}), thus obtaining
\begin{equation}
\mathbb{E}(\bar{M}U_i+\bar{N}V_i+\beta_i)^2
={D_1^i}^2\left(\delta_i^2+\frac{\sigma_1^2}{n}+\frac{\sigma_2^2}{m}\right)-2D_1^i\frac{\sigma_1^2}{n}+\frac{\sigma_1^2}{n},\label{39}
\end{equation}
for $i=1,\cdots,k$.	

For each $i\in\{1,\cdots, k\}$, the right-hand side of equation in (\ref{39}) is a standard quadratic form with respect to $D_1^i$, which means that the point ${D_1^i}^*=\frac{\frac{\sigma_1^2}{n}}{\delta_i^2+\frac{\sigma_1^2}{n}+\frac{\sigma_2^2}{m}}$  minimizes (\ref{39}).

In other words, if the tuning parameters $\lambda_1, \lambda_2, \lambda_3^i$ satisfy the condition
\begin{equation*}
{D_1^i}^*(\lambda_1,\lambda_2,\lambda_3^i)=\frac{\lambda_2\lambda_3^i}{4\lambda_1\lambda_2n+\lambda_2\lambda_3^i+\frac{n}{m}\lambda_1\lambda_3^i}=\frac{\frac{\sigma_1^2}{n}}{\delta_i^2+\frac{\sigma_1^2}{n}+\frac{\sigma_2^2}{m}},
\end{equation*}
which implies that $\lambda_1^*=\sigma_2^2, \lambda_2^*=\sigma_1^2, {\lambda_3^i}^*=\frac{4\sigma_1^2\sigma_2^2}{\delta_i^2}$,
we can then get the minimum value of $F(D_1^i):=\mathbb{E}(\bar{M}U_i+\bar{N}V_i+\beta_i)^2$ as follows
\begin{equation*}
F({D_1^i}^*)=\frac{\sigma_1^2}{n}-\frac{(\frac{\sigma_1^2}{n})^2}{\delta_i^2+\frac{\sigma_1^2}{n}+\frac{\sigma_2^2}{m}}>0
\end{equation*}
for $i=1,\cdots,k$.
\end{proof}

\begin{proof}{\bf of Corollary \ref{corollary 13}}
\\
Consider any relevant feature, say the $t$-th one, whose corresponding true regression coefficient is $\beta_t\neq0$. Let the tuning parameters $\lambda_1, \lambda_2, \lambda_3^t$ in the orthogonal TLCp be chosen to satisfy $\lambda_1^*=\sigma_2^2, \lambda_2^*=\sigma_1^2, {\lambda_3^t}^*=\frac{4\sigma_1^2\sigma_2^2}{\delta_t^2}$. Then,
we have
\begin{equation*}
(\tilde{D}^t)^*=\frac{\sigma_1^2\sigma_2^2}{\delta_t^2+\frac{\sigma_1^2}{n}+\frac{\sigma_2^2}{m}},
\end{equation*}
\begin{equation*}
M_t^*=\sigma_1^2m\frac{\delta_t^2+\frac{\sigma_1^2}{n}}{\delta_t^2+\frac{\sigma_1^2}{n}+\frac{\sigma_2^2}{m}},
\end{equation*}
\begin{equation*}
N_t^*=\sigma_2^2n\frac{\delta_t^2+\frac{\sigma_2^2}{m}}{\delta_t^2+\frac{\sigma_1^2}{n}+\frac{\sigma_2^2}{m}},
\end{equation*}
where $(\tilde{D}^t)^{*},M_t^*,N_t^*$ are obtained by substituting
 $\lambda_1^*,\lambda_2^*, {\lambda_3^t}^*$ into the expressions of $\tilde{D}^t, M_t, N_t$.
Therefore, we have
\begin{equation*}
\sqrt{M_t^*}+\sqrt{N_t^*}=\frac{\sigma_1\sqrt{m\delta_t^2+\frac{m}{n}\sigma_1^2}+\sigma_2\sqrt{n\delta_t^2+\frac{n}{m}\sigma_2^2}}{\sqrt{\delta_t^2+\frac{\sigma_1^2}{n}+\frac{\sigma_2^2}{m}}}
\end{equation*}
If we let $\delta_t=0$, then
\begin{equation}\label{40}
\sqrt{M_t^*}+\sqrt{N_t^*}=\sqrt{m\sigma_1^2+n\sigma_2^2},
\end{equation}
and
\begin{equation}\label{41}
 G_t^*=\sqrt{\frac{mn}{nM_t^*\sigma_2^2+mN_t^*\sigma_1^2}}=\frac{1}{\sigma_1\sigma_2}.
\end{equation}
%Recall that, if $\delta_t=0$, we have
%\begin{equation}\label{42}
 %\frac{Q_{t}}{\sqrt{M_{t}N_{t}}}=\frac{-2\tilde{D}^t}{\sqrt{M_{t}N_{t}}}=-2.
% \end{equation}
% Then substituting this result into the formula (\ref{14}), we can also have $\lambda_4^*=2\sigma_1^2\sigma_2^2$(where we set $\lambda=2\sigma_1^2$).

%Combining the results of (\ref{40}), (\ref{41}) together, we can conclude that, if $\delta_t=0$, then the selected parameters $\lambda_1^*,\lambda_2^*, {\lambda_3^t}^*$ satisfy the conditions (\ref{13}) given by Theorem \ref{tune parameters}.
In other words,
\begin{equation*}
F_1(\delta_t):=|\sqrt{n}\beta_{t}|-|G_{t}^*\sigma_1|\cdot\left|(\sqrt{M_{t}^*}+\sqrt{N_{t}^*})\beta_{t}+\sqrt{M_{t}^*}\delta_{t}\right|,
 \end{equation*}
where $M_{t}^*,N_{t}^*,G_t^*$ are all functions with respect to $\delta_t$, so by (\ref{40}) and (\ref{41}), we can see that $F_1(0)<0$. Further, due to the continuity of $F_1(\cdot)$, we can find a constant $\kappa(\sigma_1,\sigma_2, m,n)>0$, such that $F_1(\delta_t)<0$ holds for any $|\delta_t|\leq\kappa$.

So, by now, we have proved that, if we tune the parameters as $\lambda_1^*=\sigma_2^2, \lambda_2^*=\sigma_1^2, {\lambda_3^t}^*=\frac{4\sigma_1^2\sigma_2^2}{\delta_t^2}$, then for any relevant feature (whose corresponding coefficient is nonzero), there exists an constant $\kappa>0$, such that the formula (\ref{13}) in Theorem \ref{tune parameters} holds for any $|\delta_t|\leq\kappa$.

 Finally, since $\lambda_4^*=\min_{i\in\{1,\cdots,k\}}\left\{\lambda\left(2-\frac{Q_i^*}{\sqrt{M_i^*N_i^*}}\right)/
 4\sigma_1^2(G_i^*)^2\right\}$, where $Q_t^*=-2(\tilde{D}^t)^{*}$, satisfies (\ref{14}), by Theorem \ref{tune parameters}, we can conclude the desired results.
\end{proof}

\begin{proof}{\bf of Proposition \ref{proposition14}}
\\
As illustrated in Corollary \ref{TLCp scale-up}, the probability for the orthogonal TLCp to pick the $\gamma$-th feature if $\delta_\gamma=0$ is
\begin{align}
&Pr^{TLCp}\{\gamma\}=\frac{\sqrt{2}}{\sqrt{\pi}\sigma_1}\int_{4y^2>\frac{4\lambda_4\sigma_1^2G_\gamma^2}{2-\frac{Q_\gamma}{\sqrt{M_{\gamma}N_{\gamma}}}}}  ~~\exp\left\{-\frac{1}{2\sigma_1^2}\left[2y-G_\gamma(\sqrt{M_\gamma}+\sqrt{N_\gamma})\sigma_1\beta_\gamma\right]^2\right\}dy,\label{42}
\end{align}
where $M_\gamma, N_\gamma, Q_\gamma, G_\gamma$ are defined as previously.

We tune the model parameters of the orthogonal TLCp as $\lambda_1^*=\sigma_2^2, \lambda_2^*=\sigma_1^2, {\lambda_3^{\gamma}}^*=\infty$, and $\lambda_4^*=2\sigma_1^2\sigma_2^2$, where we assume $\lambda=2\sigma_1^2$. Also, we let $\sigma_1=\sigma_2=\sigma$. Then, we can rewrite the equality above as,
\begin{align}
&Pr^{TLCp}\{\gamma\}=\frac{\sqrt{2}}{\sqrt{\pi}\sigma_1}\int_{4y^2>2\sigma^2}  ~~\exp\left\{-\frac{1}{2\sigma^2}\left[2y-\sqrt{n+m}\beta_\gamma\right]^2\right\}dy,\notag
\end{align}
which is obtained by substituting (\ref{40}) and (\ref{41}) in the proof of Corollary \ref{corollary 13} into (\ref{42}).

By variable replacement, we obtain
\begin{align}
&Pr^{TLCp}\{\gamma\}=
\frac{\sqrt{n+m}}{\sqrt{2\pi}\sigma}\int_{(n+m)x^2>2\sigma^2}\exp\left\{-\frac{1}{2}\frac{\left(x-\beta_{\gamma}\right)^2}{\frac{\sigma^2}{n+m}}\right\}dx\notag \\
&=1-\int_{\frac{-\sqrt{n+m}\beta_\gamma-\sqrt{2\sigma^2}}{\sigma}}^{\frac{-\sqrt{n+m}\beta_\gamma+\sqrt{2\sigma^2}}{\sigma}}\frac{1}{\sqrt{2\pi}}\exp\left\{-\frac{z^2}{2}\right\}dz,\label{43}
\end{align}	
where the first equality is obtained by letting $y=\frac{\sqrt{n+m}x}{2}$, and the second one is obtained by substituting $z=\frac{\left(x-\beta_{\gamma}\right)}{\frac{\sigma}{\sqrt{n+m}}}$
into the first equality, together with the definition of probability.

Finally, to obtain the probability of the orthogonal TLCp to select $\gamma$-th feature whose regression coefficient satisfies $\beta_{\gamma}^2=\frac{2\sigma_1^2}{n}$, we can substitute the value of $\beta_{\gamma}^2$ into (\ref{43}) and perform similar derivations as above.
\end{proof}

\begin{proof}{\bf of Lemma \ref{lemma17}}
\\
If we set the parameters of orthogonal TLCp as $\lambda_1^*=\sigma_2^2, \lambda_2^*=\sigma_1^2, {\lambda_3^i}^*=\frac{4\sigma_1^2\sigma_2^2}{\delta_i^2}$, then the second term in the summation  of (\ref{MSE expression}) can be expressed as
	\begin{align}
	&\tilde{F}_i(\delta_i):=\notag\\&\iint_{\left(2-\frac{Q_i^*}{\sqrt{M_i^*N_i^*}}\right)x^2+\left(2+\frac{Q_i^*}{\sqrt{M_i^*N_i^*}}\right)y^2<\lambda_4}\left[\beta_i^2-(\bar{M}_i^*x+\bar{N}_i^*y+\beta_i)^2\right]p\left(U_i^*=x,V_i^*=y\right)dxdy\notag,
	\end{align}
due to the reason that $Q_i^*, M_i*, N_i^*, \bar{M}_i^*, U_i^*, V_i^*$ vary with $\delta_i$, for $i=1,\cdots,k$. Above, $Q_i^*, M_i^*, N_i^*, \bar{M}_i^*, U_i^*$, and $V_i^*$ represent the values of $Q_i, M_i, N_i, \bar{M}_i, U_i$, and $V_i$ after substituting $\lambda_1^*=\sigma_2^2, \lambda_2^*=\sigma_1^2, {\lambda_3^i}^*=\frac{4\sigma_1^2\sigma_2^2}{\delta_i^2}$ into the defining equations for $i=1,\cdots,k$.
	
To get an upper bound for $\tilde{F}_i(0)$, we begin to simplify the expression of $\tilde{F}(\delta_i)$, when $\delta_i=0$, for $i=1,\cdots,k$.
	
For $\delta_i=0$, we have $2+\frac{Q_i^*}{\sqrt{M_i^*N_i^*}}=0$, for $i=1,\cdots,k$. Then 
	\begin{equation*}
	\tilde{F}_i(0)=\iint_{\left(2-\frac{Q_i^*}{\sqrt{M_i^*N_i^*}}\right)x^2<\lambda_4}\left[\beta_i^2-(\bar{M}_i^*x+\bar{N}_i^*y+\beta_i)^2\right]p\left(U_i^*=x,V_i^*=y\right)dxdy \notag\\
	\end{equation*}
By the proof of Corollary \ref{TLCp scale-up}(see (\ref{28})), when $\delta_i=0$, we have 
	\begin{equation}\label{44}
	p(U_i^*=x)=\frac{\sqrt{2}G_i^*}{\sqrt{\pi}}\exp\left\{-\frac{({G_i}^*)^2}{2}\left[2x-(\sqrt{M_i^*}+\sqrt{N_i^*})\beta_i\right]^2\right\},
	\end{equation}
	where $G_i^*=\sqrt{\frac{mn}{nM_i^*\sigma_2^2+mN_i^*\sigma_1^2}}$, for $i=1,\cdots, k$.
	Further, we notice that $\bar{N}_i^*=0$, if $\delta_i=0$ holds for $i=1,\cdots, k$.
Therefore, by substituting (\ref{44}) into the expression for $\tilde{F}_i(0)$, we obtain
	\begin{align}\label{45}
	&\tilde{F}_i(0)=\notag\\&\frac{\sqrt{2}G_i^*}{\sqrt{\pi}}\int_{\left(2-\frac{Q_i^*}{\sqrt{M_i^*N_i^*}}\right)x^2<\lambda_4}\left[\beta_i^2-(\bar{M}_i^*x+\beta_i)^2\right]\exp\left\{-\frac{({G_i}^*)^2}{2}\left[2x-(\sqrt{M_i^*}+\sqrt{N_i^*})\beta_i\right]^2\right\}dx,
	\end{align}
for $i=1,\cdots, k$.
	
Notice that $\delta_i=0$ for $i=1,\cdots,k$, implies that the learning tasks in the target and source domains are equivalent, except possibly for the noise in the datasets. Thus, we can expect similar structure in the second term in the summation of the expression for the MSE of orthogonal TLCp estimator and that of the orthogonal Cp estimator. In view of this, we are going to use variable replacement to make the expression for $\tilde{F}_i(0)$ similar to that for $\tilde{H}_i:=\frac{1}{\sqrt{2\pi}\sigma_1}\int_{(y+\sqrt{n}\beta_i)^2<\lambda}\left(\beta_i^2-\frac{1}{n}y^2\right)\exp{\left\{-\frac{y^2}{2\sigma_1^2}\right\}}dy$, for $i=1,\cdots, k$.

Specifically, we substitute $y=\sigma_1G_i^*x$ into (\ref{45}) and obtain
	\begin{align}
	&\tilde{F}_i(0)=\notag\\&\frac{\sqrt{2}}{\sqrt{\pi}\sigma_1}\int_{y^2<\frac{\lambda_4\sigma_1^2(G_i^*)^2}{\left(2-\frac{Q_i^*}{\sqrt{M_i^*N_i^*}}\right)}}\left[\beta_i^2-(\frac{\bar{M}_i^*y}{\sigma_1G_i^*}+\beta_i)^2\right]\exp{\left\{-\frac{1}{2\sigma_1^2}\left[2y-G_i^*(\sqrt{M_i^*}+\sqrt{N_i^*})\sigma_1\beta_i\right]^2\right\}}dy,\notag
	\end{align}
for $i=1,\cdots, k$.
Letting $z=2y-G_i^*(\sqrt{M_i^*}+\sqrt{N_i^*})\sigma_1\beta_i$, we can further rewrite $\tilde{F}_i(0)$ as
	\begin{equation*}
	\tilde{F}_i(0)=\frac{\sqrt{2}}{2\sqrt{\pi}\sigma_1}\int_{(z+G_i^*(\sqrt{M_i^*}+\sqrt{N_i^*})\sigma_1\beta_i)^2<\frac{4\lambda_4\sigma_1^2(G_i^*)^2}{2-\frac{Q_i^*}{\sqrt{M_i^*N_i^*}}}}\tilde{g}(z,\delta_i,\sigma_1,\sigma_2,m,n)\exp{\left\{-\frac{z^2}{2\sigma_1^2}\right\}}dz,
	\end{equation*}
where $\tilde{g}(z,\delta_i,\sigma_1,\sigma_2,m,n)=\beta_i^2-\left[\frac{(z+G_i^*(\sqrt{M_i^*}+\sqrt{N_i^*})\sigma_1\beta_i)\bar{M}_i^*}{2\sigma_1G_i^*}+\beta_i\right]^2$, for $i=1,\cdots, k$.
	
When $\delta_i=0$, it is easy to check that $\frac{\bar{M}_i^*(\sqrt{M_i^*+N_i^*})}{2}=-1$, and also  $\frac{\bar{M}_i^*}{4\sigma_1^2(G_i^*)^2}=\frac{\sigma_2^2}{m\sigma_1^2+n\sigma_2^2}$, for $i=1,\cdots, k$. Thus, we have  $\tilde{g}(z,0,\sigma_1,\sigma_2,m,n)=\beta_i^2-\frac{\sigma_2^2}{m\sigma_1^2+n\sigma_2^2}z^2$, for $i=1,\cdots, k$. Additionally, we can verify that $G_i^*(\sqrt{M_i^*}+\sqrt{N_i^*})\sigma_1=\frac{\sqrt{m\sigma_1^2+n\sigma_2^2}}{\sigma_2}$, $i=1,\cdots, k$. Therefore, we can finally simplify $\tilde{F}_i(0)$ as follows
	\begin{equation}\label{46}
	\tilde{F}_i(0)=\frac{1}{\sqrt{2\pi}\sigma_1}\int_{\left(z+\frac{\sqrt{m\sigma_1^2+n\sigma_2^2}}{\sigma_2}\beta_i\right)^2<\lambda}\left[\beta_i^2-\frac{\sigma_2^2}{m\sigma_1^2+n\sigma_2^2}z^2\right]\exp{\left\{-\frac{z^2}{2\sigma_1^2}\right\}}dz,
	\end{equation}
in which we set $\lambda_4^*=\min_{i\in\{1,\cdots,k\}}\left\{\lambda\left(2-\frac{Q_i^*}{\sqrt{M_i^*N_i^*}}\right)/4\sigma_1^2(G_i^*)^2\right\}$. We can see that (\ref{46}) has a structure similar to $\tilde{H}_i$, for $i=1,\cdots, k$.
	
Let us now define $\tilde{f}(z):=\left(\beta_i^2-\frac{z^2}{\tilde{M}}\right)\exp{\left\{-\frac{z^2}{2\sigma_1^2}\right\}}$, where $\tilde{M}=\frac{m\sigma_1^2+n\sigma_2^2}{\sigma_2^2}$. Then, the derivative of $\tilde{f}(z)$ is
	\begin{equation*}
	\frac{\tilde{f}(z)}{dz}=\left[-\frac{2}{\tilde{M}}-\frac{\beta_i^2}{\sigma_1^2}z+\frac{1}{\tilde{M}}z^2\right]\exp{\left\{-\frac{z^2}{2\sigma_1^2}\right\}}.
	\end{equation*}
If $\beta_i^2\geq \frac{2\sigma_1^2}{n}$, we can easily get the maximizer  of $\tilde{f}(z)$ as $-\sqrt{\tilde{M}}\beta_i+\text{sgn}(\beta_i)\sqrt{2\sigma_1^2}$ within the integral interval $(-\sqrt{\tilde{M}}\beta_i-\sqrt{2\sigma_1^2}, -\sqrt{\tilde{M}}\beta_i+\sqrt{2\sigma_1^2})$, where $\text{sgn}(\cdot)$ indicates the sign function which is defined as follows,
\begin{eqnarray}
\operatorname { sgn } ( x ) : = \left\{ \begin{array} { l l } { - 1 } & { \text { if } x < 0 } \\ { 0 } & { \text { if } x = 0 } \\ { 1 } & { \text { if } x > 0 } \end{array} \right.
\end{eqnarray}
Thus, the maximum value of $\tilde{f}(z)$ is $\left[\frac{2|\beta_i|\sqrt{2\sigma_1^2}}{\sqrt{\tilde{M}}}-\frac{2\sigma_1^2}{\tilde{M}}\right]\exp\left\{-\frac{\left(\sqrt{\tilde{M}}|\beta_i|-\sqrt{2\sigma_1^2}\right)^2}{2\sigma_1^2}\right\}$.

Finally, if we assume $\lambda=2\sigma_1^2$, we can obtain the following upper bound on $\tilde{F}_i(0)$:
\begin{eqnarray}
\tilde{F}_i(0)&\leq& \frac{1}{\sqrt{2\pi}\sigma_1}\int_{\left(z+\sqrt{\tilde{M}}\beta_i\right)^2<\lambda}\left[\frac{2|\beta_i|\sqrt{2\sigma_1^2}}{\sqrt{\tilde{M}}}-\frac{2\sigma_1^2}{\tilde{M}}\right]\exp\left\{-\frac{\left(\sqrt{\tilde{M}}|\beta_i|-\sqrt{2\sigma_1^2}\right)^2}{2\sigma_1^2}\right\}dz\notag\\
&=&\frac{2}{\sqrt{\pi}}\left[\frac{2|\beta_i|\sqrt{2\sigma_1^2}}{\sqrt{\tilde{M}}}-\frac{2\sigma_1^2}{\tilde{M}}\right]\exp\left\{-\frac{\left(\sqrt{\tilde{M}}|\beta_i|-\sqrt{2\sigma_1^2}\right)^2}{2\sigma_1^2}\right\}\notag.
\end{eqnarray}
\end{proof}

%Firstly, as illustrated in Proposition \ref{explicit rule} that if we tune the parameters $\lambda_1, \lambda_2, \lambda_3^i$ as $\lambda_1^*=\sigma_2^2, \lambda_2^*=\sigma_1^2, {\lambda_3^t}^*=\frac{4\sigma_1^2\sigma_2^2}{\delta_t^2}$, for $i=1,\cdots, k$, then we can guarantee that
%\begin{equation*}
%\mathbb{E}[(\bar{M}_i^*{U_i^*}+\bar{N}_i^*{V_i^*}+\beta_i)^2]<\sigma_1^2/n, i=1,\cdots, k,
%\end{equation*}
%where $\bar{M}_i^*=\frac{\sqrt{M_i^*}-\sqrt{N_i^*}}{\sqrt{M_i^*N_i^*}}(D_1^i)^*-\frac{1}{\sqrt{N_i^*}}$, $\bar{N}_i^*=\frac{\sqrt{M_i^*}+\sqrt{N_i^*}}{\sqrt{M_i^*N_i^*}}(D_1^i)^*-\frac{1}{\sqrt{N_i^*}}$ as denoted in  Theorem \ref{MSE TLCp}, for $i=1,\cdots, k$. Therefore, by Theorem \ref{MSE TLCp}, we can now only focus on the second summation term of formulas (\ref{MSE expression}) and (\ref{7}).

%Now, we are going to compare each summation term of formulas (\ref{7}) and (\ref{MSE expression}) in the special case of $\delta_i=0$ and parameters $\lambda_1, \lambda_2, \lambda_3^i, \lambda_4$ are tuned as $\lambda_1^*=\sigma_2^2, \lambda_2^*=\sigma_1^2, {\lambda_3^i}^*=\frac{4\sigma_1^2\sigma_2^2}{\delta_i^2}, \lambda_4^*=\min_{i\in\{1,\cdots,k\}}\left\{\lambda\left(2-\frac{Q_i^*}{\sqrt{M_i^*N_i^*}}\right)/4\sigma_1^2(G_i^*)^2\right\}$, for $i=1,\cdots,k$.
\begin{proof}{\bf of Theorem \ref{theorem18}}

In the first part of this theorem, we want to verify that, if %$\beta_i^2>\frac{2\sigma_1^2}{n}\left[1+\frac{1+\sqrt{1-4\tilde{K}_2^2+4\tilde{K}_1\tilde{K}_2}}{2\tilde{K}_2}\right]^2$, 
$\beta_i^2\geq \frac{2\sigma_1^2}{n}\left[1+\sqrt{-\ln\left(\frac{\sqrt{\pi}}{8}K\right)}\right]^2$,
where $K=\frac{\frac{\sigma_1^2}{n}}{\frac{\sigma_1^2}{n}+\frac{\sigma_2^2}{m}}$, %$\tilde{K}_1=\frac{\tilde{G}}{4\sqrt{n}}-\frac{\sqrt{n}}{4}+\frac{1}{2}$, $\tilde{K}_2=\frac{\sqrt{\pi}\sigma_1^2m}{16\sigma_2^2\sqrt{n}}$, and $\tilde{G}=\frac{m\sigma_1^2+n\sigma_2^2}{\sigma_2^2}$, 
then each summation term of (\ref{7}) is strictly greater than that of (\ref{MSE expression}), provided that $\|\delta\|_2\leq\kappa_1$, and $\kappa_1$ is a constant. That is, we want to prove that, for $\delta$ within the region $\|\delta\|_2\leq\kappa_1$, the following holds
\begin{equation}\label{48}
\mathbb{E}(\bar{M}^*U_i^*+\bar{N}^*V_i^*+\beta_i)^2+\tilde{F}_i(\delta_i) <\frac{\sigma_1^2}{n}+\tilde{H}_i,
\end{equation}
where $\tilde{H}_i=\frac{1}{\sqrt{2\pi}\sigma_1}\int_{(y+\sqrt{n}\beta_i)^2<\lambda}\left(\beta_i^2-\frac{1}{n}y^2\right)\exp{\left\{-\frac{y^2}{2\sigma_1^2}\right\}}dy$ and $\tilde{F}_i(0)$ is defined as in Lemma \ref{lemma17}, for $i=1,\cdots,k$. By the proof of Proposition \ref{explicit rule}, if we set the parameters $\lambda_1, \lambda_2, \lambda_3^i$ to $\lambda_1^*=\sigma_2^2, \lambda_2^*=\sigma_1^2, {\lambda_3^i}^*=\frac{4\sigma_1^2\sigma_2^2}{\delta_i^2}$, for $i=1,\cdots, k$, then $\mathbb{E}(\bar{M}^*U_i^*+\bar{N}^*V_i^*+\beta_i)^2=\frac{\sigma_1^2}{n}-\frac{(\frac{\sigma_1^2}{n})^2}{\delta_i^2+\frac{\sigma_1^2}{n}+\frac{\sigma_2^2}{m}}$, for $i=1,\cdots, k$.
Therefore, (\ref{48}) amounts to
\begin{equation}\label{49}
\tilde{F}_i(\delta_i)-\frac{(\frac{\sigma_1^2}{n})^2}{\delta_i^2+\frac{\sigma_1^2}{n}+\frac{\sigma_2^2}{m}}<\tilde{H}_i,
\end{equation}
for $i=1,\cdots, k$.

Now, we'd like to find a region for $\delta$, such that inequality (\ref{49}) holds for $i=1,\cdots, k$. First, we are going to prove this inequality in the special case of $\delta=0$. Thus, we begin by analyzing the lower bound of $\tilde{H}_i$, for $i=1,\cdots, k$. By Proposition \ref{proposition5}, the minimum value of the function $f(x)=\left(\beta_i^2-\frac{1}{n}x^2\right)\exp\left\{-\frac{x^2}{2\sigma_1^2}\right\}$ is $f\left(\pm\sqrt{n\beta_i^2+2\sigma_1^2}\right)=-\frac{2\sigma_1^2}{n}\exp\left\{-\frac{n\beta_i^2}{2\sigma_1^2}-1\right\}$. Then, we have
\begin{eqnarray}
\tilde{H}_i &\geq& \frac{1}{\sqrt{2\pi}\sigma_1}\int_{(y+\sqrt{n}\beta_i)^2<\lambda}-\frac{2\sigma_1^2}{n}\exp\left\{-\frac{n\beta_i^2}{2\sigma_1^2}-1\right\}dy\notag\\
&=&-\frac{2}{\sqrt{\pi}}\cdot \frac{2\sigma_1^2}{n}\exp\left\{-\frac{n\beta_i^2}{2\sigma_1^2}-1\right\}\notag,
\end{eqnarray}
for $i=1,\cdots, k$.

Due to the assumption that $\beta_i^2\geq \frac{2\sigma_1^2}{n}\left[1+\sqrt{-\ln\left(\frac{\sqrt{\pi}}{8}K\right)}\right]^2$ in this case, where $K=\frac{\frac{\sigma_1^2}{n}}{\frac{\sigma_1^2}{n}+\frac{\sigma_2^2}{m}}$, by Lemma \ref{lemma17}, we can obtain the following upper bound for $\tilde{F}_i(0)$:
\begin{equation*}
\tilde{F}_i(0)\leq \frac{2}{\sqrt{\pi}}\left[\frac{2|\beta_i|\sqrt{2\sigma_1^2}}{\sqrt{\tilde{M}}}-\frac{2\sigma_1^2}{\tilde{M}}\right]\exp\left\{-\frac{\left(\sqrt{\tilde{M}}|\beta_i|-\sqrt{2\sigma_1^2}\right)^2}{2\sigma_1^2}\right\}\notag,
\end{equation*}
where $\tilde{M}=\frac{m\sigma_1^2+n\sigma_2^2}{\sigma_2^2}$, for $i=1,\cdots, k$.

Combining these inequalities, if we can prove the relationships below,
\begin{align}
&\frac{2}{\sqrt{\pi}}\left[\frac{2|\beta_i|\sqrt{2\sigma_1^2}}{\sqrt{\tilde{M}}}-\frac{2\sigma_1^2}{\tilde{M}}\right]\exp\left\{-\frac{\left(\sqrt{\tilde{M}}|\beta_i|-\sqrt{2\sigma_1^2}\right)^2}{2\sigma_1^2}\right\}\notag\\&<\frac{(\frac{\sigma_1^2}{n})^2}{\frac{\sigma_1^2}{n}+\frac{\sigma_2^2}{m}}-\frac{2}{\sqrt{\pi}}\cdot \frac{2\sigma_1^2}{n}\exp\left\{-\frac{n\beta_i^2}{2\sigma_1^2}-1\right\},\label{50}
\end{align}
for $i=1,\cdots, k$, then the desired inequality (\ref{49}) holds naturally.
To this end, in the following, we investigate the difference between the right-hand and left-hand sides of inequality (\ref{50}).

For simplicity, we denote the right-hand side of (\ref{50}) by $\tilde{G}_r^{i}$, the left-hand side by $\tilde{G}_l^{i}$. If $\beta_i^2\geq \frac{2\sigma_1^2}{n}\left[1+\sqrt{-\ln\left(\frac{\sqrt{\pi}}{8}K\right)}\right]^2$, where $K=\frac{\frac{\sigma_1^2}{n}}{\frac{\sigma_1^2}{n}+\frac{\sigma_2^2}{m}}$, which also implies that $\beta_i^2\geq \frac{2\sigma_1^2}{n}$ for $i=1,\cdots, k$, then the following holds:
\begin{eqnarray}
\tilde{G}_r^{i}-\tilde{G}_l^{i}&\geq& \frac{(\frac{\sigma_1^2}{n})^2}{\frac{\sigma_1^2}{n}+\frac{\sigma_2^2}{m}}-\frac{2}{\sqrt{\pi}}\cdot \frac{2\sigma_1^2}{n}\exp\left\{-\frac{n\beta_i^2}{2\sigma_1^2}-1+\frac{\sqrt{2n}|\beta_i|}{\sigma_1}\right\}-\tilde{G}_l\notag\\&\geq&\frac{(\frac{\sigma_1^2}{n})^2}{\frac{\sigma_1^2}{n}+\frac{\sigma_2^2}{m}}-\frac{2}{\sqrt{\pi}}\left[\frac{2\sigma_1^2}{n}+\frac{2|\beta_i|\sqrt{2\sigma_1^2}}{\sqrt{\tilde{M}}}-\frac{2\sigma_1^2}{\tilde{M}}\right]\exp\left\{-\frac{\left(\sqrt{n}|\beta_i|-\sqrt{2\sigma_1^2}\right)^2}{2\sigma_1^2}\right\}\notag\\
&\geq&\frac{(\frac{\sigma_1^2}{n})^2}{\frac{\sigma_1^2}{n}+\frac{\sigma_2^2}{m}}-\frac{2}{\sqrt{\pi}}\left[\frac{2\sigma_1^2}{n}+\frac{2|\beta_i|\sqrt{2\sigma_1^2}}{\sqrt{\tilde{M}}}-\frac{2\sigma_1^2}{\tilde{M}}\right]\frac{\sqrt{\pi}}{8}K\notag\\&=&\frac{\frac{\sigma_1^2}{n}}{\frac{\sigma_1^2}{n}+\frac{\sigma_2^2}{m}}\left[\frac{\sigma_1^2}{2n}-\frac{|\beta_i|\sqrt{2\sigma_1^2}}{2\sqrt{\tilde{M}}}+\frac{\sigma_1^2}{2\tilde{M}}\right]\notag\\
&\geq&\frac{\frac{\sigma_1^2}{n}}{\frac{\sigma_1^2}{n}+\frac{\sigma_2^2}{m}}\left(\frac{\sigma_1^2}{2n}-\frac{\sigma_1^2}{2\tilde{M}}\right)>0\notag,
\end{eqnarray}
for $i=1,\cdots, k$.
The first inequality is obtained due to $\beta_i^2\geq \frac{2\sigma_1^2}{n}$, the third inequality holds by substituting $\beta_i^2\geq \frac{2\sigma_1^2}{n}\left[1+\sqrt{-\ln\left(\frac{\sqrt{\pi}}{8}K\right)}\right]^2$ into the second equality, and the last inequality is obvious due to $n<\tilde{M}$.

So far, we have proved that the inequality (\ref{49}) holds under the special case of $\delta=0$. Due to the continuity of the function in the left-hand side of the inequality, there exists a radius $\kappa_1>0$, such that (\ref{49}) holds for any $\|\delta\|_2\leq\kappa_1$, for $i=1,\cdots,k$.

Thus, until now, we have verified that the MSE of the orthogonal TLCp estimator will be strictly less than that of orthogonal Cp estimator, provided that $\beta_i^2\geq \frac{2\sigma_1^2}{n}\left[1+\sqrt{-\ln\left(\frac{\sqrt{\pi}}{8}K\right)}\right]^2$, and $\|\delta\|_2\leq\kappa_1$, for $i=1,\cdots, k$.

To prove the remaining this theorem, firstly, if we assume $\beta=0$, then the MSE of orthogonal Cp estimator $\hat{\boldsymbol{a}}$ in (\ref{7}) can be reduced to:
\begin{equation}\label{51}
\text{MSE}(\hat{\boldsymbol{a}})=\sum_{i=1}^{k}\frac{\sigma_1^2}{n}+\frac{1}{\sqrt{2\pi}\sigma_1}\int_{x<\lambda}-\frac{x^2}{n}\exp\left\{-\frac{x^2}{2\sigma_1^2}\right\}dx,
\end{equation}
and the MSE of orthogonal TLCp estimator $\hat{\boldsymbol{w}}_1$ in (\ref{MSE expression}) can be reduced to the following for $\delta=0$:
\begin{equation}\label{52}
\text{MSE}(\hat{\boldsymbol{w}}_1)=\sum_{i=1}^{k}\left[\frac{\sigma_1^2}{n}-\frac{(\frac{\sigma_1^2}{n})^2}{\frac{\sigma_1^2}{n}+\frac{\sigma_2^2}{m}}\right]+\frac{1}{\sqrt{2\pi}\sigma_1}\int_{z<\lambda}-\frac{\sigma_2^2}{m\sigma_1^2+n\sigma_2^2}z^2\exp{\left\{-\frac{z^2}{2\sigma_1^2}\right\}}dz.
\end{equation}

We are going to analyze the difference between the (\ref{51}) and (\ref{52}) as follows:
\begin{align}
&\text{MSE}(\hat{\boldsymbol{w}}_1)-\text{MSE}(\hat{\boldsymbol{a}})\notag\\&=\sum_{i=1}^{k}\left(-\frac{(\frac{\sigma_1^2}{n})^2}{\frac{\sigma_1^2}{n}+\frac{\sigma_2^2}{m}}+\frac{1}{\sqrt{2\pi}\sigma_1}\int_{z<\lambda}\left[-\frac{\sigma_2^2}{m\sigma_1^2+n\sigma_2^2}+\frac{1}{n}\right]x^2\exp\left\{-\frac{x^2}{2\sigma_1^2}\right\}dx\right)\notag\\
&<\sum_{i=1}^{k}\left(-\frac{(\frac{\sigma_1^2}{n})^2}{\frac{\sigma_1^2}{n}+\frac{\sigma_2^2}{m}}+\frac{1}{\sqrt{2\pi}\sigma_1}\int_{-\infty}^{+\infty}\left[-\frac{\sigma_2^2}{m\sigma_1^2+n\sigma_2^2}+\frac{1}{n}\right]x^2\exp\left\{-\frac{x^2}{2\sigma_1^2}\right\}dx\right)\notag\\
&=\sum_{i=1}^{k}\left(-\frac{(\frac{\sigma_1^2}{n})^2}{\frac{\sigma_1^2}{n}+\frac{\sigma_2^2}{m}}+\left[-\frac{\sigma_2^2}{m\sigma_1^2+n\sigma_2^2}+\frac{1}{n}\right]\mathbb{E}(X^2)\right)\notag\\
&=\sum_{i=1}^{k}\left(-\frac{(\frac{\sigma_1^2}{n})^2}{\frac{\sigma_1^2}{n}+\frac{\sigma_2^2}{m}}+\left[-\frac{\sigma_2^2}{m\sigma_1^2+n\sigma_2^2}+\frac{1}{n}\right]\sigma_1^2\right)=0\label{53},
\end{align}
where the third equality is obtained by the definition of the expected value of the random variable $X\sim \mathcal { N } \left( 0 , \sigma_1^ { 2 } \right)$.

If we let  $\hat{R}(\beta,\delta):=\text{MSE}(\hat{\boldsymbol{w}}_1)-\text{MSE}(\hat{\boldsymbol{a}})$, then (\ref{53}) indicates that $\hat{R}(0, 0)<0$. By the continuity of $\hat{R}(\beta, \delta)$ with respect to $\beta$ and $\delta$, there exist two constants $\kappa_2(\sigma_1,\sigma_2,n,m)>0$ and $\rho(\sigma_1,\sigma_2,n,m)>0$ such that  $\hat{R}(\beta, \delta)<0$ holds for any $\|\beta\|_2<\rho$, $\|\delta\|_2<\kappa_2$.

Finally, combining the results from the above two parts, and letting $\tilde{\kappa}=\min\{\kappa_1,\kappa_2\}$, we can conclude that when the parameters of orthogonal TLCp are tuned as $\lambda_1^*=\sigma_2^2, \lambda_2^*=\sigma_1^2, {\lambda_3^i}^*=\frac{4\sigma_1^2\sigma_2^2}{\delta_i^2}$ and $\lambda_4^*=\min_{i\in\{1,\cdots,k\}}\left\{\lambda\left(2-\frac{Q_i^*}{\sqrt{M_i^*N_i^*}}\right)/ 4\sigma_1^2(G_i^*)^2\right\}$ for $i=1,\cdots, k$, then the MSE of the orthogonal TLCp estimator will be strictly less than that of the orthogonal Cp estimator, provided that $\|\delta\|_2<\tilde{\kappa}$ and
\begin{equation*}
\beta_i^2>\frac{2\sigma_1^2}{n}\left[1+\sqrt{-\ln\left(\frac{\sqrt{\pi}K}{8}\right)}\right]^2 \quad \text{or} \quad \beta_i^2<\rho^2, \quad \text{for } i=1,\cdots, k.
\end{equation*}
\end{proof}	

%\begin{proof}{\bf of Lemma \ref{lemma19}}
%\\
%To prove this lemma, we first verify the result that if the matrix $\boldsymbol{\bar{X}}$ is full column rank, then there is an invertible $k$ by $k$ matrix $\boldsymbol{A}$ such that $\boldsymbol{\bar{X}}^{\top}\boldsymbol{\bar{X}}=\boldsymbol{A}^{\top}\boldsymbol{A}$.

%In fact, $\boldsymbol{\bar{X}}^{\top}\boldsymbol{\bar{X}}$ is a positive definite matrix which is due to the full column rank property of $\boldsymbol{\bar{X}}$. Thus, there exists an invertible matrix $\boldsymbol{P}$, such that $\boldsymbol{P}^{\top}\boldsymbol{\bar{X}}^{\top}\boldsymbol{\bar{X}P}=\text{diag}(\lambda_1,\cdots, \lambda_k)$, where $\lambda_i>0$, for $i=1,\cdots, k$. Letting $\boldsymbol{A}=\boldsymbol{P}^{-1}\text{diag}(\sqrt{\lambda_1},\cdots,\sqrt{\lambda_k})$, then we have $\boldsymbol{\bar{X}}^{\top}\boldsymbol{\bar{X}}=\boldsymbol{A}^{\top}\boldsymbol{A}$.

%Letting $\boldsymbol{Q}=\boldsymbol{A}^{-1}\sqrt{n}$, then  $\boldsymbol{Q}^{\top}\boldsymbol{\bar{X}}^{\top}\boldsymbol{\bar{X}Q}=n(\boldsymbol{A}^{-1})^{\top}\boldsymbol{A}^{\top}\boldsymbol{A}\boldsymbol{A}^{-1}=n(\boldsymbol{A}^{\top})^{-1}\boldsymbol{A}^{\top}\boldsymbol{A}\boldsymbol{A}^{-1}=n\boldsymbol{I}$, which is the desired result.
%\end{proof}

\begin{proof}{\bf of Proposition \ref{proposition20}}
\\
To get the optimal solution of the orthogonalized Cp problem (\ref{orthogonalized Cp}), notice that $\boldsymbol{Q}^{\top}\boldsymbol{\bar{X}}^{\top}\boldsymbol{\bar{X}Q}=n\boldsymbol{I}$ and $\boldsymbol { \bar{y} } = \boldsymbol { \bar{X}Q } (\boldsymbol{Q}^{-1}\boldsymbol { \beta }) + \boldsymbol { \varepsilon }$. Then, we can rewrite the orthogonalized Cp problem (\ref{orthogonalized Cp}) as follows
	\begin{equation}\label{57}
	\boldsymbol{\hat{\alpha}}_1=\text{argmin}_{\boldsymbol{\alpha}_1}~ \sum_{i=1}^{k} \left\{-2n\boldsymbol{\beta}^{\top}\tilde{Q}_i\alpha_1^{i}-2\boldsymbol { \varepsilon }^{\top}Z_i\alpha_1^{i}+n(\alpha_1^{i})^2+\lambda\|\alpha_1^{i}\|_0\right\},
	\end{equation}
where $\tilde{Q_i}^{\top}$ represents the $i$-th row of the invertible matrix $\boldsymbol{Q}^{-1}$, for $i=1,\cdots,k$, and $Z_j$ is the $j$-th column of the design matrix $\boldsymbol{\bar{X}Q}$, for $j=1,\cdots,k$.
	
Due to the independence of each summand in the objective function in (\ref{57}), the orthogonalized Cp problem (\ref{orthogonalized Cp}) can be solved by solving the following $k$ one-dimensional optimization problems:
		\begin{equation}\label{58}
\hat{\alpha}_1^i=\text{argmin}_{\alpha_1^i}~\bar{g}(\alpha_1^i) \triangleq \left\{-2n\boldsymbol{\beta}^{\top}\tilde{Q}_i\alpha_1^{i}-2\boldsymbol { \varepsilon }^{\top}Z_i\alpha_1^{i}+n(\alpha_1^{i})^2+\lambda\|\alpha_1^{i}\|_0\right\},
	\end{equation}
where $i=1,\cdots,k$.
	
For the $i$-th subproblem, if $\|\alpha_1^{i}\|_0=1$, then the gradient of $\bar{g}(\alpha_1^i)$ vanishes at $	\zeta_i=\tilde{Q_i}^{\top}\boldsymbol { \beta }+\frac{Z_i^{\top}\boldsymbol { \varepsilon }}{n}$, and $\bar{g}(\zeta_i)=-n\left(\tilde{Q_i}^{\top}\boldsymbol { \beta }+\frac{Z_i^{\top}\boldsymbol { \varepsilon }}{n}\right)^2+\lambda$. If $\|\alpha_1^{i}\|_0=0$, we have $\alpha_1^{i}=0$, and $\bar{g}(0)=0$. Comparing the objective values in these two cases and picking the smaller one, we obtain the following expression for the optimal solution of the $i$-th problem:
	 \begin{eqnarray}
	\hat{\alpha}_1^{i}=\left\{
	\begin{matrix}
	\tilde{Q_i}^{\top}\boldsymbol { \beta }+\frac{Z_i^{\top}\boldsymbol { \varepsilon }}{n},
	& \text{if~} n\left[\tilde{Q_i}^{\top}\boldsymbol { \beta }+\frac{Z_i^{\top}\boldsymbol { \varepsilon }}{n}\right]^2>\lambda\\
	0,
	&  \text{otherwise}
	\end{matrix}
	\right.
	\end{eqnarray}
for $i=1,\cdots,k$. 
\end{proof}

\begin{proof}{\bf of Theorem \ref{theorem21}}
\\
Notice that the back-transformed orthogonalized Cp estimator $\hat{\boldsymbol{\alpha}}_2=\boldsymbol{Q}\hat{\boldsymbol{\alpha}}_1$, where $\hat{\boldsymbol{\alpha}}_1$ is the solution of orthogonalized Cp problem, can be rewritten as follows,
\begin{eqnarray}
\hat{\boldsymbol{\alpha}}_2&=&\sum_{i=1}^{k}Q_i\hat{\alpha}_1^{i}\notag\\
&=&\sum_{i=1}^{k}Q_i\left(\tilde{Q_i}^{\top}\boldsymbol{\beta}+\frac{Z_i^{\top}\boldsymbol { \varepsilon }}{n}\right)\cdot\boldsymbol{1}_{A_i},\label{60}
\end{eqnarray}
where we denote $\boldsymbol{1}_{A_i}$ as the indication function with respect to the random variable set $A_i=\left\{\frac{Z_i^{\top}\boldsymbol { \varepsilon }}{n}:n\left[\tilde{Q_i}^{\top}\boldsymbol{\beta}+\frac{Z_i^{\top}\boldsymbol { \varepsilon }}{n}\right]^2>\lambda\right\}$, for $i=1,\cdots,k$.
To analyze $\hat{\boldsymbol{\alpha}}_2$, we decompose (\ref{60}) into two terms as below,
\begin{equation}\label{61}
\hat{\boldsymbol{\alpha}}_2=\sum_{i=1}^{k}\left\{Q_i\tilde{Q_i}^{\top}\cdot \boldsymbol{\beta}\cdot\boldsymbol{1}_{A_i}\right\}+\sum_{i=1}^{k}\left\{Q_i\frac{Z_i^{\top}\boldsymbol { \varepsilon }}{n}\cdot\boldsymbol{1}_{A_i}\right\}.
\end{equation}	
Then, we are going to estimate these two terms. For the first summand of (\ref{61}), there holds
\begin{align*}
&\left\|\sum_{i=1}^{k}\left\{Q_i\tilde{Q_i}^{\top}\cdot \boldsymbol{\beta}\cdot\boldsymbol{1}_{A_i}\right\}-\boldsymbol{\beta}\right\|_2\\&=\left\|\sum_{i=1}^{k}\left\{Q_i\tilde{Q_i}^{\top}\cdot \boldsymbol{\beta}\cdot\boldsymbol{1}_{A_i}\right\}-\sum_{i=1}^{k}\left\{Q_i\tilde{Q_i}^{\top}\boldsymbol{\beta}\right\}\right\|_2\\&=\left\|  \sum_{i=1}^{k}(\boldsymbol{1}_{A_i}-1)\cdot Q_i\tilde{Q_i}^{\top}\boldsymbol{\beta}\right\|_2\\&\leq M\sum_{i=1}^{k}\left|\boldsymbol{1}_{A_i}-1\right|,
\end{align*}
where $M=\text{max}_{i=1,\cdots,k}\left\{\|Q_i\tilde{Q_i}^{\top}\boldsymbol{\beta}\|_2\right\}$ and the first equality can be obtained by noticing that $\sum_{i=1}^{k}Q_i\tilde{Q_i}^{\top}=\boldsymbol{I}$.
Further, for any $1>\epsilon>0$, we have
\begin{align*}
\left\{M\sum_{i=1}^{k}\left|\boldsymbol{1}_{A_i}-1\right|>\epsilon\right\}\subseteq \left\{\exists ~i, s.t. \left|\boldsymbol{1}_{A_i}-1\right|>\frac{\epsilon}{kM}\right\}.
\end{align*}
Moreover, by the definition of indicator $\boldsymbol{1}_{A_i} $ above, there holds
\begin{align}\label{62}
P_r\left\{\exists ~i, s.t. \left|\boldsymbol{1}_{A_i}-1\right|>\frac{\epsilon}{kM}\right\}\leq \sum_{i=1}^{k}P_r\left\{\left[\tilde{Q_i}^{\top}\boldsymbol{\beta}+\frac{Z_i^{\top}\boldsymbol { \varepsilon }}{n}\right]^2\leq \frac{\lambda}{n}\right\}.
\end{align}
To evaluate the right-hand side of (\ref{62}), if there exist $\tilde{Q_i}^{\top}\boldsymbol{\beta}\neq0$, then for any large enough $n$, we have
\begin{align*}
\sum_{i=1}^{k}P_r\left\{\left[\tilde{Q_i}^{\top}\boldsymbol{\beta}+\frac{Z_i^{\top}\boldsymbol { \varepsilon }}{n}\right]^2\leq \frac{\lambda}{n}\right\}\leq \sum_{i=1}^{k}P_r\left\{\left|\frac{Z_i^{\top}\boldsymbol { \varepsilon }}{n}\right|\geq|\tilde{Q_i}^{\top}\boldsymbol{\beta}|-\sqrt{\frac{\lambda}{n}}\right\}.
\end{align*}
Note that
\begin{equation*}
\mathbb{E}\left(\frac{Z_i^{\top}\boldsymbol { \varepsilon }}{n} \right)^2=\frac{1}{n^2}\mathbb{E}\left(Z_i^{\top}\boldsymbol{\varepsilon}\boldsymbol{\varepsilon}^{\top}Z_i\right)=\frac{1}{n^2}\mathbb{E}\text{trace}(Z_iZ_i^{\top}\boldsymbol{\varepsilon}\boldsymbol{\varepsilon}^{\top})=\frac{\sigma^2}{n^2}\text{trace}(Z_i^{\top}Z_i)=\frac{\sigma^2}{n},
\end{equation*}
for $i=1,\cdots,k$. Thus, we have $\frac{Z_i^{\top}\boldsymbol { \varepsilon }}{n} \sim \mathcal { N } \left( 0 , \frac{\sigma^2}{n} \right)$, for $i=1,\cdots,k$.\\
By Chebyshev's Inequality, there holds
\begin{align*}
\sum_{i=1}^{k}P_r\left\{\left|\frac{Z_i^{\top}\boldsymbol { \varepsilon }}{n}\right|\geq|\tilde{Q_i}^{\top}\boldsymbol{\beta}|-\sqrt{\frac{\lambda}{n}}\right\}\leq \frac{\frac{k\sigma^2}{n}}{\left(W-\sqrt{\frac{\lambda}{n}}\right)^2},
\end{align*}
where $W=\text{min}_{i=1,\cdots,k}\left\{|\tilde{Q_i}^{\top}\boldsymbol{\beta}|\neq0\right\}$.\\
Therefore, for any $1>\epsilon>0$, if there exist $\tilde{Q_i}^{\top}\boldsymbol{\beta}\neq0$, then for any large enough $n$, there holds
\begin{align}\label{63}
&P_r\left\{\left\|\sum_{i=1}^{k}\left\{Q_i\tilde{Q_i}^{\top}\cdot \boldsymbol{\beta}\cdot\boldsymbol{1}_{A_i}\right\}-\boldsymbol{\beta}\right\|_2>\epsilon\right\}\notag\\&\leq P_r\left\{M\sum_{i=1}^{k}\left|\boldsymbol{1}_{A_i}-1\right|>\epsilon\right\}\\&\leq\frac{\frac{k\sigma^2}{n}}{\left(W-\sqrt{\frac{\lambda}{n}}\right)^2}\notag.
\end{align}
As for the case $\tilde{Q_i}^{\top}\boldsymbol{\beta}=0$ ($i=1,\cdots,k$), then the desired result holds naturally. \\
For the second summand of the (\ref{61}), due to
\begin{equation}\notag
\left|\sum_{i=1}^{k}\left\{Q_i\frac{Z_i^{\top}\boldsymbol { \varepsilon }}{n}\cdot\boldsymbol{1}_{A_i}\right\}\right|\leq \sum_{i=1}^{k}\left|Q_i\frac{Z_i^{\top}\boldsymbol { \varepsilon }}{n}\right|\leq\sum_{i=1}^{k}\tilde{M} \left|\frac{Z_i^{\top}\boldsymbol { \varepsilon }}{n}\right|,
\end{equation}
where $\tilde{M}=\text{max}_{i=1,\cdots,k}\left\{\left\|Q_i\right\|_2\right\}$. %Notice that $\frac{Z_i^{\top}\boldsymbol { \varepsilon }}{n} \sim \mathcal { N } \left( 0 , \frac{1}{n} \right)$, for $i=1,\cdots,k$.
By Chebyshev's Inequality, for any $1>\epsilon>0$, we have
\begin{align*}
P_r\left\{\sum_{i=1}^{k}\left|\frac{Z_i^{\top}\boldsymbol { \varepsilon }}{n}\right|\geq \frac{k\epsilon}{\tilde{M}} \right\}\leq\sum_{i=1}^{k}P_r\left\{\left|\frac{Z_i^{\top}\boldsymbol { \varepsilon }}{n}\right|>\frac{\epsilon}{\tilde{M}} \right\}\leq \frac{\tilde{M}^2k}{\epsilon^2}\frac{\sigma^2}{n},
\end{align*}
cause $\frac{Z_i^{\top}\boldsymbol { \varepsilon }}{n} \sim \mathcal { N } \left( 0 , \frac{\sigma^2}{n} \right)$, for $i=1,\cdots,k$.
Thus, there holds
\begin{align}\label{64}
&P_r\left\{\left|\sum_{i=1}^{k}\left\{Q_i\frac{Z_i^{\top}\boldsymbol { \varepsilon }}{n}\cdot\boldsymbol{1}_{A_i}\right\}\right|>k\epsilon \right\}\notag\\&\leq P_r\left\{\sum_{i=1}^{k}\left|\frac{Z_i^{\top}\boldsymbol { \varepsilon }}{n}\right|\geq \frac{k\epsilon}{\tilde{M}} \right\}\\&\leq \frac{\tilde{M}^2k}{\epsilon^2}\frac{\sigma^2}{n},\notag
\end{align}
{where the second line is obtained by the triangle inequality and  $\tilde{M}=\text{max}_{i=1,\cdots,k}\left\{\left\|Q_i\right\|_2\right\}$.}
Finally, combine the results of (\ref{63}) and (\ref{64}), we have
\begin{align}\label{65}
&P_r\left\{\|\hat{\boldsymbol{\alpha}}_2-\boldsymbol{\beta}\|_2>(k+1)\epsilon\right\}\notag\\&\leq
P_r\left\{\left\|\sum_{i=1}^{k}\left\{Q_i\tilde{Q_i}^{\top}\cdot \boldsymbol{\beta}\cdot\boldsymbol{1}_{A_i}\right\}-\boldsymbol{\beta}\right\|_2>\epsilon\right\}+P_r\left\{\left|\sum_{i=1}^{k}\left\{Q_i\frac{Z_i^{\top}\boldsymbol { \varepsilon }}{n}\cdot\boldsymbol{1}_{A_i}\right\}\right|>k\epsilon \right\}\notag\\&\leq
\frac{\frac{k\sigma^2}{n}}{\left(W-\sqrt{\frac{\lambda}{n}}\right)^2}+\frac{\tilde{M}^2k}{\epsilon^2}\frac{\sigma^2}{n}
\end{align}
For every $1>\eta>0$, let $\eta=\frac{\frac{k\sigma^2}{n}}{\left(W-\sqrt{\frac{\lambda}{n}}\right)^2}+\frac{\tilde{M}^2k}{\epsilon^2}\frac{\sigma^2}{n}$, which means $\epsilon=\sqrt{\frac{\tilde{M}^2k}{\eta-\frac{\frac{k\sigma^2}{n}}{\left(W-\sqrt{\frac{\lambda}{n}}\right)^2}}}\sqrt{\frac{\sigma^2}{n}}$. Then, based on (\ref{65}), with probability at least $1-\eta$, we have
\begin{align*}
\|\hat{\boldsymbol{\alpha}}_2-\boldsymbol{\beta}\|_2\leq (k+1)\sqrt{\frac{\tilde{M}^2k}{\eta-\frac{\frac{k\sigma^2}{n}}{\left(W-\sqrt{\frac{\lambda}{n}}\right)^2}}}\sqrt{\frac{\sigma^2}{n}}.
\end{align*}
That is, the desired result holds.
 \end{proof}

\begin{proof}{\bf of Theorem \ref{theorem22}}
	\\
First, using the non-orthogonal Cp is equivalent to solving the following problem,
\begin{equation}\label{Cp2}
\hat{\mathcal{J}}=\text{argmin}_{\mathcal{J} \in \mathcal{A}}~ \frac{1}{n}[ \boldsymbol { \bar{y} } - \boldsymbol { \bar{X} }(\mathcal{J})  \hat{\boldsymbol { \beta }} (\mathcal{J} ) ]^ { \top } [ \boldsymbol { \bar{y} } - \boldsymbol { \bar{X} }(\mathcal{J})  \hat{\boldsymbol { \beta }} (\mathcal{J} ) ]+\frac{\lambda p(\mathcal{J})}{n},
\end{equation}
where $\boldsymbol{\beta}(\mathcal{J})$ or ($\bar{\boldsymbol{X}}(\mathcal{J})$) contains the components of $\boldsymbol{\beta}$ (or columns of $\bar{\boldsymbol{X}}$) that are indexed by the integers in $\mathcal{J}$, and $p(\mathcal{J})$ is the number of elements of the considering index set $\mathcal{J}$. Also, $\hat{\boldsymbol{\beta}}(\mathcal{J})$ is the least squares estimator of the true regression coefficient vector $\boldsymbol{\beta}$ under the index set $\mathcal{J}$. In this case, the original non-orthogonal Cp estimator $\boldsymbol{\alpha}$ is a vector with its nonzero elements determined by $\hat{\boldsymbol{\beta}}(\hat{\mathcal{J}})$ and the remaining elements are all zeros.

Second, notice that $\boldsymbol { \bar{y} } = \boldsymbol { \bar{X} }\boldsymbol { \beta }+\boldsymbol { \varepsilon }$ and we denote $L_n(\mathcal{J})=\frac{\|\mu_n-\hat{\mu}_n(\mathcal{J})\|_2^2}{n}$, where $\hat{\mu}_n(\mathcal{J})$ is the least squares estimator of the true model $\mu_n=\bar{\boldsymbol{X}}\boldsymbol{\beta}$ under the index subset $\mathcal{J}$. The objective function (if denoted as $Cp(\mathcal{J})$) of the non-orthogonal Cp problem (\ref{Cp2}) can be re-expressed as follows,
\begin{equation}
Cp(\mathcal{J})=\frac{\boldsymbol{\varepsilon}^{\top}\boldsymbol{\varepsilon}}{n}+L_n(\mathcal{J})+\frac{2[\frac{\lambda}{2}p(\mathcal{J})-\boldsymbol{\varepsilon}^{\top}H(\mathcal{J})\boldsymbol{\varepsilon}]}{n}+\frac{2\boldsymbol{\varepsilon}^{\top}[I-H(\mathcal{J})]\mu_n}{n},
\end{equation}
where $H(\mathcal{J})=\bar{\boldsymbol{X}}(\mathcal{J})[\bar{\boldsymbol{X}}(\mathcal{J})^{\top}\bar{\boldsymbol{X}}(\mathcal{J})]^{-1}\bar{\boldsymbol{X}}(\mathcal{J})^{\top}$.

Then, by considering that  $L_n(\mathcal{J})=Q_n(\mathcal{J})+\frac{\boldsymbol{\varepsilon}^{\top}H(\mathcal{J})\boldsymbol{\varepsilon}}{n}$, where $Q_n(\mathcal{J})=\frac{\|\mu_n-H(\mathcal{J})\mu_n\|_2^2}{n}$. Thus, when $\mathcal{J}\in \mathcal{A}^c$, by Markov inequality, there holds $Cp(\mathcal{J})-\frac{\boldsymbol{\varepsilon}^{\top}\boldsymbol{\varepsilon}}{n}=\frac{{\lambda}p(\mathcal{J})}{n}-\frac{\boldsymbol{\varepsilon}^{\top}H(\mathcal{J})\boldsymbol{\varepsilon}}{n}\xrightarrow{P}0 (n\rightarrow\infty)$. Note that $Q_n(\mathcal{J})=0$ if $\mathcal{J}\in \mathcal{A}^c$, and $\mathbb{E}[(\boldsymbol{\varepsilon}^{\top}H(\mathcal{J})\boldsymbol{\varepsilon})/n]=\sigma^2p(\mathcal{J})/n$.

However, if $\mathcal{J}\in \mathcal{A}/\mathcal{A}^c$, we have $Cp(\mathcal{J})-\frac{\boldsymbol{\varepsilon}^{\top}\boldsymbol{\varepsilon}}{n}= Q_n(\mathcal{J})+\frac{{\lambda}p(\mathcal{J})}{n}-\frac{\boldsymbol{\varepsilon}^{\top}H(\mathcal{J})\boldsymbol{\varepsilon}}{n}+\frac{2\boldsymbol{\varepsilon}^{\top}[I-H(\mathcal{J})]\mu_n}{n}\stackrel{P}\nrightarrow0(n\rightarrow\infty)$. Note that $\lim_{n\rightarrow\infty}Q_n(\mathcal{J})>0$ when the assumption holds \citep{nishii1984asymptotic}, and $\frac{2\boldsymbol{\varepsilon}^{\top}[I-H(\mathcal{J})]\mu_n}{n}\xrightarrow{P}0 (n\rightarrow\infty)$ by the Chebyshev's inequality. Finally, we can conclude the desired result.
\end{proof}

\begin{proof}{\bf of Corollary \ref{corollary22}}
\\
First, the least squares estimate of the regression coefficient vector $\boldsymbol{\beta}$ under the index set $\hat{\mathcal{J}}$ selected by the non-orthogonal Cp criterion satisfies
\begin{equation}\notag
\hat{\boldsymbol{\beta}}(\hat{\mathcal{J}})=[\bar{\boldsymbol{X}}(\hat{\mathcal{J}})^{\top}\bar{\boldsymbol{X}}(\hat{\mathcal{J}})]^{-1}\bar{\boldsymbol{X}}(\hat{\mathcal{J}})^{\top}\bar{\boldsymbol{y}}.
\end{equation}	
Further, we can rewrite it as follows,
\begin{equation}\notag
\hat{\boldsymbol{\beta}}(\hat{\mathcal{J}})=\boldsymbol{\beta}(\hat{\mathcal{J}})+[\bar{\boldsymbol{X}}(\hat{\mathcal{J}})^{\top}\bar{\boldsymbol{X}}(\hat{\mathcal{J}})]^{-1}\bar{\boldsymbol{X}}(\hat{\mathcal{J}})^{\top}\boldsymbol{\varepsilon}.
\end{equation}
Second, we have
\begin{align}
&\mathbb{E}\left([\bar{\boldsymbol{X}}(\hat{\mathcal{J}})^{\top}\bar{\boldsymbol{X}}(\hat{\mathcal{J}})]^{-1}\bar{\boldsymbol{X}}(\hat{\mathcal{J}})^{\top}\boldsymbol{\varepsilon}\right)^{\top}\left([\bar{\boldsymbol{X}}(\hat{\mathcal{J}})^{\top}\bar{\boldsymbol{X}}(\hat{\mathcal{J}})]^{-1}\bar{\boldsymbol{X}}(\hat{\mathcal{J}})^{\top}\boldsymbol{\varepsilon}\right)\notag\\
&=\mathbb{E}\left(\boldsymbol{\varepsilon}^{\top}\bar{\boldsymbol{X}}(\hat{\mathcal{J}})[\bar{\boldsymbol{X}}(\hat{\mathcal{J}})^{\top}\bar{\boldsymbol{X}}(\hat{\mathcal{J}})]^{-2}\bar{\boldsymbol{X}}(\hat{\mathcal{J}})^{\top}\boldsymbol{\varepsilon}\right)\notag\\
&=\mathbb{E}\text{trace}\left(\bar{\boldsymbol{X}}(\hat{\mathcal{J}})[\bar{\boldsymbol{X}}(\hat{\mathcal{J}})^{\top}\bar{\boldsymbol{X}}(\hat{\mathcal{J}})]^{-2}\bar{\boldsymbol{X}}(\hat{\mathcal{J}})^{\top}\boldsymbol{\varepsilon}\boldsymbol{\varepsilon}^{\top}\right)\notag\\
&=\sigma^2\text{trace}\left([\bar{\boldsymbol{X}}(\hat{\mathcal{J}})^{\top}\bar{\boldsymbol{X}}(\hat{\mathcal{J}})]^{-2}\bar{\boldsymbol{X}}(\hat{\mathcal{J}})^{\top}\bar{\boldsymbol{X}}(\hat{\mathcal{J}})\right)\notag\\
&=\sigma^2\text{trace}\left(\bar{\boldsymbol{X}}(\hat{\mathcal{J}})^{\top}\bar{\boldsymbol{X}}(\hat{\mathcal{J}})\right)^{-1}\notag\\&=\mathcal{O}\left(\frac{1}{n}\right),\notag
\end{align}
where the {third line follows from the equality $\text{trace}(AB)=\text{trace}(BA)$} and the last equality is due to the assumption that $\lim_{n\rightarrow \infty}\frac{\bar{\boldsymbol{X}}^{\top}\bar{\boldsymbol{X}}}{n}$ exists.

Therefore, by Chebyshev's inequality, there holds $\hat{\boldsymbol{\beta}}(\hat{\mathcal{J}})-\boldsymbol{\beta}(\hat{\mathcal{J}})\xrightarrow{p}0 (n\rightarrow \infty)$. Moreover, according to the relationship between the $\hat{\boldsymbol{\beta}}(\hat{\mathcal{J}})$ and the non-orthogonal Cp estimator $\hat{\boldsymbol{\alpha}}$, also by Theorem \ref{theorem22}, we have $\hat{\boldsymbol{\alpha}}\xrightarrow{P}\boldsymbol{\beta}(n\rightarrow\infty)$.
Thus, the desired result holds.
\end{proof}

\begin{proof}{\bf of Proposition \ref{orthogonalized TLCp solution}}
	\\
	{First, we denote by $\boldsymbol{1}^*_i$ the indicator function of whether the $i$-th feature is selected by the orthogonalized TLCp model (\ref{orthogonalized TLCp}) or not. Specifically,
	\begin{eqnarray}\notag
	\boldsymbol{{1}}^*_i=
	\left\{
	\begin{matrix}
	0
	& \text{if~} \|\boldsymbol{w}_1^i\|_0=\|\boldsymbol{w}_2^i\|_0=0\\
	1
	&  \text{otherwise}
	\end{matrix}
	\right.
	\end{eqnarray}
	To get the optimal solution of the orthogonalized TLCp problem (\ref{orthogonalized TLCp}), first, for the samples in the source domain, we have $\boldsymbol{Q}_1^{\top}\boldsymbol{X}_1^{\top}\boldsymbol{X}_1\boldsymbol{Q}_1=n\boldsymbol{I}$ and $\boldsymbol {y}_1 = \boldsymbol {X}_1\boldsymbol{Q}_1 (\boldsymbol{Q}_1^{-1}\boldsymbol { \beta }) + \boldsymbol { \varepsilon }$. Then, we can rewrite the orthogonalized TLCp problem (\ref{orthogonalized TLCp}) as minimizing the following objective function,
	\begin{equation}\notag
	\sum_{i=1}^{k}\left\{\tilde{f}_i(\lambda_1,Z_1^i,\boldsymbol{w}_1^i)+\tilde{g}_i(\lambda_2,Z_2^i,\boldsymbol{w}_2^i)+\tilde{h}_i(\lambda_3^i,\boldsymbol{v}_1^i,\boldsymbol{v}_2^i,\lambda_4,\boldsymbol{1}^*_i)\right\}, %\lambda_1\left(-2n\boldsymbol{\beta}^{\top}\tilde{Q}_1^i\boldsymbol{w}_1^{i}-2\boldsymbol { \varepsilon }^{\top}Z_1^i\boldsymbol{w}_1^{i}+n(\boldsymbol{w}_1^{i})^2\right)+\lambda_2\left(-2n\boldsymbol{\beta}^{\top}\tilde{Q}_1^i\boldsymbol{w}_1^{i}-2\boldsymbol { \varepsilon }^{\top}Z_1^i\boldsymbol{w}_1^{i}+n(\boldsymbol{w}_1^{i})^2\right)+\frac{\lambda_3^i((\boldsymbol{v}_1^i)^2+(\boldsymbol{v}_2^i)^2}{2},
	\end{equation}
	where
	\begin{eqnarray}
	\begin{array}{l}\notag
	\tilde{f}_i(\lambda_1,Z_1^i,\boldsymbol{w}_1^i)=-2n\boldsymbol{\beta}^{\top}\tilde{Q}_1^i\boldsymbol{w}_1^{i}-2\boldsymbol { \varepsilon }^{\top}Z_1^i\boldsymbol{w}_1^{i}+n(\boldsymbol{w}_1^{i})^2,\\
	\tilde{g}_i(\lambda_2,Z_2^i,\boldsymbol{w}_2^i)=-2m(\boldsymbol{\beta}+\boldsymbol{\delta})^{\top}\tilde{Q}_2^i\boldsymbol{w}_2^{i}-2\boldsymbol { \eta }^{\top}Z_2^i\boldsymbol{w}_2^{i}+m(\boldsymbol{w}_2^{i})^2,\\
	\tilde{h}_i(\lambda_3^i,\boldsymbol{v}_1^i,\boldsymbol{v}_2^i,\lambda_4,\boldsymbol{1}^*_i)=\frac{\lambda_3^i((\boldsymbol{v}_1^i)^2+(\boldsymbol{v}_2^i)^2)}{2}+\lambda_4\boldsymbol{1}^*_i.
	\end{array}
	\end{eqnarray}
	Among them, $(\tilde{Q}_1^i)^{\top}$ represents the $i$-th row of the invertible matrix $\boldsymbol{Q}_1^{-1}$, and $Z_1^i$ is the $i$-th column of the design matrix $\boldsymbol{X}_1\boldsymbol{Q}_1$, for $i=1,\cdots,k$.	Also, $(\tilde{Q}_2^i)^{\top}$ represents the $i$-th row of the invertible matrix $\boldsymbol{Q}_2^{-1}$, and $Z_2^i$ is the $i$-th column of the design matrix $\boldsymbol{X}_2\boldsymbol{Q}_2$, for $i=1,\cdots,k$.}
	
{Due to the independence of each summand in the objective function above, the orthogonalized TLCp problem (\ref{orthogonalized TLCp}) can be solved by solving $k$ one-dimensional optimization problems below,
	\begin{equation}\label{orthogonalized TLCp objective}
	\text{min}_{\boldsymbol{v}_1^i,\boldsymbol{v}_2^i,\boldsymbol{w}_0^i}~\left\{ \tilde{f}_i(\lambda_1,Z_1^i,\boldsymbol{w}_1^i)+\tilde{g}_i(\lambda_2,Z_2^i,\boldsymbol{w}_2^i)+\tilde{h}_i(\lambda_3^i,\boldsymbol{v}_1^i,\boldsymbol{v}_2^i,\lambda_4,\boldsymbol{1}^*_i)\right\}
	\end{equation}
	for $i=1,\cdots,k$.}
	
	{For the $i$-th problem above, if $\boldsymbol{1}^*_i=1$, setting the gradient of the corresponding objective function equal to zero, we can obtain the estimators
	with respect to the $i$-th coefficients $\boldsymbol{w}_1^i$, $\boldsymbol{w}_2^i$ for the target and source domains by solving the following equations,
	\begin{eqnarray}\notag
	\left\{
	\begin{array}{l}
	2n\lambda_1\boldsymbol{w}_0^i+(2n\lambda_1+\lambda_3^i)\boldsymbol{ v }_1^i=2\lambda_1(n\boldsymbol{\beta}^{\top}\tilde{Q}_1^i+\boldsymbol{\varepsilon}^{\top}Z_{1}^{i})\\
	2m\lambda_2\boldsymbol{w}_0^i+(2m\lambda_2+\lambda_3^i)\boldsymbol{ v }_2^i=2\lambda_2(m(\boldsymbol{\beta}+\boldsymbol{\delta}_1)^{\top}\tilde{Q}_2^i+\boldsymbol{\eta}^{\top}Z_{2}^{i})\\
	\boldsymbol{v}_1^i=-\boldsymbol{v}_2^i
	\end{array}
	\right.
	\end{eqnarray}
	Then, the estimator for the target task is
	\begin{equation}\label{solution1}
	\hat{\boldsymbol{w}}_1^i=(\tilde{{Q}}_1^{i})^{\top}\boldsymbol { \beta }+\frac{(Z_1^i)^{\top}\boldsymbol { \varepsilon }}{n}+D_1^i\left[(\tilde{{Q}}_2^{i})^{\top}(\boldsymbol{\delta}+\boldsymbol{\beta})-(\tilde{{Q}}_1^{i})^{\top}\boldsymbol{\beta}+\frac{(Z_2^i)^{\top}\boldsymbol{\eta}}{m}-\frac{(Z_1^i)^{\top}\boldsymbol { \varepsilon }}{n}\right],
	\end{equation}
	where $D_1^i=\frac{\lambda_2\lambda_3^i}{4\lambda_1\lambda_2n+\lambda_2\lambda_3^i+\frac{n}{m}\lambda_1\lambda_3^i}$, for $i=1,\cdots,k$.}
	
	{The estimator for the source domain task is
	\begin{equation}\label{solution2}
	\hat{\boldsymbol{w}}_2^i=(\tilde{{Q}}_2^{i})^{\top}(\boldsymbol { \beta }+\boldsymbol{\delta})+\frac{(Z_2^i)^{\top}\boldsymbol { \eta }}{m}-D_2^i\left[(\tilde{{Q}}_2^{i})^{\top}(\boldsymbol{\delta}+\boldsymbol{\beta})-(\tilde{{Q}}_1^{i})^{\top}\boldsymbol{\beta}+\frac{(Z_2^i)^{\top}\boldsymbol{\eta}}{m}-\frac{(Z_1^i)^{\top}\boldsymbol { \varepsilon }}{n}\right],
	\end{equation}
	where $D_2^i=\frac{\lambda_1\lambda_3^i}{4\lambda_1\lambda_2m+\lambda_1\lambda_3^i+\frac{m}{n}\lambda_2\lambda_3^i}$, for $i=1,\cdots,k$.}
	
	{Also, we have
	\begin{equation}\label{solution3}
	\boldsymbol{v}_1^i=-D_3^i\left[(\tilde{{Q}}_2^{i})^{\top}(\boldsymbol{\delta}+\boldsymbol{\beta})-(\tilde{{Q}}_1^{i})^{\top}\boldsymbol{\beta}+\frac{(Z_2^i)^{\top}\boldsymbol{\eta}}{m}-\frac{(Z_1^i)^{\top}\boldsymbol { \varepsilon }}{n}\right],
	\end{equation}
	where $D_3^i=\frac{2\lambda_1\lambda_2}{4\lambda_1\lambda_2+\frac{1}{n}\lambda_2\lambda_3^i+\frac{1}{m}\lambda_1\lambda_3^i}$, for $i=1,\cdots,k$.\\
	Then, substituting the relations (\ref{solution1}), (\ref{solution2}), (\ref{solution3}), and $\boldsymbol{v}_1^i=-\boldsymbol{v}_2^i$ into the objective function in (\ref{orthogonalized TLCp objective}), we have
	\begin{align}\notag
	&\tilde{f}_i(\lambda_1,Z_1^i,\boldsymbol{w}_1^i)+\tilde{g}_i(\lambda_2,Z_2^i,\boldsymbol{w}_2^i)+\tilde{h}_i(\lambda_3^i,\boldsymbol{v}_1^i,\boldsymbol{v}_2^i,\lambda_4,\boldsymbol{1}^*_i)\\ \notag&=(\tilde{D}^i-\lambda_2m)\tilde{H}_i^2+(\tilde{D}^i-\lambda_1n)\tilde{R}_i^2-2\tilde{D}^i\tilde{R}_i\tilde{H}_i+\lambda_4,
	\end{align}
	where $\tilde{H}_i=(\boldsymbol{\delta}+\boldsymbol{\beta})\tilde{{Q}}_2^{i}+\frac{\boldsymbol{\eta}^{\top}Z_2^i}{m}$,  $\tilde{R}_i=\boldsymbol{\beta}^{\top}\tilde{{Q}}_1^{i}+\frac{\boldsymbol { \varepsilon }^{\top}Z_i^i}{n}$ %and $\tilde{J}_i=m\lambda_2\tilde{H}_i+n\lambda_1\tilde{R}_i$,
	for $i=1,\cdots,k$. }
	
	{Further, we notice that $2D_3^i+D_2^i+D_1^i=1$ and reorganize the calculation results, thus obtaining
	\begin{eqnarray}\notag
	\tilde{f}_i(\lambda_1,Z_1^i,\boldsymbol{w}_1^i)+\tilde{g}_i(\lambda_2,Z_2^i,\boldsymbol{w}_2^i)+\tilde{h}_i(\lambda_3^i,\boldsymbol{v}_1^i,\boldsymbol{v}_2^i,\lambda_4,1)=\lambda_4-A_i\tilde{H}_i^2-B_i\tilde{R}_i^2-C_i\tilde{J}_i^2,
	\end{eqnarray}
	where $\tilde{J}_i=m\lambda_2\tilde{H}_i+n\lambda_1\tilde{R}_i$. And,  $A_i=\frac{4\lambda_1\lambda_2^2m^2n}{4\lambda_1\lambda_2mn+m\lambda_2\lambda_3+n\lambda_1\lambda_3^i}$, $B_i=\frac{4\lambda_2\lambda_1^2mn^2}{4\lambda_1\lambda_2mn+m\lambda_2\lambda_3^i+n\lambda_1\lambda_3^i}$, and $C_i=\frac{\lambda_3^i}{4\lambda_1\lambda_2mn+m\lambda_2\lambda_3^i+n\lambda_1\lambda_3^i}$ are functions with respect to the parameters $\lambda_1, \lambda_2, \lambda_3^i$.}
	
	{If $\boldsymbol{1}^*_i=0$ in the $i$-th optimization problem, the estimators for the parameters $\boldsymbol{w}_0^i,\boldsymbol{v}_1^i,\boldsymbol{v}_2^i$ satisfy $\bar{\boldsymbol{w}}_0^i=\bar{\boldsymbol{v}}_1^i=\bar{\boldsymbol{v}}_2^i=0$. So the corresponding objective value is
	\begin{equation*}
	\tilde{f}_i(\lambda_1,Z_1^i,\boldsymbol{w}_1^i)+\tilde{g}_i(\lambda_2,Z_2^i,\boldsymbol{w}_2^i)+\tilde{h}_i(\lambda_3^i,\boldsymbol{v}_1^i,\boldsymbol{v}_2^i,\lambda_4,0)=0.
	\end{equation*}
	Finally, we can derive the optimal solution for the $i$-th optimization problem (\ref{orthogonalized TLCp}) by finding two estimators $\hat{\boldsymbol{w}}_1^i, \hat{\boldsymbol{w}}_2^i$ that can pick the smaller one between the random value $\lambda_4-A_i\tilde{H}_i^2-B_i\tilde{R}_i^2-C_i\tilde{J}_i^2$ and $0$.}
	
	{Therefore, we can obtain the solution for the target regression task as follows,
	\begin{align}
	&\hat{\boldsymbol{w}}_1^{i}= \notag\\
	&\left\{
	\begin{array}{cc}
	(\tilde{{Q}}_1^{i})^{\top}\boldsymbol { \beta }+\frac{(Z_1^i)^{\top}\boldsymbol { \varepsilon }}{n}+D_1^i\left[(\tilde{{Q}}_2^{i})^{\top}(\boldsymbol{\delta}+\boldsymbol{\beta})-(\tilde{{Q}}_1^{i})^{\top}\boldsymbol{\beta}+\frac{(Z_2^i)^{\top}\boldsymbol{\eta}}{m}-\frac{(Z_1^i)^{\top}\boldsymbol { \varepsilon }}{n}\right]
	& %\text{if~}
	 \tilde{F}(\tilde{H}_i,\tilde{R}_i,\tilde{J}_i)>\lambda_4\notag\\
	%A_i\tilde{H}_i^2+B_i\tilde{Z}_i^2+C_i\tilde{J}_i^2>\lambda_4
	0
	&  \text{otherwise}
	\end{array}
	\right.
	\end{align}
	$\tilde{F}(\tilde{H}_i,\tilde{R}_i,\tilde{J}_i)=A_i\tilde{H}_i^2+B_i\tilde{R}_i^2+C_i\tilde{J}_i^2$ for $i=1,\cdots,k$.}
	\end{proof}

\begin{proof}{\bf of Theorem \ref{theorem 28}}
	\\
	{First, we notice that the approximate TLCp estimator is $\tilde{\boldsymbol{w}}_1=\boldsymbol{Q}_1\hat{\boldsymbol{w}}_1$. We can further rewrite it as follows,
	\begin{eqnarray}
	\tilde{\boldsymbol{w}}_1&=&\sum_{i=1}^{k}Q_1^i\hat{\boldsymbol{w}}_1^{i}\label{94}\\
	&=&\sum_{i=1}^{k}Q_1^i\left(\tilde{Q_1^i}^{\top}\boldsymbol{\beta}+\frac{(Z_1^i)^{\top}\boldsymbol { \varepsilon }}{n}+D_1^i\left[(\tilde{{Q}}_2^{i})^{\top}(\boldsymbol{\delta}+\boldsymbol{\beta})-(\tilde{{Q}}_1^{i})^{\top}\boldsymbol{\beta}+\frac{(Z_2^i)^{\top}\boldsymbol{\eta}}{m}-\frac{(Z_1^i)^{\top}\boldsymbol { \varepsilon }}{n}\right]\right)\cdot\boldsymbol{1}_{\tilde{A}_i},\notag
	\end{eqnarray}
	where $\boldsymbol{1}_{\tilde{A}_i}$ is the indication function with respect to the random variable set \\ $\tilde{A}_i=\left\{A_i\tilde{H}_i^2+B_i\tilde{R}_i^2+C_i\tilde{J}_i^2>\lambda_4 \right\}$, for $i=1,\cdots,k$.}
	
	{Then, following the same technique as in the proof the Theorem \ref{theorem21}, we decompose (\ref{94}) into three terms as below,
	\begin{equation}\label{95}
	\tilde{\boldsymbol{w}}_1=\sum_{i=1}^{k}\left\{Q_1^i\tilde{Q_1^i}^{\top}\cdot \boldsymbol{\beta}\cdot\boldsymbol{1}_{\tilde{A}_i}\right\}+\sum_{i=1}^{k}\left\{Q_1^i\frac{Z_i^{\top}\boldsymbol { \varepsilon }}{n}\cdot\boldsymbol{1}_{\tilde{A}_i}\right\}+\sum_{i=1}^{k}\left\{D_1^iQ_1^iR_i\boldsymbol{1}_{\tilde{A}_i}\right\},
	\end{equation}	
	where $R_i=(\tilde{{Q}}_2^{i})^{\top}(\boldsymbol{\delta}+\boldsymbol{\beta})-(\tilde{{Q}}_1^{i})^{\top}\boldsymbol{\beta}+\frac{(Z_2^i)^{\top}\boldsymbol{\eta}}{m}-\frac{(Z_1^i)^{\top}\boldsymbol { \varepsilon }}{n}$, for $i=1,\cdots,k$.}
	
	{For the first summand of (\ref{95}), we have	
	\begin{align*}
	\left\|\sum_{i=1}^{k}\left\{Q_1^i\tilde{Q_1^i}^{\top}\cdot \boldsymbol{\beta}\cdot\boldsymbol{1}_{\tilde{A}_i}\right\}-\boldsymbol{\beta}\right\|_2\leq M\sum_{i=1}^{k}\left|\boldsymbol{1}_{\tilde{A}_i}-1\right|,
	\end{align*}
where $M=\text{max}_{i=1,\cdots,k}\left\{\|Q_1^i\tilde{Q_1^i}^{\top}\boldsymbol{\beta}\|_2\right\}$. Further, for any $1>\epsilon>0$, we have
\begin{align*}
\left\{M\sum_{i=1}^{k}\left|\boldsymbol{1}_{\tilde{A}_i}-1\right|>\epsilon\right\}\subseteq \left\{\exists ~i, s.t. \left|\boldsymbol{1}_{\tilde{A}_i}-1\right|>\frac{\epsilon}{kM}\right\}.
\end{align*}
Moreover, by the definition of the indicator function $\boldsymbol{1}_{\tilde{A}_i} $ above, there holds
\begin{align*}
&P_r\left\{\exists ~i, s.t. \left|\boldsymbol{1}_{\tilde{A}_i}-1\right|>\frac{\epsilon}{kM}\right\}\\\leq& \sum_{i=1}^{k}P_r\left\{A_i\tilde{H}_i^2+B_i\tilde{R}_i^2+C_i\tilde{J}_i^2\leq\lambda_4\right\}\\\leq&
\sum_{i=1}^{k}P_r\left\{B_i\tilde{R}_i^2\leq\lambda_4\right\}.
\end{align*}}

{By noticing that $\tilde{R}_i=\boldsymbol{\beta}^{\top}\tilde{{Q}}_1^{i}+\frac{\boldsymbol { \varepsilon }^{\top}Z_i^i}{n}$ and $B_i=\frac{4\lambda_2\lambda_1^2mn^2}{4\lambda_1\lambda_2mn+m\lambda_2\lambda_3^i+n\lambda_1\lambda_3^i}$ (for $i=1,\cdots,k$), we can obtain an upper bound on the first summand of (\ref{95}) by following the same procedure as was done in the proof of Theorem \ref{theorem21}.}

{Similarly, we can estimate the second summand of (\ref{95}) by referring to the proof of Theorem \ref{theorem21}.}

{For the third summand of (\ref{95}), we notice that $D_1^i=\frac{\lambda_2\lambda_3^i}{4\lambda_1\lambda_2n+\lambda_2\lambda_3^i+\frac{n}{m}\lambda_1\lambda_3^i}\rightarrow0$, when $n\rightarrow\infty$, therefore, we can easily obtain the desired result by combining the results on the three summands.	}
\end{proof}

\begin{proof}{\bf of Theorem \ref{theorem30}}\\	
Notice that $\boldsymbol { \bar{y} } = \boldsymbol { \bar{X} }\boldsymbol { \beta }+\boldsymbol { \varepsilon }$, we can decompose Mallows' Cp statistic as follows,
\begin{equation}
\frac{(\boldsymbol{\beta}-{\boldsymbol{\alpha}})^{\top}\boldsymbol { \bar{X} }^{\top}\boldsymbol { \bar{X} }(\boldsymbol{\beta}-{\boldsymbol{\alpha}})}{n}+\frac{2(\boldsymbol{\beta}-{\boldsymbol{\alpha}})^{\top}\boldsymbol { \bar{X} }^{\top}\boldsymbol { \varepsilon }}{n}+\frac{\boldsymbol { \varepsilon }^{\top}\boldsymbol { \varepsilon }}{n}+\frac{2\sigma_1^2}{n}p,
\end{equation}
where $\boldsymbol{\alpha}$ is any fixed estimator. Recall that the approximate TLCp estimator satisfies $\hat{\boldsymbol{\alpha}}_2\xrightarrow{P}\boldsymbol{\beta}(n\rightarrow\infty)$, and $\frac{\boldsymbol { \varepsilon }^{\top}\boldsymbol { \varepsilon }}{n}\xrightarrow{P}\sigma_1^2(n\rightarrow\infty)$ by the law of large numbers.
To show that the approximate TLCp estimator $\hat{\boldsymbol{\alpha}}_2$ asymptotically achieves the lowest value of Mallows' Cp statistic, we only need to show that $\frac{2(\boldsymbol{\beta}-{\boldsymbol{\alpha}})^{\top}\boldsymbol { \bar{X} }^{\top}\boldsymbol { \varepsilon }}{n}\xrightarrow{P}0(n\rightarrow\infty)$ holds for any fixed estimator $\boldsymbol{\alpha}$. 

In fact, 
\begin{equation*}
\mathbb{E}\left(\frac{\boldsymbol { \bar{X} }^{\top}\boldsymbol { \varepsilon }}{n} \right)^2=\frac{1}{n^2}\mathbb{E}\left(\boldsymbol { \bar{X} }^{\top}\boldsymbol{\varepsilon}\boldsymbol{\varepsilon}^{\top}\boldsymbol { \bar{X} }\right)=\frac{1}{n^2}\mathbb{E}\text{trace}(\boldsymbol { \bar{X} }\boldsymbol { \bar{X} }^{\top}\boldsymbol{\varepsilon}\boldsymbol{\varepsilon}^{\top})=\frac{\sigma_1^2}{n^2}\text{trace}(\boldsymbol { \bar{X} }^{\top}\boldsymbol { \bar{X} })=\frac{\sigma_1^2}{n}.
\end{equation*}
Thus, we have $\frac{\boldsymbol { \bar{X} }^{\top}\boldsymbol { \varepsilon }}{n} \sim \mathcal { N } \left( 0 , \frac{\sigma_1^2}{n}\text{trace}(\boldsymbol { \bar{X} }^{\top}\boldsymbol { \bar{X} }) \right)$.
By Chebyshev's Inequality, for any $\eta>0$, there holds
\begin{equation*}
P_r\left\{\left|\frac{\boldsymbol { \bar{X} }^{\top}\boldsymbol { \varepsilon }}{n}\right|<\eta\right\}\geq 1-\frac{\sigma_1^2}{n}\frac{\text{trace}(\boldsymbol { \bar{X} }^{\top}\boldsymbol { \bar{X} })}{\eta^2}.
\end{equation*}
Then, we have $\frac{\boldsymbol { \bar{X} }^{\top}\boldsymbol { \varepsilon }}{n}\xrightarrow{P}0(n\rightarrow\infty)$ by the condition that $\lim_{n\rightarrow \infty}\frac{\bar{\boldsymbol{X}}^{\top}\bar{\boldsymbol{X}}}{n}$ exists. Therefore, it can be concluded that $\frac{2(\boldsymbol{\beta}-{\boldsymbol{\alpha}})^{\top}\boldsymbol { \bar{X} }^{\top}\boldsymbol { \varepsilon }}{n}\xrightarrow{P}0(n\rightarrow\infty)$, for any fixed estimator $\boldsymbol{\alpha}$. 

The second statement of this theorem can be proved similarly.
\end{proof}

\begin{proof}{\bf of Theorem \ref{theorem31}}\\
Notice that the TLCp statistic, $\frac{1}{n+m}\sum _ { t = 1 } ^ { 2 } \left[\lambda_t( \boldsymbol { y }_t - \boldsymbol { X }_t \boldsymbol { w}_t ) ^ { \top } ( \boldsymbol { y }_t - \boldsymbol { X }_t \boldsymbol { w}_t  )+\frac { 1 } { 2 } \boldsymbol { v }_t^{\top}\boldsymbol{\lambda} _ { 3 }\boldsymbol{v}_t +\frac { 1 } { 2 }\lambda_4\bar{p}\right]$ can be viewed as a weighted sum of two Mallows' Cp statistics with respect to the target and source tasks, and plus a parameter sharing term ($\frac{1}{n+m}\sum _ { t = 1 } ^ { 2 }\frac { 1 } { 2 } \boldsymbol { v }_t^{\top}\boldsymbol{\lambda} _ { 3 }\boldsymbol{v}_t$) to leverage the target and source tasks. Thus, we can follow the same proof scheme as of Theorem \ref{theorem30}. Concretely, to show that the approximate TLCp estimators for the target and source tasks achieve the lowest value of the TLCp statistic, we only need to verify that $\frac{1}{n+m}\sum _ { t = 1 } ^ { 2 }\frac { 1 } { 2 } \boldsymbol { \tilde{v }}_t^{\top}\boldsymbol{\lambda} _ { 3 }\boldsymbol{\tilde{v}}_t\xrightarrow{P}0(n\rightarrow\infty)$, where $\boldsymbol{\tilde{v}}_1$ ($\boldsymbol{\tilde{v}}_2$) indicates the individual parameter of the approximate TLCp estimator $\boldsymbol{\tilde{w}}_1$ ($\boldsymbol{\tilde{w}}_2$) for the target (source) task.

Recall that the approximate TLCp estimator with respect to the target task satisfies $\boldsymbol{\tilde{w}}_1\xrightarrow{P}\boldsymbol{\beta}(n\rightarrow\infty)$ by Theorem \ref{theorem 28}. Follow the same proof framework as of Theorem \ref{theorem 28} and notice that $\lim_{n\rightarrow \infty}m/n=C$ ($C>0$), it can be concluded that $\boldsymbol{\tilde{w}}_2\xrightarrow{P}\boldsymbol{\beta}+\boldsymbol{\delta}(n\rightarrow\infty)$, where $\boldsymbol{\delta}$ is the dissimilarity between the target and source tasks. 

Now, let's focus on the analysis of $\frac{1}{n+m}\sum _ { t = 1 } ^ { 2 }\frac { 1 } { 2 } \boldsymbol { \tilde{v }}_t^{\top}\boldsymbol{\lambda} _ { 3 }\boldsymbol{\tilde{v}}_t$. Notice that  $\boldsymbol{\tilde{v}}_1=\frac{\boldsymbol{\tilde{w}}_1-\boldsymbol{\tilde{w}}_2}{2}$ and $\boldsymbol{\tilde{v}}_2=\frac{\boldsymbol{\tilde{w}}_2-\boldsymbol{\tilde{w}}_1}{2}$, then we have $\sum _ { t = 1 } ^ { 2 }\frac { 1 } { 2 } \boldsymbol { \tilde{v }}_t^{\top}\boldsymbol{\lambda} _ { 3 }\boldsymbol{\tilde{v}}_t\xrightarrow{P}\boldsymbol{\delta}^{\top}\boldsymbol{\lambda_3}\boldsymbol{\delta}(n\rightarrow\infty)$. Due to that $\boldsymbol{\delta}$ is a fixed constant vector in our settings, there holds $\frac{1}{n+m}\sum _ { t = 1 } ^ { 2 }\frac { 1 } { 2 } \boldsymbol { \tilde{v }}_t^{\top}\boldsymbol{\lambda} _ { 3 }\boldsymbol{\tilde{v}}_t\xrightarrow{P}0(n\rightarrow\infty)$.

The second statement of this theorem can be proved similarly.
\end{proof}

\section{Additional simulations with the orthogonal TLCp method}\label{appendixE}

{This part includes two additional simulations to show the efficacy of the orthogonal TLCp method (with its parameters well-tuned) when the true model is generated randomly. Specifically, we follow the same simulation setting as in Section $6.2$ and assume the non-zero attributes of $\boldsymbol { \beta }_2=[0.42, 0.89, 0.96, 0.20, 0, 0.65, 0.84, 0, 0.29, 0]^{\top}$ are i.i.d. sampled from the $[0,1]$ uniform distribution. Note that there are no critical features in this example.}

{We can see from Figure \ref{simulation_expansion2} that, even without critical features, the orthogonal TLCp method outperforms the Cp criterion at each sample size both in terms of MSE and the number of correctly identified features. In particular, the MSE value of the TLCp estimator increases with the relative dissimilarity of tasks differs from the case when there are critical features in the true model previously analyzed. Moreover, as depicted in the top right subfigure, TLCp is remarkably competitive with Cp when the sample size and the relative dissimilarity of tasks are relatively small. Specifically, TLCp improves Cp $24\% \sim 36\%$ in terms of MSE when the sample size is $20$, and the relative task dissimilarity is less than $3.95$. Further, we observe that the ``effective sample size'' (in the sense of MSE) decreases as the relative dissimilarity of tasks increases (e.g., see the contour line at the level $0.006$ in the top right panel). Without any critical features, the ``effective sample size'' approximately equals $170$ when the relative dissimilarity of tasks is $0.10$, and approximately $140$ when the relative task dissimilarity is $3.00$. This observation indicates that, without critical features, Cp needs fewer examples to perform as well as TLCp. Similar results are seen if we compare the performance of Cp and TLCp with respect to the number of correctly identified features. The ``effective sample size'' in terms of the number of correctly selected features (e.g., see the contour line at the level $0.40$ in the bottom right panel) is about $30$ when the relative dissimilarity of tasks is $0.10$ and $85$ when the relative task dissimilarity is $3.00$.}

{We also illustrate the importance of choosing appropriate hyper-parameters in the orthogonal TLCp model by investigating the performance of TLCp with its hyper-parameters randomly selected. Figure \ref{simulation_expansion3} shows that, if the hyper-parameters of the TLCp method are tuned randomly, the MSE performance of the TLCp model degrades compared to that of TLCp whose parameters are well-tuned. In this case, the number of correctly identified features of TLCp is less than that of Cp. These observations demonstrate the significance of the proposed tuning of parameters for the TLCp method.}

\begin{figure}[htbp]
	\begin{minipage}{\textwidth}
		\centering{\includegraphics[width=1\textwidth]{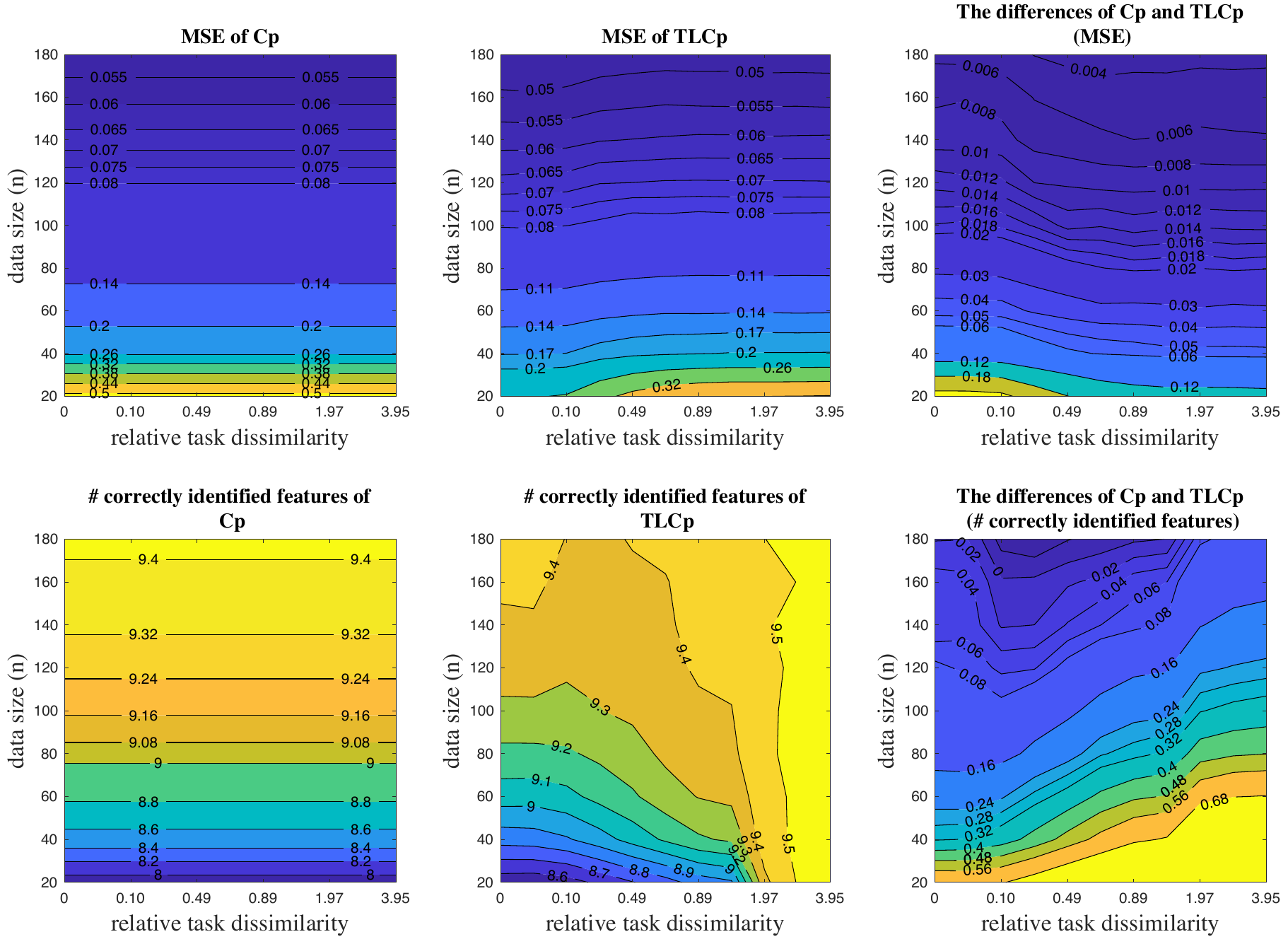}} 
	\end{minipage}
	\caption{\small Performance (in the sense of MSE and number of correctly identified features) of Cp and TLCp methods as the number of target data and the relative task dissimilarity vary simultaneously.  The true model was generated randomly (without the existence of critical features in this example).} \label{simulation_expansion2}
\end{figure}

\begin{figure}[htbp]
	\begin{minipage}{\textwidth}
		\centering{\includegraphics[width=1\textwidth]{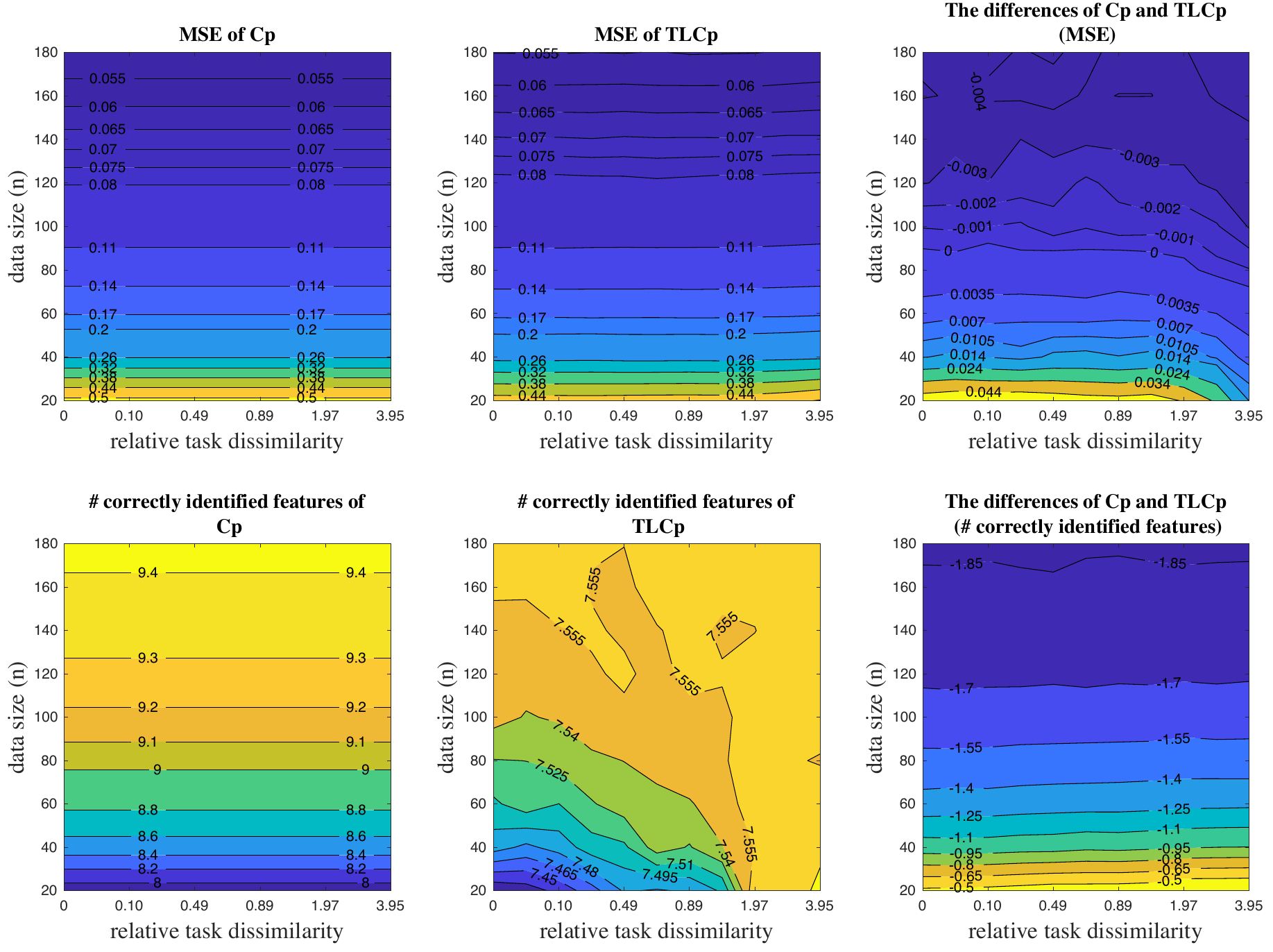}} 
	\end{minipage}
	\caption{ \small Performance (in the sense of MSE and number of correctly identified features) of Cp and TLCp methods as the number of target data and the relative task dissimilarity vary simultaneously.  The hyper-parameters of the TLCp model were tuned randomly.} \label{simulation_expansion3}
\end{figure}

\section{Additional tables }\label{appendixD}

\begin{table}[htbp]
	\centering
	\begin{tabular}{llll}
		\hline
		\multicolumn{1}{|l|}{} & \multicolumn{1}{l|}{original TLCp} & \multicolumn{1}{c|}{\begin{tabular}[c]{@{}c@{}}approximate TLCp\\ cutoff\end{tabular}}  \\ \hline
		\multicolumn{1}{|c|}{original Cp} & \multicolumn{1}{c|}{$\textbf{0.00}$} & \multicolumn{1}{c|}{$\textbf{0.00}$}              \\ \hline
		\multicolumn{1}{|c|}{stepwise FS} & \multicolumn{1}{c|}{$\textbf{0.03}$ } & \multicolumn{1}{c|}{$\textbf{0.00}$}               \\ \hline
		\multicolumn{1}{|c|}{univariate FS} & \multicolumn{1}{c|}{$\text{0.46}$} & \multicolumn{1}{c|}{$\textbf{0.00}$}               \\ \hline
		\multicolumn{1}{|c|}{LASSO} & \multicolumn{1}{c|}{$\text{0.52}$} & \multicolumn{1}{c|}{$\textbf{0.00}$}                \\ \hline
		\multicolumn{1}{|c|}{aggregate original Cp} & \multicolumn{1}{c|}{${1.00}$} & \multicolumn{1}{c|}{$\textbf{0.83}$}                             \\ \hline
		\multicolumn{1}{|c|}{aggregate stepwise FS} & \multicolumn{1}{c|}{$\text{1.00}$} & \multicolumn{1}{c|}{$\textbf{0.94}$}                     \\ \hline
		\multicolumn{1}{|c|}{aggregate univariate FS} & \multicolumn{1}{c|}{$\text{1.00}$} & \multicolumn{1}{c|}{$\text{0.69}$}                       \\ \hline
		\multicolumn{1}{|c|}{aggregate LASSO} & \multicolumn{1}{c|}{$\text{1.00}$} & \multicolumn{1}{c|}{$\text{0.86}$}                        \\ \hline
		\multicolumn{1}{|c|}{least $\ell_{2,1}$-norm} & \multicolumn{1}{c|}{$\text{0.51}$} & \multicolumn{1}{c|}{$\textbf{0.00}$}                     \\ \hline
		\multicolumn{1}{|c|}{multi-level LASSO} & \multicolumn{1}{c|}{$\text{0.50}$} & \multicolumn{1}{c|}{$\textbf{0.00}$}                    \\ \hline
		\multicolumn{1}{|c|}{original TLCp} & \multicolumn{1}{c|}{$--$} & \multicolumn{1}{c|}{$\textbf{0.00}$}                                  \\ \hline
		\multicolumn{1}{|c|}{approximate TLCp cutoff} & \multicolumn{1}{c|}{$\text{1.00}$} & \multicolumn{1}{c|}{$--$}                                \\ \hline       
		\multicolumn{1}{|c|}{\begin{tabular}[c]{@{}c@{}}original TLCp\\ with three tasks\end{tabular}} & \multicolumn{1}{c|}{$1.00$} & \multicolumn{1}{c|}{$\text{1.00}$}                                                          \\ \hline
		\multicolumn{1}{|c|}{\begin{tabular}[c]{@{}c@{}}approximate TLCp cutoff\\ with three tasks\end{tabular}} & \multicolumn{1}{c|}{$\text{1.00}$} & \multicolumn{1}{c|}{$1.00$}                                                                \\ \hline                 
	\end{tabular}
	\caption{The table shows the $p$-value of the pairwise $t$-test of different methods on school data with the relative task dissimilarity $1.51$ when $n=170$. Boldface means the proposed TLCp methods statistically outperform the compared methods ($p$-value $<0.05$). }\label{table9}
\end{table}

\begin{table}[htbp]
	\centering
	\begin{tabular}{llll}
		\hline
		\multicolumn{1}{|l|}{} & \multicolumn{1}{l|}{original TLCp} & \multicolumn{1}{c|}{\begin{tabular}[c]{@{}c@{}}approximate TLCp\\ cutoff\end{tabular}}  \\ \hline
		\multicolumn{1}{|c|}{original Cp} & \multicolumn{1}{c|}{$\textbf{0.01}$} & \multicolumn{1}{c|}{$\textbf{0.00}$}              \\ \hline
		\multicolumn{1}{|c|}{stepwise FS} & \multicolumn{1}{c|}{$\text{0.85}$ } & \multicolumn{1}{c|}{$\text{0.18}$}               \\ \hline
		\multicolumn{1}{|c|}{univariate FS} & \multicolumn{1}{c|}{$\text{1.00}$} & \multicolumn{1}{c|}{$\text{0.81}$}               \\ \hline
		\multicolumn{1}{|c|}{LASSO} & \multicolumn{1}{c|}{$\text{1.00}$} & \multicolumn{1}{c|}{$\text{0.85}$}                \\ \hline
		\multicolumn{1}{|c|}{aggregate original Cp} & \multicolumn{1}{c|}{${0.64}$} & \multicolumn{1}{c|}{$\textbf{0.05}$}                             \\ \hline
		\multicolumn{1}{|c|}{aggregate stepwise FS} & \multicolumn{1}{c|}{$\text{0.41}$} & \multicolumn{1}{c|}{$\textbf{0.01}$}                     \\ \hline
		\multicolumn{1}{|c|}{aggregate univariate FS} & \multicolumn{1}{c|}{$\text{1.00}$} & \multicolumn{1}{c|}{$\text{0.72}$}                       \\ \hline
		\multicolumn{1}{|c|}{aggregate LASSO} & \multicolumn{1}{c|}{$\text{1.00}$} & \multicolumn{1}{c|}{$\text{1.00}$}                        \\ \hline
		\multicolumn{1}{|c|}{least $\ell_{2,1}$-norm} & \multicolumn{1}{c|}{$\textbf{0.00}$} & \multicolumn{1}{c|}{$\textbf{0.00}$}                     \\ \hline
		\multicolumn{1}{|c|}{multi-level LASSO} & \multicolumn{1}{c|}{$\text{0.95}$} & \multicolumn{1}{c|}{$\text{0.35}$}                    \\ \hline
		\multicolumn{1}{|c|}{original TLCp} & \multicolumn{1}{c|}{$--$} & \multicolumn{1}{c|}{$\textbf{0.02}$}                                  \\ \hline
		\multicolumn{1}{|c|}{approximate TLCp cutoff} & \multicolumn{1}{c|}{$\text{0.98}$} & \multicolumn{1}{c|}{$--$}                                \\ \hline       
		\multicolumn{1}{|c|}{\begin{tabular}[c]{@{}c@{}}original TLCp\\ with three tasks\end{tabular}} & \multicolumn{1}{c|}{$1.00$} & \multicolumn{1}{c|}{$\text{1.00}$}                                                          \\ \hline
		\multicolumn{1}{|c|}{\begin{tabular}[c]{@{}c@{}}approximate TLCp cutoff\\ with three tasks\end{tabular}} & \multicolumn{1}{c|}{$\text{1.00}$} & \multicolumn{1}{c|}{$1.00$}                                                                \\ \hline                 
	\end{tabular}
	\caption{The table shows the $p$-value of the pairwise $t$-test of different methods on school data with the relative task dissimilarity $2.58$ when $n=170$. Boldface means the proposed TLCp methods statistically outperform the compared methods ($p$-value $<0.05$). }\label{table10}
\end{table}

\begin{table}[htbp]
	\centering
	\begin{tabular}{llll}
		\hline
		\multicolumn{1}{|l|}{} & \multicolumn{1}{l|}{original TLCp} & \multicolumn{1}{c|}{\begin{tabular}[c]{@{}c@{}}approximate TLCp\\ cutoff\end{tabular}}  \\ \hline
		\multicolumn{1}{|c|}{original Cp} & \multicolumn{1}{c|}{$\textbf{0.00}$} & \multicolumn{1}{c|}{$\textbf{0.00}$}              \\ \hline
		\multicolumn{1}{|c|}{stepwise FS} & \multicolumn{1}{c|}{$\text{1.00}$ } & \multicolumn{1}{c|}{$\text{1.00}$}               \\ \hline
		\multicolumn{1}{|c|}{univariate FS} & \multicolumn{1}{c|}{$\text{1.00}$} & \multicolumn{1}{c|}{$\text{1.00}$}               \\ \hline
		\multicolumn{1}{|c|}{LASSO} & \multicolumn{1}{c|}{$\text{1.00}$} & \multicolumn{1}{c|}{$\text{1.00}$}                \\ \hline
		\multicolumn{1}{|c|}{aggregate original Cp} & \multicolumn{1}{c|}{$\textbf{0.00}$} & \multicolumn{1}{c|}{$\textbf{0.00}$}                             \\ \hline
		\multicolumn{1}{|c|}{aggregate stepwise FS} & \multicolumn{1}{c|}{$\textbf{0.00}$} & \multicolumn{1}{c|}{$\textbf{0.00}$}                     \\ \hline
		\multicolumn{1}{|c|}{aggregate univariate FS} & \multicolumn{1}{c|}{$\textbf{0.00}$} & \multicolumn{1}{c|}{$\textbf{0.00}$}                       \\ \hline
		\multicolumn{1}{|c|}{aggregate LASSO} & \multicolumn{1}{c|}{$\textbf{0.00}$} & \multicolumn{1}{c|}{$\textbf{0.00}$}                        \\ \hline
		\multicolumn{1}{|c|}{least $\ell_{2,1}$-norm} & \multicolumn{1}{c|}{$\text{0.76}$} & \multicolumn{1}{c|}{$\text{0.96}$}                     \\ \hline
		\multicolumn{1}{|c|}{multi-level LASSO} & \multicolumn{1}{c|}{$\text{1.00}$} & \multicolumn{1}{c|}{$\text{1.00}$}                    \\ \hline
		\multicolumn{1}{|c|}{original TLCp} & \multicolumn{1}{c|}{$--$} & \multicolumn{1}{c|}{$\text{0.84}$}                                  \\ \hline
		\multicolumn{1}{|c|}{approximate TLCp cutoff} & \multicolumn{1}{c|}{$\text{0.16}$} & \multicolumn{1}{c|}{$--$}                                \\ \hline       
		\multicolumn{1}{|c|}{\begin{tabular}[c]{@{}c@{}}original TLCp\\ with three tasks\end{tabular}} & \multicolumn{1}{c|}{$1.00$} & \multicolumn{1}{c|}{$\text{1.00}$}                                                          \\ \hline
		\multicolumn{1}{|c|}{\begin{tabular}[c]{@{}c@{}}approximate TLCp cutoff\\ with three tasks\end{tabular}} & \multicolumn{1}{c|}{$\text{0.98}$} & \multicolumn{1}{c|}{$1.00$}                                                                \\ \hline                 
	\end{tabular}
	\caption{The table shows the $p$-value of the pairwise $t$-test of different methods on Parkinson's data with the relative task dissimilarity $2.02$. Boldface means the proposed TLCp methods statistically outperform the compared methods ($p$-value $<0.05$). }\label{table12}
\end{table}

\begin{table}[htbp]
	\centering
	\begin{tabular}{llll}
		\hline
		\multicolumn{1}{|l|}{} & \multicolumn{1}{l|}{original TLCp} & \multicolumn{1}{c|}{\begin{tabular}[c]{@{}c@{}}approximate TLCp\\ cutoff\end{tabular}}  \\ \hline
		\multicolumn{1}{|c|}{original Cp} & \multicolumn{1}{c|}{$\text{1.00}$} & \multicolumn{1}{c|}{$\text{0.56}$}              \\ \hline
		\multicolumn{1}{|c|}{stepwise FS} & \multicolumn{1}{c|}{$\text{1.00}$ } & \multicolumn{1}{c|}{$\text{1.00}$}               \\ \hline
		\multicolumn{1}{|c|}{univariate FS} & \multicolumn{1}{c|}{$\text{1.00}$} & \multicolumn{1}{c|}{$\text{1.00}$}               \\ \hline
		\multicolumn{1}{|c|}{LASSO} & \multicolumn{1}{c|}{$\text{1.00}$} & \multicolumn{1}{c|}{$\text{1.00}$}                \\ \hline
		\multicolumn{1}{|c|}{aggregate original Cp} & \multicolumn{1}{c|}{$\textbf{0.00}$} & \multicolumn{1}{c|}{$\textbf{0.00}$}                             \\ \hline
		\multicolumn{1}{|c|}{aggregate stepwise FS} & \multicolumn{1}{c|}{$\textbf{0.00}$} & \multicolumn{1}{c|}{$\textbf{0.00}$}                     \\ \hline
		\multicolumn{1}{|c|}{aggregate univariate FS} & \multicolumn{1}{c|}{$\textbf{0.00}$} & \multicolumn{1}{c|}{$\textbf{0.00}$}                       \\ \hline
		\multicolumn{1}{|c|}{aggregate LASSO} & \multicolumn{1}{c|}{$\textbf{0.00}$} & \multicolumn{1}{c|}{$\textbf{0.00}$}                        \\ \hline
		\multicolumn{1}{|c|}{least $\ell_{2,1}$-norm} & \multicolumn{1}{c|}{$\text{1.00}$} & \multicolumn{1}{c|}{$\text{1.00}$}                     \\ \hline
		\multicolumn{1}{|c|}{multi-level LASSO} & \multicolumn{1}{c|}{$\text{1.00}$} & \multicolumn{1}{c|}{$\text{1.00}$}                    \\ \hline
		\multicolumn{1}{|c|}{original TLCp} & \multicolumn{1}{c|}{$--$} & \multicolumn{1}{c|}{$\textbf{0.00}$}                                  \\ \hline
		\multicolumn{1}{|c|}{approximate TLCp cutoff} & \multicolumn{1}{c|}{$\text{1.00}$} & \multicolumn{1}{c|}{$--$}                                \\ \hline       
		\multicolumn{1}{|c|}{\begin{tabular}[c]{@{}c@{}}original TLCp\\ with three tasks\end{tabular}} & \multicolumn{1}{c|}{$1.00$} & \multicolumn{1}{c|}{$\text{1.00}$}                                                          \\ \hline
		\multicolumn{1}{|c|}{\begin{tabular}[c]{@{}c@{}}approximate TLCp cutoff\\ with three tasks\end{tabular}} & \multicolumn{1}{c|}{$\text{1.00}$} & \multicolumn{1}{c|}{$1.00$}                                                                \\ \hline                 
	\end{tabular}
	\caption{The table shows the $p$-value of the pairwise $t$-test of different methods on Parkinson's data with the relative task dissimilarity $21.22$. Boldface means the proposed TLCp methods statistically outperform the compared methods ($p$-value $<0.05$). }\label{table13}
\end{table}

%\begin{comment}
\begin{table}[htbp]
	\caption{{Symbols and Mathematical Notation}} \label{NOTE} \centering
	\begin{tabular}{l l}
		\toprule
		Notation & Meaning \\
		$\boldsymbol{X} (\tilde{\boldsymbol{X}})$ & Design matrix of $\mathbb{R}_{n\times k} (\mathbb{R}_{m\times k}) $ for the orthogonal target (source) task \\
		$\boldsymbol{\beta}$ & True regression coefficients of $\mathbb{R}_{k\times 1}$\\
		$\boldsymbol{I}$ & Identity matrix of $\mathbb{R}_{k\times k}$\\
		$\boldsymbol { \varepsilon } (\boldsymbol{\eta})$ & Multivariate Gaussian noise of $\mathbb{R}_{n\times 1} (\mathbb{R}_{m\times 1} )$\\
		$\lambda$ & Regularization parameter of the Cp criterion \\
		$\boldsymbol{y}(\tilde{\boldsymbol{y}})$ & Response vector of $\mathbb{R}_{n\times 1} (\mathbb{R}_{m\times 1})$ for the target (source) task \\
		$\boldsymbol{\delta}$ & Task dissimilarity vector of $\mathbb{R}_{k\times 1}$\\
		$\sigma_1 (\sigma_2)$ & Residual variance for the target (source) task \\
		$W_j (\tilde{W}_j)$ & $j$-th column vector of the design matrix $\boldsymbol{X} (\tilde{\boldsymbol{X}})$\\
		$\hat{\boldsymbol{a}}$ & Estimated regression coefficients of the orthogonal Cp problem \\
		$\hat{\boldsymbol{w}}_1 (\hat{\boldsymbol{w}}_2)$ & Estimated regression coefficients of the orthogonal TLCp problem for the target (source) task \\
		$\lambda_1 (\lambda_2)$ & Weight parameter of the residual sum of squares for the target (source) task \\
		$\boldsymbol{\lambda}_3$ & Parameter matrix in the TLCp problem\\
		$\lambda_4$ & Regularization parameter in the TLCp problem\\
		$\boldsymbol{\hat{\alpha}}$ & Estimated regression coefficients of the non-orthogonal Cp problem\\
		$\boldsymbol{\hat{\alpha}}_1$ & Estimated regression coefficients of the non-orthogonal Cp problem after orthogonalization\\
		$\boldsymbol{\hat{\alpha}}_2$ & Estimated regression coefficients of the approximate Cp procedure\\
		$\boldsymbol{\tilde{\alpha}}_2$ & The approximate Cp cutoff estimator\\
		$\bar{\boldsymbol{X}}$ & Design matrix of the non-orthogonal Cp problem\\
		$\boldsymbol{Q}$ & Transformation matrix of $\mathbb{R}_{k\times k}$ s.t. $(\bar{\boldsymbol{X}}\boldsymbol{Q})^{\top}\bar{\boldsymbol{X}}\boldsymbol{Q}=n\boldsymbol{I} $ \\
		$\hat{\mathcal{J}}$ & Subscripts for nonzero elements of $\boldsymbol{\hat{\alpha}}$\\
		$\mathcal{J}^*$ & Subscripts for nonzero elements of $\boldsymbol{\beta}$\\
		$\mathcal{A}$ & All the nonempty subsets of $\{1,\cdots,k\}$\\
		$\mathcal{A}^{c}$ & Subsets of $\mathcal{A}$ that contain $\mathcal{J}^*$\\
		$u_{\tau/2}$ & The $(1-\tau/2)$-percentile of the standard normal distribution\\
		$\boldsymbol{X}_1 (\boldsymbol{X}_2)$ & Design matrix of the non-orthogonal TLCp problem for the target (source) task \\
		$\boldsymbol{Q}_1 (\boldsymbol{Q}_2)$ & Transformation matrix of $\mathbb{R}_{k\times k}$ s.t. $(\boldsymbol{X}_1\boldsymbol{Q}_1)^{\top}\boldsymbol{X}_1\boldsymbol{Q}_1=n\boldsymbol{I} ((\boldsymbol{X}_2\boldsymbol{Q}_2)^{\top}\boldsymbol{X}_2\boldsymbol{Q}_2=m\boldsymbol{I})$ \\
		$\bar{\boldsymbol{w}}_1$ & Estimated regression coefficients of the non-orthogonal TLCp problem after orthogonalization\\
		$\hat{\boldsymbol{w}}_1$ & Estimated regression coefficients of the approximate TLCp procedure\\
		$\tilde{\boldsymbol{w}}_1$ & The approximate TLCp cutoff estimator\\
		$\tilde{\boldsymbol{w}}_1^*$ & Estimated regression coefficients of the non-orthogonal TLCp problem\\
		$\hat{\boldsymbol{\beta}}$ & Least squares estimation of $\boldsymbol{\beta}$\\
		\bottomrule
	\end{tabular}
\end{table}
%\end{comment}

\end{document}